%% file: paper.tex
\documentclass[]{bytedance_seed}

\usepackage[toc,page,header]{appendix}
\usepackage{subcaption}
\usepackage{wrapfig}

\usepackage{minitoc}
\usepackage{booktabs}
\usepackage{graphicx}
\usepackage[table]{xcolor}
\usepackage{array}
\usepackage{multirow}
\usepackage{multicol}
\usepackage{longtable}
\usepackage{amsmath}
\usepackage{amssymb}
\usepackage{amsthm}
\usepackage{enumitem}
\usepackage{mathtools}
\usepackage{dsfont}
\usepackage{microtype}

\theoremstyle{plain}
\newtheorem{theorem}{Theorem}[section]
\newtheorem{lemma}{Lemma}[section]
\newtheorem{proposition}{Proposition}[section]

\theoremstyle{definition}
\newtheorem{definition}{Definition}[section]
\newtheorem{condition}{Condition}[section]
\newtheorem{assumption}{Assumption}[section]

\theoremstyle{remark}

\newcommand{\E}{\mathbb{E}}
\newcommand{\Prob}{\mathbb{P}}

\newcommand{\1}{\mathds{1}}
\newcommand{\sigmoid}{\mathsf{sigmoid}}

\newenvironment{readingbox}[1]
{\par\medskip\noindent\begin{minipage}{\linewidth}\hrule\smallskip\noindent\textbf{#1.}\ }
{\par\smallskip\hrule\end{minipage}\medskip}

\newcommand{\benchmark}{EdgeBench}

\newcommand{\BenchTopRule}{\toprule[1.1pt]}
\newcommand{\BenchMidRule}{\midrule[0.65pt]}
\newcommand{\BenchBottomRule}{\bottomrule[1.1pt]}
\newcommand{\BenchCMidRule}[1]{\cmidrule[0.65pt](lr){#1}}

\title{\benchmark: Unveiling Scaling Laws of Learning from Real-World Environments}

\affiliation{ByteDance Seed}

\abstract{
Pretraining scaling laws reveal that model capability improves predictably with data and compute. But learning from real world environments after deployment remains far less understood. Analyzing roughly 38{,}000 hours of agent interaction with the environment across 134 real world tasks, we find, to the best of our knowledge, the first evidence that overall performance during environment learning follows a log-sigmoid scaling law with remarkably high precision, reaching $R^2 = 0.998$. Across model generations, we also find that agent learning speed roughly doubles every three months. This discovery stems from \benchmark{}, a suite of 134 real world tasks with ultra-long horizons, spanning scientific discovery, software engineering, combinatorial optimization, professional knowledge work, formal mathematics, and interactive games. Each task sustains at least 12 hours of continuous agent operation under rich, multilevel feedback, and is built through substantial expert effort. We publicly release 51 tasks and our full evaluation framework to accelerate the study of how agents learn from real world experience.
}

\titlefigure{%
\begin{center}
\centering
\includegraphics[width=\linewidth]{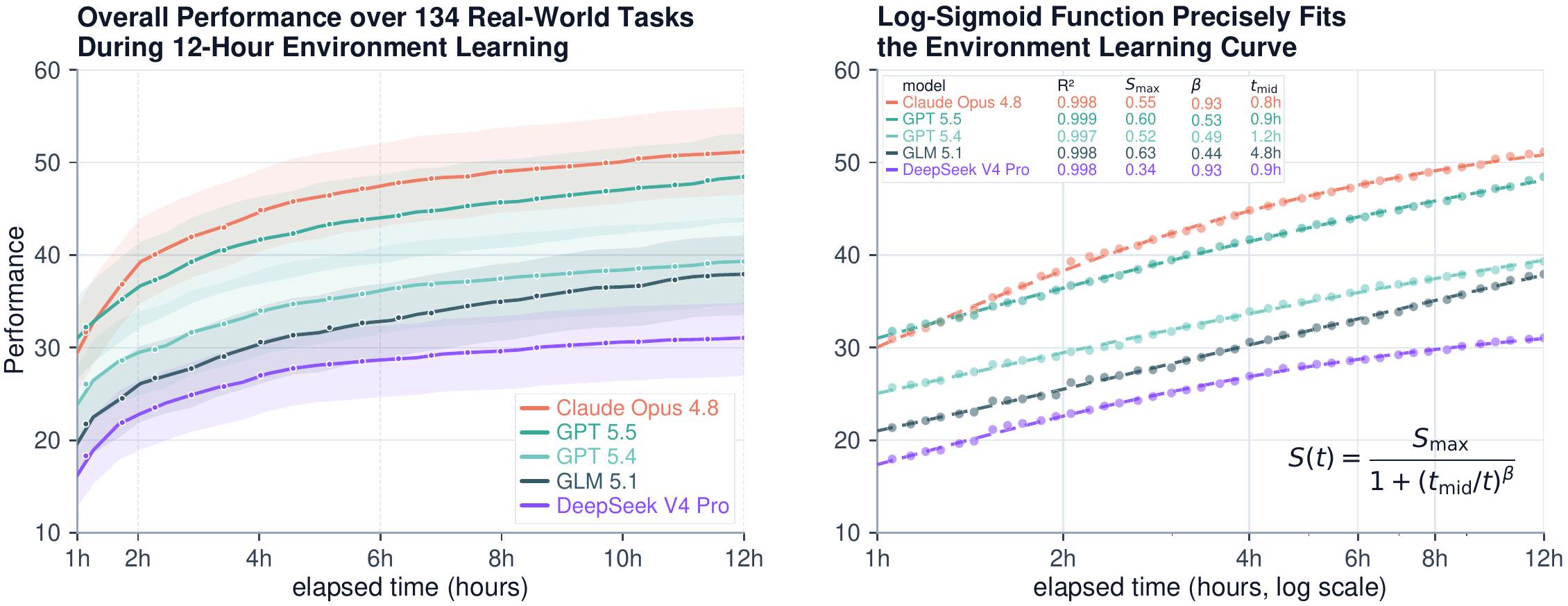}
\captionof{figure}{\textbf{Overall performance follows a precise log-sigmoid
relationship with environment interaction time} when averaged over all 134 tasks
(mean $R^2 = 0.998$).}
\label{fig:logtime-sigmoid-all-tasks}
\end{center}
}

\date{\today}
\correspondence{Shu Zhong at \email{zhongshu@bytedance.com}}

\checkdata[Project Page]{\url{https://edge-bench.org}}

\begin{document}
\maketitle

\input{sections/introduction}
\input{sections/approach}
\input{sections/emerging_scaling}
\input{sections/scaling}
\input{sections/experiments}
\input{sections/relatedwork}
\input{sections/conclusion}

\clearpage
\section*{Author Contributions}

\newcommand{\contribrole}[2]{\noindent\textbf{#1}\nopagebreak\par#2\par\medskip}

\begin{multicols}{2}
\raggedright
\setlength{\parskip}{1pt}
\contribrole{Core Contributors}{Deyao Zhu\textsuperscript{*}\\ Xin Zhou\textsuperscript{*}\\ Shengling Qin\textsuperscript{*}\\ Xuekai Zhu\textsuperscript{*}\\ Hangliang Ding\textsuperscript{*}\\ Shu Zhong\textsuperscript{*\textdagger}}
\contribrole{Theory}{Zixin Wen}
\contribrole{Data Collection and Curation}{Zhonglin Xie\\ Chenhui Gou\\ Linxuan Ren\\ Yueyang Wang\\ Junfeng Zhong\\ Rui Liu\\ Tian Gao\\ Yangguang Lin\\ Jingyuan Zhang\\ Maojia Song\\ Xuan Qi\textsuperscript{\ensuremath{\ddagger}}\\ Jinhong Wu\textsuperscript{\ensuremath{\ddagger}}\\ Chenyang Zhang\textsuperscript{\ensuremath{\ddagger}}\\ Yinzhu Piao\\ Ziru Niu\\ Hongbin Lin\\ Lingxiang Meng\\ Peng Tang\\ Chengyao Tang\\ Shanyu Wu\\ Huanyu Zheng\\ Yu Liu\\ Liya Zhu\\ He Wang\\ Ming Ding}
\contribrole{Data Infrastructure}{Ziyu Wan\\ Hao Liu\\ Sibo Wang\\ Haotian Zhu\\ Xintian Zhang\\ Nan Chai\\ Yipeng Liu\\ Panhao Lai}
\contribrole{Referrals}{Sihang Yuan\\ Zixin Su\\ Ge Zhang\\ Wangchunshu Zhou\\ Yantao Du}
\contribrole{Advisors}{Wenhao Huang\\ Guang Shi}
\end{multicols}

\medskip
\noindent\textsuperscript{*}Equal contribution\hspace{1em}\textsuperscript{\textdagger}Project Lead\hspace{1em}\textsuperscript{\ensuremath{\ddagger}}External contributor

\clearpage

\bibliographystyle{plainnat}
\bibliography{main}

\clearpage

\beginappendix

\section*{Appendix Contents}
\begingroup
\small
\setlength{\parskip}{0.15em}
\newcommand{\appentry}[2]{\noindent\hyperref[#1]{\makebox[1.7em][l]{\ref*{#1}}#2}\dotfill\pageref{#1}\par}
\newcommand{\appsubentry}[2]{\noindent\hspace*{1.5em}\hyperref[#1]{\makebox[2.5em][l]{\ref*{#1}}#2}\dotfill\pageref{#1}\par}

\appentry{sec:construction}{Evaluation Harness}
\appentry{sec:gpt54-serving-stability}{Serving and API Stability}
\appentry{sec:appendix-evaluation-hacking}{Evaluation Hacking}
\appentry{sec:theory}{A Comprehensive Derivation of the Log-Sigmoid Law}
\appsubentry{subsec:theory-prelim}{Preliminaries: Latent Capability Graph and the Attainable Support}
\appsubentry{subsec:single-task-frontier}{Environment Learning as a Frontier Expansion Process}
\appsubentry{subsec:aggregate-frontier-limit}{Many-task Aggregation Reveals the Smooth Log-Sigmoid Law}
\appsubentry{subsec:self-similar-log-time}{Graph Self-similarity Induces Log Scale for Time Axis}
\appsubentry{subsec:theory-discussion-limitations}{Discussion and Limitations}
\appentry{sec:scaling-law-discussion}{More Discussion on the Scaling Law Shapes}
\appentry{sec:additional-related-work}{Additional Related Work}
\appsubentry{sec:additional-related-work-not-self-evolution}{Benchmarks Not Suitable for Measuring Self-Evolution}
\appsubentry{sec:additional-related-work-learning}{Benchmarks Suitable for Measuring Learning or Self-Evolution}
\appsubentry{sec:additional-related-work-scaling-laws}{Scaling Laws for LLMs and Agents}
\appentry{sec:additional-benchmark-details}{Additional Benchmark and Experiment Details}
\appsubentry{sec:experience-gain-estimation}{Estimating the With- and Without-Experience Curves}
\appsubentry{sec:appendix-gw-case-study-details}{Gravitational-Wave Case Study Details}
\appsubentry{sec:ablation-experiment-setting}{Harness-Level Continuation Ablations}
\appsubentry{sec:task-by-task-design-notes}{Per-Task Design Notes}
\appsubentry{sec:all-curves}{Per-Task Learning Curves}
\appsubentry{sec:appendix-category-scores}{Per-Task Score Tables}
\appentry{sec:acknowledgements}{Acknowledgements}
\endgroup
\clearpage

\input{sections/construction}
\input{sections/serving_api_stability}
\input{sections/evaluation_hacking}

\input{sections/theory}
\input{sections/scaling_law_discussion}
\input{sections/additional_relatedwork}
\input{sections/appendix}

\clearpage
\input{sections/acknowledgements}

\end{document}

%% file: sections/introduction.tex
\section{Introduction}

\begin{figure}[!t]
\centering
\includegraphics[width=\linewidth]{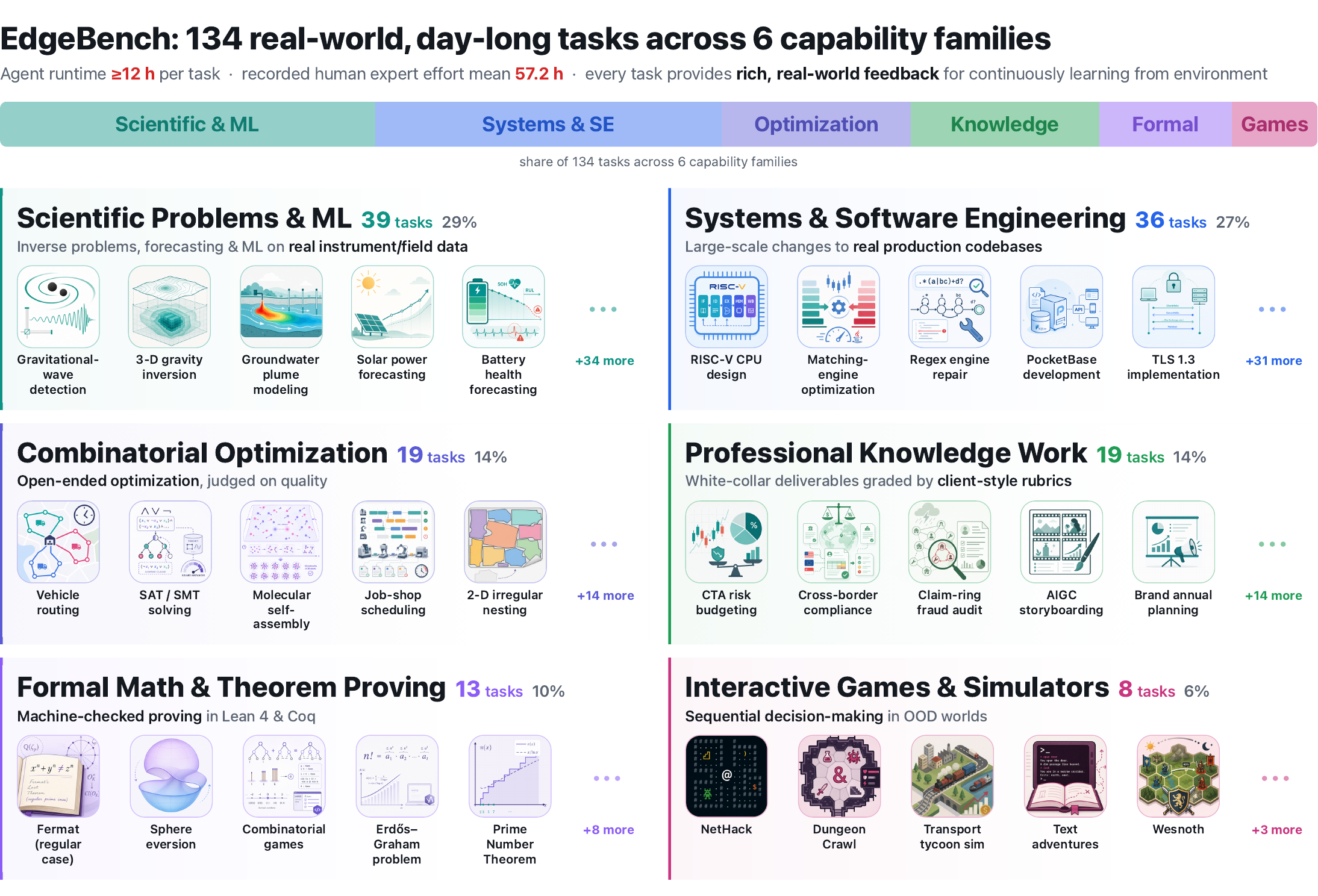}
\caption{\benchmark{} task taxonomy. 134 real world tasks across six capability families, with feedback channels designed to support within-run improvement. Recorded human expert effort estimates: mean 57.2 h.}
\label{fig:taxonomy}
\end{figure}

Scaling laws for pretraining~\citep{kaplan2020scaling,hoffmann2022training} revealed that model capability improves predictably with data and compute, an insight that guided much of the progress in the years that followed~\citep{openai2023gpt4,openai2025gpt45systemcard,openai2025gpt5systemcard,openai2025gpt52systemcard,openai2026gpt54thinkingsystemcard,openai2026gpt55systemcard,dubey2024llama3,anthropic2024claude35,anthropic2026claudeopus48systemcard,geminiteam2024gemini15,geminiteam2025gemini25}. Now, as large language models enter the agent era, they are being deployed into a growing range of real world environments where they can try to learn from interaction, yet whether learning from such environments obeys a clean scaling law remains unknown. Analyzing roughly 38{,}000 hours of agent environment interaction across 134 diverse real world tasks, we find, to the best of our knowledge, the first evidence that \textbf{when agents learn from real world environments, overall performance follows a log-sigmoid scaling law as a function of environment interaction time}, achieving remarkably high precision with $R^2 = 0.998$.

Why study agents' ability to learn from their environments? Real world use of AI depends on more than what a model learned during training. Some needed knowledge never appears in training data, such as private records and internal tools. Even when raw data exists, it omits the human process behind it: the trial-and-error, the interpretation of evidence, and the adaptation to feedback through which experts actually reach results. The real world also never stands still: human knowledge keeps advancing, and new tools, discoveries, and problems continually emerge that no fixed training corpus can anticipate. Therefore, an agent's ability to learn from its environment and improve task performance is central to deploying AI systems at scale in the real world.

Studying this ability requires task environments that resemble real use, provide informative feedback for learning, and allow agents enough time to learn through interaction. However, existing benchmarks often lack such feedback or limit agents to only minutes or a few hours, making them not well suited to this goal~\citep{jimenez2024swebench,chan2024mlebench,merrill2026terminalbench,sun2026agentslastexam,chu2026frontierswe,xu2026autolab}. To address this gap, we created \textbf{\benchmark{}}. Our benchmark contains \textbf{134 realistic and diverse tasks spanning six capability families}, from scientific research and software engineering to formal mathematics and interactive games, as shown in Figure~\ref{fig:taxonomy}. Each task runs in an executable workspace that combines fast local exploration with slower judge feedback on submitted artifacts, mirroring real-world workflows. Agents can work for at least \textbf{12 hours} on each task (by contrast, Agents' Last Exam~\citep{sun2026agentslastexam} averages roughly one hour per task), while we record their submissions and track how performance changes throughout the run. These tasks are substantial even for human experts: recorded expert effort averages \textbf{57.2 hours} per task and reaches up to \textbf{320 hours}. This makes \benchmark{} a natural testbed for studying how agents learn from their environments over long horizons.

Using \benchmark{}, we evaluate frontier agents over roughly 38{,}000 hours of
environment interaction. Our study makes four main observations:

\begin{itemize}
    \item \textbf{Environment learning exhibits a precise log-sigmoid scaling law.}
    Averaged performance follows the same functional form across the full
    benchmark, across task families, under longer interaction horizons up
    to 72 hours, and when forecasting later performance from early
    trajectories.

    \item \textbf{A theoretical derivation of the log-sigmoid law.}
    We propose a theory that models environment learning as a frontier expansion process on latent task graphs, which explains why benchmark-averaged progress takes the
    observed log-sigmoid form.

    \item \textbf{Agent learning speed doubles roughly every three months.}
    Studying frontier models released since September 2025, we find a rapid
    scaling trend in how quickly frontier agents learn from their environments.

    \item \textbf{Learning dynamics strongly shape long-horizon performance.}
    Long-horizon performance depends on how agents use accumulated experience,
    not only on how many attempts they make. Continuous experience outperforms
    independent restarts, longer context improves retention, and detailed
    case studies show that feedback can turn many failed probes into a few
    durable gains.
\end{itemize}

%% file: sections/approach.tex
\section{\benchmark{}}

\subsection{Design Goals}

\benchmark{} aims to measure whether an autonomous agent can learn from experience in an unfamiliar real world environment. This requires two properties from a benchmark that existing evaluations lack:

\begin{itemize}
    \item \textbf{Ultra-long-horizon, diverse tasks.} Learning behaviors such as exploration, strategy revision, and experience accumulation need time and complexity to emerge. Short tasks are usually solved from memory rather than learning, so measuring learning calls for long-horizon tasks. Because learning is a general capability, these tasks must also span diverse domains.
    \item \textbf{Realistic, multi-level feedback.} In practice, human experts learn from rich feedback: test failures, experimental results, unexpected phenomena, authoritative judgments, and more. A benchmark that cannot offer such rich feedback cannot measure learning, and leaves the agent guessing what the evaluation actually rewards. We need feedback that approximates the real world, so we can measure true, general-purpose learning.
\end{itemize}

The first requirement motivates our task taxonomy (\S\ref{sec:taxonomy}): 134 tasks across six capability families, each designed as a day-scale challenge that supports frontier models running for at least 12 hours. The second motivates our evaluation protocol (\S\ref{sec:evaluation-protocol}): each task individually simulates its own slice of the real world, providing isolated work and judge environments, local agent-driven feedback, submission-gated judge feedback, and host-side trajectory measurement.

\subsection{Task Taxonomy}
\label{sec:taxonomy}

We searched for real world tasks that satisfy two criteria: a performance ceiling high enough that no current agent can saturate it, and a workflow that supports continuous learning rather than one-shot completion. This search, conducted in collaboration with domain experts across fields, identified six capability families and yielded 134 curated tasks (Figure~\ref{fig:taxonomy}):

\begin{itemize}
    \item \textbf{Scientific Problems \& ML} (39 tasks). Each task uses real world research data and experimental settings sourced from working scientists. Domain expertise is essential: agents must formulate hypotheses, choose models, validate against noisy observations, and refine iteratively. Many problems are open-ended, with no known optimal solution.

    \item \textbf{Systems \& Software Engineering} (36 tasks). Agents work on production-grade codebases where a single task may require thousands of lines of change, with over 100{,}000 lines in the largest cases. Because the code spans interdependent modules, an agent must reason about cross-module coupling while meeting both correctness and performance targets.

    \item \textbf{Combinatorial Optimization} (19 tasks). These are open-ended, predominantly NP-hard problems where exact methods are intractable and progress depends on designing, tuning, and iterating on heuristic search strategies. Even strong solvers have room to improve with additional time and feedback.

    \item \textbf{Professional Knowledge Work} (19 tasks). These tasks reproduce real white-collar deliverables across finance, education, healthcare, and legal domains, matching work that would take a human professional with three or more years of experience roughly three full days to complete. Many tasks feature carefully designed rubrics and multi-round delivery feedback that approximate real client review cycles, so agents can learn from structured critique and revise iteratively.

    \item \textbf{Formal Math \& Theorem Proving} (13 tasks). These tasks sit at the frontier of mathematical difficulty and require building large-scale machine-checked proofs in Lean, coupling deep mathematical insight with substantial formal-verification engineering. Most are newly created for \benchmark{} and designed to support iterative progress: agents receive structured intermediate guidance and can extend partial proofs incrementally.

    \item \textbf{Interactive Games \& Simulators} (8 tasks). These are real games designed for human players, where proficient humans typically invest tens of hours to master the mechanics. The state spaces are enormous and each run is procedurally distinct, so agents face strong out-of-distribution pressure. Agents must develop and refine strategies through high-frequency interaction across many episodes.
\end{itemize}

Tasks whose primary difficulty lies in visual understanding, especially GUI operation, are excluded. When success depends on the vision backbone rather than iterative reasoning, learning ability and perceptual capability are hard to separate.

\subsection{Feedback Loop and Evaluation Protocol}
\label{sec:evaluation-protocol}
\label{sec:feedback-design}

\begin{figure}[!t]
\centering
\includegraphics[width=\linewidth]{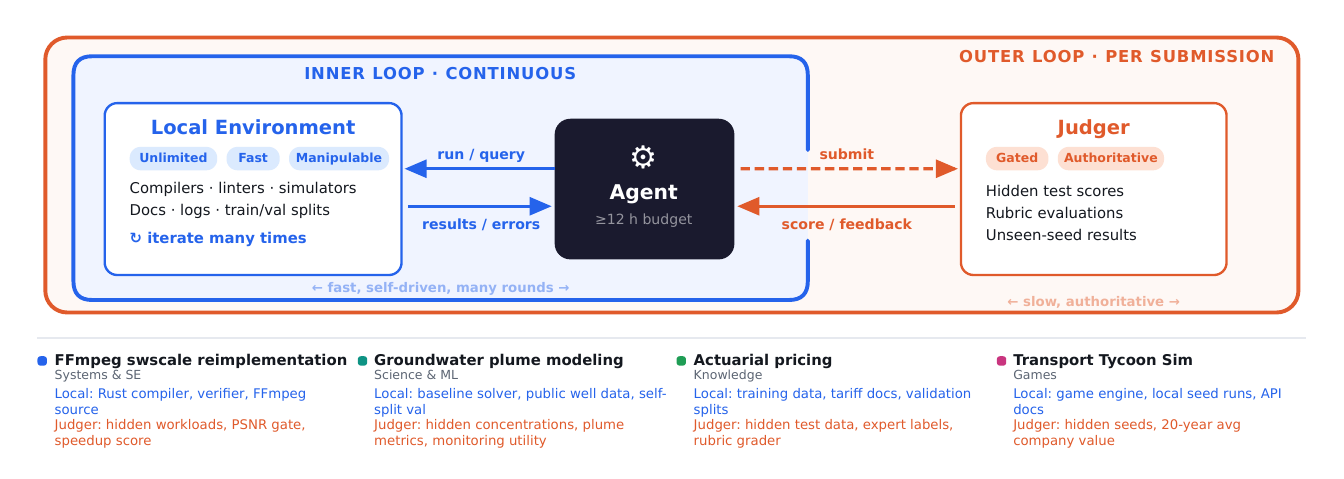}
\caption{The informative feedback loop in \benchmark{}. The inner loop (blue) lets agents iterate freely with local feedback; the outer loop (orange) gates authoritative judge feedback behind submissions. Bottom: representative tasks showing how both feedback channels are instantiated across capability families.}
\label{fig:dual-loop}
\end{figure}

\begin{figure}[p]
\centering
\includegraphics[width=0.31\linewidth]{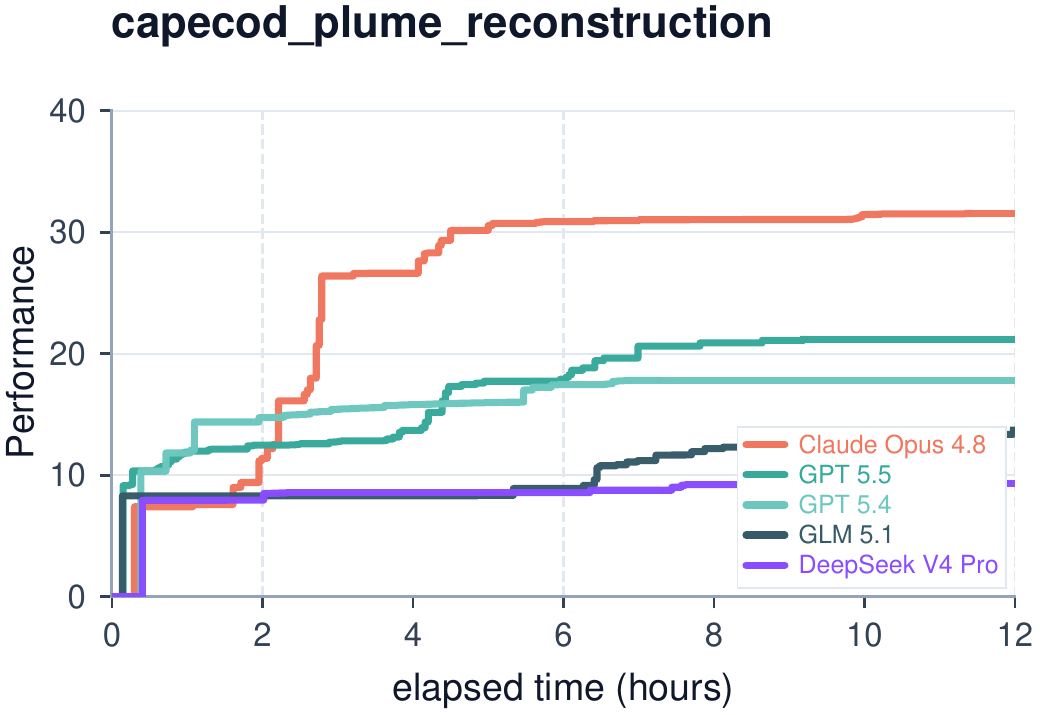}\hfill
\includegraphics[width=0.31\linewidth]{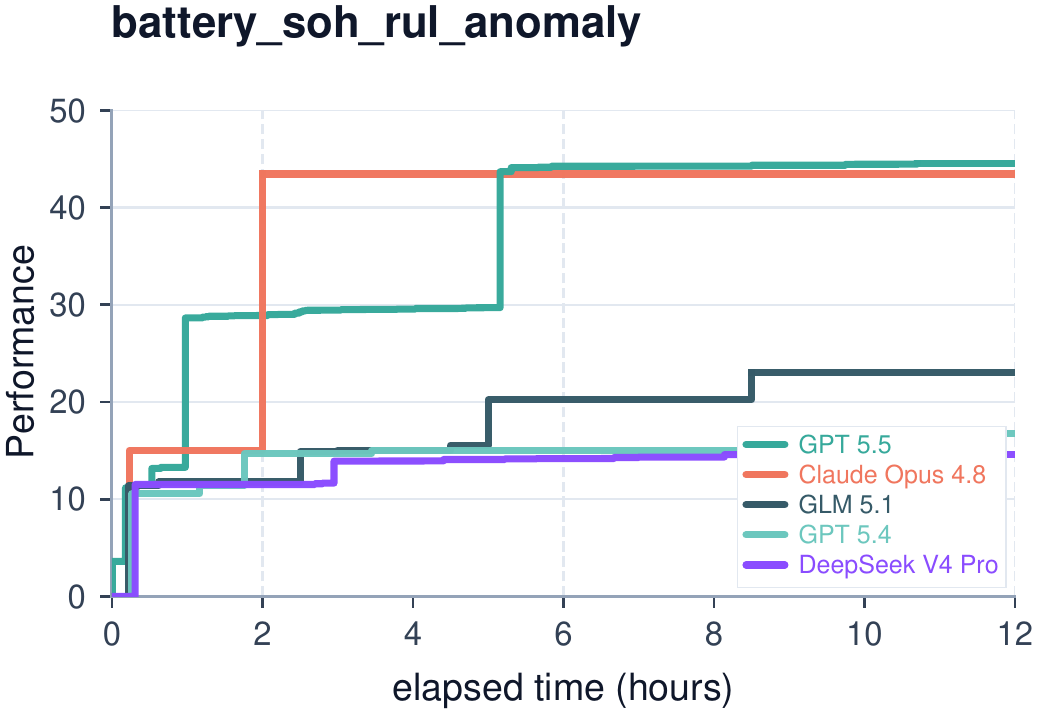}\hfill
\includegraphics[width=0.31\linewidth]{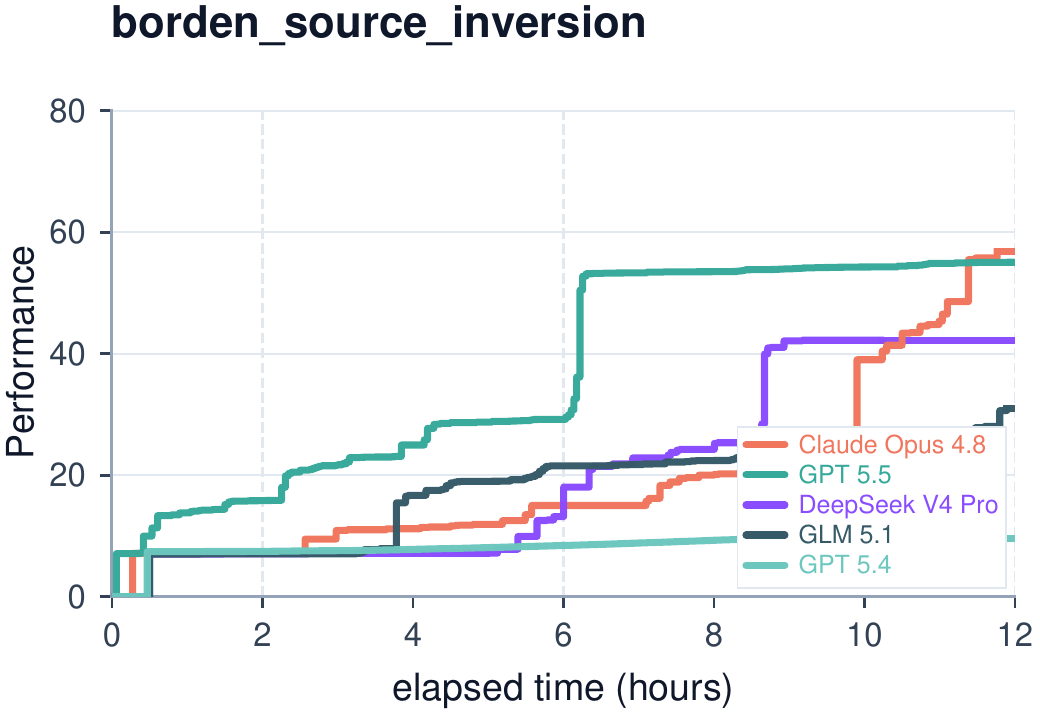}

\vspace{0.05em}
\includegraphics[width=0.31\linewidth]{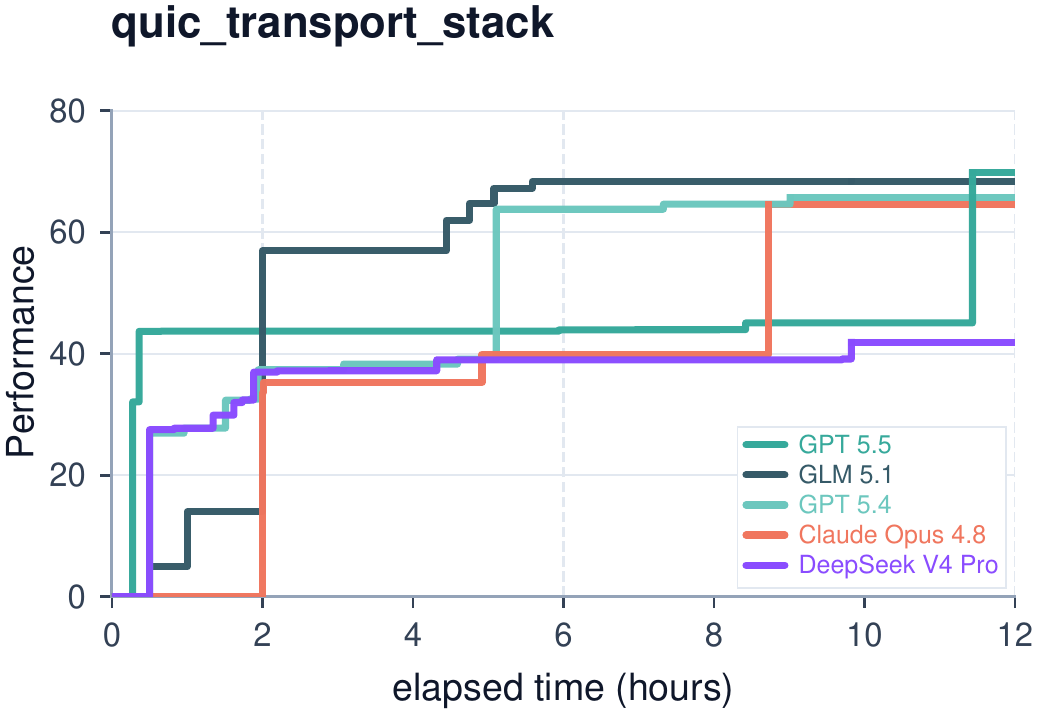}\hfill
\includegraphics[width=0.31\linewidth]{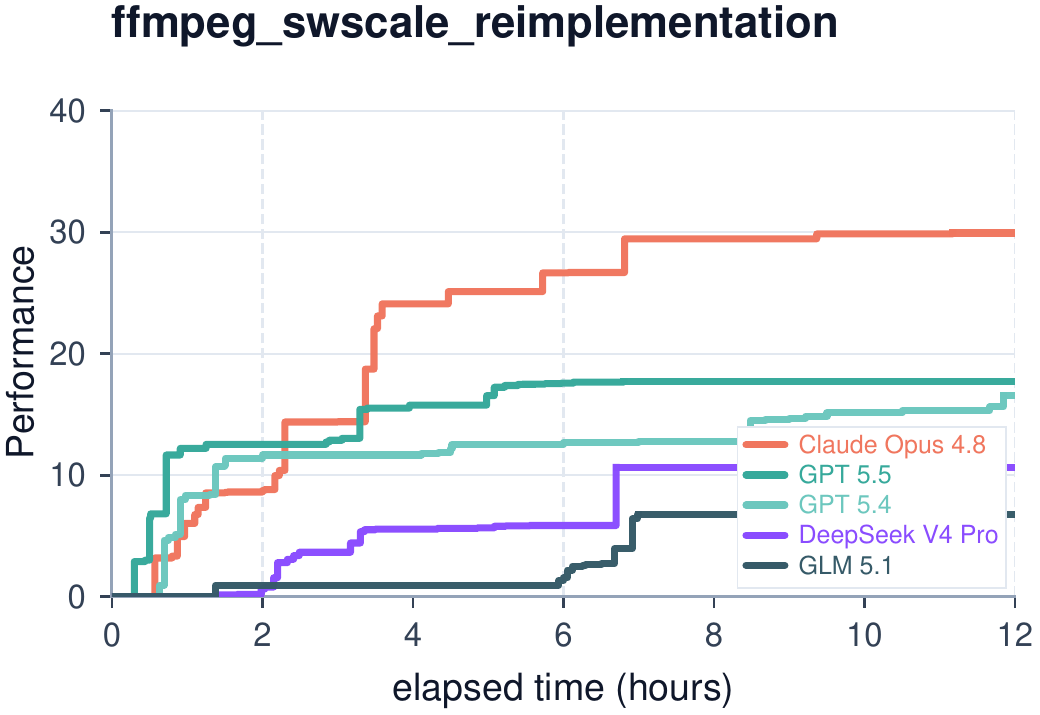}\hfill
\includegraphics[width=0.31\linewidth]{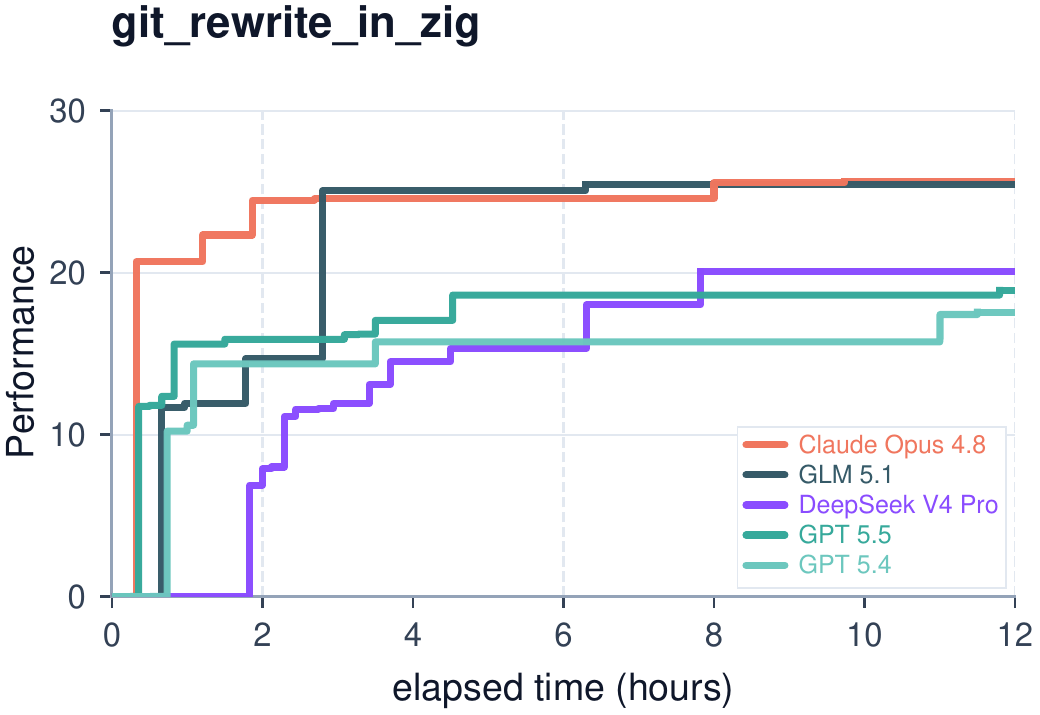}

\vspace{0.05em}
\includegraphics[width=0.31\linewidth]{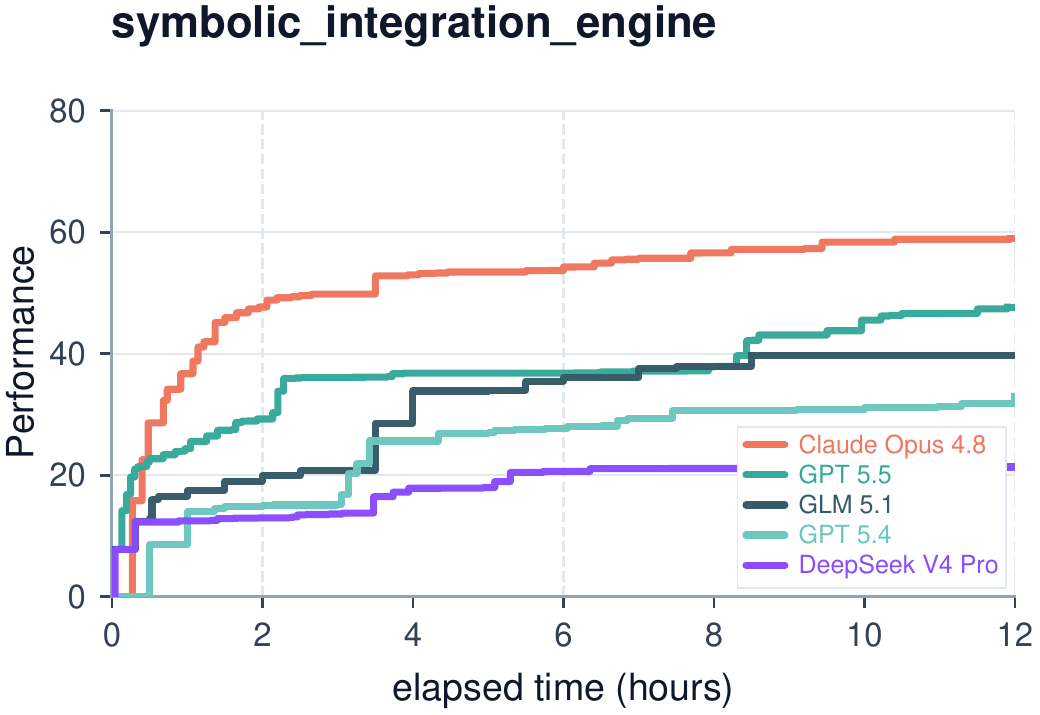}\hfill
\includegraphics[width=0.31\linewidth]{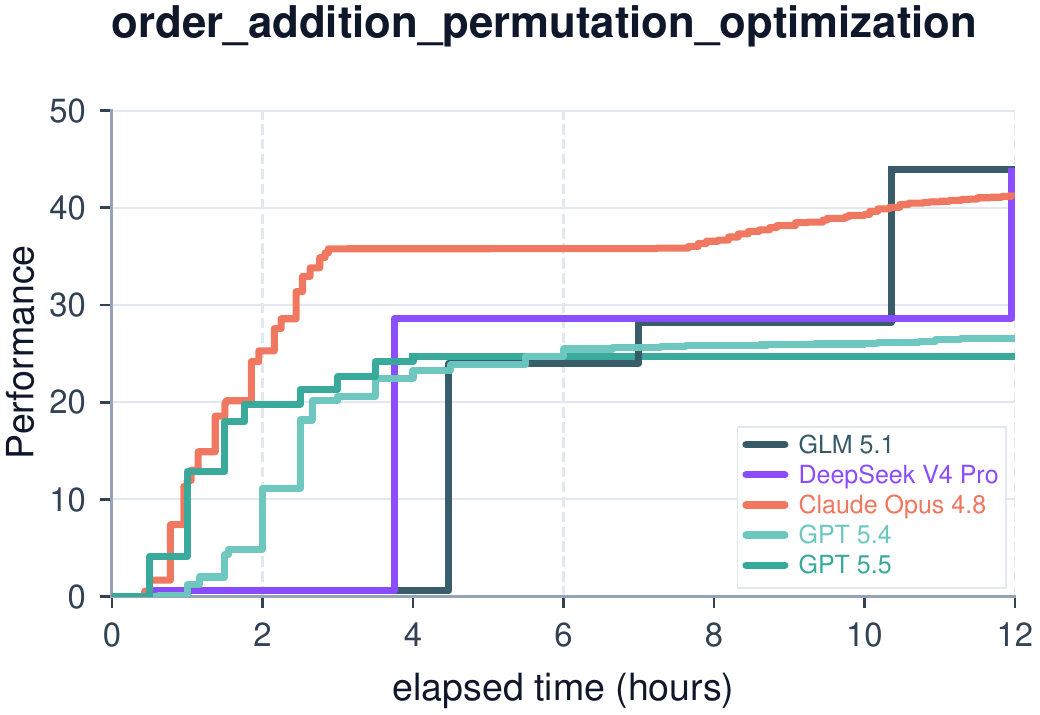}\hfill
\includegraphics[width=0.31\linewidth]{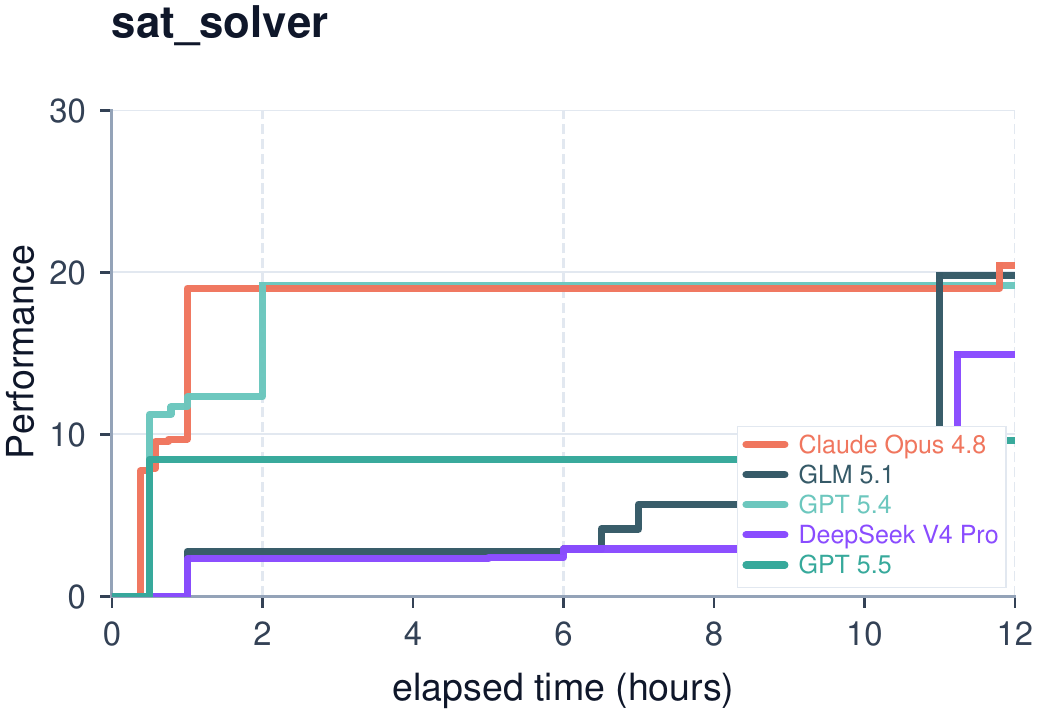}

\vspace{0.05em}
\includegraphics[width=0.31\linewidth]{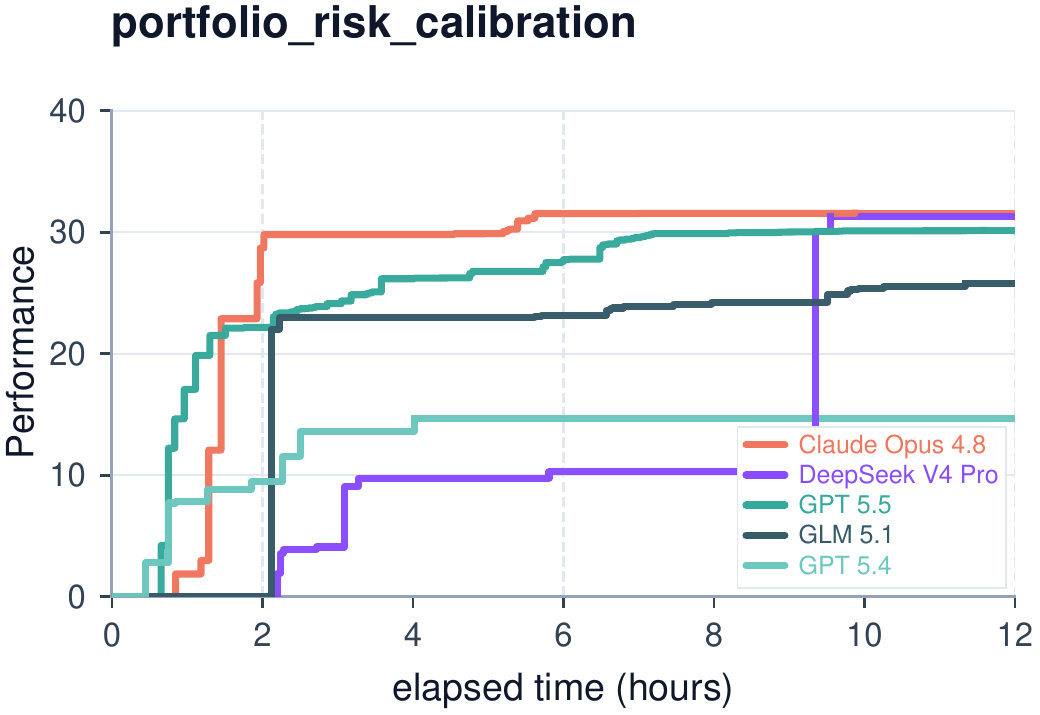}\hfill
\includegraphics[width=0.31\linewidth]{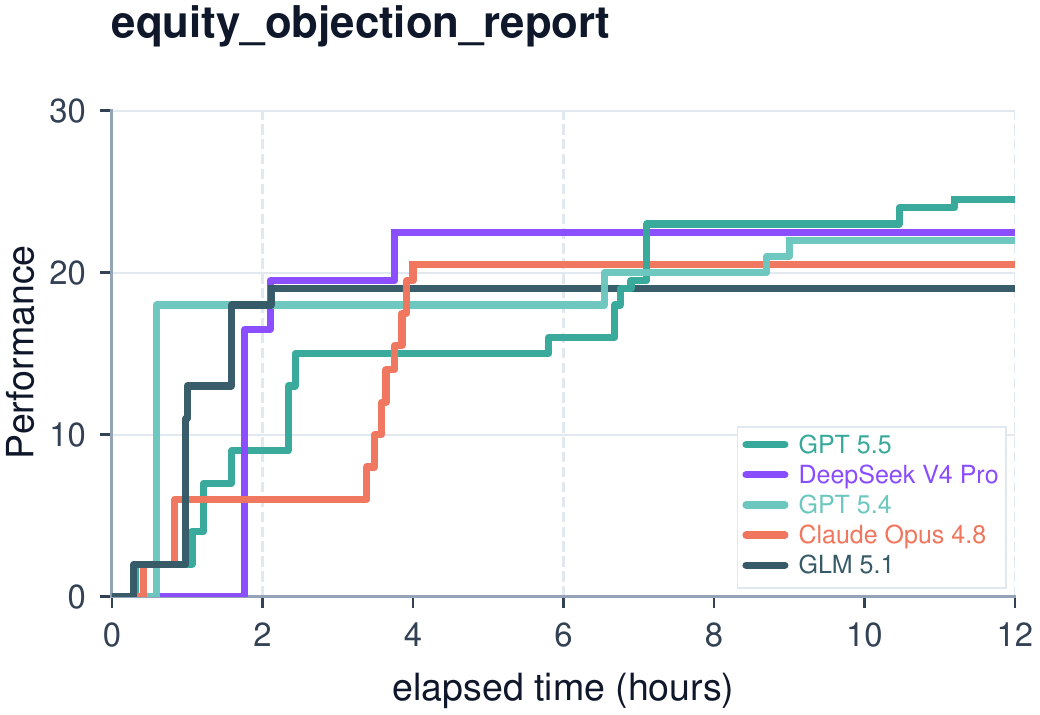}\hfill
\includegraphics[width=0.31\linewidth]{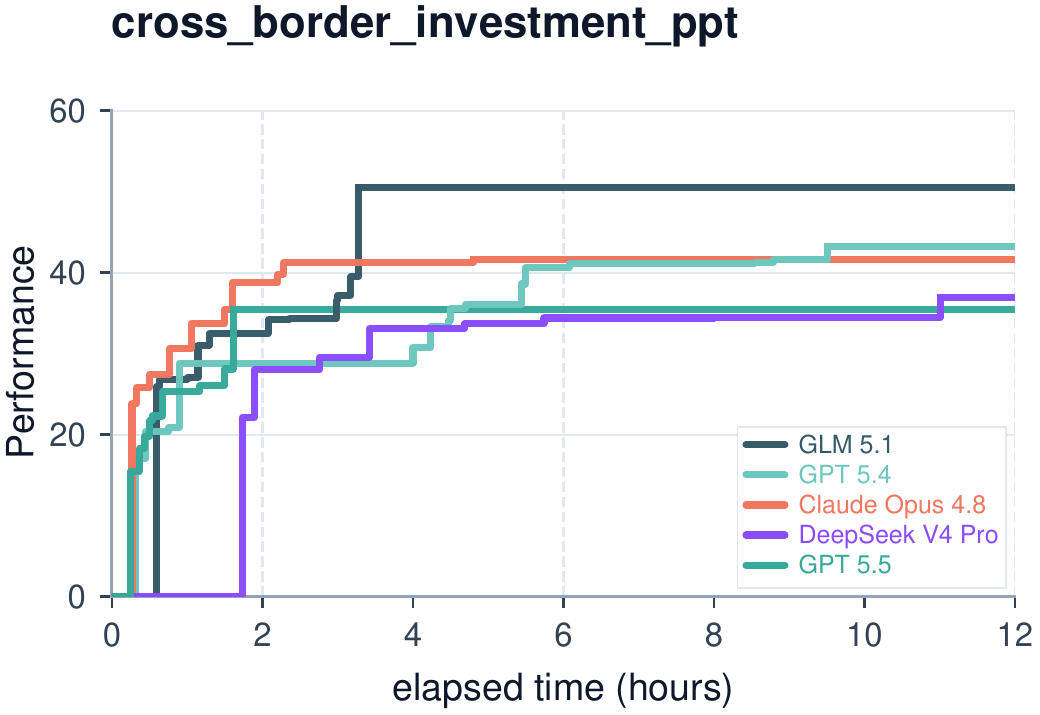}

\vspace{0.05em}
\includegraphics[width=0.31\linewidth]{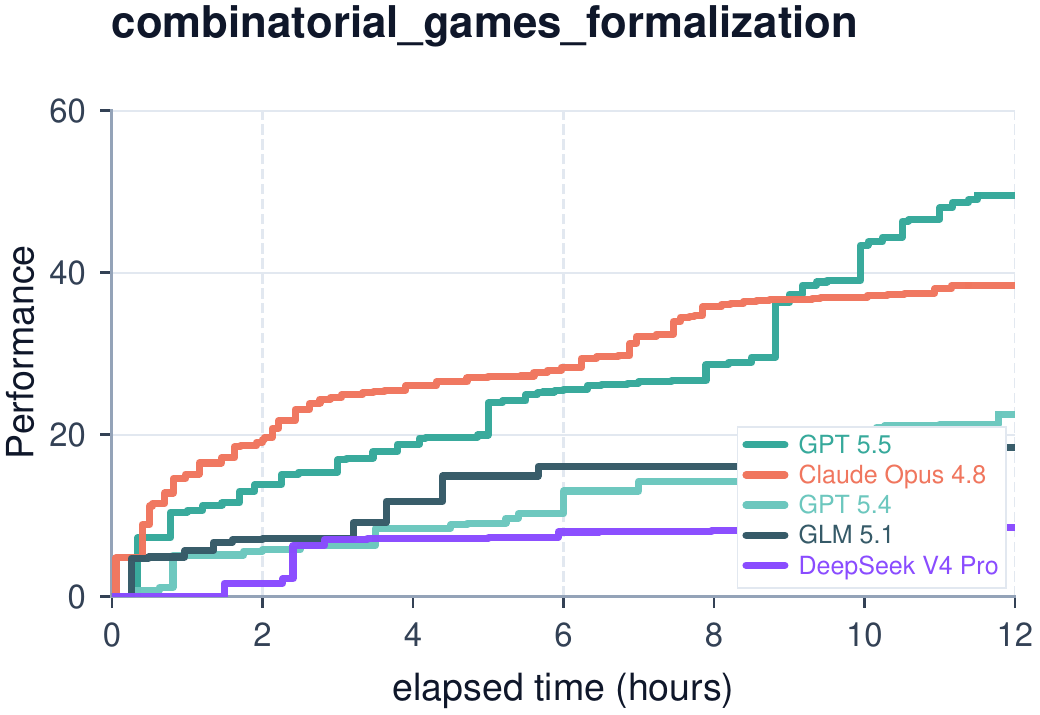}\hfill
\includegraphics[width=0.31\linewidth]{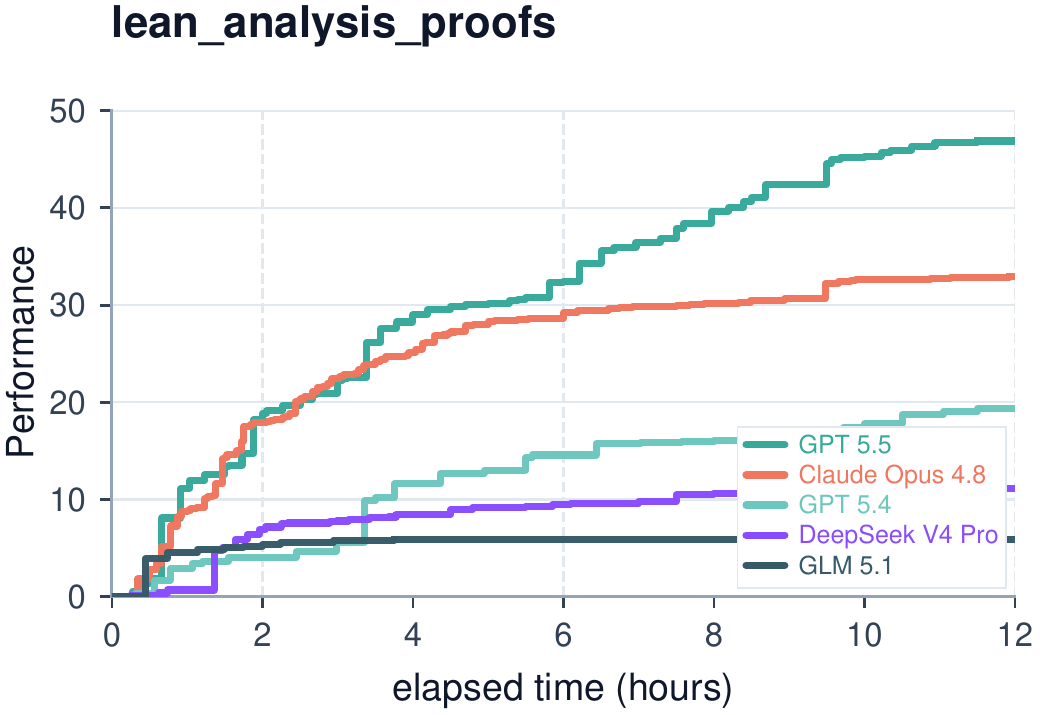}\hfill
\includegraphics[width=0.31\linewidth]{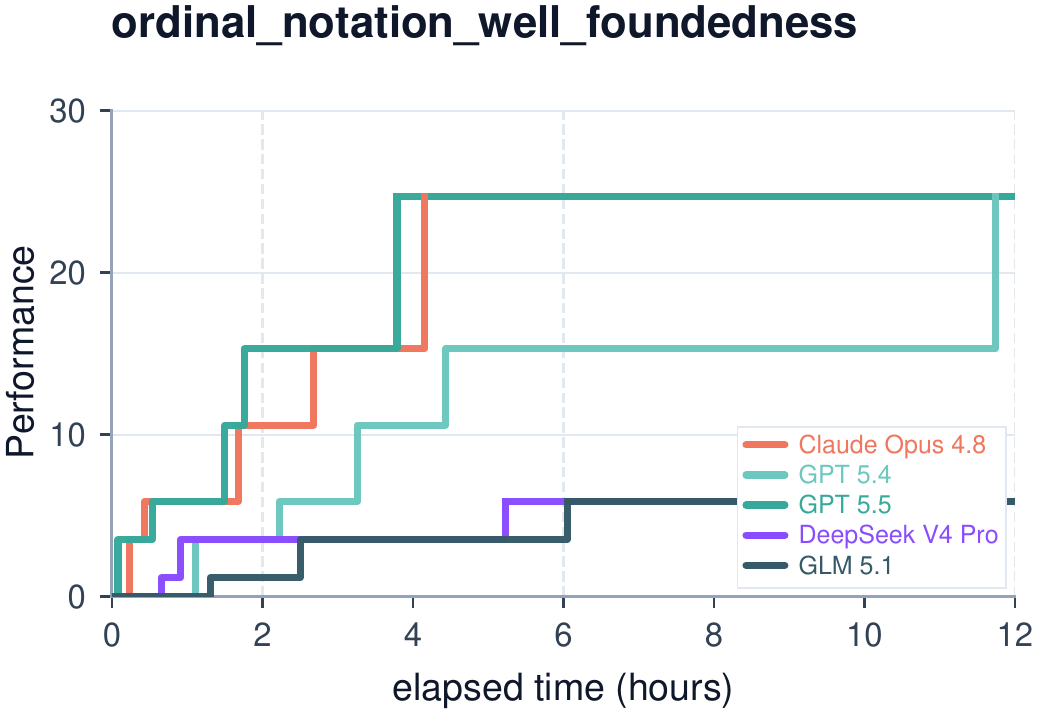}

\vspace{0.05em}
\includegraphics[width=0.31\linewidth]{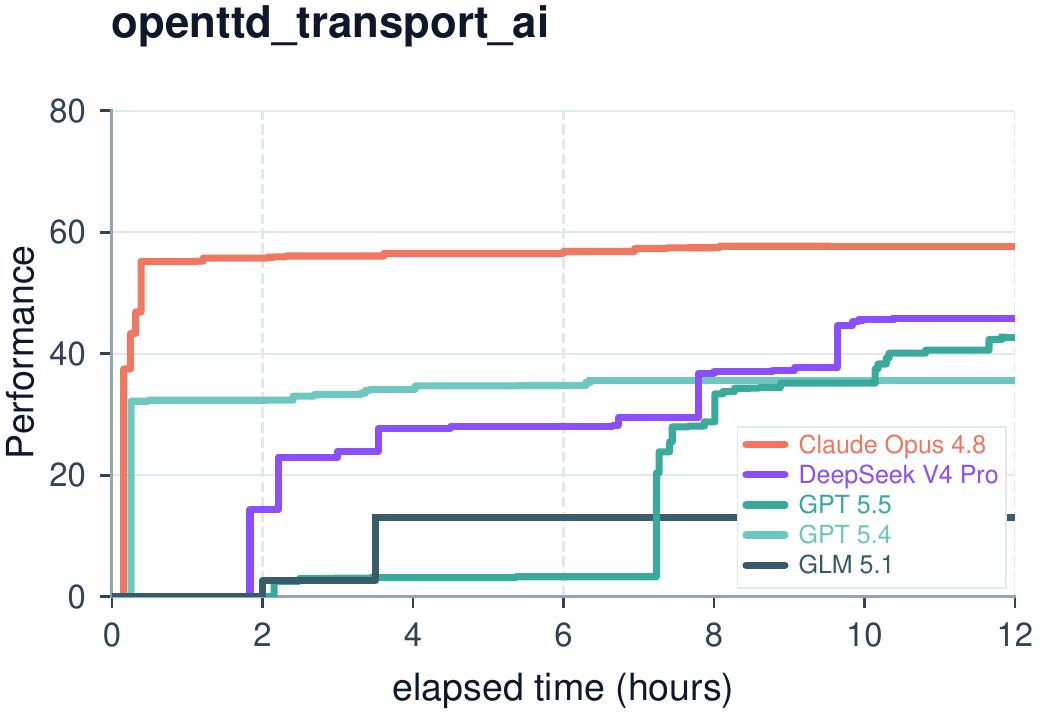}\hfill
\includegraphics[width=0.31\linewidth]{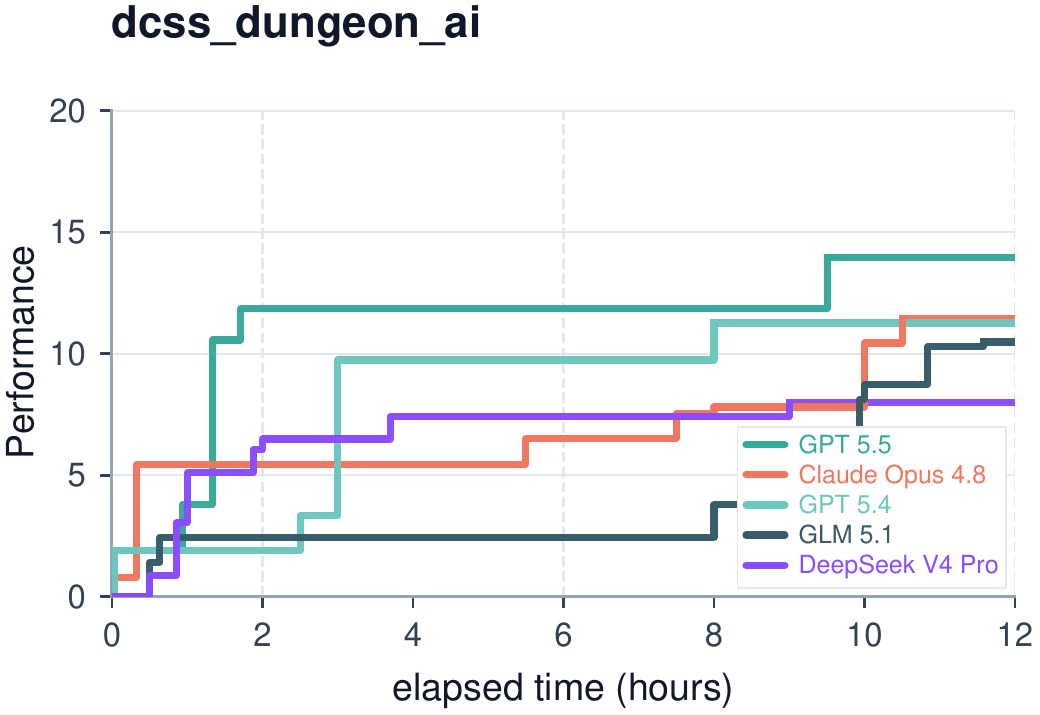}\hfill
\includegraphics[width=0.31\linewidth]{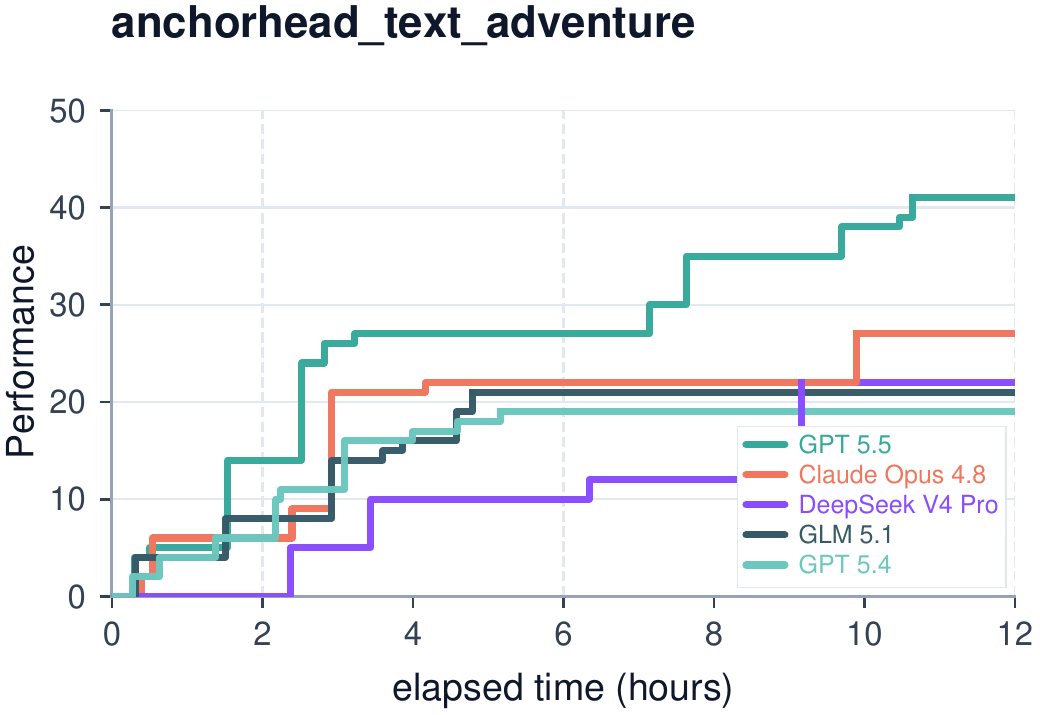}
\caption{Learning curves over 12 hours for 18 representative tasks across six capability families. The remaining task curves are provided in Figures~\ref{fig:curves-all-1}--\ref{fig:curves-all-21}.}
\label{fig:environment-data-curves-copy}
\end{figure}

Real world engineering and research workflows rarely provide a single final answer check. Instead, practitioners iterate through \textbf{two complementary feedback loops}: a fast local loop for exploration, debugging, and refinement, and a slower external loop that provides authoritative calibration through deployment, peer review, benchmark evaluation, or stakeholder feedback. The local loop enables rapid progress, while the external loop guards against overfitting to visible checks and exposes failures that are not captured by the developer's own tests.

\benchmark{} adopts this dual-loop structure to measure learning rather than endpoint success (Figure~\ref{fig:dual-loop}). The inner loop is local and agent-driven: agents can inspect a writable workspace, run tests or simulators, observe errors, and revise their artifacts. The outer loop is judge-mediated: submitted artifacts are evaluated against hidden cases or private grading criteria, returning calibrated scores, verdicts, or diagnostics. Across task families, this pattern is instantiated through different mechanisms: tests and profilers for software tasks, development splits and validators for scientific tasks, local testers and hidden seeds for optimization tasks, proof-checker states for theorem proving, episode scores for games, and rubrics for professional knowledge work.

The protocol is implemented with an isolated work--judge evaluation harness. During a run, the agent works inside a work container holding the task materials and local validation tools, but no hidden evaluation assets. The agent can actively submit its current artifacts to a separate judge container, which runs the hidden evaluation and returns the feedback specified by the task. A host-side judge server mediates this outer loop, including submission queues, cooldowns, authentication, and support for asynchronous grading on long-running evaluations, allowing agents to continue working while submitted jobs are being judged. Appendix~\ref{sec:appendix-evaluation-hacking} gives examples of task-design failure modes that motivated this isolation and submission-mediated design.

For trajectory measurement, the evaluation harness also performs host-side auto-evaluation at fixed intervals. These snapshots are scored through the hidden judge and recorded for analysis, but the results are not shown to the agent. This lets \benchmark{} measure improvement even between explicit submissions, while preserving the distinction between agent-visible feedback and evaluator-only measurement.

%% file: sections/emerging_scaling.tex
\section{Scaling Laws of Learning from Real World Environments}
\label{sec:emerging-scaling}

Pretraining scaling laws classically model language-model loss as a power-law
function of training scale, including the amount of pretraining
data~\citep{kaplan2020scaling,hoffmann2022training}. By contrast, benchmark
performance is a task-level readout of model capability, shaped by the
difficulty thresholds induced by tasks, examples, and scoring criteria. Prior
work has found that it is well described under pretraining scale-up by a
log-sigmoid curve
~\citep{dubey2024llama3,bhagia2024taskscaling,owen2024predictable,
ruan2024observational}.

Beyond learning from human-collected training data, modern agents such as
GPT-5.5 and Claude Opus 4.8 can continue to learn from their environments after
deployment. On a
task, they can acquire new information through interaction and improve their
performance over time. Yet it remains
unclear whether this form of learning obeys any similarly simple scaling law.
\benchmark{} provides 134 diverse real world tasks with executable
environments, informative feedback, and at least 12-hour interaction windows, allowing
us to study how agents improve through interaction with their environments.
Analyzing five frontier agents over roughly \textbf{38,000 hours of environment
interaction} across 134 diverse real world tasks, we find that \textbf{a
log-sigmoid curve fits environment learning performance precisely}. Thus,
learning from the environment and learning from pretraining data induce the
same mathematical scaling form.

\subsection{From Task Trajectories to Predictable Scaling Curves}
\label{sec:environment-learning-data}

\underline{\textit{Experimental setting.}} We evaluate 134 \benchmark{} tasks
with five frontier models: Claude Opus 4.8~\cite{anthropic2026claudeopus48systemcard}, GPT-5.5~\cite{openai2026gpt55systemcard}, GPT-5.4~\cite{openai2026gpt54thinkingsystemcard}, GLM-5.1~\cite{zai2026glm51,glm5team2026glm5report}, and
DeepSeek-V4-Pro (preview)~\cite{deepseekai2026deepseekv4}. For each task--model pair, we run three independent 12-hour trials and record the
full submission trajectory. GPT models are run with Codex using a 256k compact
window, while GLM-5.1 and DeepSeek-V4-Pro are run with Claude Code using 200k
compact windows. Claude Opus 4.8 is primarily run with a 1M Claude Code compact
window; we additionally include a 200k versus 1M Opus ablation in
Section~\ref{sec:context-length-analysis}.

\underline{\textit{Per-task trajectory diversity.}} Figure~\ref{fig:environment-data-curves-copy} shows per-task learning curves
for 18 representative tasks, illustrating how different models improve on the
same task over time. The full set of per-task curves is provided in
Appendix~\ref{sec:all-curves} for all 134 tasks. The tasks span diverse domains
and exhibit heterogeneous learning dynamics, with trajectories ranging from
smooth incremental gains to long plateaus, abrupt breakthroughs, and irregular
regressions.

\underline{\textit{Aggregate curves reveal a common structure.}} Although these individual curves are heterogeneous across both tasks and
models, their cross-task averages are unexpectedly smooth and share a common
structure. Motivated by prior log-sigmoid fits of benchmark performance under
pretraining scale-up, we fit the averaged environment learning curves with the
following three-parameter log-sigmoid model:
\begin{equation}
    S(t)
    =
    \frac{S_{\max}}{1 + (t_{\mathrm{mid}} / t)^\beta},
    \label{eq:log-sigmoid-learning}
\end{equation}
where $t$ is elapsed interaction time and $S(t)$ is best-so-far
performance. The fitted parameters serve as empirical descriptors of the
aggregate learning curve: $S_{\max}$ is the attainable score ceiling,
$t_{\mathrm{mid}}$ is the interaction time at which the curve reaches half of
that ceiling, and $\beta$ controls how sharply progress concentrates in log
time. Thus a smaller $t_{\mathrm{mid}}$ means the model reaches the bulk of its
attainable score sooner, while a larger $\beta$ corresponds to a steeper
learning transition. Section~\ref{sec:log-sigmoid-fit} shows that this simple
functional form fits the empirical learning curves surprisingly well.

\begin{figure}[!t]
\centering
\includegraphics[width=\linewidth]{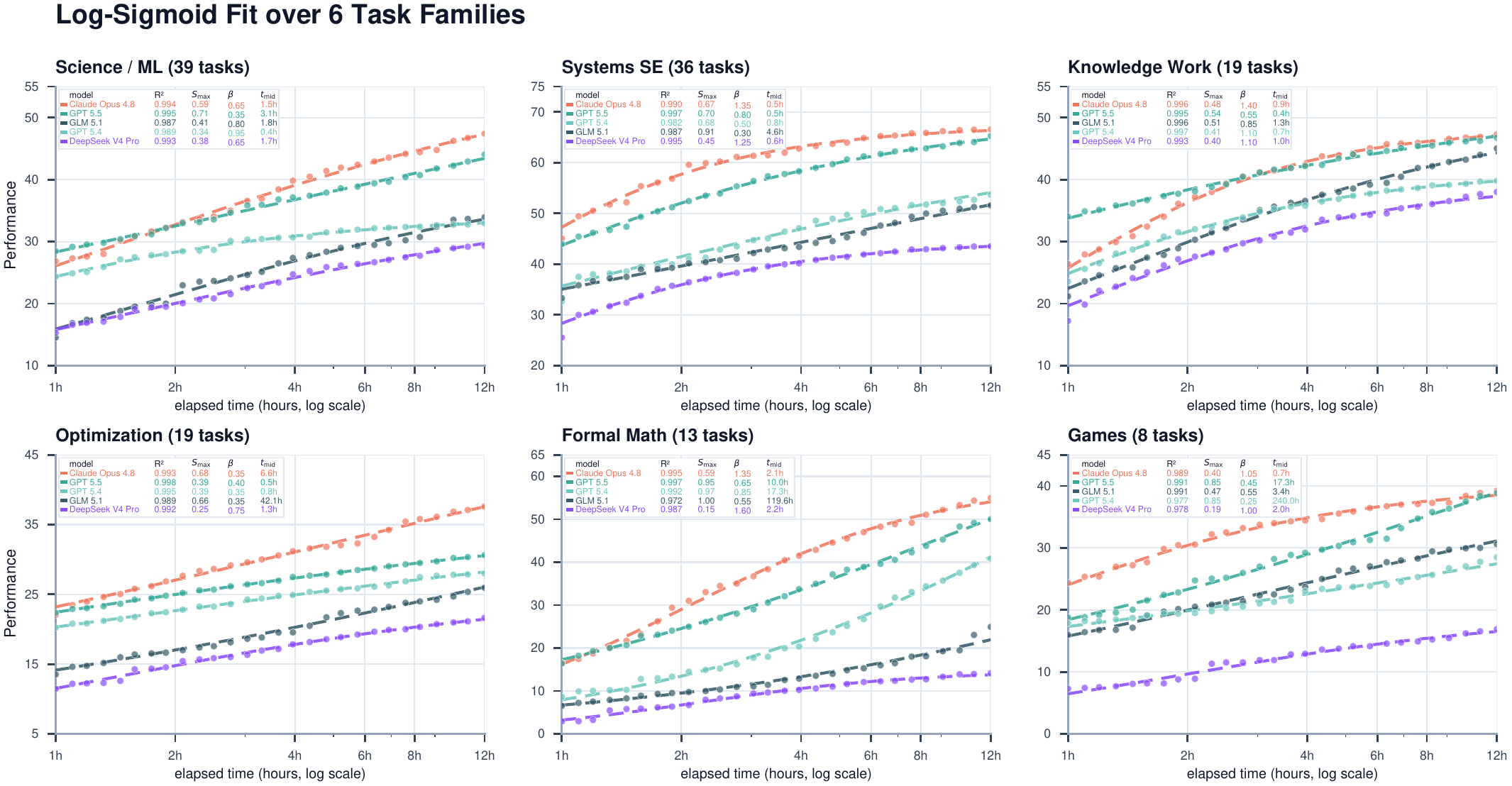}
\caption[Task-family-level log-sigmoid fits across the six \benchmark{} task families.]{
Task-family-level log-sigmoid fits across the six \benchmark{} task families.
Despite large differences in task type and scoring function, the same
log-time sigmoidal form fits the average trajectory in each task family well.\protect\footnotemark}
\label{fig:logtime-sigmoid-domains}
\end{figure}
\footnotetext{Fit precision improves as the number of fitted tasks increases; see
Figure~\ref{fig:log-sigmoid-rmse-task-count}.}

\begin{figure}[!t]
\centering
\begin{minipage}[t]{0.64\linewidth}
\centering
\includegraphics[width=\linewidth]{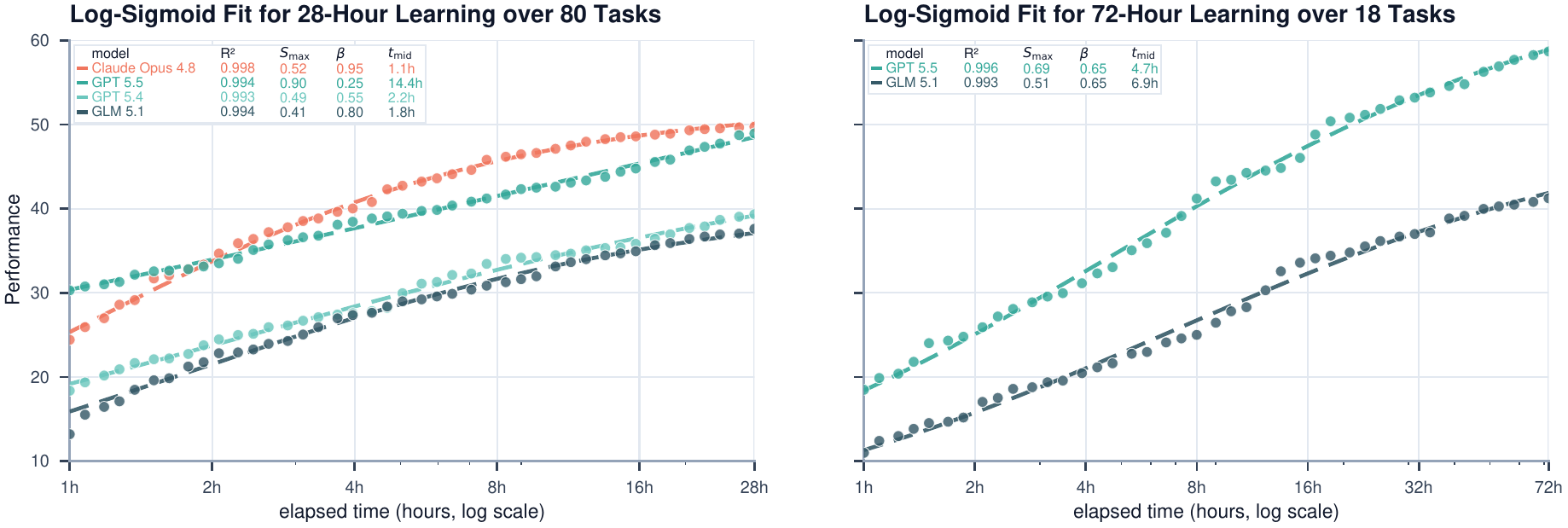}
\caption{Long-horizon log-sigmoid fits beyond the main 12-hour benchmark window.
The left panel fits 28-hour trajectories averaged over 80 tasks; the right panel fits
72-hour trajectories averaged over 18 tasks. All fitted curves remain highly precise,
with $R^2 \geq 0.993$.}
\label{fig:long-horizon-sigmoid-fits}
\end{minipage}
\hfill
\begin{minipage}[t]{0.32\linewidth}
\centering
\includegraphics[width=\linewidth]{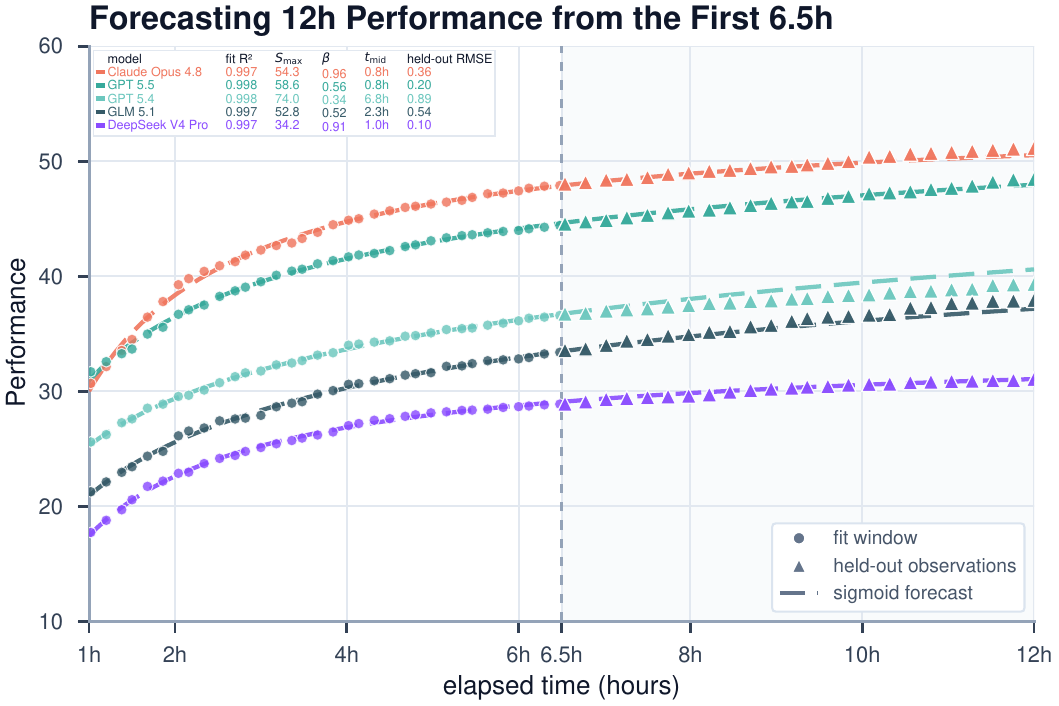}
\caption[Forecasting 12-hour performance from the first 6.5 hours.]{
Forecasting 12-hour performance from the first 6.5 hours. Log-sigmoid
fits on this early window accurately predict the held-out remainder for all five
models.\protect\footnotemark}
\label{fig:sigmoid-forecast-6p5h}
\end{minipage}
\end{figure}
\footnotetext{GPT-5.4 service availability dropped in the latter half of the
experiment, which may partly explain its forecast deviation; see
Appendix~\ref{sec:gpt54-serving-stability}.}

\subsection{Log-Sigmoid Curves Fit Environment Learning with Remarkable Precision}
\label{sec:log-sigmoid-fit}

We find that the log-sigmoid fit is precise and robust across every setting we tested:

\begin{itemize}
    \item \textbf{Log-sigmoid curves precisely fit the 134-task average for all
    five models.}
    As shown in Figure~\ref{fig:logtime-sigmoid-all-tasks}, after averaging
    over all 134 tasks, the fitted log-sigmoid curve closely tracks each
    model's 12-hour learning trajectory. The fit is uniformly tight across all
    five models, with $R^2 \geq 0.997$ in every case.

    \item \textbf{The same log-sigmoid form persists across heterogeneous task families.}
    Figure~\ref{fig:logtime-sigmoid-domains} repeats the fit separately across
    the six capability families in \benchmark{}. These families require different knowledge,
    exercise different capabilities, and produce visibly different aggregate
    learning curves. Yet each family is still well described by the same
    log-sigmoid form, including smaller families where fewer tasks make the
    averaged trajectories noisier.

    \item \textbf{The fit remains precise under substantially longer
    interaction horizons.}
    Figure~\ref{fig:long-horizon-sigmoid-fits} extends the analysis beyond the
    main 12-hour window to 28-hour and 72-hour interaction horizons. Due to
    resource limits, the 28-hour fit covers 80 tasks and four models, while
    the 72-hour fit covers 18 tasks and two models. The same log-sigmoid
    structure remains stable across both longer settings, with every fitted
    curve reaching $R^2 \geq 0.993$.

    \item \textbf{The log-sigmoid law exhibits predictive power.}
    Figure~\ref{fig:sigmoid-forecast-6p5h} tests this by fitting each
    12-hour aggregate curve using only the first 6.5 hours, then evaluates the forecast on
    held-out observations from 6.5 to 12 hours. The extrapolated curves remain
    close to the later observed trajectories across all five models, with
    $R^2 \geq 0.997$ and RMSE below 1.0 performance point in every case.

\end{itemize}

These precise fits raise two questions: is the log-sigmoid genuinely the right curve, or would any S-shape do, and where does so clean a law come from?

\underline{\textit{Log-sigmoid fits best among common S-curves.}}
Many saturating growth processes are S-shaped~\citep{wikipedia_scurve}. We therefore
compare the log-sigmoid against other common S-curves, each viewing an
S-shaped learning curve as a cumulative distribution over a time coordinate:
log-probit~\citep{bliss1934method}, log-Gompertz~\citep{gompertz1825nature}, and
a Weibull CDF~\citep{weibull1951statistical} on raw time, together with a
two-parameter log-linear baseline; for the log-time families the independent
variable is $x=\ln t$. We fit every family on the 12h, 28h, and 72h full windows
and pool the error. As Table~\ref{tab:scurve-family-comparison} shows, the log-sigmoid
family attains the lowest RMSE, while the log-linear baseline is substantially worse. The empirical
signal is thus robustly sigmoidal rather than tied to a single link function,
and is not explained by mere linear improvement in $\ln t$. Appendix~\ref{sec:scurve-mechanistic-reading}
discusses why we nevertheless prefer the log-sigmoid on mechanistic grounds.

\begin{figure}[!t]
\centering
\begin{minipage}[t]{0.57\textwidth}
\vspace{0pt}
\small
\renewcommand{\arraystretch}{1.42}
\setlength{\tabcolsep}{4pt}
\resizebox{\ifdim\width>\linewidth\linewidth\else\width\fi}{!}{%
\begin{tabular}{@{}>{\raggedright\arraybackslash}p{0.22\linewidth}
                >{\raggedright\arraybackslash}p{0.58\linewidth}
                r@{}}
\BenchTopRule
Family & Functional form & RMSE \\
\BenchMidRule
\textbf{Log-Sigmoid} &
$\displaystyle S(t)=\frac{S_{\max}}{1+(t_{\mathrm{mid}}/t)^{\beta}}$ &
\textbf{0.390} \\
\addlinespace[2pt]
Log-Probit &
$\displaystyle S(t)=S_{\max}\,\Phi\!\left(\frac{\ln t-\mu}{\sigma}\right)$ &
0.398 \\
\addlinespace[2pt]
Log-Gompertz &
$\displaystyle S(t)=S_{\max}\exp\!\left[-\exp\{-c(\ln t-x_0)\}\right]$ &
0.402 \\
\addlinespace[2pt]
Weibull CDF &
$\displaystyle S(t)=S_{\max}\left(1-\exp\{-(t/\lambda)^\beta\}\right)$ &
0.404 \\
\addlinespace[2pt]
Log-Linear &
$\displaystyle S(t)=a+b\ln t$ &
0.717 \\
\BenchBottomRule
\end{tabular}%
}
\captionof{table}{Full-window fit error for three-parameter S-curve families and a
two-parameter log-linear baseline. RMSE is measured in \underline{performance
points} on the 0--100 score scale; lower is better.}
\label{tab:scurve-family-comparison}
\end{minipage}
\hfill
\begin{minipage}[t]{0.39\textwidth}
\vspace{0pt}
\centering
\includegraphics[width=\linewidth]{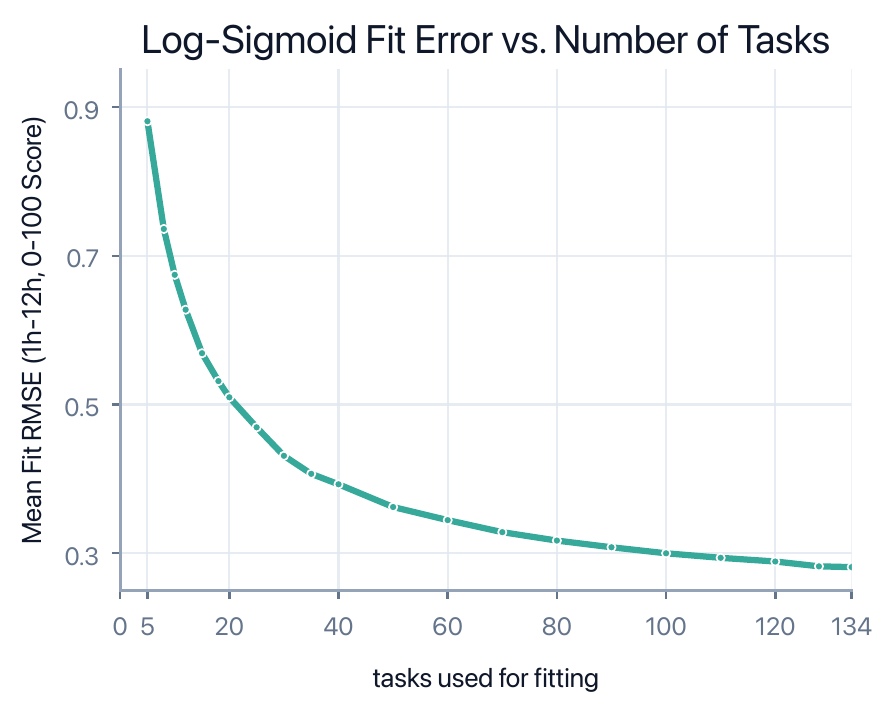}
\captionof{figure}{Log-sigmoid fit error decreases as more tasks are averaged.}
\label{fig:log-sigmoid-rmse-task-count}
\end{minipage}
\end{figure}

\underline{\textit{The law emerges from a population of tasks.}}
A single task's trajectory is noisy and idiosyncratic, yet
Figure~\ref{fig:log-sigmoid-rmse-task-count} shows that the log-sigmoid fit
becomes steadily more precise as we average over more tasks: the residual error
falls monotonically as tasks accumulate, from $1$ task to all $134$. The clean
scaling law is therefore an \emph{emergent, population-level} regularity, sharp
only across many diverse tasks rather than within any one of them. 

\FloatBarrier

\subsection{A Theory of the Log-Sigmoid Law}
\label{sec:why-log-sigmoid}

The empirical fits show a robust log-sigmoid shape, but do not by themselves explain why this form should appear. We propose a theoretical model: environment learning is a frontier expansion process on the underlying task graphs. In this view, each task is modeled by a latent graph of score units, the already-unlocked units exert influence on the locked neighbors to unlock them, and progress occurs when the frontier between unlocked and locked score nodes advances. Appendix~\ref{sec:theory} gives the full derivation; here we summarize the mechanism. Throughout the section, we use 
\[
u = \log t - \log t_{\mathrm{mid}}
\] 
as a change of coordinate of the time axis.

\underline{\textit{Environment learning is a frontier expansion process.}}
For a single task, let us consider the task score is composed of many score units, representing nodes \(i\) on the task graph \(G\) with score \(w_i\) and normalized score weights \(\mu_i = w_i / \sum_i w_i\). Let \(n_i(u)\in\{0,1\}\) indicate whether unit \(i\) has been unlocked at time \(u\), the normalized score obtained is
\[
    x(u)=\sum_i \mu_i n_i(u).
\]
On the task graph \(G\), an edge weight \(K_{ij}\ge0\) measures how much an unlocked source unit \(j\) helps unlock a target unit \(i\). Thus a locked unit \(i\) receives an influence field
\[
    h_i(u)=\sum_j K_{ij}n_j(u).
\]
If locked units unlock randomly at an expected rate proportional to this field, then conditioned on the current state,
\begin{equation}
    \label{eq:frontier-expansion-naive}
    \frac{\mathrm d}{\mathrm du}\mathbb E[x(u)\mid n(u)]= \eta \sum_{i\in L(u)}\sum_{j\in U(u)}\mu_iK_{ij}.
\end{equation}
The expected score-growth rate is therefore exactly the weighted frontier cut from unlocked units \(U(u)\) to locked units \(L(u)\).

\underline{\textit{Frontier process moves at a speed proportional to \(x(1-x)\).}}
The exact frontier cut still depends on the task graph structure. The mean-field approximation is to assume that, at the aggregate level, every macroscopic unlocked--locked cut has approximately product-measure influence:
\[
    \sum_{i\in L}\sum_{j\in U}\mu_iK_{ij}
    \approx
    \kappa \mu(L)\mu(U).
\]
where \(\mu(A) = \sum_{i\in A} \mu_i\) for any set \(A \subseteq G\). This, along with \eqref{eq:frontier-expansion-naive}, gives
\begin{equation}
    \label{eq:logistic-dynamics-naive}
    \frac{\mathrm dx}{\mathrm du} = \beta x(1-x),\qquad \beta=\eta\kappa.
\end{equation}
The two factors have direct interpretations: unlocked score mass supplies reusable capability, while locked score mass measures the remaining opportunity for improvement. Appendix~\ref{subsec:single-task-frontier} formalizes this approximation using a weighted cut-mixing condition, which is weaker than assuming that all edge weights are individually equal.

\underline{\textit{The effective time coordinate is logarithmic.}}
The frontier equation is written in an effective task-graph coordinate \(u\). A natural reason for \(u\) to be approximately \(\log t\) is self-similar graph structure\footnote{This resembles scale-free dynamics in physical complex systems, such as self-organized criticality and critical phenomena, where behavior is not governed by a single characteristic scale~\citep{bak1987self,newman2005power}.}. If each additive increase in task difficulty exposes a multiplicatively larger amount of relevant graph structure, then search volume needed to traverse the graph grows exponentially with difficulty scale. If the search effort is approximately constant across time horizon, then the difficulty scale reached by time \(t\) grows as 
\[
    u \sim \log t
\]
Substituting this coordinate into the frontier equation~\eqref{eq:logistic-dynamics-naive} gives
\begin{equation}
    \label{eq:log-sigmoid-curve-naive}
    \frac{\mathrm dx}{\mathrm d\log t} = \beta x(1-x).
\end{equation}

\underline{\textit{Solving the frontier equation.}}
Separating variables in \eqref{eq:log-sigmoid-curve-naive} gives
\[
    \log\frac{x(t)}{1-x(t)}
    =
    \beta\log\frac{t}{t_{\mathrm{mid}}},
\]
where \(t_{\mathrm{mid}}\) is chosen so that \(x(t_{\mathrm{mid}})=1/2\). Hence
\[
    x(t) = \frac{1}{1+(t_{\mathrm{mid}}/t)^\beta}, \quad \implies \quad S(t) = \frac{S_{\max}}{1+(t_{\mathrm{mid}}/t)^\beta}.
\]

\underline{\textit{Benchmark average is smoother than individual tasks.}}
The argument above describes the limiting frontier drift. It does not imply that every finite task should visibly follow a smooth sigmoid. A task with a small number of score units can have long plateaus and sudden jumps. The empirical scaling law is instead a statement about the task-aggregate curve. If many independently evaluated tasks each follow approximate frontier dynamics, then averaging removes finite-task jaggedness. The aggregate becomes a single log-sigmoid when residual task midpoints and task speeds concentrate:
\[
    x_M(u)
    =
    \frac1M\sum_{b=1}^M x_b(u)
    \approx
    \frac1M\sum_{b=1}^M
    \frac{1}{1+e^{-\beta_b(u-\delta_b)}}
    \stackrel{P}{\longrightarrow}
    \frac{1}{1+e^{-\beta u}}.
\]
Appendix~\ref{subsec:aggregate-frontier-limit} makes this statement precise under assumptions of blockwise cut-mixing, vanishing average jump noise, midpoint alignment, and speed concentration.

\underline{\textit{Interpretation of the fitted parameters.}}
The fitted rate \(\beta\) has a natural interpretation as an effective frontier-propagation speed in log time. A larger \(\beta\) means that score unlocks over a narrower range of interaction scales, producing a steeper transition from low to high performance. A smaller \(\beta\) means that progress is spread over more multiplicative time, producing a more gradual learning curve. The fitted ceiling \(S_{\max}\) should be interpreted as the attainable score support over the fitted regime, not necessarily as an absolute upper bound on performance.

\underline{\textit{Applicability and limitations.}}
The log-sigmoid law is not expected to hold for every environment-learning process. It can fail when task graphs contain strong bottlenecks, dispersed task midpoints, heterogeneous frontier speeds, or non-scale-free graph structures. These failure modes clarify the scope of the theory: the log-sigmoid is most natural when progress behaves like a sufficiently mixed frontier expansion process on an approximately fractal structured graph. In this sense, learning from environments tests whether an agent can convert diverse feedback into reusable structure that accelerates subsequent discovery. We view this as central to why environment learning is worth measuring and scaling.

\begin{figure}[!t]
\centering
\includegraphics[width=0.9\linewidth]{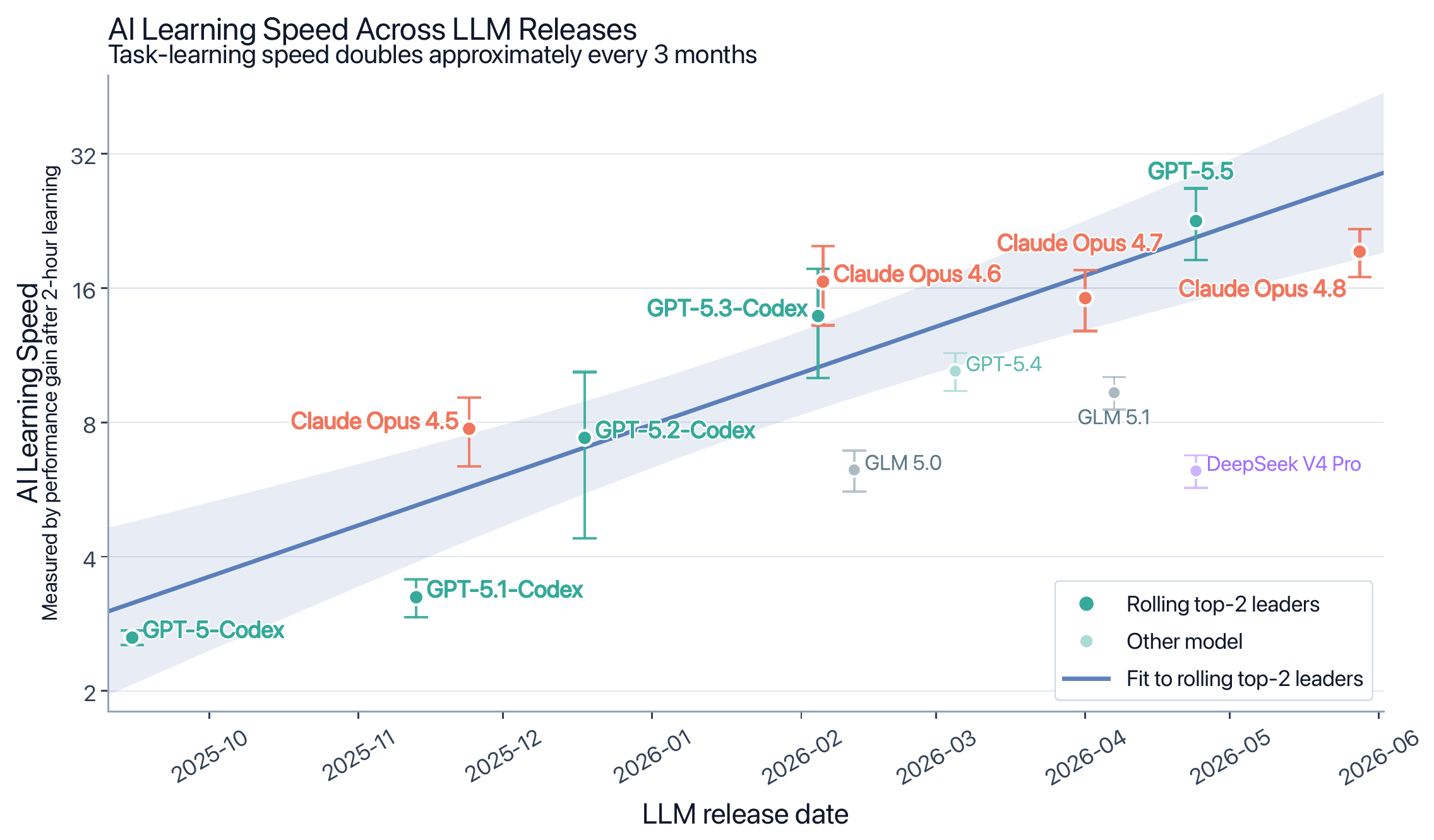}
\caption{Learning speed across evaluated LLM releases. Learning speed denotes the two-hour performance gain on the fixed 18-task slice. The blue line fits the rolling top-2 leaders by release date and indicates an approximately three-month doubling trend.}
\label{fig:scaling}
\end{figure}

%% file: sections/scaling.tex
\section{Agent Learning Speed Doubles Approximately Every Three Months}
\label{sec:scaling}

Frontier agents can now improve through interaction with task environments. This
raises a separate question: are newer models learning from their
environments faster? We measure this by comparing how much performance each
agent gains over a fixed interaction budget across model release dates.

\subsection{Experimental Design}
\label{sec:measuring}

\underline{\textit{Disentangling prior knowledge from environment learning.}}
A high score may reflect what the model already knew rather than what it learned during the run. To reduce this confound, we select an 18-task slice from \benchmark{} where models show similar first-attempt performance. With comparable starting points, later gains provide a cleaner measure of environment learning. We measure task-learning speed as the \textbf{average performance gain} over a fixed two-hour budget.

\underline{\textit{Evaluation protocol.}}
We evaluate frontier open- and closed-source AI systems released from September 2025 through the current evaluation window, when frontier systems became capable of sustained autonomous runs. Each model is run three times per task. GPT models are run with Codex; all other models are run with Claude Code. As shown in Figure~\ref{fig:performance-effort} (left), models start from comparable first-attempt performance on this 18-task slice, with an average of $6.87 \pm 0.97$.

\begin{figure}[!t]
\centering
\includegraphics[width=\linewidth]{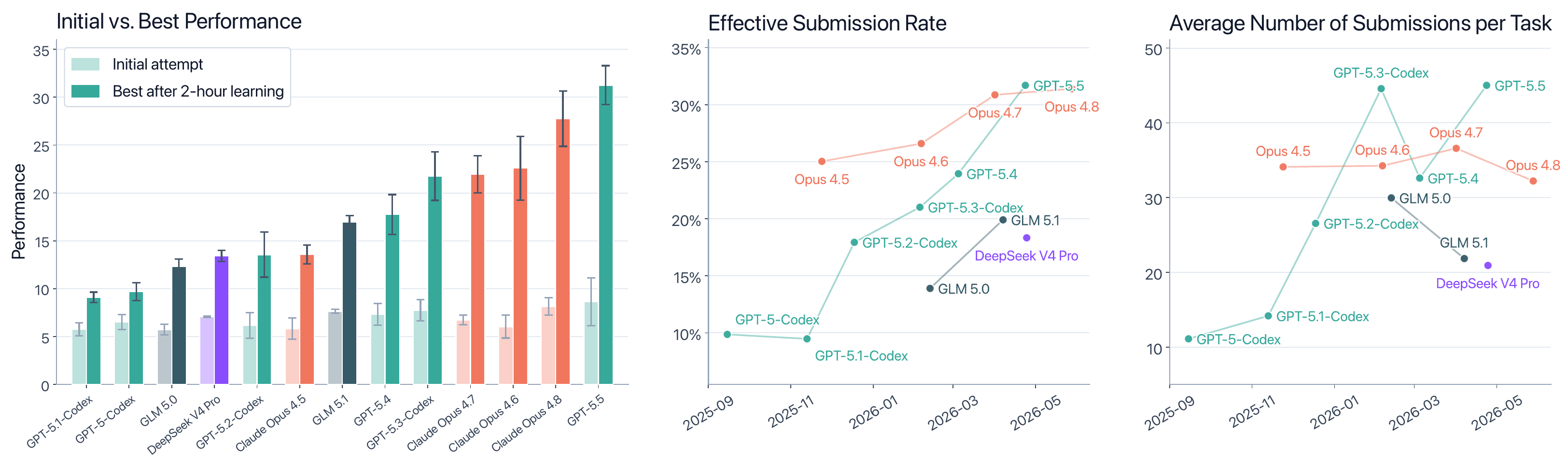}
\caption{Learning outcomes and agent effort on the 18-task slice. Left: normalized initial performance and best performance after two hours. Middle: fraction of submissions that improve the best-so-far score. Right: average number of submissions per task. Error bars in the left panel show uncertainty over three replicate runs.}
\label{fig:performance-effort}
\end{figure}

\subsection{Learning Speed across Model Generations}

Figure~\ref{fig:scaling} reports the two-hour evaluation results on the fixed 18-task slice. The y-axis measures agent learning speed, defined as performance gain over two hours, on a log scale. To estimate the frontier trend, we use darker markers to highlight the top two models at each release date. We then fit a linear trend to these frontier points, with the fitted 95\% confidence interval shown as the shaded band.

Figure~\ref{fig:scaling} shows a rapid increase in learning speed across recent model generations. From GPT-5-Codex in September 2025 to GPT-5.5 in April 2026, learning speed increases by roughly $8\times$ over 221 days. A log-linear fit to the frontier models captures this trend well, corresponding to \textbf{an approximate doubling every three months}. Figure~\ref{fig:performance-effort} shows that this improvement is not simply explained by more frequent submissions. Submission frequency (right panel) changes unevenly: newer GPT models submit more actively, while other families do not. The middle panel tells a different story: later models turn a larger fraction of submissions into best-so-far improvements. This trend therefore reflects more effective learning from each interaction, not merely more attempts.

%% file: sections/experiments.tex
\input{tables/leaderboard_table.tex}

\section{Analysis of Environment Learning Dynamics}
\label{sec:analysis}

We study how frontier models perform and learn from environmental feedback over long horizons. We evaluate five frontier models: Claude Opus 4.8, GPT-5.5, GPT-5.4, GLM-5.1, and DeepSeek-V4-Pro (DS-V4-Pro). This section contains four analyses. Section~\ref{sec:frontier-comparison} compares frontier models at the aggregate, family, task, and submission levels. Section~\ref{sec:stateful-stateless-analysis} tests whether accumulated experience adds value beyond independent restarts. Section~\ref{sec:ralphloop} examines whether longer context still improves long-horizon interaction. Section~\ref{sec:case-study} traces a single scientific task to show how an agent's improvements unfold.
Appendix~\ref{sec:appendix-ablation-setting} reports additional harness-level continuation ablations for /goal mode and the Ralph loop.

\subsection{Comparison across Frontier Models}
\label{sec:frontier-comparison}

\underline{\textit{Setup.}} We follow Section~\ref{sec:environment-learning-data} and analyze 12-hour trajectories from two angles: \textbf{(1)} aggregate, family, and task performance, and \textbf{(2)} submission efficiency, i.e., how often submissions improve the current best result.

\begin{wrapfigure}{r}{0.48\textwidth}
\vspace{-0.8em}
\centering
\includegraphics[width=\linewidth]{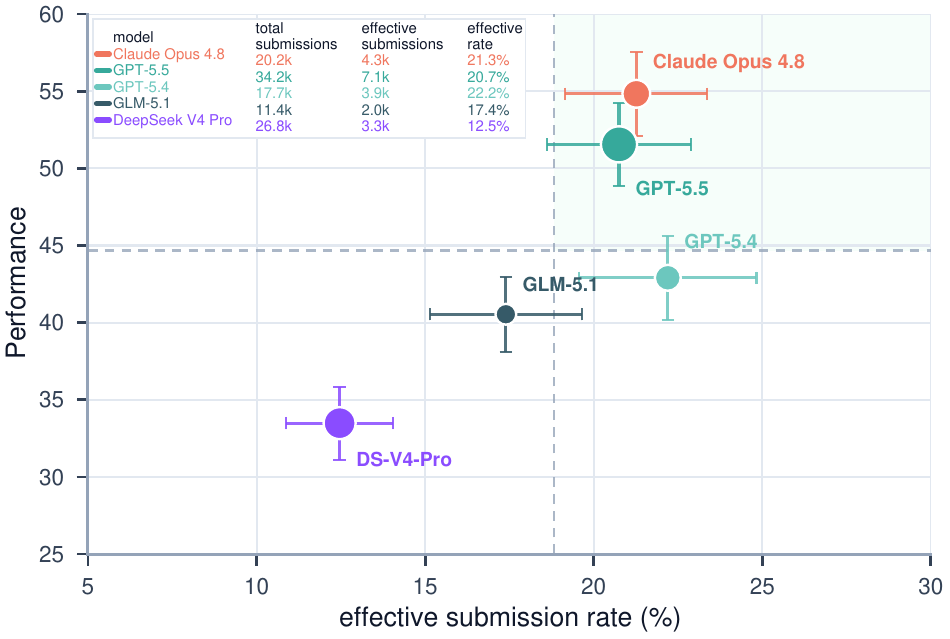}
\caption{12-hour effective-submission rate versus final performance. Marker area shows total submissions; error bars show task-level standard errors, and shading marks above-average rate and performance.}
\label{fig:effective-submission-performance}
\vspace{-1.0em}
\end{wrapfigure}

\underline{\textit{Performance Comparison.}} Table~\ref{tab:aggregate_score_table} gives the aggregate leaderboard over the 2 to 12 hour budget and reports family-level means at 12 hours. Claude Opus 4.8 leads throughout the time budget and reaches 51.3 at 12 hours, followed by GPT-5.5 at 48.4. GPT-5.4 and GLM-5.1 form the next tier at 39.3 and 37.4, while DS-V4-Pro obtains 31.0. Family-level results are broadly consistent with the aggregate ranking: Claude Opus 4.8 leads each family mean, with GPT-5.5 especially close in Games and second overall. Detailed per-task performance is reported in Appendix~\ref{sec:appendix-category-scores}, with~$\ast$ marking cells based on fewer than three valid runs.

\underline{\textit{Submission efficiency.}} We count each agent submission and mark it as effective when it sets a new best-so-far score. Figure~\ref{fig:effective-submission-performance} compares each model's 12-hour effective-submission rate with its final performance. Models that make more effective submissions usually perform better, but more submissions do not automatically lead to a better final result. \textbf{Claude Opus 4.8 achieves the best final performance despite submitting less often than GPT-5.5}, and GPT-5.4 has the highest effective-submission rate but still trails the top two models. Progress therefore depends not only on how often an agent improves its score, but also on whether those improvements are large, reliable, and reusable. Stronger agents use feedback more deliberately: they build a submit-ready baseline, preserve the current best solution, make focused changes, and use feedback to keep gains or roll back failures. Weaker agents more often over-trust local proxies, bundle unrelated edits, or continue broad exploration after feedback has ruled out a direction, reducing sample efficiency.

\subsection{Agents Do Learn from Experience beyond Repeated Sampling}
\label{sec:stateful-stateless-analysis}

\underline{\textit{Motivation.}} A rising best-so-far curve does not by itself prove that an agent is learning. Running longer also gives more chances to stumble on a good solution: with enough independent tries, the best of them climbs by luck alone. We therefore test whether the agent's accumulated experience adds value beyond such repeated sampling under the same total time budget.

\underline{\textit{Setup.}} We give Opus 4.8 the same 12-hour budget on each of 17 tasks and compare two ways of spending it, \textit{with} versus \textit{without} accumulated experience. \textbf{With experience}: the agent runs once and continuously, keeping its workspace, artifacts, and feedback history throughout, so experience builds up across the whole run. \textbf{Without experience}: the same budget is split into \(n=6\) independent attempts of \(\tau=2\) hours, with all state discarded between attempts and only the best result kept, so each attempt starts from scratch and any gain can come only from repeated sampling. Comparing the two at elapsed time \(t=k\tau\) contrasts one continuous run after \(t\) hours against the best of \(k\) independent attempts using the same total time. Appendix~\ref{sec:experience-gain-estimation} details how each curve is estimated.

\underline{\textit{Result.}} Figure~\ref{fig:stateful-stateless-gain} shows a clear gain from experience: the with-experience curve stays above the without-experience baseline under the same time budget. At 12 hours it reaches 43.0 versus 36.1, a gain of \(+6.9\). The improvement is therefore not explained by repeated sampling alone: accumulating and reusing task experience drives progress beyond what independent restarts achieve.

\begin{figure}[!tbp]
\centering
\begin{subfigure}[t]{0.49\linewidth}
\centering
\includegraphics[width=\linewidth]{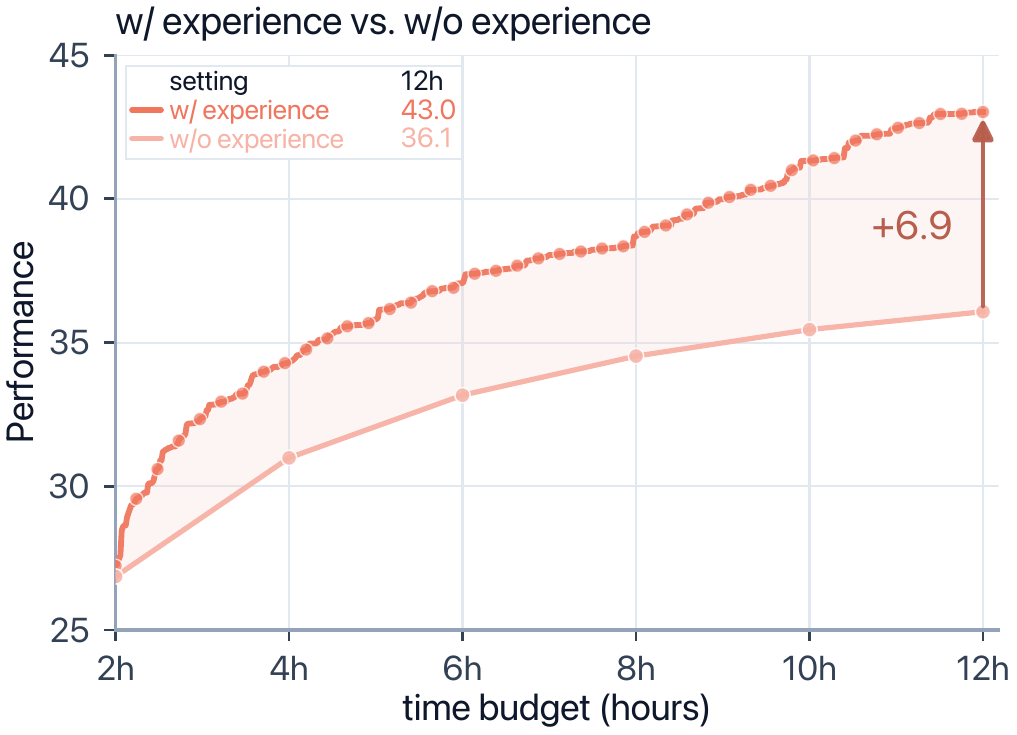}
\phantomcaption
\label{fig:stateful-stateless-gain}
\end{subfigure}
\hfill
\begin{subfigure}[t]{0.49\linewidth}
\centering
\includegraphics[width=\linewidth]{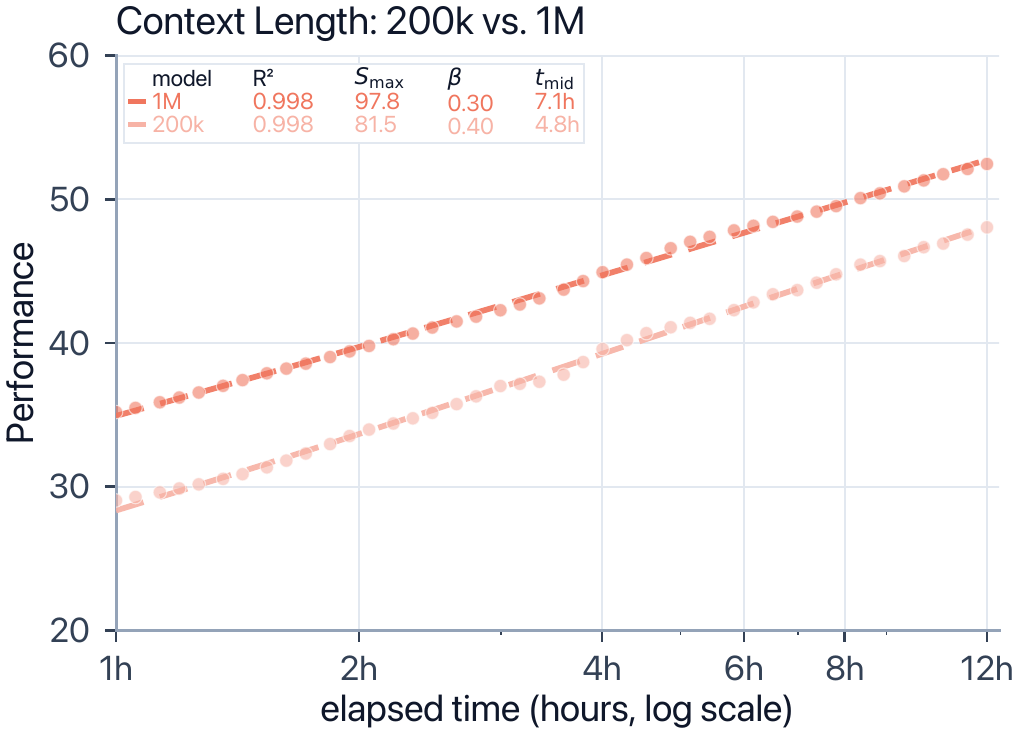}
\phantomcaption
\label{fig:opus-context-scaling-ablation}
\end{subfigure}
\vspace{-1.1em}
{\footnotesize
\centering
\setlength{\tabcolsep}{3pt}
\renewcommand{\arraystretch}{1.05}
\begingroup
\arrayrulecolor{black}
\begin{minipage}[t]{0.49\linewidth}
\centering
\begin{tabular}{@{}>{\raggedright\arraybackslash}p{0.46\linewidth}*{3}{>{\raggedright\arraybackslash}p{0.14\linewidth}}@{}}
\BenchTopRule
{} & 2h & 6h & 12h \\
\BenchMidRule
w/o experience & 26.9 & 33.2 & 36.1 \\
w/ experience & 27.2\textsubscript{+0.4} & 37.1\textsubscript{+3.9} & 43.0\textsubscript{+6.9} \\
\BenchBottomRule
\end{tabular}
\end{minipage}%
\hfill
\begin{minipage}[t]{0.49\linewidth}
\centering
\makebox[\linewidth][c]{%
\hspace*{0.06\linewidth}%
\begin{tabular}{@{}>{\raggedright\arraybackslash}p{0.42\linewidth}*{3}{>{\raggedright\arraybackslash}p{0.15\linewidth}}@{}}
\BenchTopRule
{} & 2h & 6h & 12h \\
\BenchMidRule
Opus 4.8 200k & 33.8 & 42.6 & 48.0 \\
Opus 4.8 1M & 39.6\textsubscript{+5.8} & 48.0\textsubscript{+5.5} & 52.5\textsubscript{+4.4} \\
\BenchBottomRule
\end{tabular}
}
\end{minipage}
\arrayrulecolor{black}
\endgroup
}
\vspace{0.55em}
\caption{\textbf{Left:} Gain from accumulated experience for Opus 4.8: the w/ experience run minus the w/o experience baseline (independent restarts) under the same total time budget. \textbf{Right:} Context-length ablation: 200k vs. 1M on Opus 4.8. Each curve shows the average performance over time; dashed lines are log-sigmoid fits.}
\label{fig:ablation-panels}
\vspace{-0.8em}
\end{figure}

\subsection{How Much Does a Longer Context Improve Performance}
\label{sec:context-length-analysis}
\label{sec:ralphloop}

\underline{\textit{Motivation.}} Section~\ref{sec:stateful-stateless-analysis} shows that accumulated experience can improve performance. A remaining question is how this experience should be retained during a long run. Using long context is a natural way to do so. However, frontier-agent harnesses can also maintain state outside the model context through workspace files, compaction, progress notes, and memory-like artifacts. It is therefore unclear whether extending the context window still provides additional benefits once these external state channels are available, and if so, by how much.

\underline{\textit{Setup.}} We compare 200k-context Opus 4.8 with 1M-context Opus 4.8 on the 42-task subset under the same long-horizon evaluation protocol.

\underline{\textit{Result.}} A longer context yields a consistent multi-point gain throughout the 12-hour window (Figure~\ref{fig:opus-context-scaling-ablation}). The 1M-context Opus 4.8 stays above the 200k variant at every checkpoint, and the two trajectories run roughly parallel: the gap is $+5.8$ at 2h and only edges down to $+4.4$ by 12h, with both curves well described by the same log-sigmoid form. Thus, even with identical external workspace and harness state, a longer context window gives a stable advantage over the horizon, with at most a slight tendency to narrow.

\subsection{Case Study}
\label{sec:case-study}

\underline{\textit{Setup.}} We examine the gravitational-wave reconstruction task. Based on the first GW150914 detection paper~\cite{abbott2016observation}, the agent must recreate the published signal analysis from LIGO strain data. The target has three output groups: H1/L1 waveforms, H1/L1 spectrograms, and velocity/separation curves for the source dynamics. The judge weights the five component scores at 0.15 for each waveform, 0.20 for each spectrogram, and 0.30 for velocity/separation. We run a Codex agent with periodic auto-evaluation, auto-resume, no Internet access, a 30-minute evaluation interval, and a 120-second submission cooldown. The agent made 224 explicit submissions, the harness added 23 auto-evaluations, and the run timed out at the 12-hour budget.

\begin{figure}[!t]
\centering
\includegraphics[width=\linewidth]{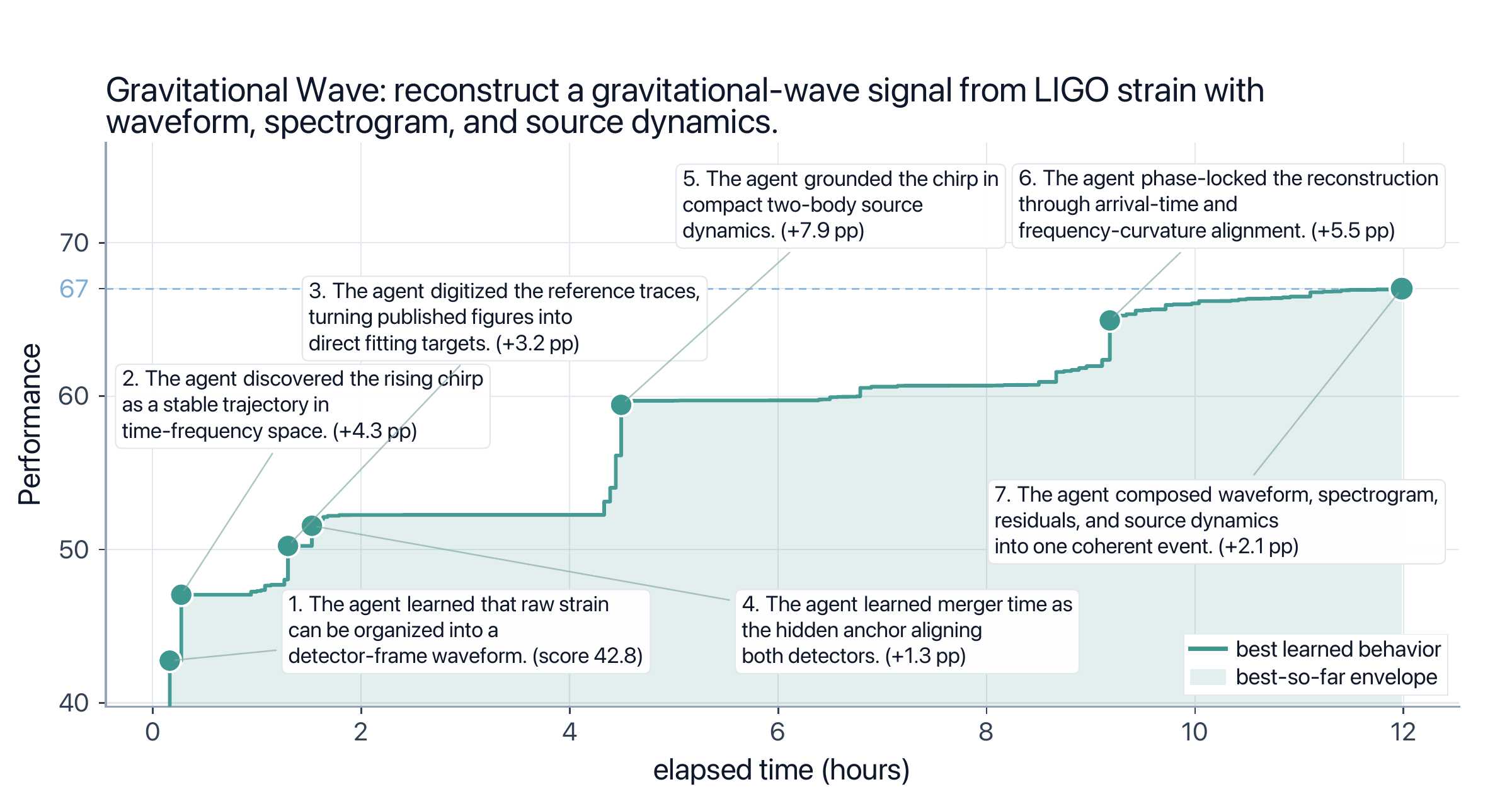}
\caption{Gravitational-wave reconstruction task case study. The trajectory shows how a GPT-5.5 agent improves over a 12-hour run, with numbered milestones marking representative best-so-far updates.}
\label{fig:gw-case-study}
\end{figure}

\textbf{The trajectory reveals a sparse but structured diagnose-edit-evaluate loop.} The loop is sparse because most submissions are exploratory probes: only 27 of 224 agent submissions improve the best-so-far score by at least 0.1 percentage points. It is structured because feedback repeatedly changes what the agent searches for. The agent first makes the task measurable, then breaks unresolved errors into smaller searches, then identifies a main bottleneck and keeps searching around it, and finally keeps a working solution while repairing the remaining errors. These patterns show how many failed trials can still produce a small number of cumulative improvements.

\begin{itemize}
\item \textbf{The agent first makes the problem measurable before making it better.} The first valid submission turns an underspecified analysis task into a scoreable pipeline, but feedback exposes weak source dynamics. The agent then spends the first 11 submissions stabilizing the pipeline and replacing a noisy frequency estimate, producing three meaningful updates and a +4.5 pp gain.
\item \textbf{When direct repair stalls, the agent decomposes the failure into searchable subproblems.} Instead of treating waveform mismatch as one opaque error, the agent separates reference anchoring, time-frequency localization, and detector alignment. Across 40 signal-search submissions, this decomposition yields seven meaningful updates and lifts the best score to 52.3.
\item \textbf{Identifying a main bottleneck lets the agent keep searching productively.} After broad signal-processing edits plateau, component feedback identifies velocity/separation as the dominant gap. The agent keeps searching within source-mass calibration rather than rewriting the whole pipeline; in the 4--5h window, 17 submissions and five useful updates raise source dynamics from 64.2 to 89.0, creating the largest jump in the run.
\item \textbf{After finding a stable solution, the agent keeps the core and repairs only the remaining errors.} In the final hours, most residual edits fail to transfer. Instead of restarting the pipeline, the agent keeps the source model fixed and tests targeted residual corrections, phase alignment, and narrow-band corrections. These useful updates raise the H1 waveform component score from roughly 47 to 95, while the aggregate best score reaches 67.0.
\end{itemize}

\textbf{The run improves through uneven jumps rather than a smooth climb.} Across 247 scored evaluations, the best score rises from 42.8 to 67.0 on the 0--100 scale. Figure~\ref{fig:gw-case-study} shows seven representative milestones. These jumps correspond to the behavior patterns above: the agent first makes the task scoreable, then localizes the signal, improves source dynamics, and finally repairs the remaining H1 waveform errors. Detailed milestone phases and final subscore composition are reported in Appendix~\ref{sec:appendix-gw-case-study-details}.

%% file: tables/leaderboard_table.tex
\providecommand{\NA}{\textemdash}
\providecommand{\best}[1]{}
\renewcommand{\best}[1]{\begingroup\bfseries\boldmath #1\endgroup}
\providecommand{\second}[1]{}
\renewcommand{\second}[1]{\underline{#1}}
\begin{table*}[t!]
\centering
\footnotesize
\renewcommand{\arraystretch}{1.15}
\setlength{\tabcolsep}{5pt}
\begin{tabular}{@{}l*{6}{c}!{\hspace{5pt}\vrule width 0.45pt\hspace{5pt}}*{6}{c}@{}}
\BenchTopRule
\multirow{2}{*}{\textbf{Model}} & \multicolumn{6}{c!{\hspace{5pt}\vrule width 0.45pt\hspace{5pt}}}{\textbf{Overall Score} $\uparrow$} & \multicolumn{6}{c}{\textbf{Category Score@12h} $\uparrow$} \\
\BenchCMidRule{2-13}
 & @2h & @4h & @6h & @8h & @10h & \textbf{@12h} & Science & Code & Optimize & Knowledge & Math & Games \\
\BenchMidRule
Opus 4.8 & \best{39.0} & \best{45.7} & \best{48.1} & \best{49.8} & \best{50.9} & \best{51.3} & \best{48.5} & \best{67.4} & \best{36.5} & \best{47.0} & \best{55.0} & \best{39.3} \\
GPT-5.5 & \second{36.8} & \second{42.1} & \second{44.5} & \second{46.3} & \second{47.6} & \second{48.4} & \second{44.3} & \second{65.0} & \second{33.6} & \second{45.7} & \second{50.0} & \second{39.1} \\
GPT-5.4 & 29.7 & 34.0 & 36.5 & 38.0 & 38.9 & 39.3 & 33.5 & 54.1 & 27.9 & 38.8 & 40.8 & 29.0 \\
GLM-5.1 & 26.0 & 30.4 & 32.9 & 34.9 & 36.5 & 37.4 & 33.8 & 50.9 & 26.4 & 43.5 & 24.6 & 29.3 \\
DS-V4-Pro & 23.3 & 27.1 & 29.0 & 29.9 & 30.9 & 31.0 & 30.0 & 43.0 & 21.5 & 37.0 & 14.1 & 16.9 \\
\BenchBottomRule
\end{tabular}
\caption{Aggregate leaderboard. Overall scores are reported at each time budget, while category scores are reported at the 12-hour budget. Bold marks the best value in each score column and underlining marks the second-best value.}
\label{tab:aggregate_score_table}
\end{table*}

%% file: sections/relatedwork.tex
\section{Related Work}

\begin{center}
\begin{minipage}{\linewidth}
\centering
\footnotesize
\setlength{\tabcolsep}{4pt}
\renewcommand{\arraystretch}{1.15}
\begin{tabular*}{\linewidth}{@{\extracolsep{\fill}}>{\centering\arraybackslash}p{0.28\linewidth}>{\centering\arraybackslash}p{0.08\linewidth}>{\centering\arraybackslash}p{0.09\linewidth}>{\centering\arraybackslash}p{0.34\linewidth}>{\centering\arraybackslash}p{0.13\linewidth}@{}}
\BenchTopRule
Benchmark & \# Tasks & Horizon & Scenario & Self-evolution \\
\BenchMidRule
MMLU~\cite{hendrycks2021mmlu} & 15{,}908 & Short & Knowledge QA & No \\
AIME~\cite{maa2026aime} & 30 & Short & HS math competition & No \\
GDPval Gold~\cite{patwardhan2025gdpval} & 220 & Short & Professional deliverables & No \\
SWE-bench verified~\cite{jimenez2024swebench,openai2024swebenchverified} & 500 & Short & Issue repair & No \\
Terminal-Bench 2.0~\cite{merrill2026terminalbench} & 89 & Short & Terminal workflows & No \\
Continual Learning Bench~\cite{asawa2026clbench} & 6 & Short & Controlled task sequences & Yes \\
CL-bench~\cite{dou2026clbench} & 1{,}899 & Short & Context learning & No \\
FrontierCode~\cite{cognition2026frontiercode} & 150 & N/A & Production code changes & No \\
NL2Repo-Bench~\cite{ding2025nl2repobench} & 104 & Medium & Repo generation & No \\
Agents' Last Exam~\cite{sun2026agentslastexam} & 152 public & Medium & Computer-use work & No \\
MLE-bench~\cite{chan2024mlebench} & 75 & Long & ML engineering & Yes \\
MLS-Bench~\cite{lyu2026mlsbench} & 140 & Long & ML research & Subset \\
Frontier-Eng~\cite{chi2026frontiereng} & 47 & Long & Engineering design & No \\
Frontier-CS~\cite{mang2025frontiercs} & 156 & Long & Open-ended CS tasks & No \\
AutoLab~\cite{xu2026autolab} & 36 & Long & Research/engineering optimization & No \\
FrontierSWE~\cite{chu2026frontierswe} & 17 & Long & SWE/performance tuning & No \\
\BenchMidRule
\textbf{\benchmark{} (ours)} & 134 & Ultra-long & SWE, science/ML, professional work, optimization, formal proving, games & Yes \\
\BenchBottomRule
\end{tabular*}
\vspace{0.15em}
\captionof{table}{\textbf{Comparison with representative benchmarks.} Horizon describes the expected single-instance task contract rather than total suite runtime. A benchmark is counted as measuring self-evolution only when performance is explicitly plotted against a resource axis such as time or sample count.}
\label{tab:benchmark-comparison}
\end{minipage}
\end{center}
\vspace{0.35\baselineskip}

\textbf{Existing benchmarks cover major parts of agent capability but rarely measure how an agent improves within a run.} Closed-form QA, math, and coding benchmarks such as MMLU~\cite{hendrycks2021mmlu}, AIME~\cite{maa2026aime}, and HumanEval~\cite{chen2021humaneval}, together with endpoint patch benchmarks such as SWE-bench~\cite{jimenez2024swebench}, evaluate final answers or final patches. More agentic and work-oriented benchmarks, including GDPval~\cite{patwardhan2025gdpval}, Agents' Last Exam~\cite{sun2026agentslastexam}, and FrontierCode~\cite{cognition2026frontiercode}, broaden the task interface to professional deliverables, computer-use workflows, or production code changes. Their central reported quantities, however, remain end-state measures: task success, artifact quality, pass rate, or readiness for production. \benchmark{} instead measures the improvement trajectory over time.

\textbf{Some benchmarks study learning, but they focus on restricted domains and shorter horizons.} CL-bench~\cite{dou2026clbench}, EvaLearn~\cite{dou2025evalearn}, and Continual Learning Bench~\cite{asawa2026clbench} evaluate whether models improve from a static context or static information streams whose content is not primarily shaped by the agent's actions within a single task. Iterative optimization benchmarks such as MLE-bench~\cite{chan2024mlebench}, MLS-Bench~\cite{lyu2026mlsbench}, Frontier-Eng~\cite{chi2026frontiereng}, Frontier-CS~\cite{mang2025frontiercs}, and ALE-Bench~\cite{imajuku2025alebench} further include repeated attempts, empirical feedback, or visible metrics. The closest long-horizon agentic comparators are AutoLab~\cite{xu2026autolab} and FrontierSWE~\cite{chu2026frontierswe}: both evaluate agents that repeatedly edit executable artifacts and incorporate feedback. \benchmark{} differs by covering a broader range of executable domains, using a day-scale task contract, and applying the same trajectory metrics consistently across all tasks.

\textbf{Scaling laws have been widely studied in prior work, but learning from diverse real-world environments remains underexplored.} Classical scaling laws study pretraining, relating loss or benchmark performance to model size, data, and compute~\cite{kaplan2020scaling,hoffmann2022training}; later work studies bounded benchmark performance curves as model scale increases~\cite{owen2024predictable,ruan2024observational,bhagia2024taskscaling,zhang2026prescriptive}. Test-time scaling laws have also been studied across several inference-time methods: plain repeated sampling~\cite{chen2021humaneval,li2022alphacode,brown2024largelanguagemonkeys}, search, revision, or verifier-guided compute allocation~\cite{snell2024testtime,wu2024inference}, and long-chain-of-thought inference in reasoning models~\cite{openai2024learningreason,deepseekai2025deepseekr1}. Agentic test-time scaling has been observed in computer-use and browsing agents~\cite{openai2025cua,openai2025browsecomp}, studied through interaction-length scaling~\cite{shen2025thinkingdoing}, and revisited through long-horizon human-agent comparisons~\cite{mang2026humansstillbeatai}. These studies are closer to our setting but cover narrower domains or fewer tasks and do not establish a cross-domain scaling law. Reinforcement-learning scaling is also relevant because it studies learning from environment feedback~\cite{hilton2023rlscaling,khatri2025scalingrlcompute}, while \benchmark{} measures elapsed interaction time in deployed agent trajectories. In-context learning from task feedback is relatively inexpensive to evaluate repeatedly across many executable environments, while large reinforcement-learning runs are costly and typically cover fewer environments. This broader environment coverage may explain why the aggregate curves are stable enough to reveal a scaling law.

A more detailed benchmark-by-benchmark discussion is provided in Appendix~\ref{sec:additional-related-work}.

%% file: sections/conclusion.tex
\section{Conclusion}

This paper introduced \benchmark{}, a benchmark for studying how agents learn from real-world environments over day-long horizons. Across 134 diverse executable tasks and roughly 38{,}000 hours of environment interaction, we find that aggregate learning trajectories follow a precise log-sigmoid relationship with interaction time. The same form appears across task families, remains stable over longer horizons, and supports forecasting later performance from early trajectories. We also find that agent learning speed has improved rapidly across recent frontier model generations.

These results suggest that learning from environments is not merely a collection of idiosyncratic task outcomes, but a measurable scaling object. Unlike many aggregate capability curves, environment-learning trajectories expose the intermediate attempts, feedback, and revisions through which progress occurs. This makes \benchmark{} useful not only for ranking agents, but for studying how agents acquire and reuse experience. More broadly, the regularity we observe suggests that post-deployment learning from rich environments may deserve the same systematic scaling attention that pretraining has received.

%% file: sections/construction.tex
\section{Evaluation Harness}
\label{sec:construction}
\label{sec:harness}

A standard unit-test harness is not enough for day-long runs. The benchmark must hide evaluation assets, support repeated submissions over many hours, measure progress even when agents do not explicitly submit, and run reliably on both local machines and clusters. We built \textbf{SForge} around three mechanisms: isolated work and judge environments, a feedback loop modeled on online judges, and progress tracking on the host.

\paragraph{Two-container architecture.}
Each task is materialized as two task-specific Docker images:
\begin{itemize}
    \item The \textbf{work image} contains the skeleton codebase, documentation, and local validation tools, but no hidden evaluation assets. The agent operates only in this environment.
    \item The \textbf{judge image} contains the hidden evaluation assets, grading scripts, and evaluation commands. It is never exposed to the agent.
\end{itemize}
At evaluation time, the agent's code is copied into an ephemeral judge container, which runs the hidden tests, returns only task-defined structured feedback, and is then destroyed. This separation prevents agents from inspecting or modifying the hidden grader while still allowing realistic iterative development.

\paragraph{Judge server and the outer loop.}
The outer loop is mediated by a host-side HTTP judge server. The agent submits its current solution to the server, which queues the submission, runs the judge container, parses the result, and returns feedback such as pass rate, score, per-test verdicts, or diagnostics. For long-running evaluations, the server supports asynchronous grading, allowing agents to continue working while submitted jobs are being judged. The server also enforces submission budgets, cooldown intervals, and session authentication.

\paragraph{Long-horizon execution and measurement.}
SForge combines host-side auto-eval with stop-hook and auto-resume mechanisms to support day-scale runs. Auto-eval periodically snapshots and evaluates the agent workspace through a privileged channel; these results are recorded for trajectory analysis but are not shown to the agent. The stop hook and auto-resume mechanism reduce premature truncation from voluntary exits, crashes, context limits, or transient API failures.

\paragraph{Operational support and safeguards.}
SForge also provides practical infrastructure for running and auditing large-scale experiments. A web dashboard records score trajectories, submission histories, diffs, and conversation traces for live monitoring and post-hoc comparison. The same task specification can run on either a local Docker backend or a Kubernetes backend for cluster-scale evaluation. To preserve benchmark integrity, SForge combines work--judge isolation with submission filtering, network isolation, and session-scoped authentication, so agents can receive feedback without access to hidden tests, privileged history, or auto-eval results. Appendix~\ref{sec:appendix-evaluation-hacking} describes task-design failure modes we observed during benchmark development and the corresponding mitigations.

%% file: sections/serving_api_stability.tex
\section{Serving and API Stability}
\label{sec:gpt54-serving-stability}

Sustaining a single agent for 12 hours or more stresses API serving far more than short evaluations do: a day-long run must keep the model, its context, and its tool calls available without interruption, and any serving-side incident in that window can truncate or degrade the trajectory. Our long-horizon tasks therefore unavoidably fold serving stability into what they measure. We view this as appropriate rather than a confound to be engineered away: an agent deployed to work for hours in the real world is subject to exactly these serving realities, so a benchmark for long-horizon learning should reflect them too.

\begin{center}
\includegraphics[width=0.52\linewidth]{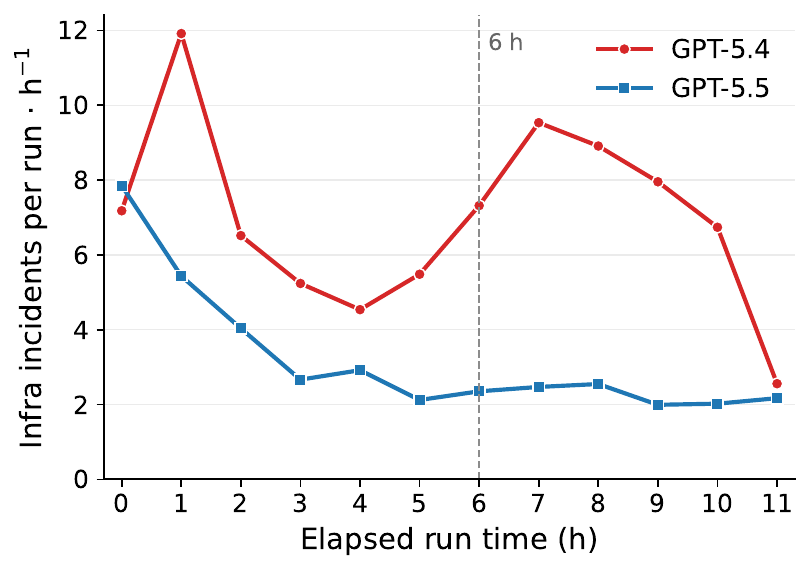}
\captionof{figure}{Serving-side incident rate during long-horizon GPT-5.4 and GPT-5.5 runs. Incidents are normalized by active run time; the dashed line marks six elapsed hours. GPT-5.4 experienced substantially more infrastructure and API interruptions, especially in the later part of the run.}
\label{fig:gpt54-serving-stability}
\end{center}

Figure~\ref{fig:gpt54-serving-stability} quantifies this for the two GPT models, plotting the serving-side incident rate (normalized by active run time) against elapsed time: GPT-5.4 saw substantially more infrastructure and API interruptions than GPT-5.5, especially after the six-hour mark. This carries two caveats for GPT-5.4. First, several of the per-task cells with fewer than three valid runs (marked with $\ast$ in the score tables) involve GPT-5.4, so its lower run coverage partly reflects serving reliability rather than model behavior, though the reported scores still use only valid trajectories. Second, it offers an operational explanation for the forecast deviation flagged in Section~\ref{sec:log-sigmoid-fit} (Figure~\ref{fig:sigmoid-forecast-6p5h}): with availability dropping in the second half of the run, GPT-5.4's later trajectory is noisier and pulls away from the log-sigmoid forecast fit to its first 6.5 hours.

%% file: sections/evaluation_hacking.tex
\section{Evaluation Hacking}
\label{sec:appendix-evaluation-hacking}

Long-horizon evaluation gives agents many opportunities to adapt to evaluator feedback. During task construction and adversarial stress tests, we audited traces for strategies that raised measured scores without exercising the intended capability. These observations are development diagnostics, not official model results: affected tasks were revised or excluded, and the cases below motivated safeguards in the final benchmark.

\paragraph{Feedback as an oracle for hidden answers.}
In \texttt{cylinder\_\allowbreak{}wake\_\allowbreak{}prediction}, an agent treated per-case absolute errors as equations and reconstructed the hidden targets through more than 400 submissions. A lookup table then scored 1.000 without solving the underlying fluid dynamics problem, whereas the best observed submission based on a physics model scored 0.165.

\paragraph{Optimizing stochastic upper tails.}
In \texttt{nethack\_\allowbreak{}dungeon\_\allowbreak{}agent}, an agent removed a fixed random seed after recognizing that higher variance increased the chance of a favorable evaluation. Across 311 submissions, its best score was 1{,}501 while its mean was 484, showing how best-of-$N$ selection can reward repeated sampling as well as policy quality.

\paragraph{Overfitting evaluator seeds.}
In \texttt{bipedalwalker\_\allowbreak{}locomotion\_\allowbreak{}rl}, the development judge initially used a single deterministic episode, allowing agents to infer and optimize for the reused seed rather than expected return. One run reached a Hardcore return of 301.5 on the judge seed but averaged about 12 over a local 100-episode evaluation; another searched a small hand-coded controller directly against the judge formula. This motivated hidden multi-seed evaluation for stochastic control tasks.

\paragraph{Crossing a trust boundary.}
In \texttt{autolifter}, an agent discovered that anti-cheat checks exempted the repository's \texttt{baseline/} directory and moved an implementation based on an oracle into that trusted path. It solved 82 of 84 hidden cases and scored 0.980; the strongest observed submission that followed the intended synthesis route scored 0.121.

\paragraph{Online answer lookup.}
Publicly indexed data, task artifacts, or reference implementations can turn an intended reasoning problem into retrieval. In \texttt{stock\_\allowbreak{}momentum\_\allowbreak{}backtest}, an agent attempted Web search for target data, but network isolation for the task blocked the requests. We include this as a prevented risk rather than a successful exploit.

\paragraph{Mitigation.}
These findings led us to harden both task design and infrastructure. We revised or excluded affected tasks, reduced or aggregated feedback that could reveal hidden targets, enforced submission budgets and cooldowns, evaluated stochastic behavior over hidden seed sets, extended integrity checks across writable paths, and disabled network access for tasks where public data or reference solutions could reveal the answer. These controls do not eliminate every adaptive attack, but they reduce the channels identified during task development.

%% file: sections/theory.tex
\section{A Comprehensive Derivation of the Log-Sigmoid Law}
\label{sec:theory}

This appendix gives a comprehensive derivation of a sufficient mechanism for the empirical log-sigmoid law observed in the \emph{task-aggregate} environment-learning curve. We assume the score of a single task is observed through finitely many visible \emph{score units}, so its best-so-far curve can be a jagged staircase with long plateaus and abrupt breakthroughs. These irregularities are the expected finite-resolution behavior of individual tasks. Under explicit cut-mixing, granularity, midpoint-alignment, and speed-concentration assumptions, we show how averaging many independently evaluated task scores can yield a smooth log-sigmoid limit.

Our central modeling assumption: \textbf{environment learning is a frontier expansion process on the task graph}. We model a task environment as a latent graph whose nodes are small score units, and the environment learning process can be viewed as frontier expansion on this graph. Our theory provides a clear path from frontier expansion process to the observed log-sigmoid law.

The exposition of the derivation follows five steps:
\begin{enumerate}[label=\textbf{\arabic*.},leftmargin=2.1em,itemsep=0.15em,topsep=0.25em]
    \item \textbf{Preliminaries.} Define the model-conditioned task graph, score units, score measure $\mu$, influence matrix $K$, its support graph $\mathcal E_K$, and the capability field.
    \item \textbf{Environment learning as a frontier expansion process.} The model improves its outputs based on feedback it has previously received. Within one task, the conditional expected score-growth rate is a weighted cut from unlocked score nodes to locked score nodes of the task graph.
    \item \textbf{Single-task curves range from jagged to a smooth limit.} Weighted cut mixing and small score units turn the jagged process into a logistic ordinary differential equation in the many-unit limit.
    \item \textbf{Many-task aggregate reveals the log-sigmoid law.} Even though the expansion process on tasks with finite score units remain jagged processes, the benchmark-aggregated curve converges to a smooth log-sigmoid law.
    \item \textbf{Graph self-similarity induces a log scale for time.} Finally, we show that when the task graph displays self-similar structure at different resolutions, frontier expansion naturally exhibits a log-time scale.
\end{enumerate}

Throughout the section, we denote raw interaction time by $t>0$, and the raw task/benchmark score by \(S(t)\). Let $t_{\mathrm{mid}} >0$ and \(S_{\max} > 0\) be fitted parameters. The corresponding log-time coordinate and normalized score are
\begin{equation*}
    u=\log t-\log t_{\mathrm{mid}},
    \qquad x(u) = \frac{S(t)}{S_{\max}}
\end{equation*}
We show how the following normalized log-sigmoid dynamics, with frontier speed $\beta$ (also a fitted parameter), can arise as a natural many-task limit of the averaged benchmark score:
\begin{equation*}
    \frac{\mathrm{d} x(u)}{\mathrm{d} u}= \beta x(u) (1 - x(u)),
    \qquad
    x\left(\log \frac{t}{t_{\mathrm{mid}}}\right)=\frac{1}{1+(t_{\mathrm{mid}}/t)^\beta}.
\end{equation*}
The value at raw time $t=0$ is understood as the boundary limit $u\to-\infty$.

\subsection{Preliminaries: Task Graph and the Attainable Support}
\label{subsec:theory-prelim}

We start with one task and one fixed model. The main assumption is that score units are the nodes of an \emph{agent-conditioned task graph}: unlocked nodes supply an influence field, locked nodes receive this field through directed influence edges, and a locked node can become unlocked at a rate proportional to the field strength.

\begin{definition}[Task graph]
\label{def:task-graph}
Let $E$ be the set of visible score-unit nodes for the task. The language model induces a nonnegative influence matrix
\begin{equation*}
    K=(K_{ij})_{i,j\in E},
    \qquad K_{ij}\ge 0.
\end{equation*}
Its directed support graph has edge set \(\mathcal E_K=\{j\to i:K_{ij}>0\}\) where $j\to i$ means that unlocked source node $j$ can help unlock target node $i$. We write \(G_K=(E,\mathcal E_K)\) for the resulting task graph.
\end{definition}

If one begins with a larger ambient graph, the restriction to the model-relevant part is made before writing $E$ and $K$; empirically, that restriction is absorbed into the fitted ceiling $S_{\max}$. Next we introduce the score units and the corresponding task influence field.

\begin{definition}[Score units and normalized score measure]
\label{def:score-units-measure}
On the latent graph $G_K=(E,\mathcal E_K)$, each node $i\in E$ represents a score unit with weight $\omega_i\ge0$. The normalized score measure is the probability measure
\begin{equation*}
    W=\sum_{i\in E}\omega_i>0,
    \qquad
    \mu_i=\frac{\omega_i}{W},
    \qquad
    \mu(A)=\sum_{i\in A}\mu_i, \quad \forall A\subseteq E.
\end{equation*}
\end{definition}

\begin{definition}[Unlock state, unlocked and locked sets]
\label{def:unlock-state-sets}
The unlock state of node $i$ at log-time $u$ is $n_i(u)\in\{0,1\}$, and once a node is unlocked, it cannot be locked again. The normalized score, unlocked set, and locked set are defined by
\begin{equation}\label{eq:frontier-unlocked-locked-sets}
    x(u)=\sum_{i\in E}\mu_i n_i(u),
    \qquad
    U(u)=\{j\in E:n_j(u)=1\},
    \qquad
    L(u)=\{i\in E:n_i(u)=0\}.
\end{equation}
For a static state $n=(n_i)_{i\in E}$, we also write
\begin{equation*}
    U(n)=\{j\in E:n_j=1\},
    \qquad
    L(n)=\{i\in E:n_i=0\},
    \qquad
    x(n)=\mu(U(n)).
\end{equation*}
\end{definition}

\begin{definition}[Task influence field]
\label{def:task-influence-field}
The entry $K_{ij}$ is the influence strength from source node $j$ to target node $i$. The capability field on target node $i$ is the incoming-edge sum
\begin{equation*}
    h_i(n)=\sum_{j:j\to i}K_{ij}n_j=\sum_{j\in E}K_{ij}n_j.
\end{equation*}
\end{definition}

When stating many-unit limits, we restore the index $N$: the objects above become $E_N$, $\mu^{(N)}$, $K^{(N)}$, $n^{(N)}(u)$, $x_N(u)$, $U_N(u)$, $L_N(u)$, and
\begin{equation*}
    \mu_N(A)=\sum_{i\in A}\mu_i^{(N)},
    \qquad
    x_N(u)=\sum_{i\in E_N}\mu_i^{(N)}n_i^{(N)}(u),
    \qquad
    h_i^{(N)}(u)=\sum_{j\in E_N}K_{ij}^{(N)}n_j^{(N)}(u).
\end{equation*}
For a static state $n\in\{0,1\}^{E_N}$, we write
\begin{equation*}
    U_N(n)=\{j\in E_N:n_j=1\},
    \qquad
    L_N(n)=\{i\in E_N:n_i=0\},
    \qquad
    x_N(n)=\mu_N(U_N(n)).
\end{equation*}

\subsection{Environment Learning as a Frontier Expansion Process}
\label{subsec:single-task-frontier}

With the latent graph and score measure fixed, we now turn the story into a process. The reason to use a frontier picture is that long-horizon environment learning is cumulative. In a software task, a runnable baseline makes later debugging meaningful; in a proof task, a lemma makes later proof obligations easier; in a scientific task, a calibrated model or validation loop makes later parameter search more useful. More generally, each partial success can become a tool for later progress.

The graph model encodes which already-solved score units can help unlock which
remaining ones. At log time \(u\), let
\[
    U(u)=\{j:n_j(u)=1\},\qquad L(u)=\{i:n_i(u)=0\}
\]
denote the unlocked and locked sets. A locked unit \(i\) can only become usable
through influence arriving from units in \(U(u)\). Thus the relevant quantity is
not the total amount of unlocked score, but the amount of influence crossing the
frontier from \(U(u)\) into \(L(u)\).

\underline{\textit{Probabilistic frontier expansion.}}
For a locked score unit \(i\), its task field is given by \(h_i(u)=\sum_{j\in E}K_{ij}n_j(u)\). We turn the influence of the field into a continuous-time stochastic process
by assuming that unit \(i\) unlocks with hazard proportional to its accumulated field:
\[
    \eta(1-n_i(u))h_i(u)
    =
    \eta(1-n_i(u))\sum_{j\in E}K_{ij}n_j(u),
\]
where \(\eta>0\) converts influence strength into progress per unit log time. Thus the field does not deterministically unlock \(i\); it determines the instantaneous probability that \(i\) unlocks. This is the sense in which learning expands the frontier: locked nodes become available at rates determined by the influence they receive from the currently unlocked side of the graph. We formalize this process in the definition below.

\begin{definition}[Single-task frontier process]
\label{def:single-task-frontier-process}
Fix a finite task graph \(G_K=(E,\mathcal E_K)\), score weights \(\mu_i\),
influence matrix \(K\), and conversion constant \(\eta>0\). The single-task
frontier process is the continuous-time process on \(\{0,1\}^E\) in which, for
each score unit \(i\), the instantaneous unlocking intensity is
\begin{align}
    \lambda_i(u) &:= \lim_{\Delta u\downarrow 0}
    \frac{1}{\Delta u}
    \Prob\left(n_i(u+\Delta u)-n_i(u)=1\mid n(u)\right) \nonumber
    \\
    & =
    \eta(1-n_i(u))\sum_{j\in E}K_{ij}n_j(u) .
    \label{eq:single-task-intensity-new}
\end{align}
Importantly, once a node unlocks, it remains unlocked.
\end{definition}

To express the resulting score growth, define for \(A,B\subseteq E\) the weighted
frontier cut
\[
    C(A,B)=\sum_{i\in A}\sum_{j\in B}\mu_iK_{ij}.
\]
This is the total influence from sources in \(B\) to targets in \(A\), with each
target weighted by the score gained if it unlocks. The active frontier at time
\(u\) is therefore \(C(L(u),U(u))\).

\begin{lemma}[Exact single-task frontier growth rate]
\label{lem:exact-single-task-frontier-growth-new}
For the single-task frontier process,
\[
    \frac{\mathrm{d}}{\mathrm{d}u}\E\bigl[x(u)\mid n(u)\bigr]
    =
    \eta C(L(u),U(u)).
\]
\end{lemma}

\begin{proof}
Fix \(u\) and condition on the state \(n(u)=n\). For \(i\in E\), let
\(A_i(\Delta u)\) denote the event that node \(i\) unlocks during
\([u,u+\Delta u]\). By the definition of the unlocking intensity,
\[
    \Pr(A_i(\Delta u)\mid n(u)=n)
    =
    \eta(1-n_i)\sum_{j\in E}K_{ij}n_j\,\Delta u
    +o(\Delta u).
\]
Since \(E\) is finite, the probability of two or more unlocking events in
\([u,u+\Delta u]\) is \(o(\Delta u)\). Therefore, to first order in
\(\Delta u\), the expected score increment is obtained by summing the score
gain from each possible single-node unlock:
\begin{align*}
    \E[x(u+\Delta u)-x(u)\mid n(u)=n]
    &=
    \sum_{i\in E}\mu_i\,
    \Pr(A_i(\Delta u)\mid n(u)=n)
    +o(\Delta u) \\
    &=
    \eta
    \sum_{i\in E}
    \mu_i(1-n_i)
    \sum_{j\in E}K_{ij}n_j\,\Delta u
    +o(\Delta u).
\end{align*}
Dividing by \(\Delta u\) and letting \(\Delta u\downarrow 0\) gives
\[
    \frac{\mathrm d}{\mathrm du}
    \E[x(u)\mid n(u)=n]
    =
    \eta
    \sum_{i\in E}
    \mu_i(1-n_i)
    \sum_{j\in E}K_{ij}n_j .
\]
Since \(1-n_i\) restricts the outer sum to \(i\in L(n)\), and \(n_j\) restricts
the inner sum to \(j\in U(n)\), this becomes
\[
    \frac{\mathrm d}{\mathrm du}
    \E[x(u)\mid n(u)=n]
    =
    \eta
    \sum_{i\in L(n)}
    \sum_{j\in U(n)}
    \mu_iK_{ij}
    =
    \eta C(L(n),U(n)).
\]
This proves the claim.
\end{proof}

Lemma~\ref{lem:exact-single-task-frontier-growth-new} is the key reduction. It says that, to obtain a logistic expected growth rate, we do not need every microscopic edge weight to be identical. Instead, we need the boundary influence to depend mainly on two coarse quantities: the unlocked mass $\mu(U)$ and the locked mass $\mu(L)$. Then the field is roughly the product of the two: available reusable capability times remaining score opportunity.

As the number of score units increases, the same fixed-resolution objects are now applied to the size-$N$ graph. In particular,
\begin{equation*}
    C_N(A,B)=\sum_{i\in A}\sum_{j\in B}\mu_i^{(N)}K_{ij}^{(N)},
    \qquad A,B\subseteq E_N.
\end{equation*}
\underline{\textit{From frontier cuts to product frontiers.}}
The remaining modeling question is when this boundary influence should look like a product of the two measures. The following quantity controls the deviation of the frontier process from the complete-mixing case.

\begin{definition}[Single-task weighted cut error]
\label{def:single-task-cut-error}
Given $(E_N,\mu^{(N)},K^{(N)})$ and an arbitrary $\kappa>0$, define
\begin{equation*}
    \varepsilon_N(\kappa)=
    \sup_{A,B\subseteq E_N}
    \left|
    \sum_{i\in A}\sum_{j\in B}
    \mu_i^{(N)}\bigl(K_{ij}^{(N)}-\kappa\mu_j^{(N)}\bigr)
    \right|.
\end{equation*}
\end{definition}

\begin{condition}[Single-task weighted cut mixing]
\label{cond:single-task-weighted-cut-mixing}
There exists a \(\kappa > 0\) such that, as $N\to\infty$ for the task graph, the weighted cut error in Definition~\ref{def:single-task-cut-error} satisfies
\begin{equation*}
    \varepsilon_N\to0.
\end{equation*}
\end{condition}

Weighted cut mixing is weaker than entrywise complete mixing. It does not say that each $K_{ij}^{(N)}$ is close to $\kappa\mu_j^{(N)}$. It says that, for every possible frontier, the aggregate capability crossing from unlocked graph nodes into locked graph nodes is close to the product-measure value.

\begin{lemma}[Cut mixing gives logistic frontier growth rate]
\label{lem:cut-mixing-logistic-growth}
Under Condition~\ref{cond:single-task-weighted-cut-mixing}, the following holds for every static state $n\in\{0,1\}^{E_N}$. Let $U_N=U_N(n)$, $L_N=L_N(n)$, and $x_N=x_N(n)$. Then
\begin{equation*}
    \left|C_N(L_N,U_N)-\kappa x_N(1-x_N)\right|\le\varepsilon_N.
\end{equation*}
Consequently,
\begin{equation*}
    b_N(n)=\eta\kappa x_N(1-x_N)+r_N(n),
    \qquad |r_N(n)|\le\eta\varepsilon_N.
\end{equation*}
\end{lemma}

\begin{proof}
Apply Definition~\ref{def:single-task-cut-error} to $A=L_N$ and $B=U_N$. Since $\mu_N(U_N)=x_N$ and $\mu_N(L_N)=1-x_N$,
\begin{align*}
    \left|C_N(L_N,U_N)-\kappa x_N(1-x_N)\right|
    &=
    \left|
    \sum_{i\in L_N}\sum_{j\in U_N}
    \mu_i^{(N)}\bigl(K_{ij}^{(N)}-\kappa\mu_j^{(N)}\bigr)
    \right| \\
    &\le \varepsilon_N.
\end{align*}
The expected-growth-rate statement follows from Lemma~\ref{lem:exact-single-task-frontier-growth-new}.
\end{proof}

Cut mixing controls the deterministic expected growth rate. A separate condition controls the visible jump size. This distinction is crucial for the interpretation of the paper: a finite task with only a few large score units can follow the right frontier mechanism and still look very jagged.

\begin{condition}[Small score units and controlled fields]
\label{cond:single-task-small-units}
Let
\begin{equation*}
    q_N=\sum_{i\in E_N}\bigl(\mu_i^{(N)}\bigr)^2.
\end{equation*}
Assume $q_N\to0$. Assume also that there is a deterministic $H_N$ such that
\begin{equation*}
    \sup_{i\in E_N,\,n\in\{0,1\}^{E_N}}\sum_{j\in E_N}K_{ij}^{(N)}n_j\le H_N,
    \qquad
    H_Nq_N\to0.
\end{equation*}
\end{condition}

For equal score units, $q_N=1/N$. Thus Condition~\ref{cond:single-task-small-units} captures the smoothening process from a task with few bulky score units to a task with many small score units. The theorem below should be read exactly in that way. It is not a claim that the realized best-so-far curve of a finite benchmark task must be visually smooth. Rather, it says that if the same task-level mechanism were observed at increasingly fine score resolution, the staircase process would converge to a logistic frontier.

\begin{theorem}[Single-task frontier, many-unit limit]
\label{thm:single-task-frontier-limit-new}
Consider the single-task frontier process of Definition~\ref{def:single-task-frontier-process}. Let $[u_0,u_1]$ be a compact log-time interval and let $x_0\in[0,1]$. Suppose Conditions~\ref{cond:single-task-weighted-cut-mixing} and~\ref{cond:single-task-small-units} hold, and suppose $x_N(u_0)\to x_0$. Then $x_N$ converges uniformly in probability on $[u_0,u_1]$ to the solution of
\begin{equation}
    \frac{dx}{du}=\eta\kappa x(1-x),
    \qquad x(u_0)=x_0.
    \label{eq:single-task-logistic-ode-new}
\end{equation}
If \(t_{\mathrm{mid}}\) is chosen so that the limiting midpoint satisfies $x(0)=1/2$, then, with $x(t):=x(\log(t/t_{\mathrm{mid}}))$, we obtain the log-sigmoid law:
\begin{equation}
    x_N \stackrel{P}{\longrightarrow}  
    x(t)=\frac{1}{1+(t_{\mathrm{mid}}/t)^\beta},
    \qquad
    \beta=\eta\kappa.
    \label{eq:single-task-logsigmoid-limit-new}
\end{equation}
\end{theorem}

\begin{proof}
Let $M_N(u)$ denote the martingale increment in the Doob-Meyer decomposition, normalized so that $M_N(u_0)=0$. The Doob-Meyer decomposition and Lemma~\ref{lem:cut-mixing-logistic-growth} give, for $u\in[u_0,u_1]$,
\begin{equation}
    x_N(u)=x_N(u_0)+M_N(u)+\eta\kappa\int_{u_0}^{u}x_N(s)(1-x_N(s))\,ds+R_N(u),
    \label{eq:single-task-doob-meyer-new}
\end{equation}
where
\begin{equation*}
    \sup_{u\in[u_0,u_1]}|R_N(u)|\le \eta(u_1-u_0)\varepsilon_N.
\end{equation*}

A jump of unit $i$ changes $x_N$ by $\mu_i^{(N)}$. The predictable quadratic variation accumulated over the interval $[u_0,u_1]$ is therefore bounded by
\begin{equation*}
    \langle M_N\rangle_{[u_0,u_1]}
    =\int_{u_0}^{u_1}\sum_{i\in E_N}\bigl(\mu_i^{(N)}\bigr)^2
      \eta\bigl(1-n_i^{(N)}(s)\bigr)\sum_{j\in E_N}K_{ij}^{(N)}n_j^{(N)}(s)\,ds
    \le
    \eta(u_1-u_0)H_Nq_N.
\end{equation*}
Doob's $L^2$ inequality therefore yields
\begin{equation*}
    \E\left[\sup_{u\in[u_0,u_1]}|M_N(u)|^2\right]
    \le 4\eta(u_1-u_0)H_Nq_N\to0.
\end{equation*}
Thus the martingale term vanishes uniformly in probability.

Let $f(x)=\eta\kappa x(1-x)$. On $[0,1]$, $f$ is Lipschitz with constant at most $\eta\kappa$. Comparing \eqref{eq:single-task-doob-meyer-new} with the integral equation for \eqref{eq:single-task-logistic-ode-new} gives
\begin{align*}
    \sup_{u_0\le r\le u}|x_N(r)-x(r)|
    &\le |x_N(u_0)-x_0|+
      \sup_{u_0\le r\le u}|M_N(r)|
      +\eta(u_1-u_0)\varepsilon_N  \\
    &\quad+\eta\kappa\int_{u_0}^{u}\sup_{u_0\le r\le s}|x_N(r)-x(r)|\,ds.
\end{align*}
Gronwall's inequality proves uniform convergence in probability on $[u_0,u_1]$. Solving \eqref{eq:single-task-logistic-ode-new} and setting $x(0)=1/2$ gives \eqref{eq:single-task-logsigmoid-limit-new}.
\end{proof}

\begin{readingbox}{Interpreting the single-task curve shapes}
For finite $N$, the task score is still a jagged jump process. The theorem says that a smooth logistic curve appears only after two effects become negligible: mixing cut error $\varepsilon_N \to 0$ and jump noise from large score units \(H_Nq_N \to 0\). Thus a jagged finite-task curve is compatible with the logistic frontier mechanism. Moreover, the growth-rate $x(1-x)$ has a direct interpretation: unlocked score measure $x$ supplies capability to help unlock new score, while locked score measure $1-x$ delimits the remaining opportunities for score improvement. In the next subsection, we explain why averaging many such finite-task curves is the natural place to expect a clean scaling law.
\end{readingbox}

\subsection{Many-task Aggregation Reveals the Smooth Log-Sigmoid Law}
\label{subsec:aggregate-frontier-limit}

We now return to the quantity fit in the benchmark: the average over many independently evaluated tasks. As we have observed empirically, the per-task curves are heterogeneous and often visibly jagged, while the cross-task averages are much smoother and become better fit by the log-sigmoid as more tasks are included. In our language, each task is its own finite frontier expansion process on a task-specific task graph. The aggregate theorem therefore answers the central question: 
\begin{center}
    \emph{How can many jagged task curves produce a clean log-sigmoid scaling law after task averaging?}    
\end{center}

The answer has two layers. First, because different task environments do not interact with each other, each task can have its own approximate logistic frontier. Second, the benchmark average removes finite-task roughness and becomes a single log-sigmoid when the task midpoints and task speeds are sufficiently concentrated in the many-task limit.

Moving from a single task to a full benchmark introduces three additional sources of variation.

\begin{itemize}
    \item \textbf{Task-level logistic frontier dynamics.} Each task has its own task graph, score atoms, influence matrix, cut error, and finite-jump noise. These quantities determine how closely that task's frontier expansion follows its own logistic limit.

    \item \textbf{Time-axis alignment.} Each task may have its own model-dependent midpoint $t_{\mathrm{mid},b}$. A benchmark-level fit, however, has only one common midpoint $t_{\mathrm{mid}}$. Therefore the task-specific midpoint shifts must concentrate across tasks.

    \item \textbf{Learning-speed homogeneity.} Each task may also have its own frontier speed $\beta_b$. A benchmark-level fit uses one common speed $\beta$, so the model-dependent task learning speeds must concentrate around a shared value.
\end{itemize}

The aggregate theorem says that the log-sigmoid law appears as a population-level limit: finite-task roughness is washed out by averaging, and the remaining task-level logistic frontiers combine into a single curve when their midpoints and speeds are sufficiently aligned.

\begin{definition}[Benchmark task graph and state]
\label{def:benchmark-task-graph-state}
For a benchmark with $M$ tasks, the tasks are indexed by $b=1,\ldots,M$. Task $b$ has its own task graph
\begin{equation*}
    G_b=(E_b,\mathcal E_b),
    \qquad
    \mathcal E_b=\{j\to i:K_{ij}^b>0\},
\end{equation*}
where $E_b$ is the task's visible score-unit node set for the model being evaluated. It has normalized score measurees $\mu_i^b$ for $i\in E_b$, influence matrix $K^b=(K_{ij}^b)_{i,j\in E_b}$, and field-to-progress conversion constant $\eta_b>0$.

The state vector of task $b$ is
\begin{equation*}
    n^b(u)=(n_i^b(u))_{i\in E_b},
\end{equation*}
and the capability field on node $i$ is
\begin{equation*}
    h_i^b(u)=\sum_{j\in E_b}K_{ij}^b n_j^b(u).
\end{equation*}
\end{definition}

\begin{definition}[Task score and benchmark average]
\label{def:benchmark-task-score-average}
The normalized score of task $b$ at log-time $u$ is
\begin{equation*}
    x_b(u)=\sum_{i\in E_b}\mu_i^b n_i^b(u).
\end{equation*}
And the benchmark-average normalized score is
\begin{equation*}
    x_B(u)=\frac1M\sum_{b=1}^{M}x_b(u).
\end{equation*}
Here \(u = \log t - \log t_{\mathrm{mid}}\) is the log-time coordinate. For raw time, we denote the normalized score and benchmark score by 
\[ 
x_B^\dagger(t)=x_B(\log(t/t_{\mathrm{mid}})),\qquad S_B(t)=S_{\max} \cdot x_B^\dagger(t),
\] 
where \(t_{\mathrm{mid}}\) is the fitted aggregate midpoint and \(S_{\max}\) denotes the fitted aggregate ceiling. Individual tasks may have residual midpoint shifts relative to this common coordinate, introduced below as \(\delta_b\) in \eqref{eq:task-specific-logistic}.
\end{definition}

\begin{definition}[Task-specific logistic frontier]
\label{def:task-specific-logistic-frontier}
In the benchmark-level log-time coordinate \(u\), task \(b\) is assigned a task-specific frontier speed \(\beta_b\ge0\) and a residual midpoint shift \(\delta_b\in\mathbb R\). The corresponding task-specific logistic curve is
\begin{equation}
    \ell_b(u)=
    \frac{1}{1+e^{-\beta_b(u-\delta_b)}}.
    \label{eq:task-specific-logistic}
\end{equation}
Here \(\beta_b\) controls the speed of task-level frontier propagation, while \(\delta_b\) measures the task midpoint relative to the benchmark-level midpoint \(t_{\mathrm{mid}}\). Thus \(\delta_b=0\) means that task \(b\)'s midpoint is aligned with the aggregate midpoint.
\end{definition}

The first condition applies the single-task cut-mixing argument inside each task graph. It is written for the influence matrix $\eta_bK^b$, so the product-frontier coefficient is directly $\beta_b$.

\begin{definition}[Blockwise weighted cut error and field bound]
\label{def:blockwise-error-granularity}
For each task $b$, define the task-conditioned weighted cut error by
\begin{equation}
    \varepsilon_b=
    \sup_{A,B\subseteq E_b}
    \left|
    \sum_{i\in A}\sum_{j\in B}
    \mu_i^b\bigl(\eta_bK_{ij}^b-\beta_b\mu_j^b\bigr)
    \right|.
    \label{eq:blockwise-epsilon-new}
\end{equation}
Define the block granularity and score-growth field bound by
\begin{equation*}
    q_b=\sum_{i\in E_b}(\mu_i^b)^2,
    \qquad
    H_b=\sup_{i\in E_b,\,n\in\{0,1\}^{E_b}}
    \eta_b\sum_{j\in E_b}K_{ij}^b n_j.
\end{equation*}
\end{definition}

\begin{condition}[Blockwise weighted cut mixing and vanishing score units]
\label{cond:blockwise-cut-mixing}
Assume that task speeds are uniformly bounded,
\begin{equation*}
    0\le\beta_b\le\beta_+<\infty,
\end{equation*}
and that the average task cut error and granularity error vanish:
\begin{equation}
    \Delta_M=\frac1M\sum_{b=1}^{M}\varepsilon_b\to0,
    \qquad
    A_M=\frac1M\sum_{b=1}^{M}\sqrt{H_b q_b}\to0
    \label{eq:blockwise-errors-new}
\end{equation}
in probability.
\end{condition}

The next condition asks the initial value variation is negligible across different tasks in the benchmark.

\begin{condition}[Blockwise initial alignment]
    \label{cond:blockwise-initial-alignment}
    We assume the average initial value of the dynamics is aligned,
    \[
        \frac{1}{M}\sum_{b=1}^M |x_b(u_0)-\ell_b(u_0)| \stackrel{P}{\longrightarrow} 0.
    \]
\end{condition}

Condition~\ref{cond:blockwise-cut-mixing} has a direct consequence: the observed benchmark trajectory is close to an average of task-specific logistic frontiers. Notice the averaging in the condition. We do not require every task to look smooth on its own, but the \emph{average} contribution of cut errors, score units, and jumps to vanish. This is the formal version of the empirical claim that a population of jagged frontier expansion processes can reveal a smooth law. For the compact interval \(I \subseteq \mathbb{R}\), define the following blockwise trajectory error supremum:
\begin{equation}
    R_M(I):=\frac1M\sum_{b=1}^{M}\sup_{u\in I}|x_b(u)-\ell_b(u)|.
    \label{eq:R-M-new}
\end{equation}

Then we have the following lemma that proves its vanishing property.

\begin{lemma}[Average error supremum, many-task limit]
\label{lem:average-task-logistic-approx}
Under Conditions~\ref{cond:blockwise-cut-mixing} and~\ref{cond:blockwise-initial-alignment}, $R_M(I)\stackrel{P}{\to} 0$ as \(M \to \infty\) for any compact interval $I \subseteq \mathbb{R}$. 
\end{lemma}

\begin{proof}
For task $b$, the same Doob-Meyer decomposition as in the single-task proof gives
\begin{equation}
    x_b(u)=x_b(u_0)+M_b(u)+\int_{u_0}^{u}\left\{\beta_b x_b(s)(1-x_b(s))+\rho_b(s)\right\}\,ds,
    \label{eq:blockwise-doob-meyer-new}
\end{equation}
where $M_b$ is a martingale and, by \eqref{eq:blockwise-epsilon-new}, $|\rho_b(s)|\le\varepsilon_b$ for all states. The task-level curve \eqref{eq:task-specific-logistic} solves
\begin{equation*}
    \frac{\mathrm{d}\ell_b}{\mathrm{d}u}=\beta_b\ell_b(1-\ell_b).
\end{equation*}
Since $z\mapsto \beta_b z(1-z)$ is Lipschitz on $[0,1]$ with constant at most $\beta_+$, Gronwall's inequality applied to \eqref{eq:blockwise-doob-meyer-new} yields
\begin{equation}
    \sup_{u\in I}|x_b(u)-\ell_b(u)|
    \le
    e^{\beta_+|I|}
    \left(
      |x_b(u_0)-\ell_b(u_0)|+
      \sup_{u\in I}|M_b(u)|+
      |I|\varepsilon_b
    \right),
    \label{eq:blockwise-gronwall-new}
\end{equation}
where $|I|=u_1-u_0$. 

The first term converges to zero by Condition~\ref{cond:blockwise-initial-alignment}. The third term as well by Condition~\ref{cond:blockwise-cut-mixing}. It remains to average the martingale terms. A jump of unit $i$ in task $b$ changes $x_b$ by $\mu_i^b$, and the total score-growth field is bounded by $H_b$. Hence the predictable quadratic variation of \(M_b\) can be bounded by
\begin{equation*}
    \langle M_b\rangle_{I}\le |I|H_b q_b.
\end{equation*}
Doob's $L^2$ inequality gives the conditional bound
\begin{equation}
    \E\left[\sup_{u\in I}|M_b(u)|\,\middle|\,\mathcal G_M\right]
    \le 2\sqrt{|I|H_b q_b},
    \label{eq:blockwise-martingale-conditional-new}
\end{equation}
where $\mathcal G_M$ denotes the task-level weights and influence matrices. This conditional form is useful because $A_M$ may itself be random. Let
\begin{equation*}
    Z_M=\frac1M\sum_{b=1}^{M}\sup_{u\in I}|M_b(u)|.
\end{equation*}
For any $\varepsilon>0$ and $a>0$, \eqref{eq:blockwise-martingale-conditional-new} implies
\begin{align*}
    \Prob(Z_M>\varepsilon)
    &\le \Prob(A_M>a)+\Prob(Z_M>\varepsilon,\,A_M\le a) \\
    &\le \Prob(A_M>a)+\varepsilon^{-1}\E[Z_M\1\{A_M\le a\}] \\
    &\le \Prob(A_M>a)+\frac{2\sqrt{|I|}\,a}{\varepsilon}.
\end{align*}
Since $A_M\to0$ in probability, first let $M\to\infty$ and then let $a\downarrow0$ to obtain $Z_M\to0$ in probability. Averaging \eqref{eq:blockwise-gronwall-new} over $b$ and using \eqref{eq:blockwise-errors-new} proves \eqref{eq:R-M-new}.
\end{proof}

The next two task-distributional assumptions serve to align the scaling curves within the benchmark. They are not consequences of blockwise cut mixing. They are what turns an average of task-level sigmoids into a single benchmark sigmoid.

\begin{assumption}[Uniform log-time bias]
\label{assum:uniform-log-time-bias}
There is a single benchmark-level midpoint \(t_{\mathrm{mid}}\) such that the residual task shifts in \eqref{eq:task-specific-logistic} satisfy
\begin{equation*}
    D_M=\frac1M\sum_{b=1}^{M}|\delta_b|\to0.
\end{equation*}
The strict uniform-bias case is $\delta_b\equiv0$.
\end{assumption}

\begin{assumption}[Concentrated environment learning speed]
\label{assum:concentrated-speed}
There is a scalar $\beta>0$ such that
\begin{equation*}
    B_M=\frac1M\sum_{b=1}^{M}|\beta_b-\beta|\to0.
\end{equation*}
In particular, $M^{-1}\sum_{b=1}^{M}\beta_b\to\beta$.
\end{assumption}

The reason these assumptions are necessary can be seen from the aggregated frontier. Averaging removes jaggedness, but it does not by itself align task difficulty or environment learning speed. When blockwise cut mixing error tends to zero, the leading expected growth rate is
\begin{equation*}
    \frac1M\sum_{b=1}^{M}\beta_b x_b(u)(1-x_b(u)).
\end{equation*}
This becomes $\beta x_B(1-x_B)$ only if tasks are aligned enough that the task scores and speeds behave like one population. Even when all speeds are equal, persistent task dispersion subtracts from the scalar frontier through
\begin{equation*}
    \frac1M\sum_{b=1}^{M}x_b(1-x_b)
    =x_B(1-x_B)-\frac1M\sum_{b=1}^{M}(x_b-x_B)^2.
\end{equation*}
The theorem below proves the core convergence once the blockwise frontier condition, midpoint alignment, and speed concentration are in place.

\begin{theorem}[Benchmark aggregate produces the log-sigmoid law in the many-task limit]
\label{thm:aggregate-logsigmoid-frontier-limit-new}
Fix a compact log-time interval \(I \subseteq \mathbb{R}\). Suppose Conditions~\ref{cond:blockwise-cut-mixing}-\ref{cond:blockwise-initial-alignment}, Assumptions~\ref{assum:uniform-log-time-bias}-\ref{assum:concentrated-speed} hold for some \(t_{\mathrm{mid}} > 0\) and \(\beta > 0\). Then 
\[
    x_B(u) \stackrel{P}{\to} \ell_\beta(u), \quad \forall u \in I.
\]
Equivalently, for raw times \(t\) whose log coordinate \(u=\log(t/t_{\mathrm{mid}})\) lies in \(I\), using the notation from Definition~\ref{def:benchmark-task-score-average}, we have
\begin{equation}
    x_B^\dagger(t) = \frac{S_B(t)}{S_{\max}}\stackrel{P}{\longrightarrow}
    \frac{1}{1+(t_{\mathrm{mid}}/t)^\beta}.
    \label{eq:aggregate-raw-time-law-new}
\end{equation}
\end{theorem}

\begin{proof}
We prove the statement for the normalized score \(x_B(u)\). Multiplication by a fixed fitted ceiling $S_{\max}$ gives the raw-score statement.

\emph{Step 1: replace observed task trajectories by task-level frontier limits.} By Lemma~\ref{lem:average-task-logistic-approx},
\begin{equation}
    \sup_{u\in I}\left|x_B(u)-\frac1M\sum_{b=1}^{M}\ell_b(u)\right|
    \le
    \frac1M\sum_{b=1}^{M}\sup_{u\in I}|x_b(u)-\ell_b(u)|
    =R_M(I)\to0
    \label{eq:aggregate-proof-step-one-new}
\end{equation}
in probability.

\emph{Step 2: align the task-level frontiers to the benchmark frontier.} Since the logistic link satisfies $|\sigmoid'(z)|\le1/4$ for all $z$, writing \(\ell_\beta(u) = \sigmoid(\beta u)\), we have
\begin{align}
    \sup_{u\in I}|\ell_b(u)-\ell_\beta(u)|
    &\le
    \frac14\sup_{u\in I}
    \left|\beta_b(u-\delta_b)-\beta u\right| \notag\\
    &\le
    \frac14\left(\sup_{u\in I}|u|\,|\beta_b-\beta|+\beta_b|\delta_b|\right).
    \label{eq:aggregate-lipschitz-new}
\end{align}
The uniform speed bound in Condition~\ref{cond:blockwise-cut-mixing} gives $\beta_b|\delta_b|\le \beta_+|\delta_b|$. Averaging \eqref{eq:aggregate-lipschitz-new} over tasks and using Assumptions~\ref{assum:uniform-log-time-bias} and~\ref{assum:concentrated-speed},
\begin{equation}
    \frac1M\sum_{b=1}^{M}\sup_{u\in I}|\ell_b(u)-\ell_\beta(u)|
    \le
    \frac14\left(\sup_{u\in I}|u|\,B_M+\beta_+D_M\right)
    \to0.
    \label{eq:aggregate-proof-step-two-new}
\end{equation}

\emph{Step 3: combine the two approximations.} The triangle inequality gives
\begin{align*}
    \sup_{u\in I}|x_B(u)-\ell_\beta(u)|
    &\le
    \sup_{u\in I}\left|x_B(u)-\frac1M\sum_{b=1}^{M}\ell_b(u)\right| \\
    &\quad +
    \frac1M\sum_{b=1}^{M}\sup_{u\in I}|\ell_b(u)-\ell_\beta(u)|.
\end{align*}
The first term converges to zero by \eqref{eq:aggregate-proof-step-one-new}; the second converges to zero by \eqref{eq:aggregate-proof-step-two-new}. This proves uniform convergence in probability of $x_B$ to $\ell_\beta$ on $I$. Substituting $u=\log(t/t_{\mathrm{mid}})$ gives \eqref{eq:aggregate-raw-time-law-new}.
\end{proof}

\begin{readingbox}{How does the benchmark aggregate produce the empirical scaling law}
This is the theorem that corresponds to the empirical scaling law. The aggregate averages many frontier expansion processes on task-specific graphs to produce an emergent smooth curve in the many-task limit. For the convergence to hold, we need blockwise cut mixing that ensures task-level product frontier growth rates, the small-score-unit condition that makes jump noise vanish in the aggregate, and the midpoint and speed assumptions that prevent a mixture of incompatible sigmoids. Under these conditions, the benchmark average can finally converges to a single log-sigmoid law in \eqref{eq:aggregate-raw-time-law-new}, instead of being merely S-shaped.
\end{readingbox}

\subsection{Graph Self-similarity Induces Log Scale for Time Axis}
\label{subsec:self-similar-log-time}

The previous subsections explain why frontier expansion gives logistic dynamics once progress is measured in an effective learning coordinate. We now explain why this coordinate should often be logarithmic in raw interaction time.

Recall that the frontier process already takes place on the model-conditioned task graph \(G_K=(E,\mathcal E_K)\). An edge \(j\to i\) exists means that an unlocked source node \(j\) can help unlock a target node \(i\), with field strength (marginal improvement in scores) \(K_{ij}\). The weighted cut \(C(L,U)\) measures the total influence crossing from the unlocked set \(U\) to the locked set \(L\).

The remaining question is how raw interaction time moves the agent through the latent task graph. The answer we propose is graph-dependent and geometric: harder regions of the task graph require more search effort to expose. If this search geometry is approximately self-similar across difficulty scales, then equal additive increases in difficulty require multiplicative increases in raw time. This is why the natural frontier coordinate becomes \(\log t\).

\begin{definition}[Difficulty scale on task-graph edges]
\label{def:edge-difficulty-scale}
For each influence edge \(j\to i\in\mathcal E_K\), let
\[
    \rho_{ij}\ge0
\]
be the difficulty scale of that edge. Smaller \(\rho_{ij}\) means that the influence from \(j\) to \(i\) becomes usable at lower search effort; larger \(\rho_{ij}\) means that the influence requires deeper search or longer interaction to become usable.

For \(r\ge0\), define the exposed influence matrix
\[
    K^{(\le r)}_{ij}
    =
    \begin{cases}
        K_{ij}, & \text{if } j\to i\in\mathcal E_K \text{ and } \rho_{ij}\le r,\\
        0, & \text{otherwise.}
    \end{cases}
\]
The corresponding exposed field is
\[
    h_i^{(\le r)}(n)
    =
    \sum_{j\in E}K^{(\le r)}_{ij}n_j .
\]
\end{definition}

Increasing \(r\) reveals more of the same task graph. At small \(r\), only low-difficulty edges contribute to the field. At larger \(r\), higher-difficulty edges also contribute to the frontier cut. Thus \(r\) is an effective coordinate for how much of the task graph has become usable.

\begin{definition}[Search volume at difficulty scale]
\label{def:search-volume-scale}
Let \(k\in \mathbb{N}^+\) index difficulty levels, with level \(k\) corresponding to the scale band \([k\Delta r,(k+1)\Delta r)\) where \(\Delta r>0\). Define the set of task-graph edges at level \(k\) by
\[
    \mathcal E_k
    =
    \{j\to i\in\mathcal E_K:
    k\Delta r\le \rho_{ij} < (k+1)\Delta r\}.
\]
Let \(N_k=|\mathcal E_k|\) be the number of edges whose difficulty lies in this scale band. We interpret \(N_k\) as a discrete proxy of search-volume: it counts how much edge structure must be exposed at difficulty level \(k\).
\end{definition}

Search volume is different from score measure. score measure \(\mu\) measures how much benchmark value is obtained once nodes are unlocked. Search volume measures how much task-graph structure must become usable before the corresponding frontier influence can appear.

\begin{assumption}[Self-similar task-graph edge growth]
\label{ass:self-similar-edge-growth}
There exist constants \(b>1\) and \(c_-,c_+>0\) such that, throughout the scale range of interest,
\[
    c_- b^k
    \le
    N_k
    \le
    c_+ b^k .
\]
\end{assumption}

This assumption says that each additive increase in difficulty level multiplies the amount of relevant task-graph edge structure by a factor comparable to \(b\). Since level \(k\) corresponds to scale \(r_k=k\Delta r\), this means
\[
    N_k
    \asymp
    b^{r_k/\Delta r}
    =
    \exp\left(\frac{\log b}{\Delta r}r_k\right).
\]
We write
\[
    h=\frac{\log b}{\Delta r}
\]
for the corresponding search-entropy parameter. Informally, \(h\) measures how quickly the amount of task-graph structure grows as difficulty scale increases.

The discrete assumption motivates the continuum approximation
\[
    V(r)=V_0e^{hr},
    \qquad V_0>0,
\]
where \(V(r)\,\mathrm dr\) is the search effort needed to expose the task-graph edge structure in the difficulty band \([r,r+\mathrm dr]\). Hence the cumulative search volume up to difficulty scale \(r\) is
\[
    \mathcal V(r)
    =
    \int_0^r V(s)\,\mathrm ds
    =
    \frac{V_0}{h}(e^{hr}-1).
\]

\begin{assumption}[Linear search-effort supply]
\label{ass:linear-search-effort-supply}
Let \(A(t)\) be the cumulative search effort supplied by raw interaction time \(t\). Assume
\[
    A(t)=\nu t
\]
for some \(\nu>0\).
\end{assumption}

This assumption says that raw interaction time is proportional to available search effort. The search may be adaptive, but one unit of raw time does not by itself expose exponentially many edges of the task graph.

\begin{proposition}[Self-similar task geometry gives logarithmic time]
\label{prop:self-similar-log-time}
Under Assumptions~\ref{ass:self-similar-edge-growth} and~\ref{ass:linear-search-effort-supply}, the difficulty scale exposed by raw time \(t\) satisfies
\[
    r(t)
    =
    \frac1h
    \log\left(1+\frac{h\nu}{V_0}t\right).
\]
In particular, in the scale-free regime \(t\gg V_0/(h\nu)\),
\[
    r(t)=\frac1h\log t+O(1).
\]
\end{proposition}

The proposition is just the inversion of cumulative search volume. The scale \(r(t)\) is defined by \(\mathcal V(r(t))=A(t)\). Since \(\mathcal V(r)\) grows exponentially in \(r\) while \(A(t)\) grows linearly in \(t\), the exposed difficulty scale grows logarithmically in raw time.

\underline{\textit{Connection to the frontier law.}}
The earlier frontier law describes score growth once progress is measured in the coordinate that makes task-graph influence usable. Here that coordinate is \(r\). As \(r\) increases, more entries of \(K\) are exposed through \(K^{(\le r)}\), the field \(h_i^{(\le r)}\) grows, and the frontier cut from unlocked nodes to locked nodes expands through the same mechanism as before.

If the exposed task graph is approximately weighted cut-mixed at each scale, with a scale-stationary effective frontier coefficient, then the coarse frontier dynamics in the scale coordinate take the product form
\[
    \frac{\mathrm dx}{\mathrm dr}
    =
    \gamma x(1-x),
\]
where \(\gamma>0\) is the frontier speed per unit difficulty scale. This is the earlier weighted-cut frontier mechanism written in the coordinate \(r\).

Since \(r(t)=h^{-1}\log t+O(1)\) in the scale-free regime, a unit increase in \(\log t\) corresponds to approximately \(1/h\) units of difficulty-scale progress. Therefore
\[
    \frac{\mathrm dx}{\mathrm d\log t}
    =
    \frac{\gamma}{h}x(1-x).
\]
Thus the fitted log-time frontier speed is
\[
    \beta=\frac{\gamma}{h}.
\]
Solving the log-time logistic equation gives
\[
    x(t)
    =
    \frac{1}{1+(t_{\mathrm{mid}}/t)^\beta}.
\]

\begin{readingbox}{Why self-similar task geometry gives log time}
The weighted-cut argument explains the frontier factor \(x(1-x)\), as previously shown. Self-similar task-graph geometry explains why the time coordinate is measured by log scale. If each additive increase in difficulty scale multiplies the number of relevant task-graph edges by a factor \(b\), then search volume grows like \(e^{hr}\), where \(h=(\log b)/\Delta r\). Since raw interaction time supplies only \(O(t)\) search effort, the exposed difficulty scale satisfies \(r(t)=h^{-1}\log t+O(1)\). Therefore a logistic frontier law in task-graph scale becomes a logistic law in \(\log t\).
\end{readingbox}

\subsection{Discussion and Limitations}
\label{subsec:theory-discussion-limitations}

The derivation above gives a sufficient mechanism for the empirical log-sigmoid law. Its assumptions are useful precisely because they identify concrete ways in which the log-sigmoid limit can fail. We therefore treat it as a mechanistic account of the observed regime rather than a claim that all environment-learning curves must be logistic.

\paragraph{Finite score granularity.}
The single-task theory allows finite tasks to remain visibly jagged. Real tasks may contain a few large hidden tests, decisive proof obligations, or high-weight rubric cells. When such score units remain macroscopic, the martingale error term need not vanish for that task, and the realized best-so-far curve can exhibit long plateaus followed by sudden jumps. The aggregate theorem only requires that such coarse units do not dominate the benchmark average. If a non-negligible fraction of benchmark score measure is carried by coarse tasks, the aggregate curve may retain visible jumps or large run-to-run variance even when the average drift is approximately logistic.

\paragraph{Weighted cut mixing.}
Weighted cut mixing is the core condition of the scaling law. It says that every macroscopic unlocked--locked frontier cut sees approximately product-measure influence. If the task graph has persistent bottlenecks, modules, prerequisite chains, or separated high-transfer and low-transfer regions, then the frontier remembers where it is in the graph, not only how much score measure has been unlocked. In that case, the natural limit is no longer a one-dimensional logistic equation. One should instead expect multi-type dynamics, delayed takeoff, multiple inflection regions, long plateaus, or a sum of sigmoids corresponding to different task modules.

\paragraph{Attainable support and the fitted ceiling.}
The analysis treats \(S_{\max}\) as the stable score measure of an effective reachable support. This is appropriate when the set of practically attainable score units is fixed over the fitted time window. However, if longer interaction changes what is effectively reachable---for example, if weak transfer routes become usable only at much longer horizons---then the denominator of the normalized score is itself moving. A short-window fit may still provide a useful effective ceiling, but that ceiling should not be interpreted as an indefinite upper bound.

\paragraph{Task midpoint alignment.}
The aggregate theorem assumes that a single benchmark-level midpoint removes most task-level log-time bias. If residual midpoint shifts \(\delta_b\) remain widely dispersed, the benchmark average becomes a convolution of shifted task frontiers. Such a curve may still be smooth and S-shaped, but it need not satisfy the scalar logistic ODE. In this regime, the fitted midpoint and slope are window-dependent summaries of the task-midpoint distribution rather than intrinsic benchmark constants.

\paragraph{Learning-speed concentration.}
The aggregate theorem also assumes that task frontier speeds concentrate around a common value. If some task families have persistently steep frontiers while others have persistently shallow frontiers, then the average of task-level sigmoids is generally not itself a sigmoid. Early progress may be dominated by fast tasks, while later progress may be governed by slower tasks. A single fitted \(\beta\) may therefore reflect task-family composition rather than a true scalar environment-learning speed.

\paragraph{Choice of time coordinate.}
The log-time coordinate is justified by the graph self-similarity hypothesis in which unit increments of difficulty require multiplicative increases in search effort. Some environments, however, have characteristic raw-time cycles: fixed evaluation delays, daily data refreshes, hard deadlines, staged curricula, or batch feedback. In such cases, another time coordinate, or a piecewise time model, may be more appropriate than a single log-time transformation.

Together, these limitations clarify the scope of the proof. The appendix gives one route from frontier expansion to the observed population-level log-sigmoid law. A fuller theory would classify the non-logistic limits produced by coarse score units, non-mixing graphs, moving attainable supports, dispersed midpoints, heterogeneous learning speeds, and non-scale-free feedback schedules.

%% file: sections/scaling_law_discussion.tex
\section{More Discussion on the Scaling Law Shapes}
\label{sec:scaling-law-discussion}
\label{sec:scurve-mechanistic-reading}

\underline{\textit{A mechanistic reading selects the log-sigmoid.}}
Because the candidate S-curves fit almost equally well
(Table~\ref{tab:scurve-family-comparison}; the log-sigmoid/log-probit degeneracy
has been known since Berkson~\citep{berkson1944application}), the choice cannot
rest on fit; we make it on mechanism. Written as a growth law for the normalized
score $y=S/S_{\max}$, each candidate's rate makes a different statement about
what drives progress and what limits it.
\begin{itemize}
  \item \textbf{Log-sigmoid:} $\frac{dy}{d\tau}=\beta\,y(1-y)$ in log time
  $\tau=\ln t$. The detailed derivation is given in
  Section~\ref{sec:why-log-sigmoid} and Appendix~\ref{sec:theory}; here the key
  mechanistic reading is that \(y\) is the unlocked attainable score mass and
  \(1-y\) is the remaining locked score mass. If the aggregate influence crossing
  the unlocked--locked frontier is approximately proportional to the product of
  these masses, then progress follows \(y(1-y)\). The log-time coordinate comes
  from the self-similar search geometry of the task graph. This reading is
  consistent with two empirical checks: Section~\ref{sec:stateful-stateless-analysis}
  shows that continuous stateful runs outperform equal-budget repeated sampling,
  so progress is not explained by independent attempts alone; and
  Section~\ref{sec:case-study} shows a finite task advancing through sparse but
  cumulative breakthroughs, where an early working pipeline makes later repairs
  searchable. Inflection occurs at \(y=0.5\) (symmetric).

  \item \textbf{Log-Gompertz:} $\frac{dy}{d\tau}=c\,y\ln(1/y)$. The
  $\ln(1/y)$ term is the signature of an engine that winds down
  multiplicatively as the system matures (e.g., tumour growth, where the
  proliferating fraction shrinks). It is front-loaded, with inflection at
  $y=1/e\approx0.37$: the relative growth rate diverges as $y\to0$, so a tiny
  system grows explosively because the limitation it will eventually hit does
  not yet exist. Experience acquisition is the opposite---its early phase is
  \emph{slow} because a foothold must first be bootstrapped, not fast for lack
  of a brake.

  \item \textbf{Log-Probit:} $\frac{dy}{d\tau}=\frac{1}{\sigma}\,
  \varphi\!\big(\Phi^{-1}(y)\big)$, a sweep across a \emph{log-normal}
  distribution of difficulties---difficulties formed as products of many
  \emph{independent} factors (a multiplicative central-limit argument).
  Experience acquisition is path-dependent rather than a product of independent
  factors, so this microfoundation does not apply, even though probit and the
  logistic are empirically near-indistinguishable.

  \item \textbf{Weibull CDF:} $y(t)=1-\exp\{-(t/\lambda)^\beta\}$, equivalently
  $\frac{dy}{dt}=h(t)(1-y)$ with
  $h(t)=\frac{\beta}{\lambda}(t/\lambda)^{\beta-1}$. Its hazard is a function
  of raw elapsed time, while accumulated progress enters only through the
  survival term $(1-y)$. This makes it the natural first-passage or repeated
  sampling baseline. For a single 0--1 reward task with independent attempts of
  success probability $p$,
  $P_{\rm pass}(k)=1-(1-p)^k=1-\exp[-k(-\ln(1-p))]$, which is an exponential
  CDF, the $\beta=1$ Weibull case (approximately $1-\exp(-pk)$ for small
  $p$). Such a curve improves because repeated independent attempts shrink the
  remaining failure mass, not because state carried across attempts makes later
  attempts more effective. Section~\ref{sec:stateful-stateless-analysis} shows
  that stateful learning beats this repeated-sampling mechanism under the same
  total budget, and the Weibull CDF also fits worse than the log-sigmoid in
  Table~\ref{tab:scurve-family-comparison}. Individual improvements may have a
  first-passage flavor, but the macroscopic best-so-far trajectory accumulates
  through dependent, stateful steps rather than through a raw-time hazard with no
  accumulated-progress factor.

  \item \textbf{Log-linear:} $\frac{dy}{d\tau}=b$ (constant). With no
  saturating term it grows without bound and cannot level off, contradicting the
  many tasks that reach a ceiling within the budget; it attains by far the worst
  fit.
\end{itemize}
We therefore read the log-sigmoid not merely as the best-fitting S-curve but as
the curve whose \(y(1-y)\) rate law matches the frontier-expansion interpretation:
unlocked score mass supplies reusable capability, while locked score mass bounds
the remaining opportunity for improvement.
This is a mechanistic preference, not an empirical exclusion---the data cannot
separate the symmetric families. It is falsifiable through the inflection: a
symmetric peak near $y=0.5$ is consistent with the logistic (and probit), a
front-loaded peak near $0.37$ would favor Gompertz, and a back-loaded peak near
$0.63$ would favor Weibull.

%% file: sections/additional_relatedwork.tex
\section{Additional Related Work}
\label{sec:additional-related-work}

\subsection{Benchmarks Not Suitable for Measuring Self-Evolution}
\label{sec:additional-related-work-not-self-evolution}

Many benchmarks evaluate whether models or agents can answer questions, complete tasks, or produce correct artifacts, without making self-evolution itself the object of measurement. We use this label in a narrow evaluation sense: these benchmarks may still involve reasoning, tools, iteration, or environment interaction, but their primary reported quantity is not related to learning. It is usually final-answer accuracy, task success, pass rate, artifact quality, reproduction fidelity, or human-time horizon. In such settings, interaction is typically a means to produce a final answer or artifact, rather than the quantity being measured.

At the shortest and least interactive end, many classic capability benchmarks, such as MMLU~\cite{hendrycks2021mmlu}, GPQA~\cite{rein2023gpqa}, AIME~\cite{maa2026aime}, and closed-form coding or math evaluations~\cite{chen2021humaneval,hendrycks2021math}, ask models to solve static problems and report final accuracy. These benchmarks can be difficult and useful, but the model's actions do not create a changing environment, and feedback from the benchmark is not exposed as experience for later adaptation. They therefore measure static knowledge and reasoning accuracy rather than self-evolution.

Agentic software benchmarks increase realism by placing agents in codebases, development workflows, or larger construction tasks, but many still reduce evaluation to the final result. SWE-bench~\cite{jimenez2024swebench} evaluates issue repair in existing repositories; development and evolution benchmarks such as RoadmapBench~\cite{xu2026roadmapbench} and SWE-EVO~\cite{le2025sweevo} evaluate incremental changes over existing projects or release histories; construction or reconstruction benchmarks such as NL2Repo-Bench~\cite{ding2025nl2repobench} and ProgramBench~\cite{yang2026programbench} evaluate complete repository generation or recovery of program behavior; and SWE-AGI~\cite{zhang2026sweagi} evaluates specification-driven system construction under a fixed MoonBit API scaffold. FrontierCode~\cite{cognition2026frontiercode} raises the bar on readiness for production by evaluating whether coding agents produce maintainable, mergeable changes under repository-specific quality rubrics. Their task formats differ, but the shared evaluation target is final patch correctness, repository quality, program behavior, maintainability, or test pass rate rather than the trajectory by which an agent improves from feedback.

Some evaluations broaden realism without becoming agentic learning benchmarks. GDPval~\cite{patwardhan2025gdpval}, for example, measures model performance on economically valuable professional deliverables across occupations. Agents' Last Exam~\cite{sun2026agentslastexam} is a close comparator in realism and autonomy: it evaluates generalist computer-use agents on economically valuable professional workflows in real Windows or Linux sandboxes, with CLI and GUI access, professional software, hidden references, and verifiable final artifacts. The distinction is the evaluation target rather than agenticity. Agents' Last Exam is oriented toward completion: agents work toward a final deliverable, and the benchmark reports final scores, pass rates, cost, or time after the artifact is graded. Its trajectories are valuable for replay, audit, and failure analysis, but learning from the benchmark's feedback over a run is not the main quantity being measured.

Other benchmarks use realistic or long-horizon workflows that could support learning, but their published protocols mainly measure final outcomes. HCAST~\cite{rein2025hcast} and METR time-horizon evaluations~\cite{kwa2025timehorizon} calibrate model or agent success by human task duration. Terminal-Bench~\cite{merrill2026terminalbench}, RE-Bench~\cite{wijk2024rebench}, PaperBench~\cite{starace2025paperbench}, CORE-Bench~\cite{siegel2024corebench}, PRBench~\cite{qiu2026prbench}, ReplicationBench~\cite{ye2025replicationbench}, Collider-Bench~\cite{faroughy2026colliderbench}, and ScienceAgentBench~\cite{chen2024scienceagentbench} involve executable software tasks, terminal workflows, or scientific reproduction. These settings are more compatible with learning than short closed-form tasks, but their central questions are typically whether a task is completed, an artifact works, a paper is reproduced, or a final score is high enough under a budget. PaperBench is a useful example of the distinction: it evaluates agentic replication of 20 ICML papers from scratch using detailed rubrics and many gradable subtasks, but the reported score is still replication quality rather than improvement per unit of feedback.

\subsection{Benchmarks Suitable for Measuring Learning or Self-Evolution}
\label{sec:additional-related-work-learning}

Another line of work is more suitable for measuring learning or self-evolution because it exposes new context, serialized streams, iterative feedback, or long executable workspaces. These benchmarks are the closest conceptual comparisons to \benchmark{}, but they differ in what creates the experience stream and what quantity is ultimately scored. We group them by evaluation interface: learning from supplied context, learning over serialized streams or task sequences, and iterative optimization. This grouping is descriptive rather than an ordering by interaction strength. Sequential benchmarks can be highly interactive, and optimization benchmarks can expose diagnostic feedback; the key question is whether the agent's own behavior shapes the future experience it receives within a long executable task.

Context-learning benchmarks study the least agentic form of adaptation. CL-bench~\cite{dou2026clbench} and CL-bench Life~\cite{dou2026clbenchlife} test whether models can use newly provided professional or personal context to answer downstream questions or perform tasks grounded in that context. This counts as learning from external information, but the stream of episodes is usually not shaped by the agent's own actions, and there is no changing environment. The model is asked to use context, not to discover and exploit feedback through sustained interaction.

Sequential learning benchmarks introduce a stream-like structure. EvaLearn~\cite{dou2025evalearn} groups 648 problems into 182 sequences and evaluates learning capability and efficiency as models solve related problems in order. Continual Learning Bench~\cite{asawa2026clbench} explicitly reports improvement over sequential experience and introduces a stateful-versus-stateless comparison similar to Section~\ref{sec:stateful-stateless-analysis}. However, its synthetic task sequences have explicit subtask boundaries, whereas \benchmark{} uses single continuous problems. We therefore partition by time rather than by subtask, which changes how the stateful and stateless settings are defined. StreamBench~\cite{wu2024streambench} also reports improvement over feedback streams. LLF-Bench~\cite{cheng2023llfbench}, LifelongAgentBench~\cite{zheng2025lifelongagentbench}, and Evo-Memory~\cite{wei2025evomemory} study related questions through language feedback, interdependent tasks, experience replay, or evolving memory. These works establish that static capability and learning ability are distinct axes. The difference from \benchmark{} is not that their interaction is necessarily weaker; rather, the sequence is usually organized by the benchmark as a series of instances, tasks, feedback events, or memory updates. In \benchmark{}, the sequence is endogenous to one long task: the agent plans, edits artifacts, submits attempts, receives diagnostics, and thereby affects what it learns next.

Iterative optimization benchmarks are closer to \benchmark{} because they make repeated attempts and empirical feedback central to performance. Frontier-Eng~\cite{chi2026frontiereng} studies improvement over generative optimization cycles, while MLS-Bench~\cite{lyu2026mlsbench} and Frontier-CS~\cite{mang2025frontiercs} evaluate test-time discovery or open-ended CS optimization with visible metrics, submissions, or trajectory-level reporting. MLE-bench~\cite{chan2024mlebench} is also naturally viewed in this family: agents work in a free Kaggle-style workspace, the paper includes studies of how performance scales with resources, and its 100-hour analysis grades snapshots of the agent's best attempt over elapsed time. ALE-Bench~\cite{imajuku2025alebench} is another adjacent long-horizon setting driven by fixed objectives, even though its original protocol primarily reports outcomes on algorithm engineering rather than learning metrics.

FrontierSWE~\cite{chu2026frontierswe} and AutoLab~\cite{xu2026autolab} are the closest prior work. Both evaluate long-horizon agentic improvement over executable artifacts through repeated edits, experiments, and empirical feedback. FrontierSWE focuses on software engineering and performance-tuning tasks, with reported average agent runtime of roughly 3--4 hours across models and task categories. AutoLab gives agents working but deliberately suboptimal programs across systems, CUDA, model development, and puzzle-style optimization tasks; it directly measures iterative improvement and should not be treated as a non-learning benchmark, although performance curves are not reported as the main result and most current tasks run for 2--4 hours, with only a small minority exceeding 6 hours. Under the shared premise of long-horizon agentic tasks, the main distinction is domain coverage: FrontierSWE is concentrated on software engineering, AutoLab emphasizes research and engineering tasks centered on optimization, and \benchmark{} spans a broader set of executable domains. \benchmark{} also uses a day-scale task contract and makes the time-aligned trajectory itself the primary object of measurement, with metrics for improvement area, regression, and active learning span.

Taken together, these benchmarks cover important pieces of the problem: realistic executable work, learning over organized streams, and iterative optimization under feedback. \benchmark{} targets their intersection. It evaluates within-run self-evolution in long-horizon executable environments where the experience stream arises from the task itself, the agent can influence what it observes next, and evaluation uses a general-purpose agent harness rather than a benchmark-specific learning scaffold. Its trajectory-level metrics report improvement, regression, active learning span, and final performance over the same continuous run.

\subsection{Scaling Laws for LLMs and Agents}
\label{sec:additional-related-work-scaling-laws}

Classical neural and language-model scaling laws mainly study pretraining, modeling reducible loss as a function of model size, data, and training compute. Kaplan et al.~\cite{kaplan2020scaling} showed that language-modeling loss follows power-law relationships over several axes of scale, and Chinchilla-style work refined the compute-optimal allocation between model parameters and training tokens~\cite{hoffmann2022training}. More recent work studies how benchmark performance, rather than loss, changes with scale; because accuracy-like metrics are bounded, these curves are often better described by sigmoidal or other saturating forms~\cite{owen2024predictable,ruan2024observational,bhagia2024taskscaling,zhang2026prescriptive}. Our log-sigmoid environment learning curves in Section~\ref{sec:emerging-scaling} are closest in mathematical form to this bounded-performance line of work, but the independent variable is elapsed environment interaction within a task rather than pretraining compute or model scale.

Test-time and inference-time scaling provide a second point of comparison: different test-time methods exhibit their own empirical scaling behavior. One line scales plain repeated sampling: Evaluating Large Language Models Trained on Code introduced pass@\(k\) as a repeated-sampling evaluation for code generation~\cite{chen2021humaneval}, AlphaCode used large-scale sampling, filtering, and selection to improve competitive-programming solve rates~\cite{li2022alphacode}, and Large Language Monkeys showed that repeated sampling can scale coverage across orders of magnitude when outputs are automatically verifiable~\cite{brown2024largelanguagemonkeys}. A second line studies how test-time compute should be allocated across search, revision, voting, and verifier- or reward-guided strategies~\cite{snell2024testtime,wu2024inference}. A third line scales long-chain-of-thought inference in reasoning models: OpenAI's o1 report made train-time reinforcement learning compute and test-time thinking compute explicit scaling axes~\cite{openai2024learningreason}, and DeepSeek-R1 studies how reinforcement learning can elicit reasoning behaviors in LLMs~\cite{deepseekai2025deepseekr1}. These studies mostly scale the compute spent producing or selecting answers on traditional math, code, proof, or game-style tasks, and the reported functional forms are often local log-linear trends, exponentiated power laws, or frontiers for allocating compute. \benchmark{} instead scales \emph{interaction time}: it measures how an agent's best-so-far performance changes as elapsed time in an executable environment increases and as the agent repeatedly reads, acts, receives feedback, and revises. This gives broader task coverage and a different scaling object than sampling from a static prompt or selecting among answers.

Recent agentic test-time scaling work moves closer to our setting by allowing agents to acquire information from an environment over time. OpenAI's CUA report observed that computer-use performance improves when more steps are allowed~\cite{openai2025cua}, and BrowseComp reports smoother gains as browsing effort increases~\cite{openai2025browsecomp}. Thinking vs.\ Doing explicitly frames interaction length as a test-time scaling axis for web agents~\cite{shen2025thinkingdoing}. Closest in spirit, Mang et al.~\cite{mang2026humansstillbeatai} compare agents and humans on a long-horizon coding contest setting where agents can try, observe, and revise over time. They report an informative human reference curve and find that current agents plateau while strong humans continue improving. However, the measurement is concentrated in one contest-style domain and a limited task set. The human improvement curve is also difficult to extrapolate: the observed segment appears closer to linear improvement over time, but it may cover only the early portion of a longer sigmoid-like trajectory.

Scaling has also been studied in reinforcement learning. Hilton et al.~\cite{hilton2023rlscaling} introduced intrinsic performance to obtain smooth power-law relations between environment interactions, model size, and RL performance. Recent LLM RL scaling work fits sigmoidal compute-performance curves for post-training and uses smaller runs to predict larger RL runs~\cite{khatri2025scalingrlcompute}. Both RL and our evaluation measure learning from environment feedback rather than from human-labeled data. The practical difference is resolution and coverage. RL learning runs are expensive, so scaling evidence is usually gathered on narrower task distributions, fewer distinct environments, and fewer long trajectories. In \benchmark{}, the learning algorithm is in-context learning (ICL): the agent absorbs observations, diagnostics, and prior attempts into its context and uses them to improve subsequent actions. Because ICL is cheaper to repeat than RL training, \benchmark{} can measure environment learning across 134 executable task environments, multiple domains, full time-aligned trajectories, and repeated trials, yielding substantially broader environment coverage than is typical in RL scaling studies. The resulting fits go beyond a bounded-performance analogy to RL scaling: they provide a higher-resolution measurement of how agents learn from interaction.

%% file: sections/appendix.tex
\section{Additional Benchmark and Experiment Details}
\label{sec:additional-benchmark-details}

\subsection{Estimating the With- and Without-Experience Curves}
\label{sec:experience-gain-estimation}

This describes how the two curves in Section~\ref{sec:stateful-stateless-analysis} are estimated. The with-experience curve averages three 12-hour best-so-far curves per task, then averages across tasks. For the without-experience curve, let $u_1,\ldots,u_n$ be the scores of the $n$ independent attempts on a task; at elapsed time $t=k\tau$ we estimate the expected best of $k$ attempts as
\begin{equation}
\hat{u}_{k\tau} = \mathbb{E}_{S}\big[\max_{i\in S} u_i\big],
\end{equation}
where $S$ is sampled uniformly from the $\binom{n}{k}$ size-$k$ subsets of the $n$ attempts. This is the score-valued, without-replacement extension of the pass@$k$ estimator~\cite{chen2021humaneval}; when each $u_i$ is binary, it reduces to the usual pass@$k$ estimate.

\input{sections/gw_case_study_details}

\clearpage

\input{sections/harness_level_continuation_ablations}

\clearpage

\input{sections/task_by_task_specifications}

\subsection{Per-Task Learning Curves}
\label{sec:all-curves}

Figures~\ref{fig:curves-all-1}--\ref{fig:curves-all-21} show learning curves for all tasks in the benchmark, grouped by capability family. Each plot shows raw task score versus elapsed time over the 12-hour budget; solid lines are the best-so-far envelope, faint dashed dots are individual submissions, and pale segments after the marked best indicate later submissions without improvement. On a few tasks the y-axis is zoomed to the main score band, with extreme outlier submissions off-axis (noted under each plot). Five agents are compared: Claude Opus 4.8, GPT-5.5, GPT-5.4, GLM-5.1, and DS-V4-Pro. The 18 representative tasks in the main text (Figure~\ref{fig:environment-data-curves-copy}) are not repeated here.

\begin{figure}[p]
\centering
\begin{subfigure}[b]{0.48\linewidth}
\includegraphics[width=\linewidth]{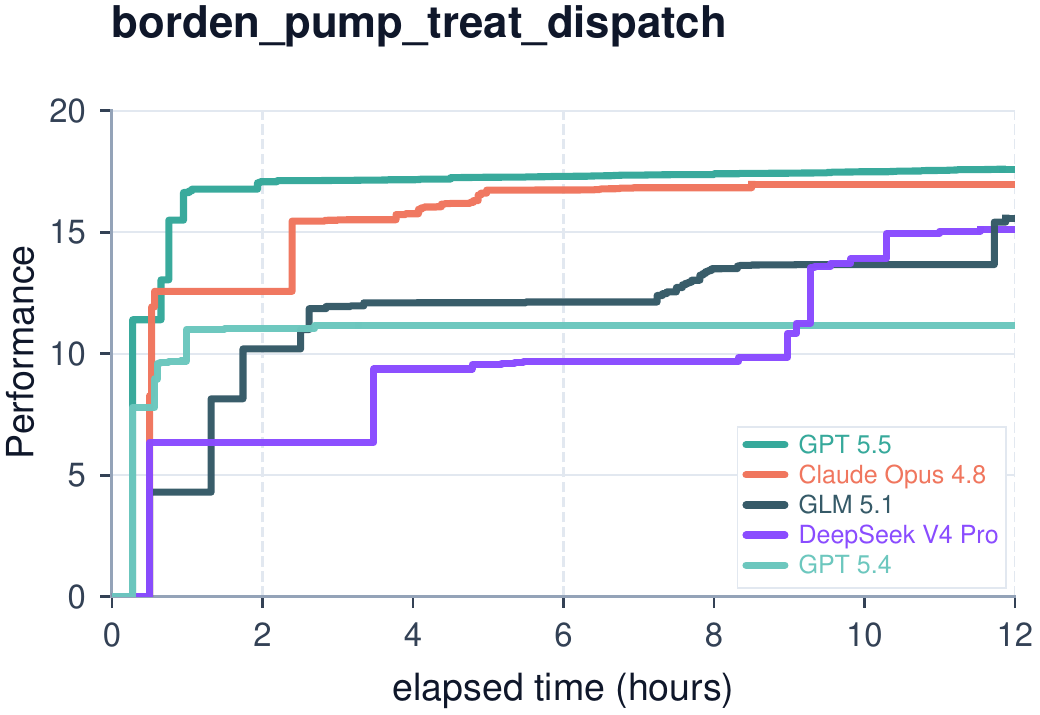}
\end{subfigure}
\hfill
\begin{subfigure}[b]{0.48\linewidth}
\includegraphics[width=\linewidth]{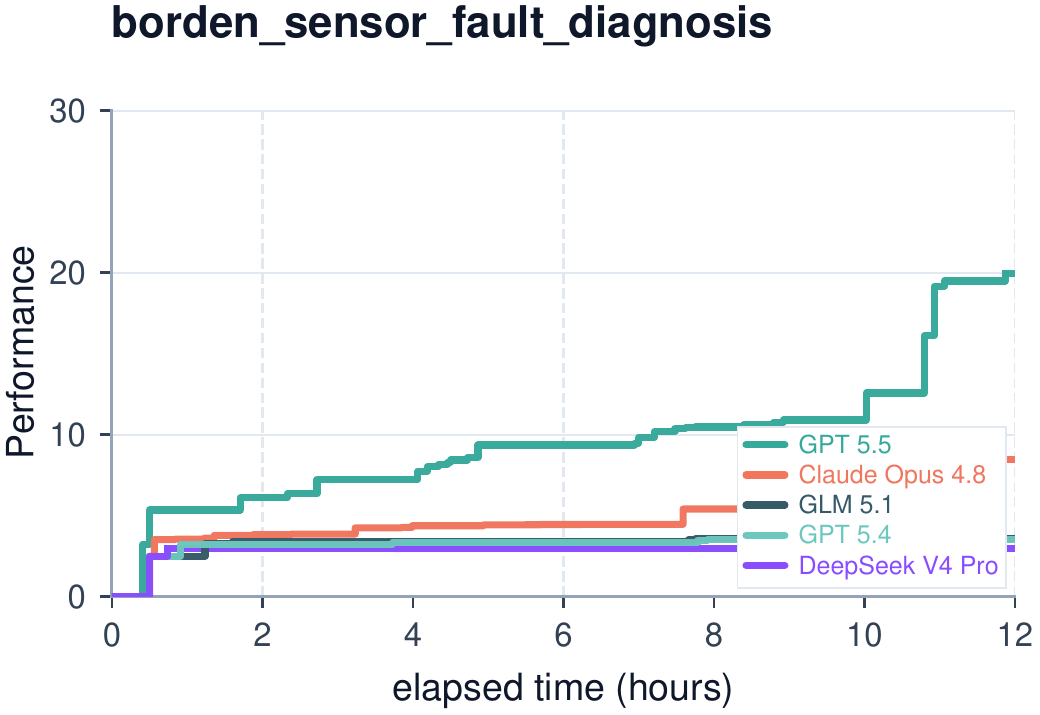}
\end{subfigure}
\vspace{0.5em}
\begin{subfigure}[b]{0.48\linewidth}
\includegraphics[width=\linewidth]{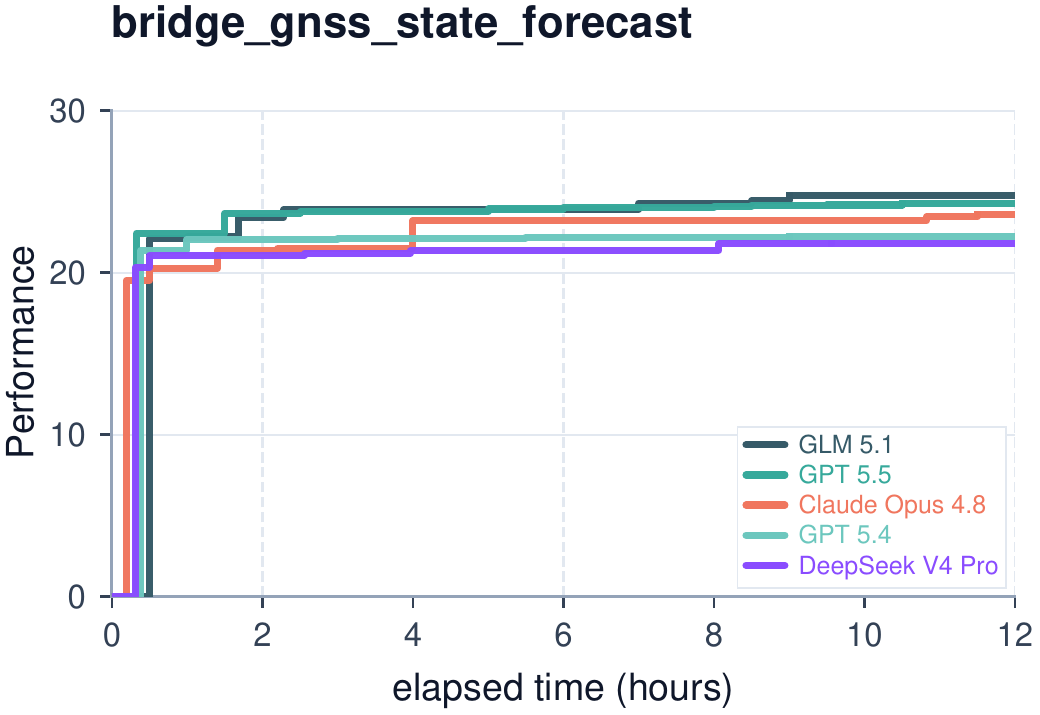}
\end{subfigure}
\hfill
\begin{subfigure}[b]{0.48\linewidth}
\includegraphics[width=\linewidth]{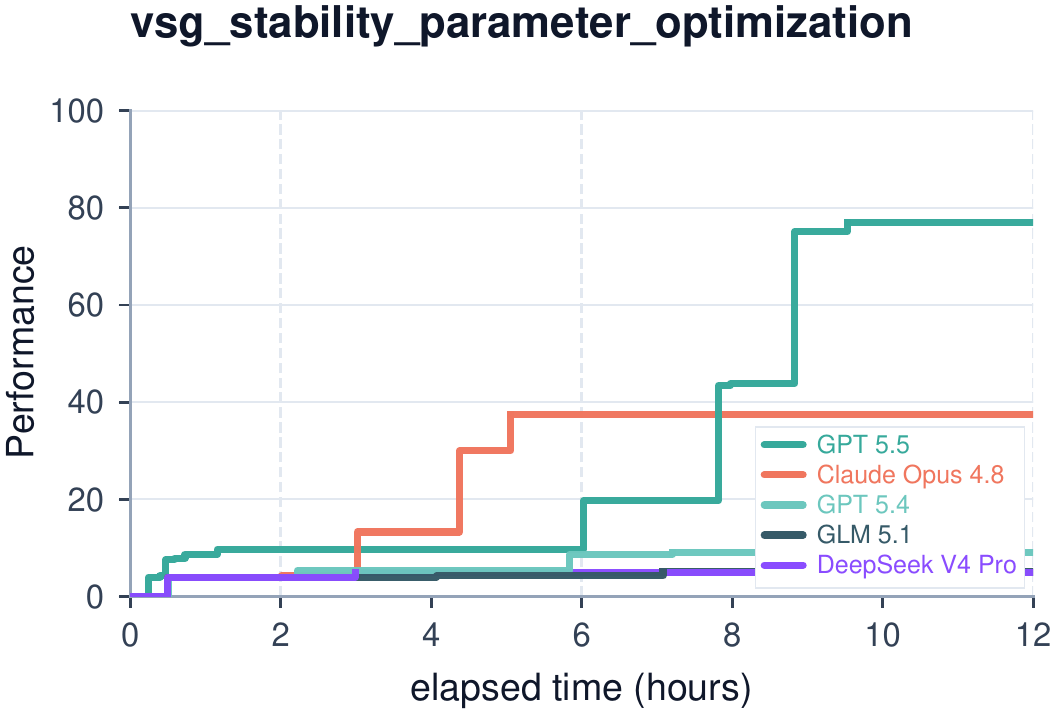}
\end{subfigure}
\vspace{0.5em}
\begin{subfigure}[b]{0.48\linewidth}
\includegraphics[width=\linewidth]{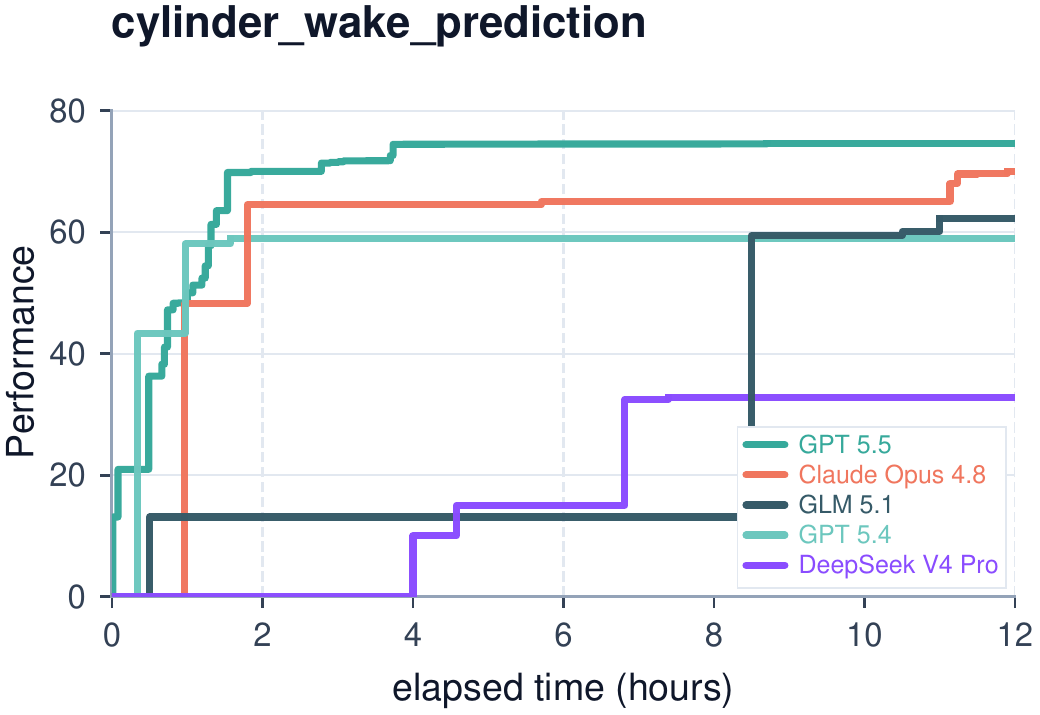}
\end{subfigure}
\hfill
\begin{subfigure}[b]{0.48\linewidth}
\includegraphics[width=\linewidth]{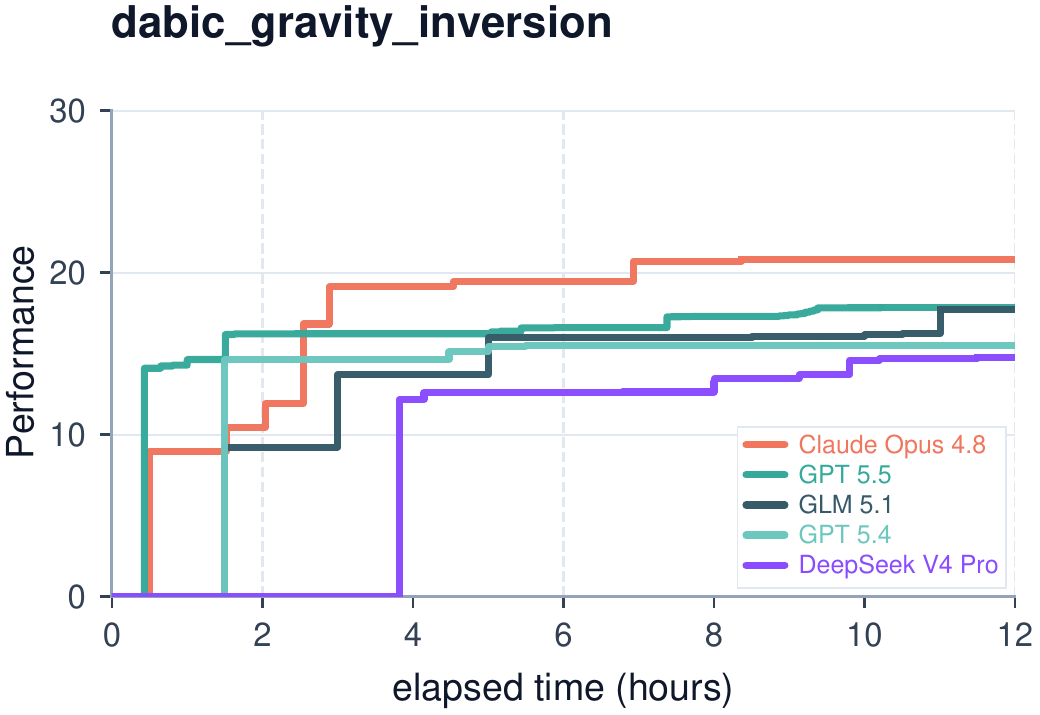}
\end{subfigure}
\caption{Per-task learning curves: Scientific Computing \& ML (1/21).}
\label{fig:curves-all-1}
\end{figure}

\begin{figure}[p]
\centering
\begin{subfigure}[b]{0.48\linewidth}
\includegraphics[width=\linewidth]{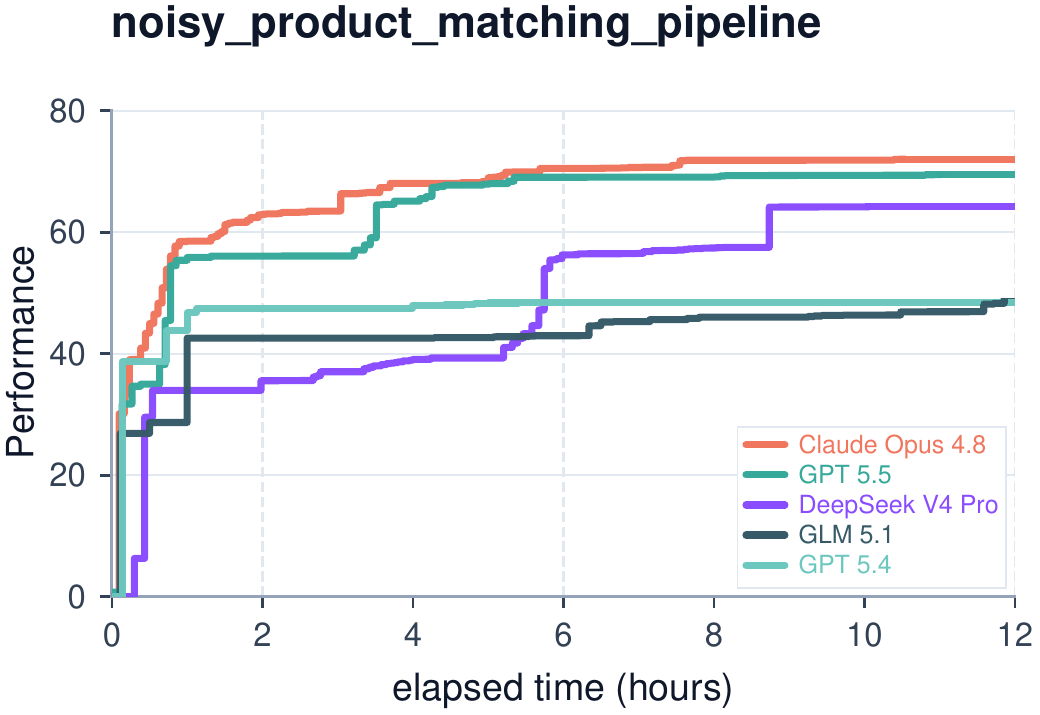}
\end{subfigure}
\hfill
\begin{subfigure}[b]{0.48\linewidth}
\includegraphics[width=\linewidth]{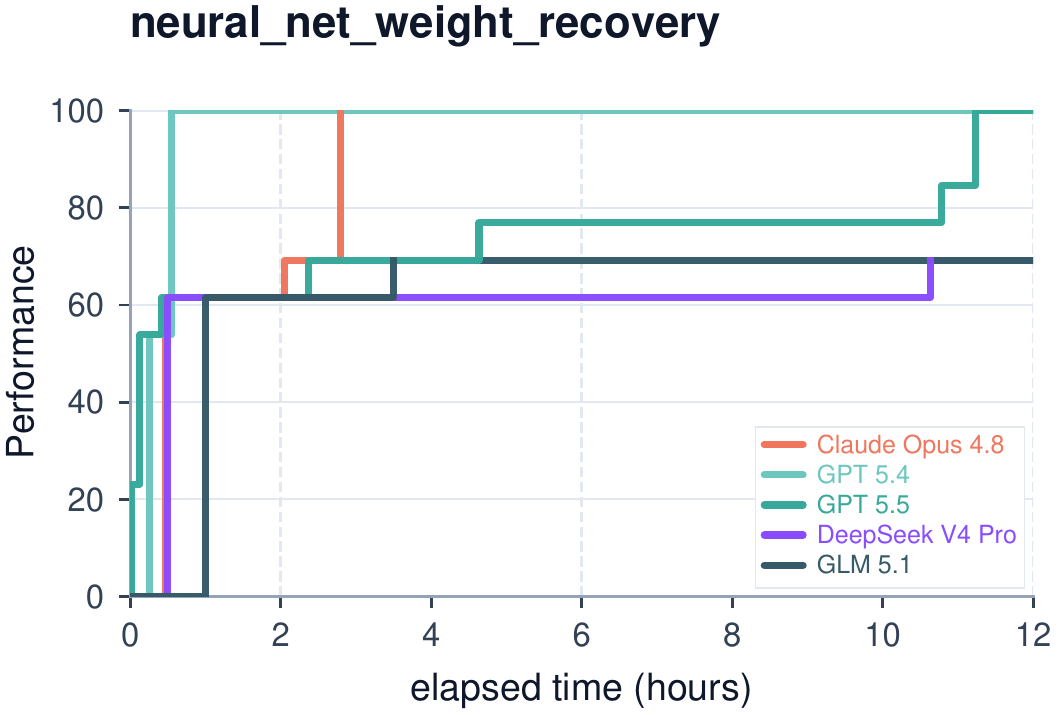}
\end{subfigure}
\vspace{0.5em}
\begin{subfigure}[b]{0.48\linewidth}
\includegraphics[width=\linewidth]{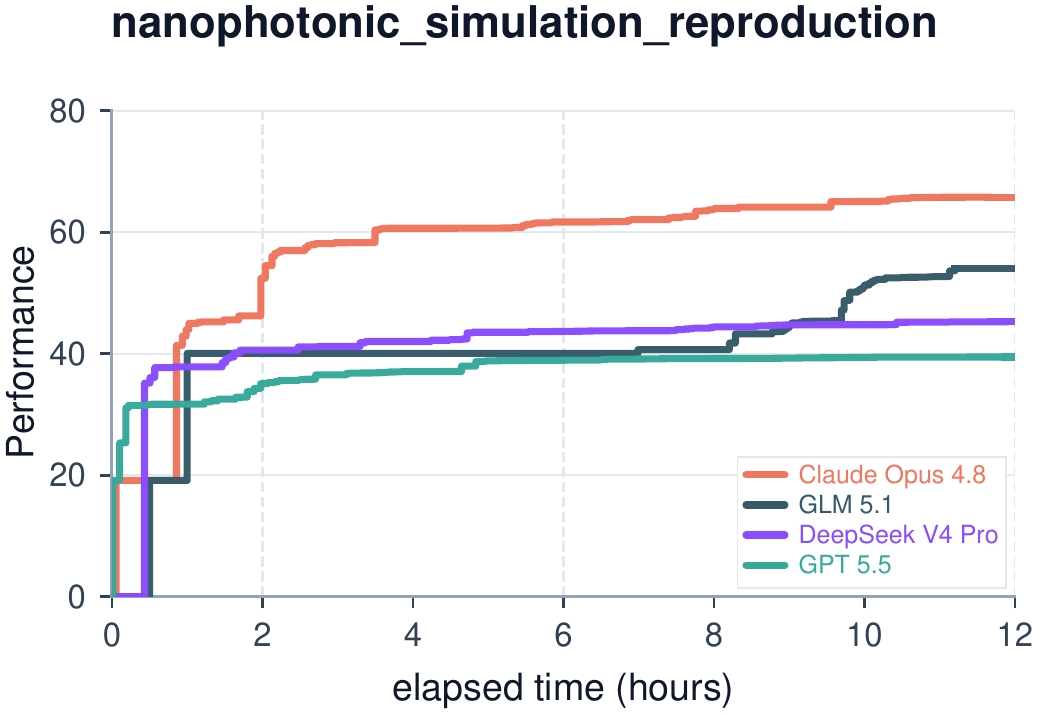}
\end{subfigure}
\hfill
\begin{subfigure}[b]{0.48\linewidth}
\includegraphics[width=\linewidth]{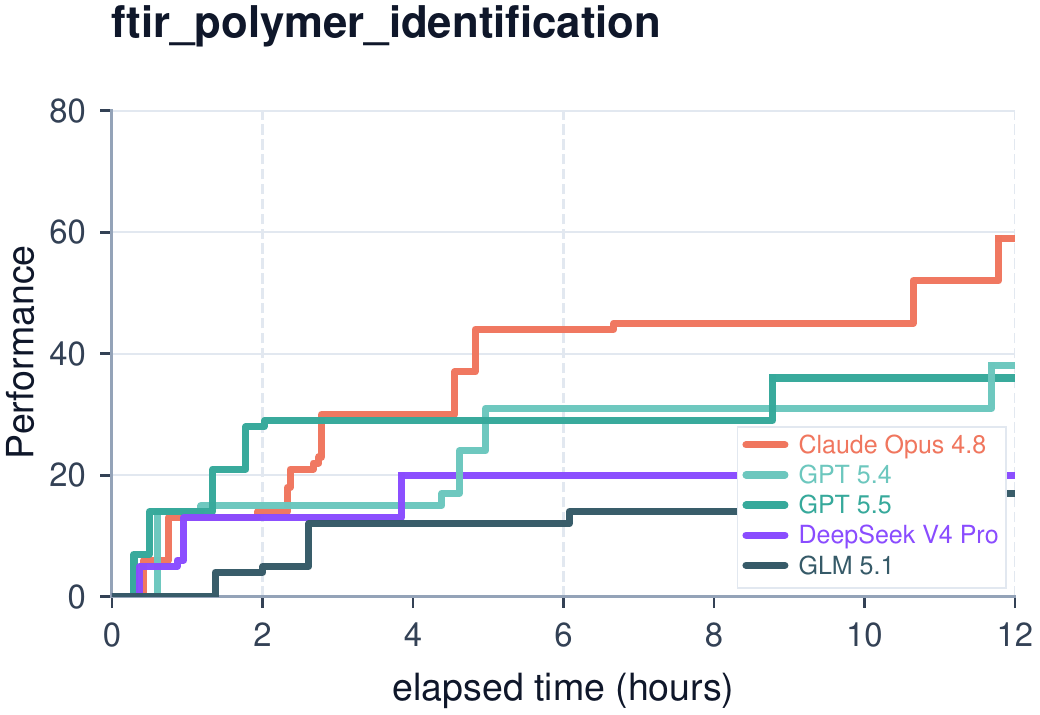}
\end{subfigure}
\vspace{0.5em}
\begin{subfigure}[b]{0.48\linewidth}
\includegraphics[width=\linewidth]{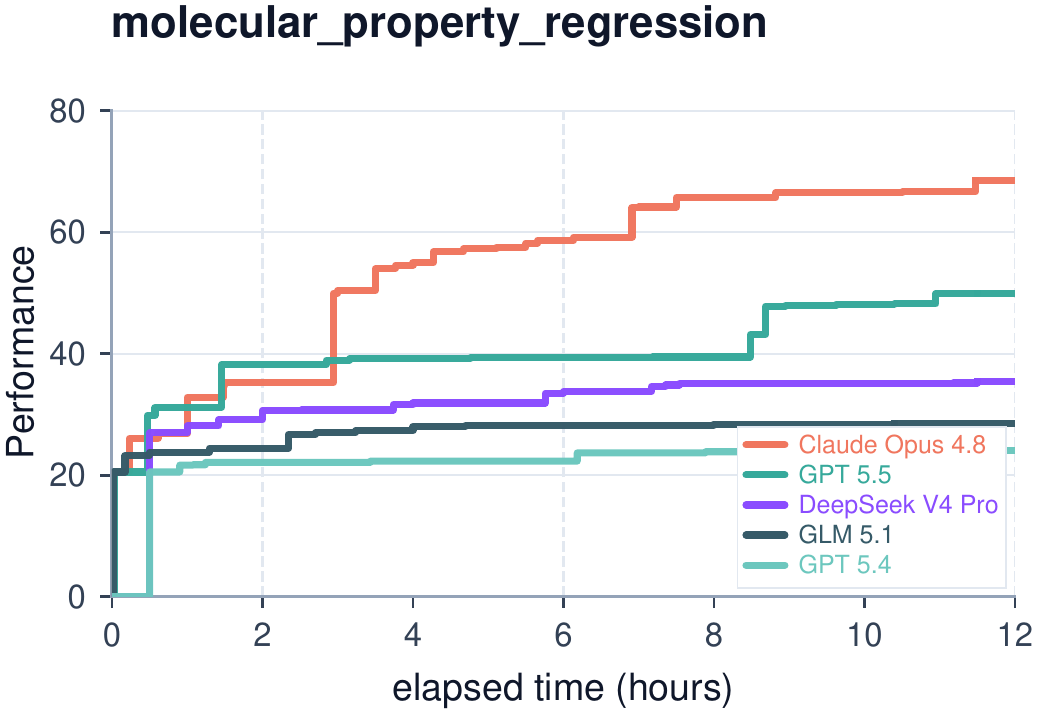}
\end{subfigure}
\hfill
\begin{subfigure}[b]{0.48\linewidth}
\includegraphics[width=\linewidth]{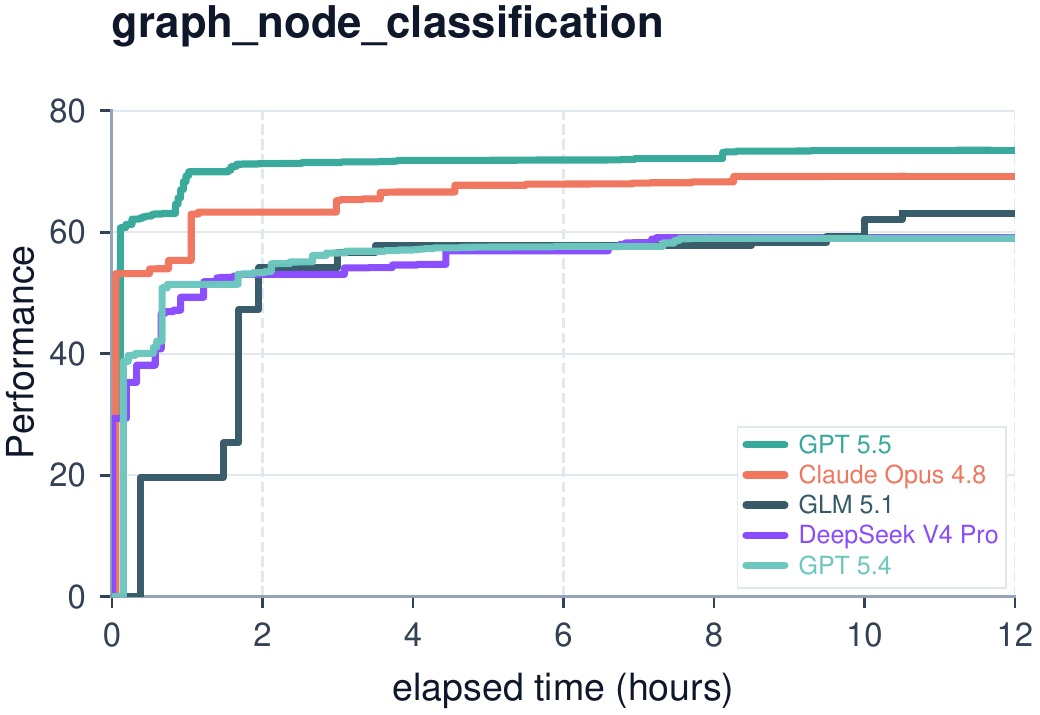}
\end{subfigure}
\caption{Per-task learning curves: Scientific Computing \& ML cont. (2/21).}
\label{fig:curves-all-2}
\end{figure}

\begin{figure}[p]
\centering
\begin{subfigure}[b]{0.48\linewidth}
\includegraphics[width=\linewidth]{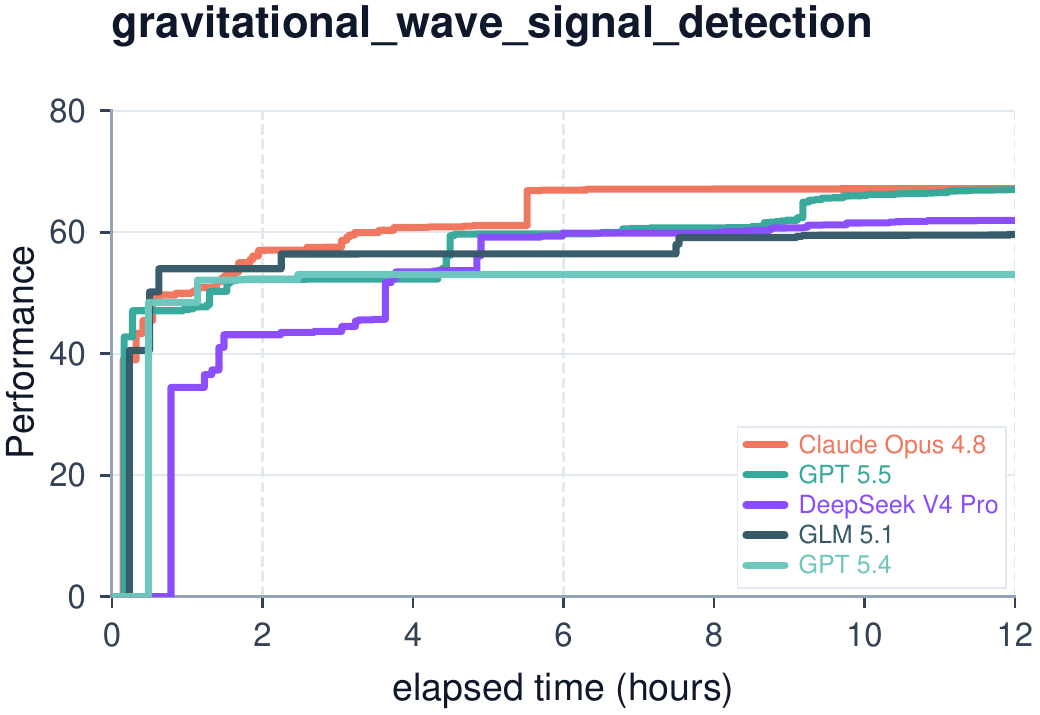}
\end{subfigure}
\hfill
\begin{subfigure}[b]{0.48\linewidth}
\includegraphics[width=\linewidth]{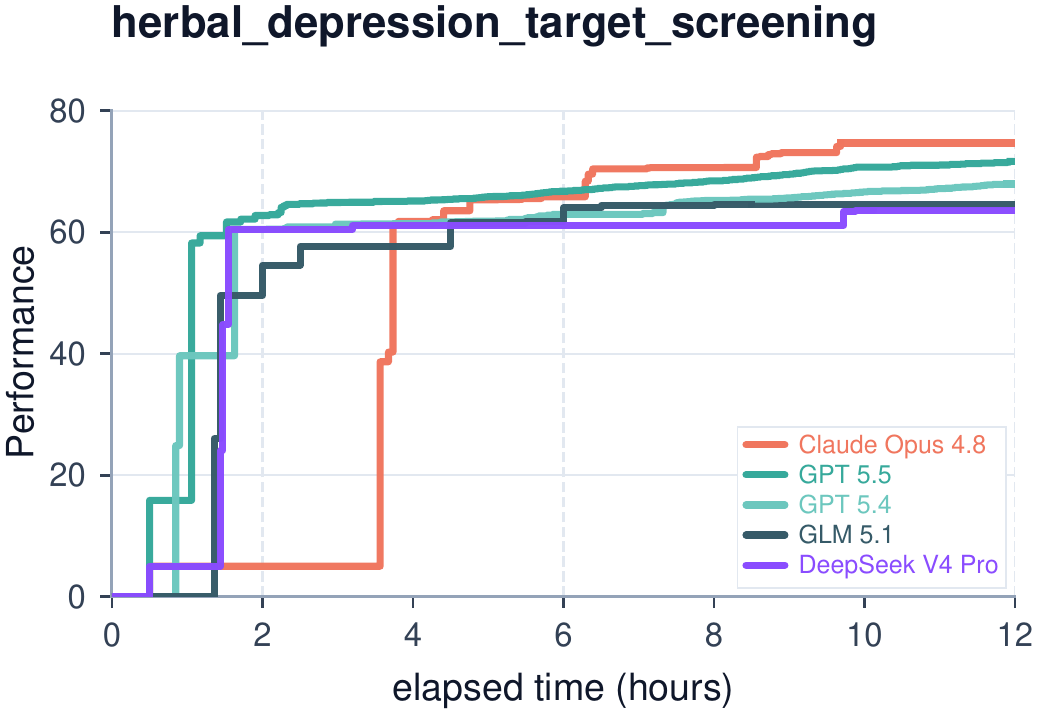}
\end{subfigure}
\vspace{0.5em}
\begin{subfigure}[b]{0.48\linewidth}
\includegraphics[width=\linewidth]{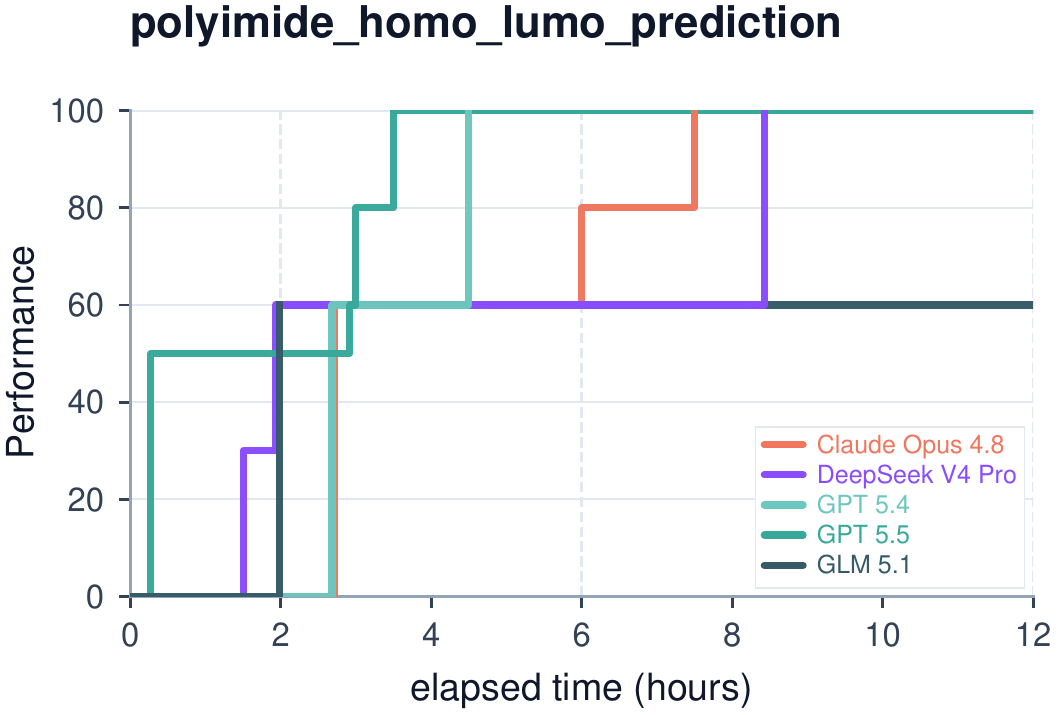}
\end{subfigure}
\hfill
\begin{subfigure}[b]{0.48\linewidth}
\includegraphics[width=\linewidth]{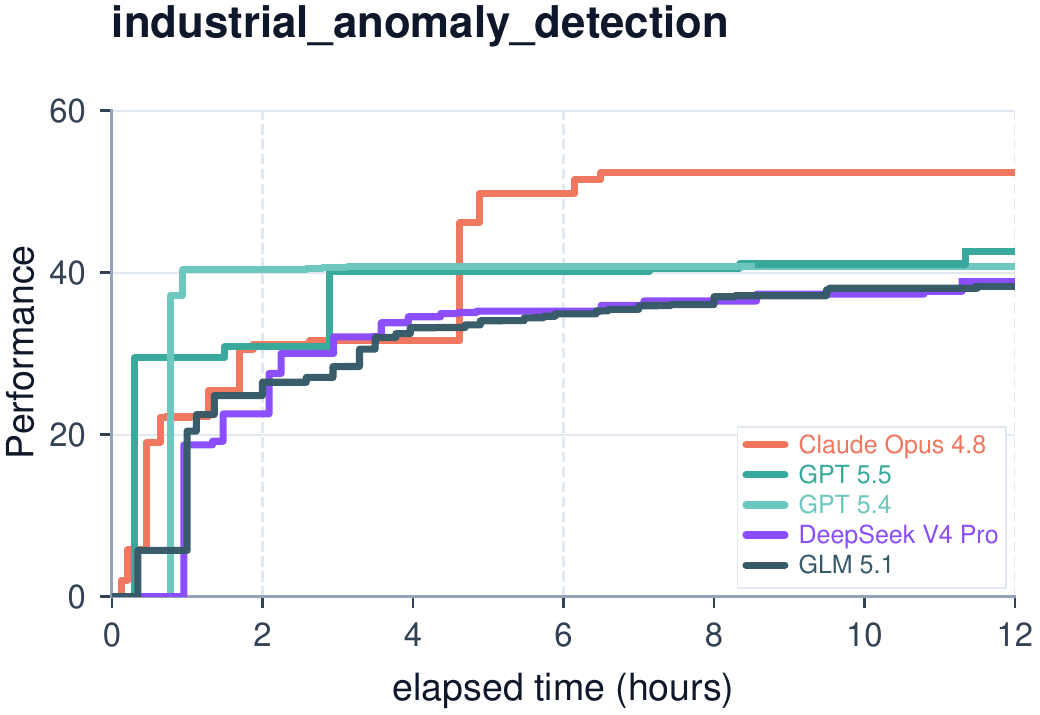}
\end{subfigure}
\vspace{0.5em}
\begin{subfigure}[b]{0.48\linewidth}
\includegraphics[width=\linewidth]{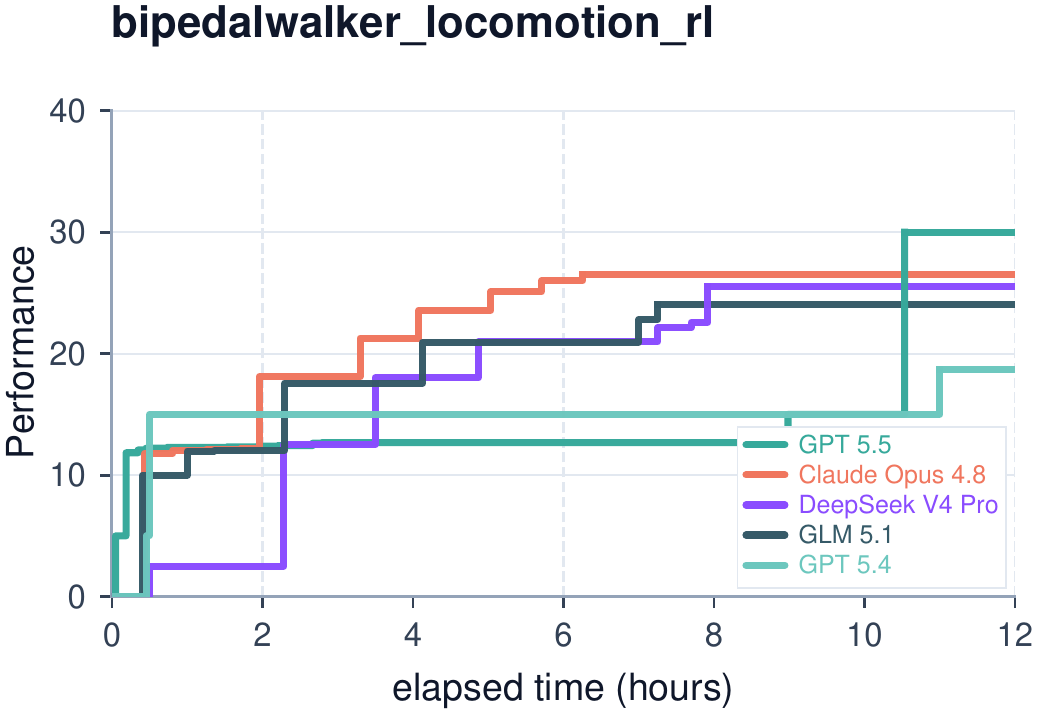}
\end{subfigure}
\hfill
\begin{subfigure}[b]{0.48\linewidth}
\includegraphics[width=\linewidth]{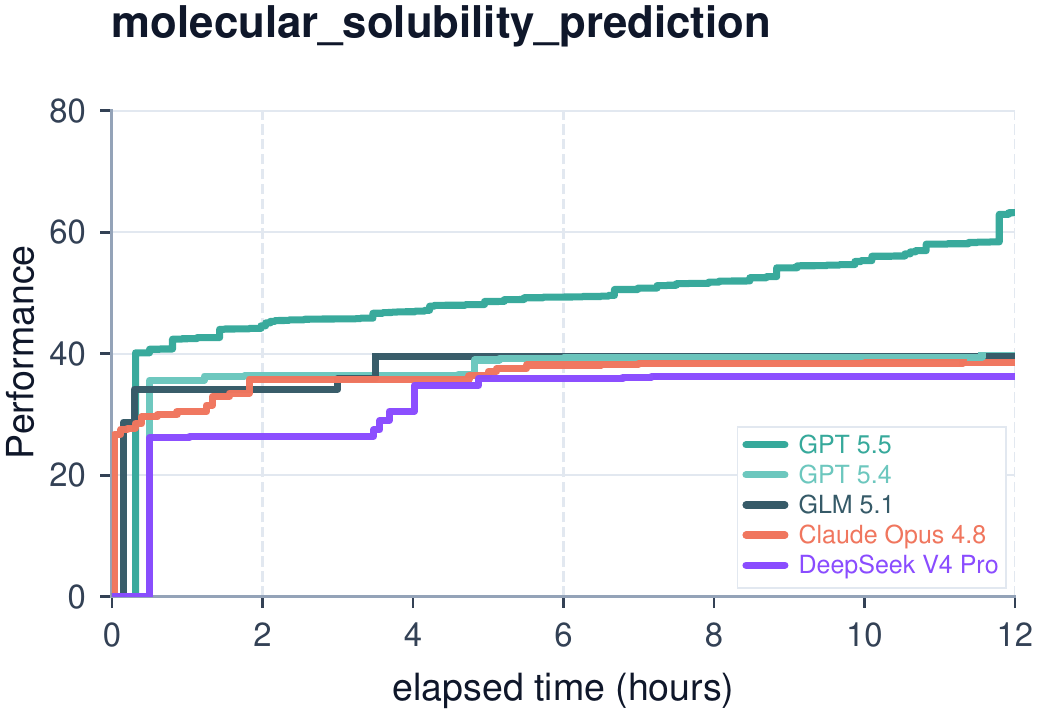}
\end{subfigure}
\caption{Per-task learning curves: Scientific Computing \& ML cont. (3/21).}
\label{fig:curves-all-3}
\end{figure}

\begin{figure}[p]
\centering
\begin{subfigure}[b]{0.48\linewidth}
\includegraphics[width=\linewidth]{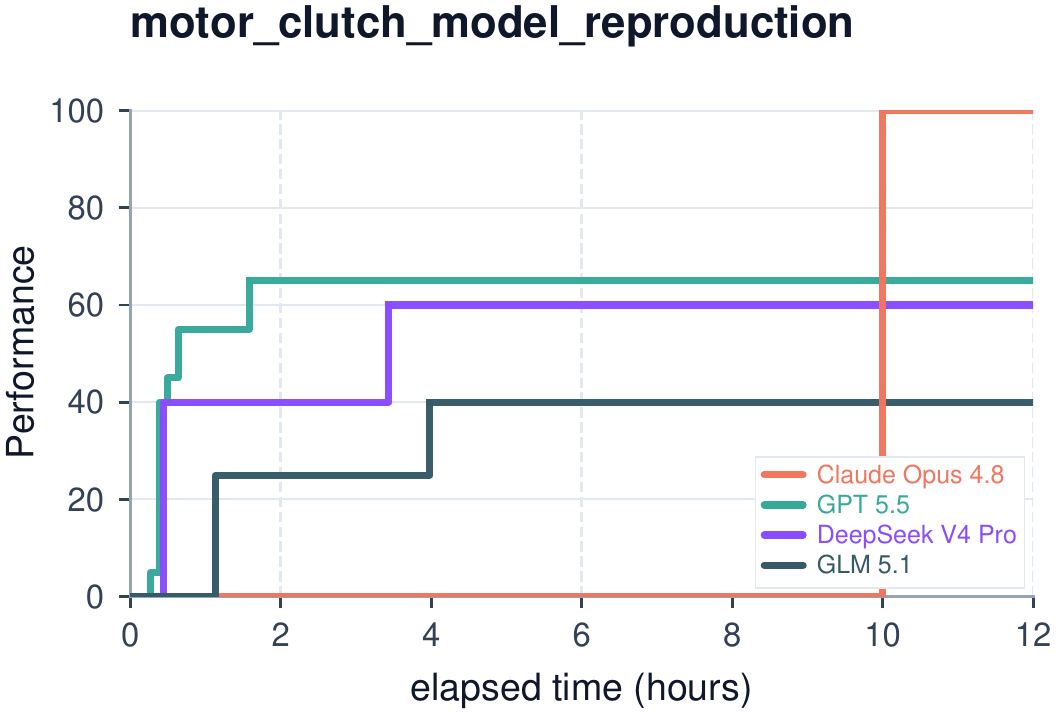}
\end{subfigure}
\hfill
\begin{subfigure}[b]{0.48\linewidth}
\includegraphics[width=\linewidth]{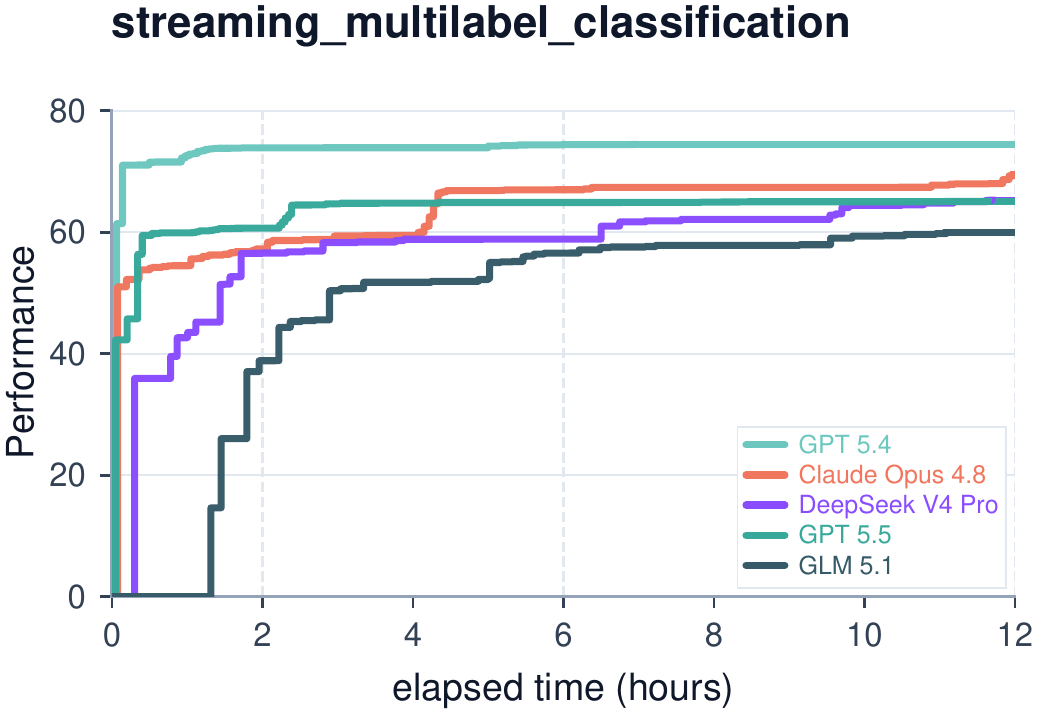}
\end{subfigure}
\vspace{0.5em}
\begin{subfigure}[b]{0.48\linewidth}
\includegraphics[width=\linewidth]{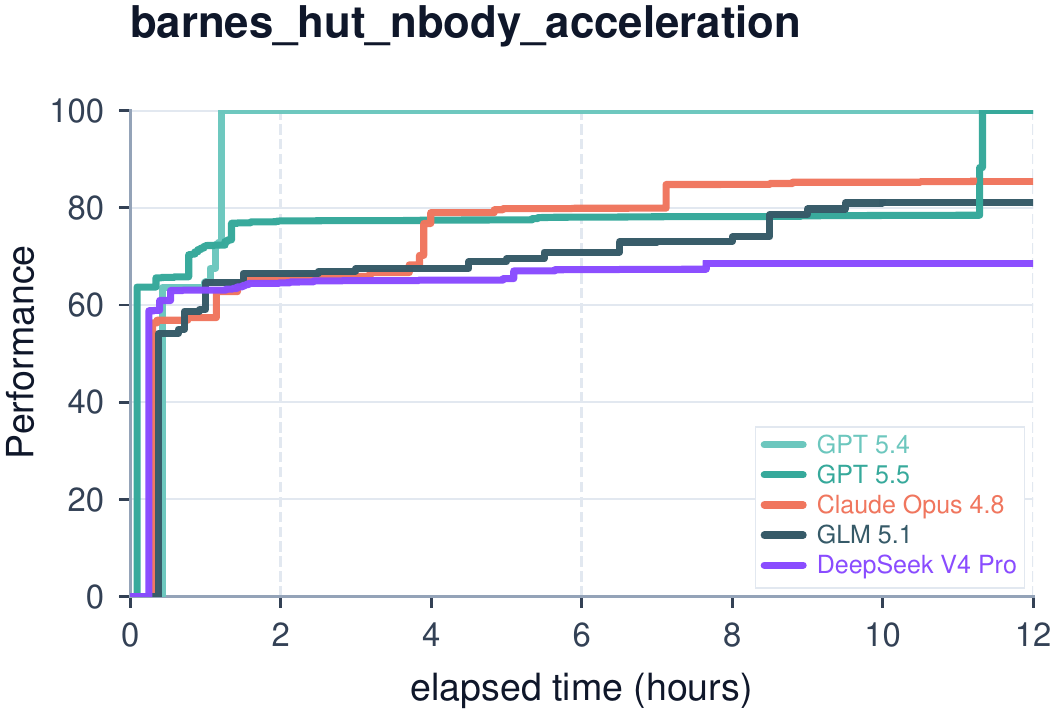}
\end{subfigure}
\hfill
\begin{subfigure}[b]{0.48\linewidth}
\includegraphics[width=\linewidth]{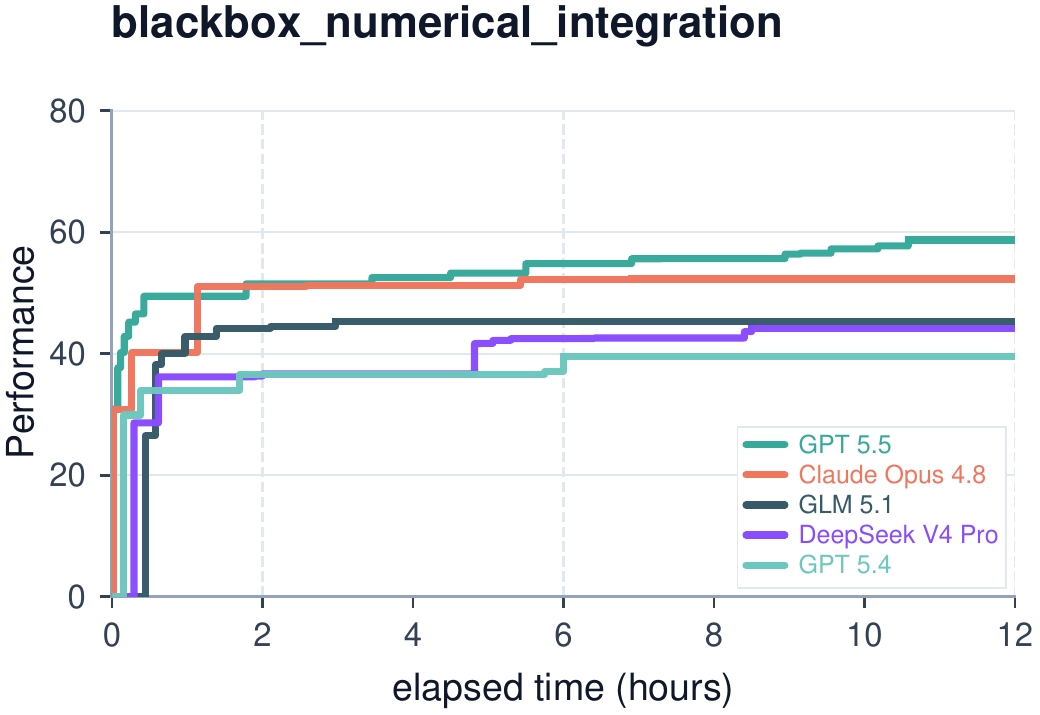}
\end{subfigure}
\vspace{0.5em}
\begin{subfigure}[b]{0.48\linewidth}
\includegraphics[width=\linewidth]{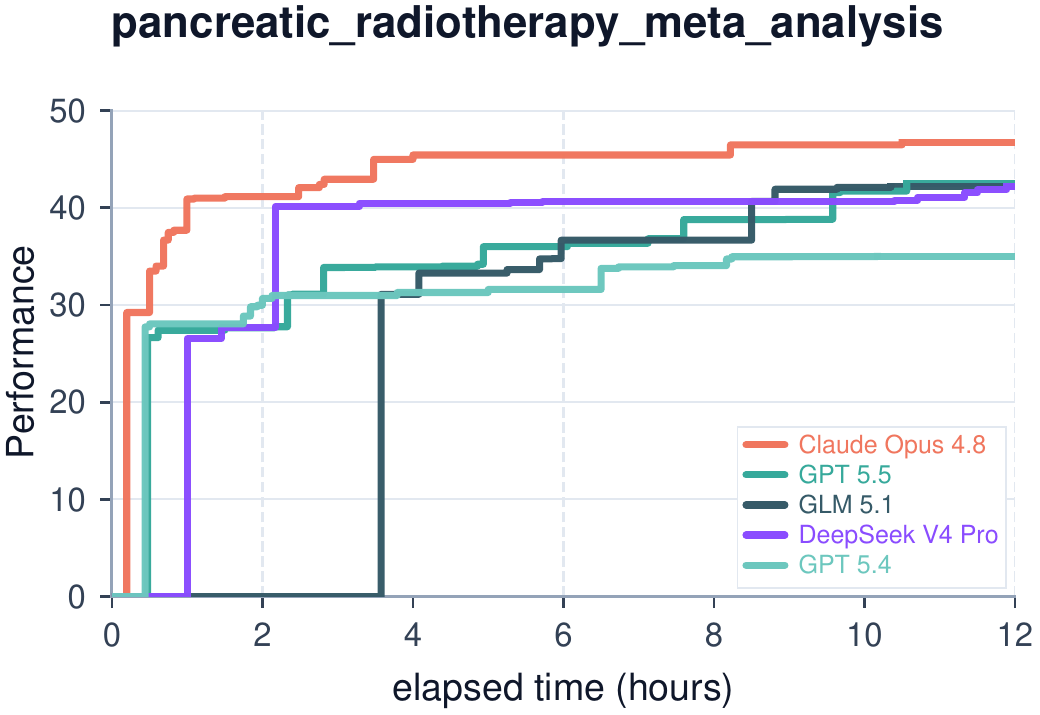}
\end{subfigure}
\hfill
\begin{subfigure}[b]{0.48\linewidth}
\includegraphics[width=\linewidth]{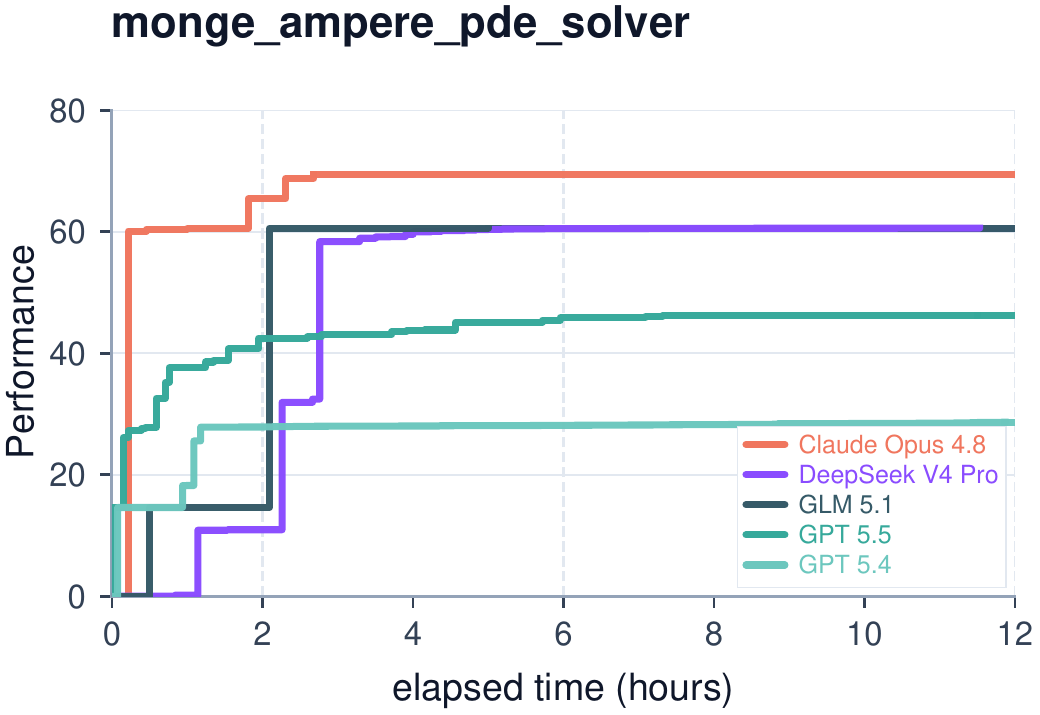}
\end{subfigure}
\caption{Per-task learning curves: Scientific Computing \& ML cont. (4/21).}
\label{fig:curves-all-4}
\end{figure}

\begin{figure}[p]
\centering
\begin{subfigure}[b]{0.48\linewidth}
\includegraphics[width=\linewidth]{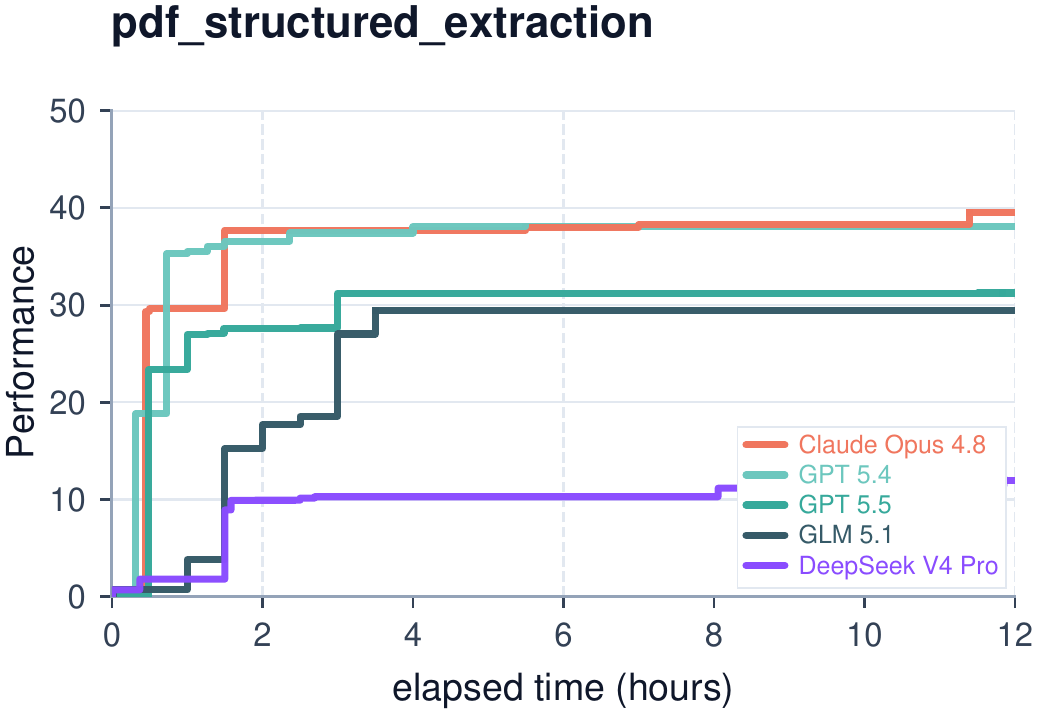}
\end{subfigure}
\hfill
\begin{subfigure}[b]{0.48\linewidth}
\includegraphics[width=\linewidth]{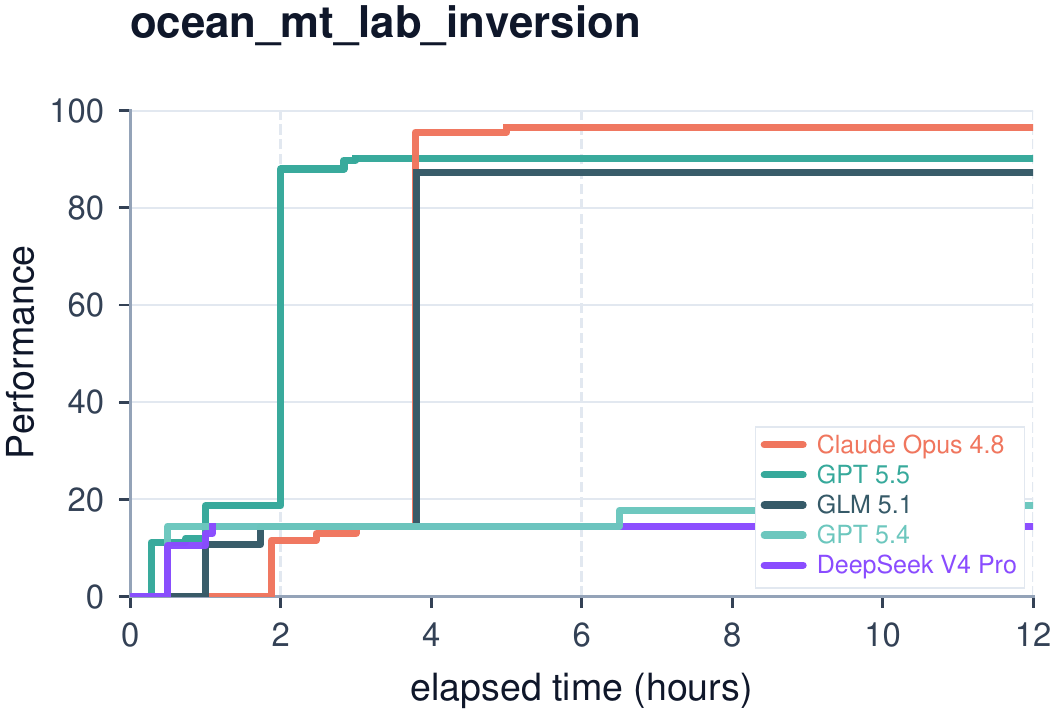}
\end{subfigure}
\vspace{0.5em}
\begin{subfigure}[b]{0.48\linewidth}
\includegraphics[width=\linewidth]{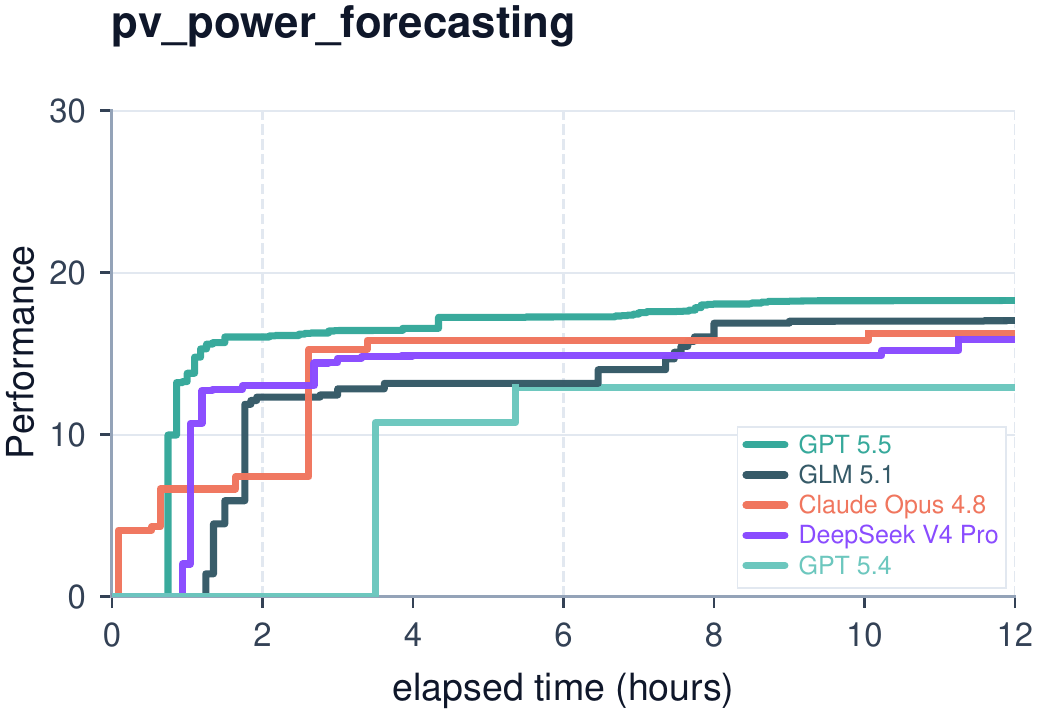}
\end{subfigure}
\hfill
\begin{subfigure}[b]{0.48\linewidth}
\includegraphics[width=\linewidth]{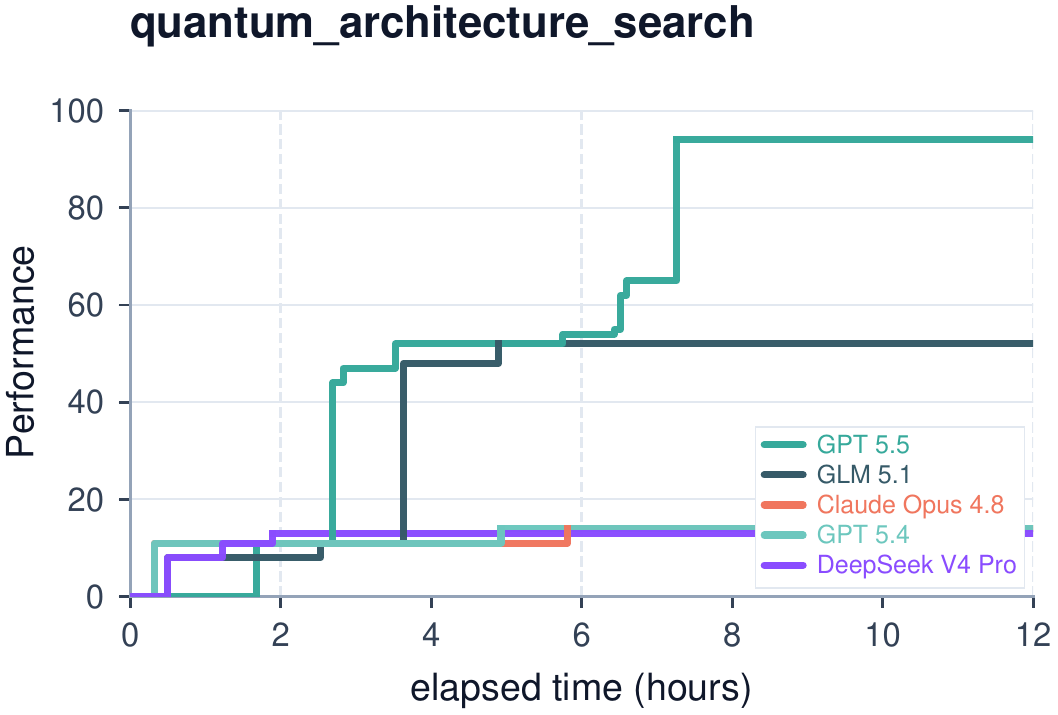}
\end{subfigure}
\vspace{0.5em}
\begin{subfigure}[b]{0.48\linewidth}
\includegraphics[width=\linewidth]{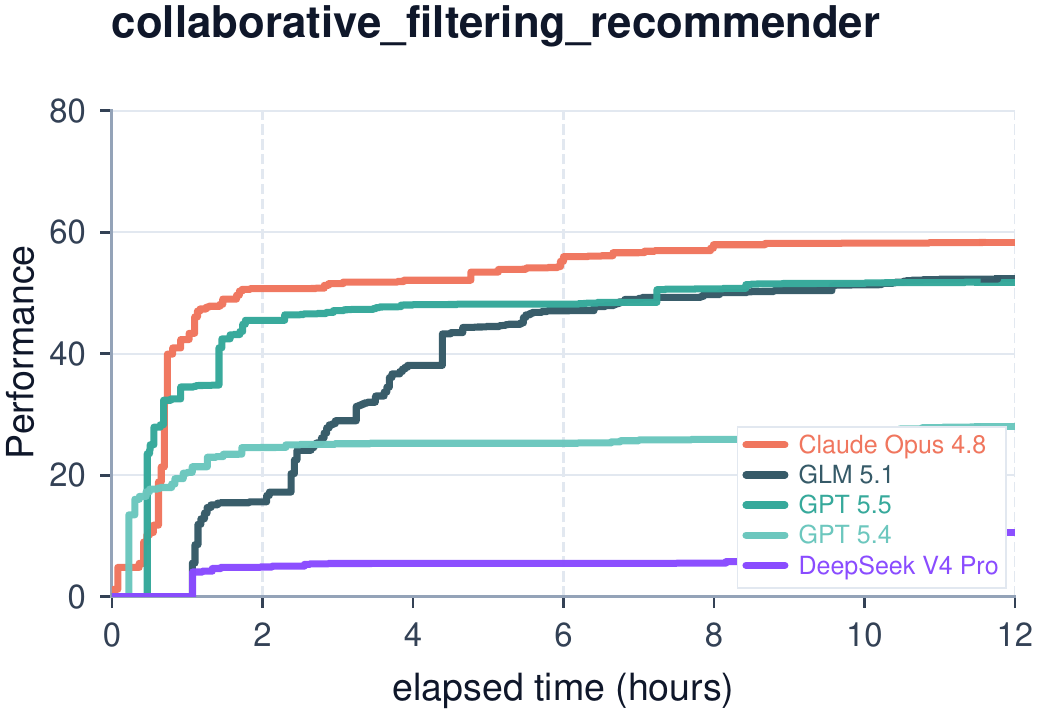}
\end{subfigure}
\hfill
\begin{subfigure}[b]{0.48\linewidth}
\includegraphics[width=\linewidth]{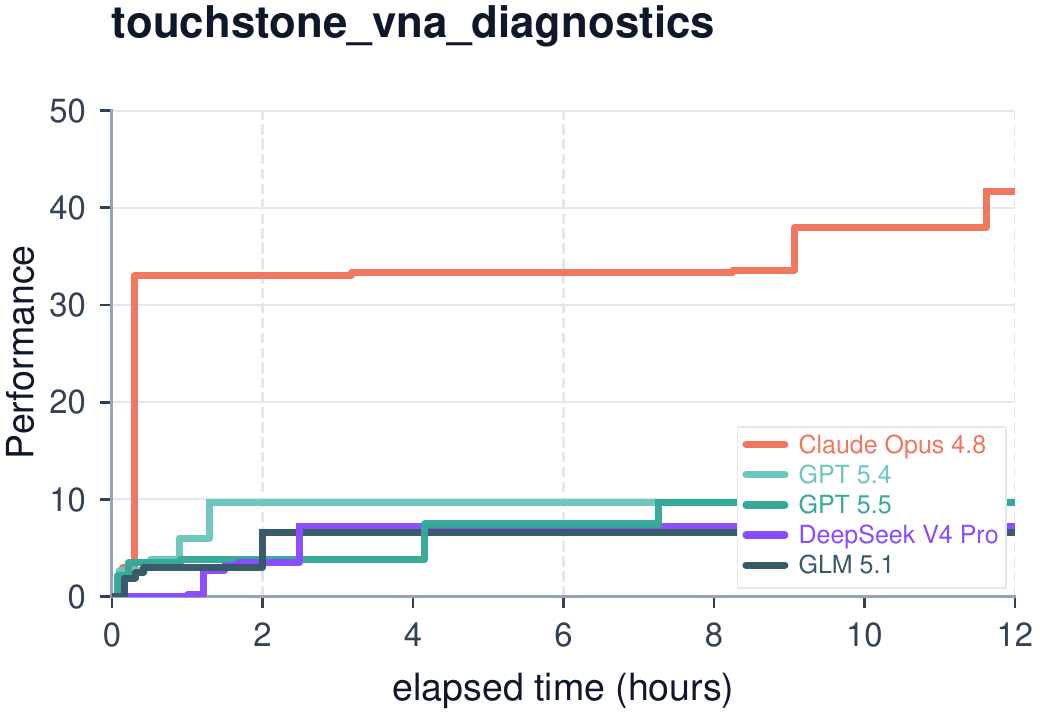}
\end{subfigure}
\caption{Per-task learning curves: Scientific Computing \& ML cont. (5/21).}
\label{fig:curves-all-5}
\end{figure}

\begin{figure}[p]
\centering
\begin{subfigure}[b]{0.48\linewidth}
\includegraphics[width=\linewidth]{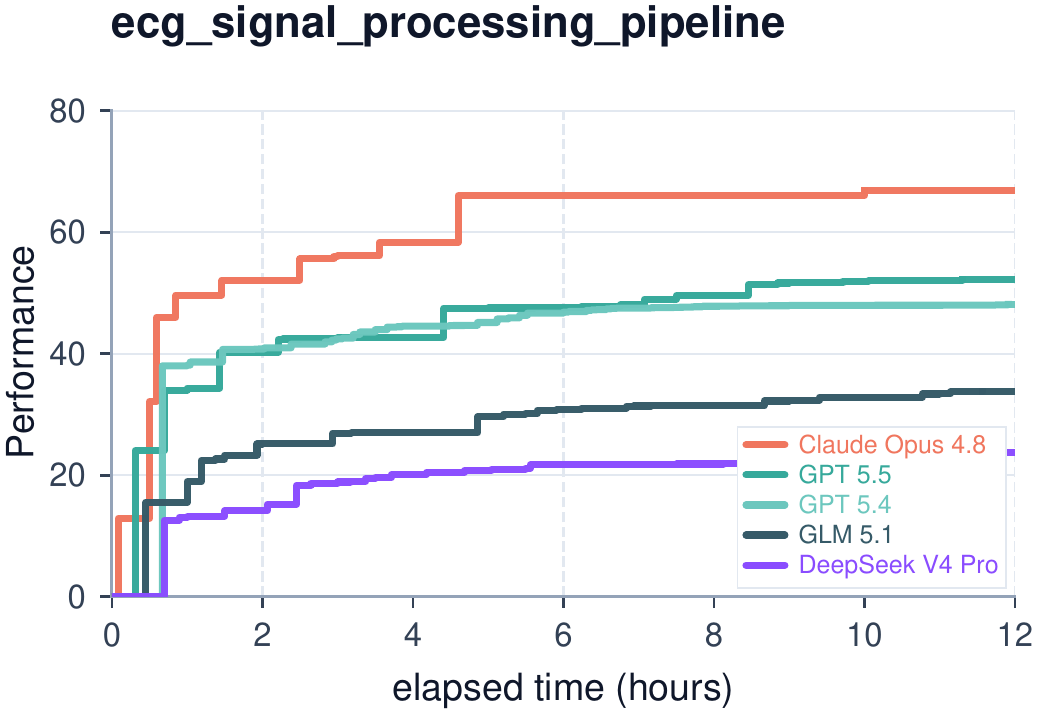}
\end{subfigure}
\hfill
\begin{subfigure}[b]{0.48\linewidth}
\includegraphics[width=\linewidth]{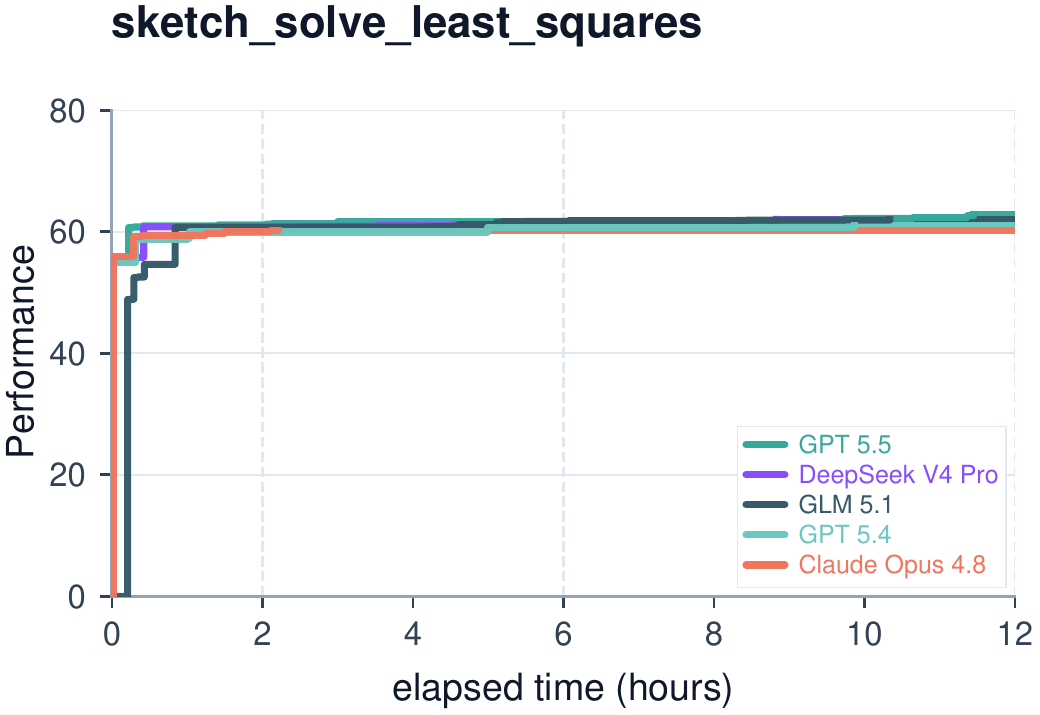}
\end{subfigure}
\vspace{0.5em}
\begin{subfigure}[b]{0.48\linewidth}
\includegraphics[width=\linewidth]{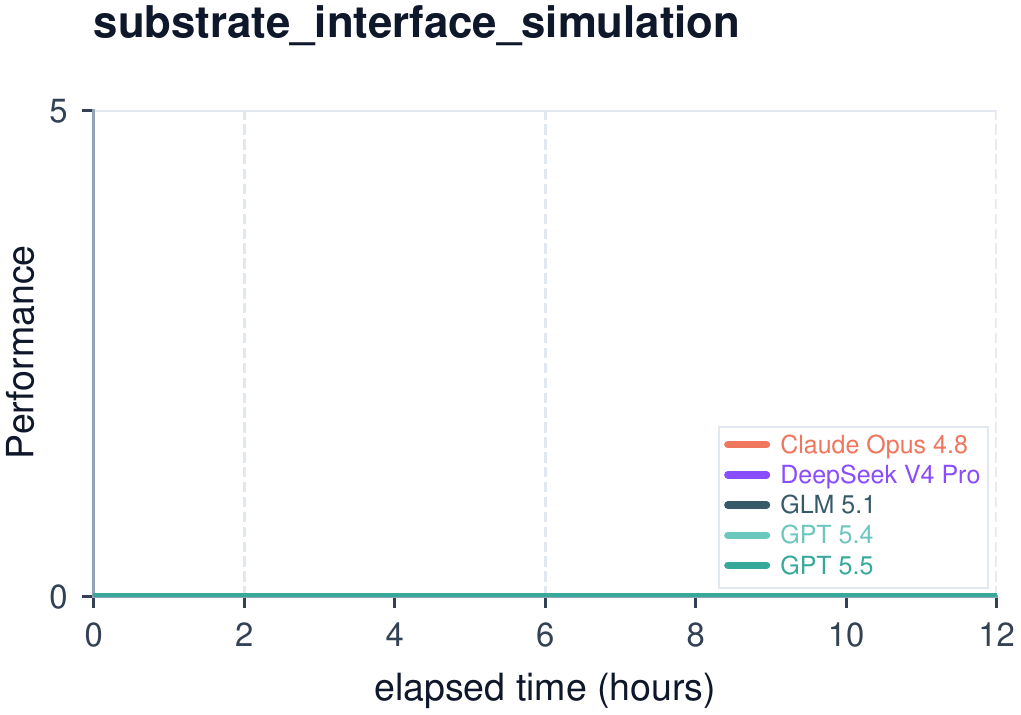}
\end{subfigure}
\hfill
\begin{subfigure}[b]{0.48\linewidth}
\includegraphics[width=\linewidth]{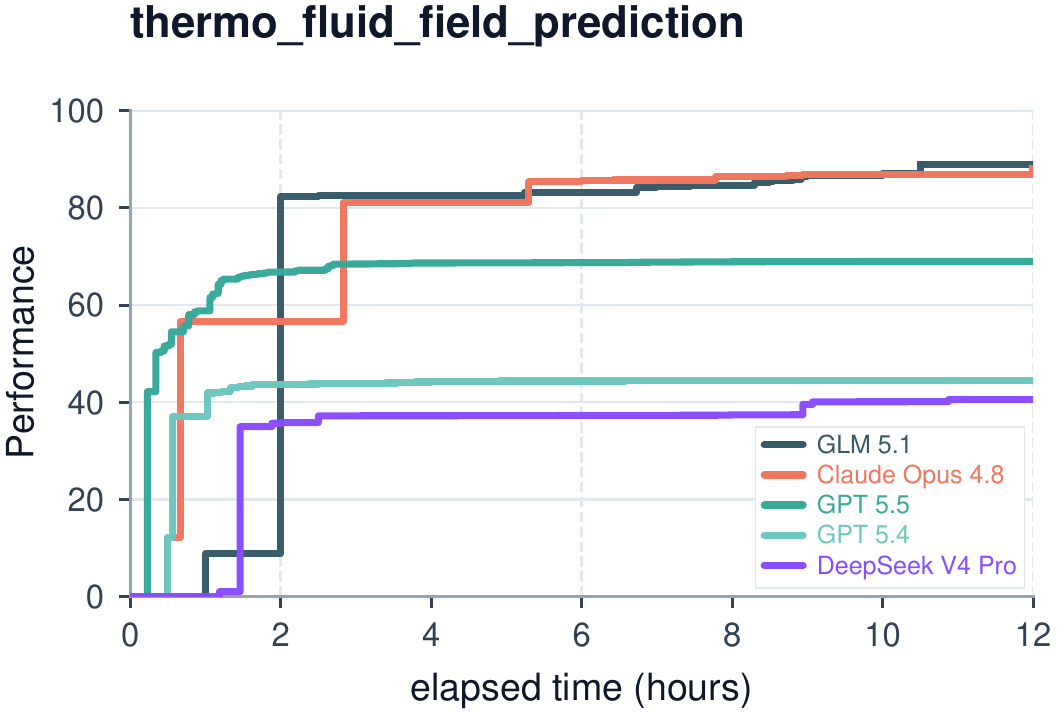}
\end{subfigure}
\vspace{0.5em}
\begin{subfigure}[b]{0.48\linewidth}
\includegraphics[width=\linewidth]{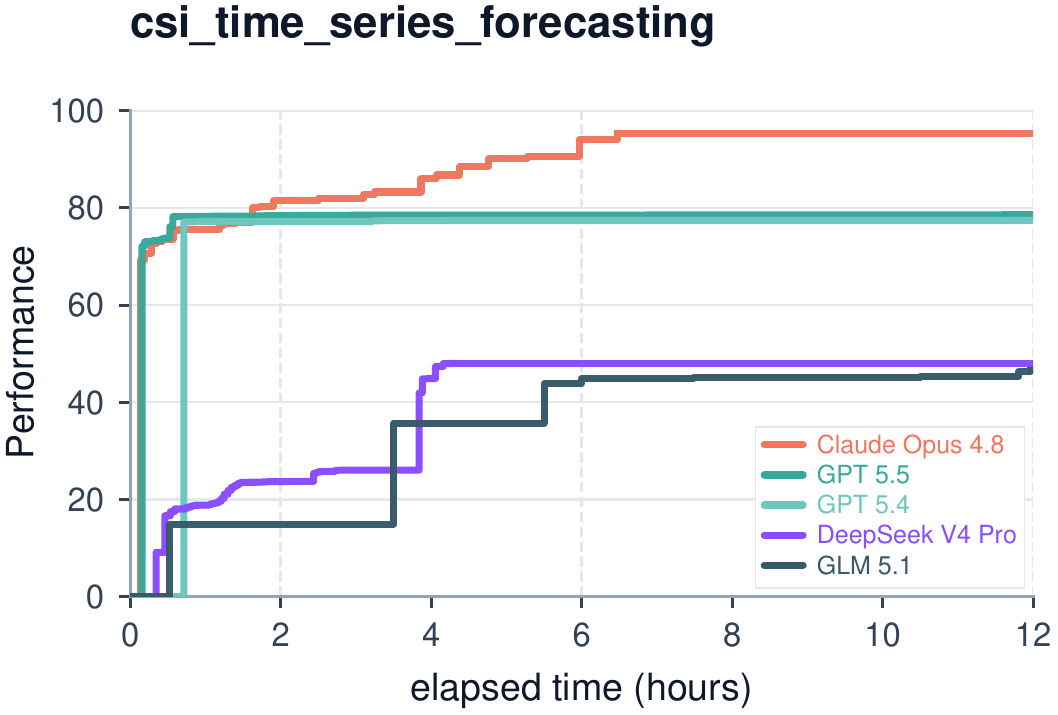}
\end{subfigure}
\hfill
\begin{subfigure}[b]{0.48\linewidth}
\includegraphics[width=\linewidth]{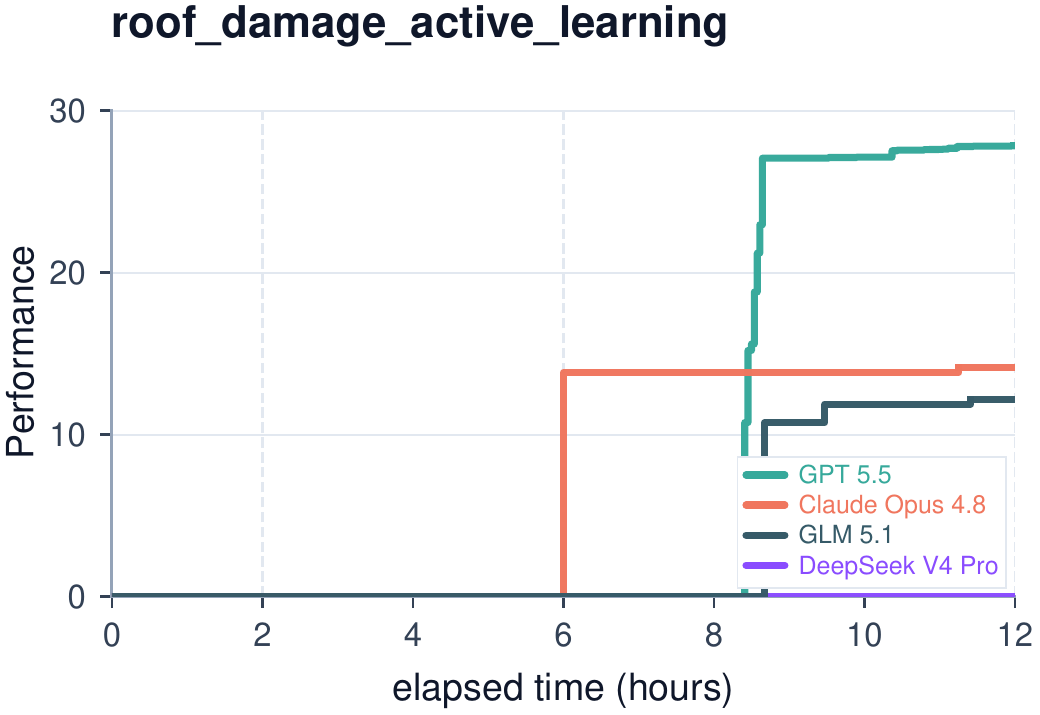}
\end{subfigure}
\caption{Per-task learning curves: Scientific Computing \& ML cont. (6/21).}
\label{fig:curves-all-6}
\end{figure}

\begin{figure}[p]
\centering
\begin{subfigure}[b]{0.48\linewidth}
\includegraphics[width=\linewidth]{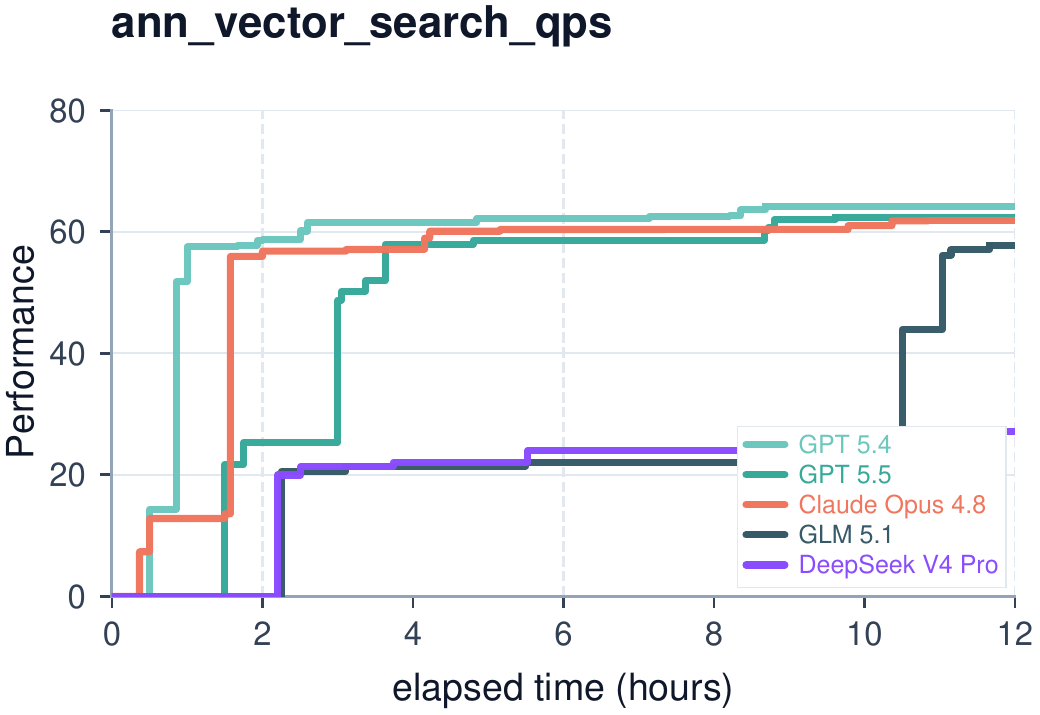}
\end{subfigure}
\hfill
\begin{subfigure}[b]{0.48\linewidth}
\includegraphics[width=\linewidth]{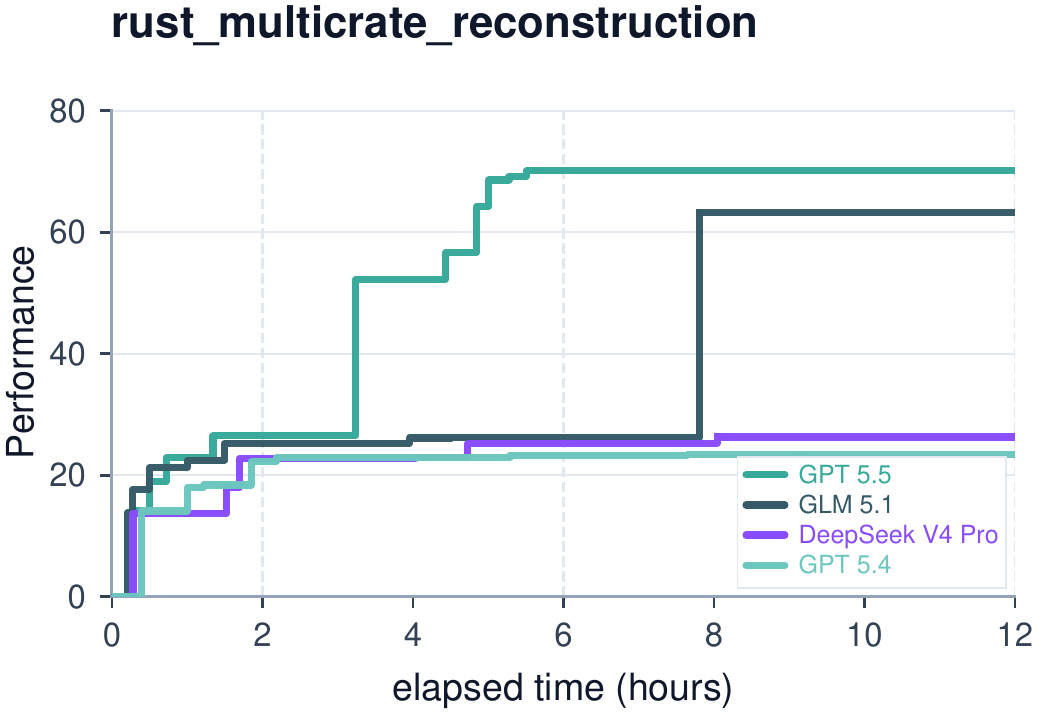}
\end{subfigure}
\vspace{0.5em}
\begin{subfigure}[b]{0.48\linewidth}
\includegraphics[width=\linewidth]{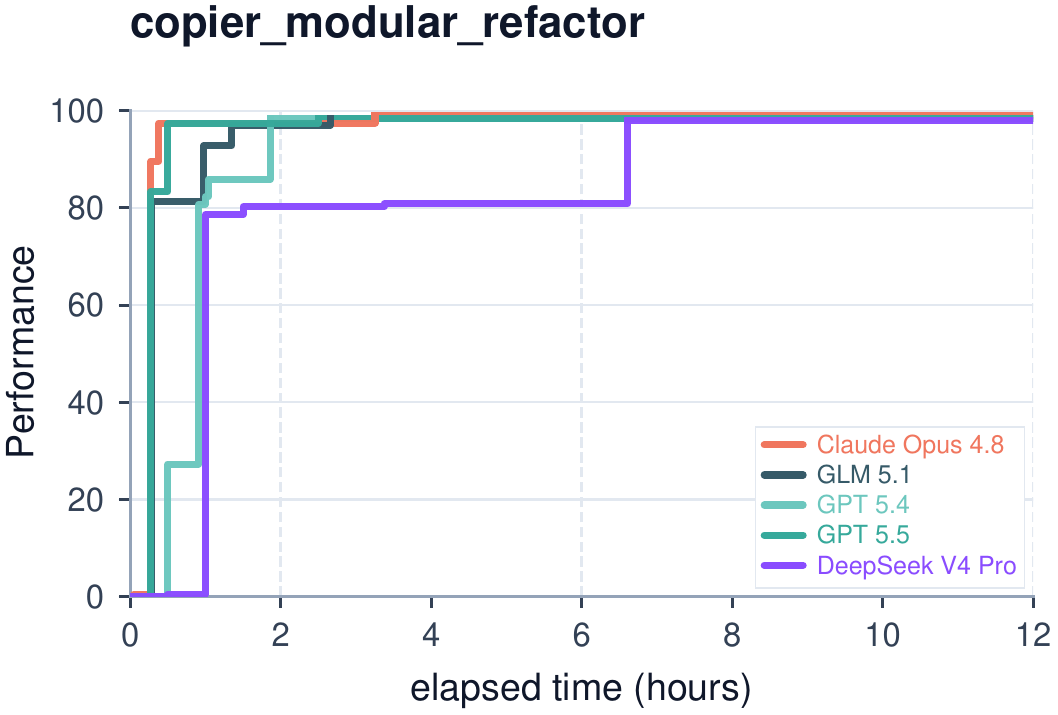}
\end{subfigure}
\hfill
\begin{subfigure}[b]{0.48\linewidth}
\includegraphics[width=\linewidth]{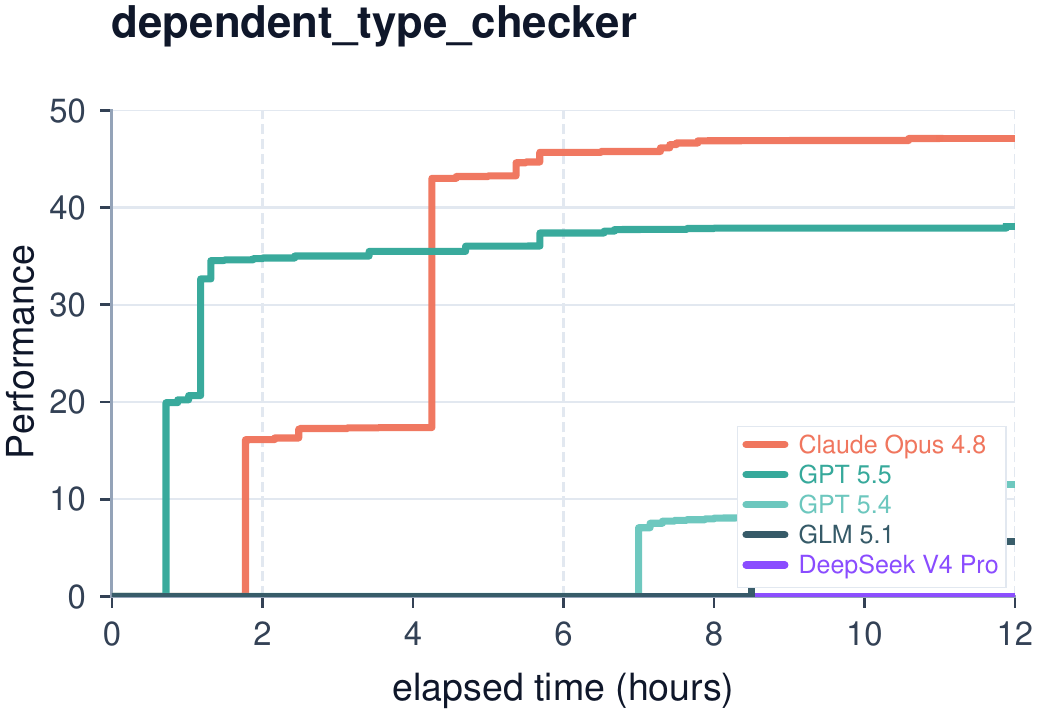}
\end{subfigure}
\vspace{0.5em}
\begin{subfigure}[b]{0.48\linewidth}
\includegraphics[width=\linewidth]{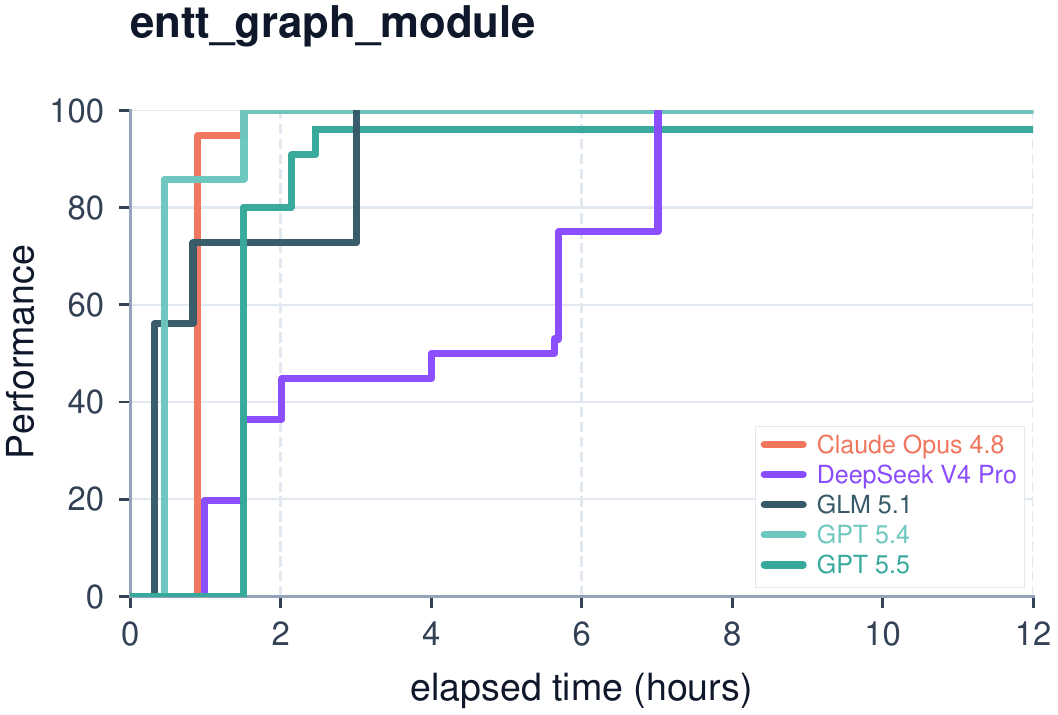}
\end{subfigure}
\hfill
\begin{subfigure}[b]{0.48\linewidth}
\includegraphics[width=\linewidth]{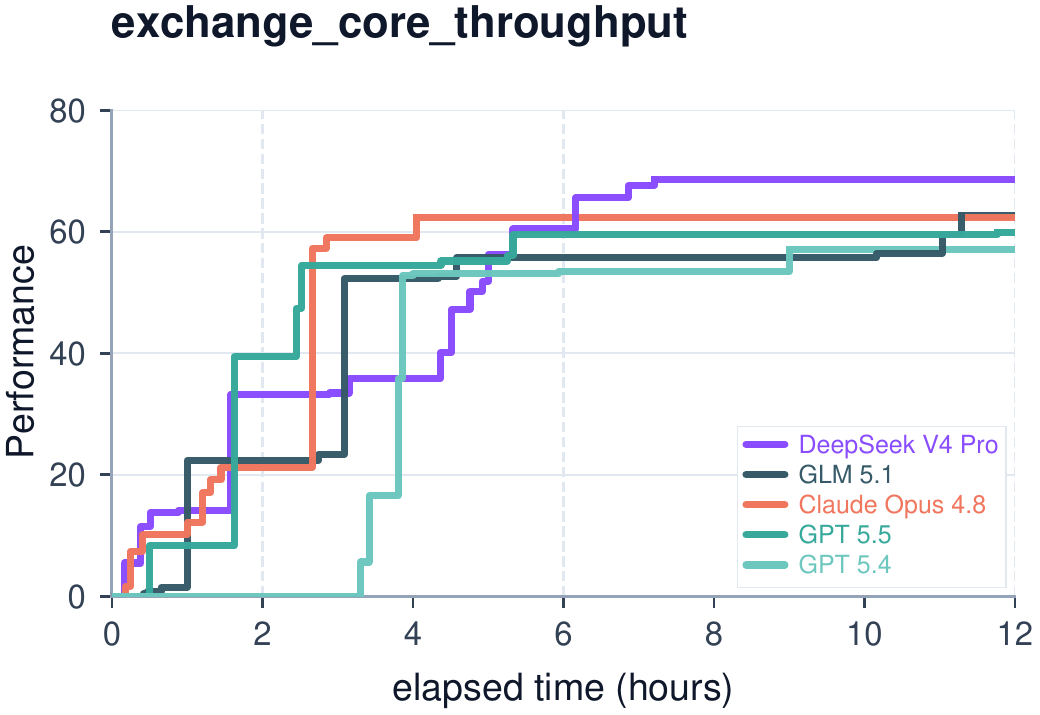}
\end{subfigure}
\caption{Per-task learning curves: Systems \& Software Engineering (7/21).}
\label{fig:curves-all-7}
\end{figure}

\begin{figure}[p]
\centering
\begin{subfigure}[b]{0.48\linewidth}
\includegraphics[width=\linewidth]{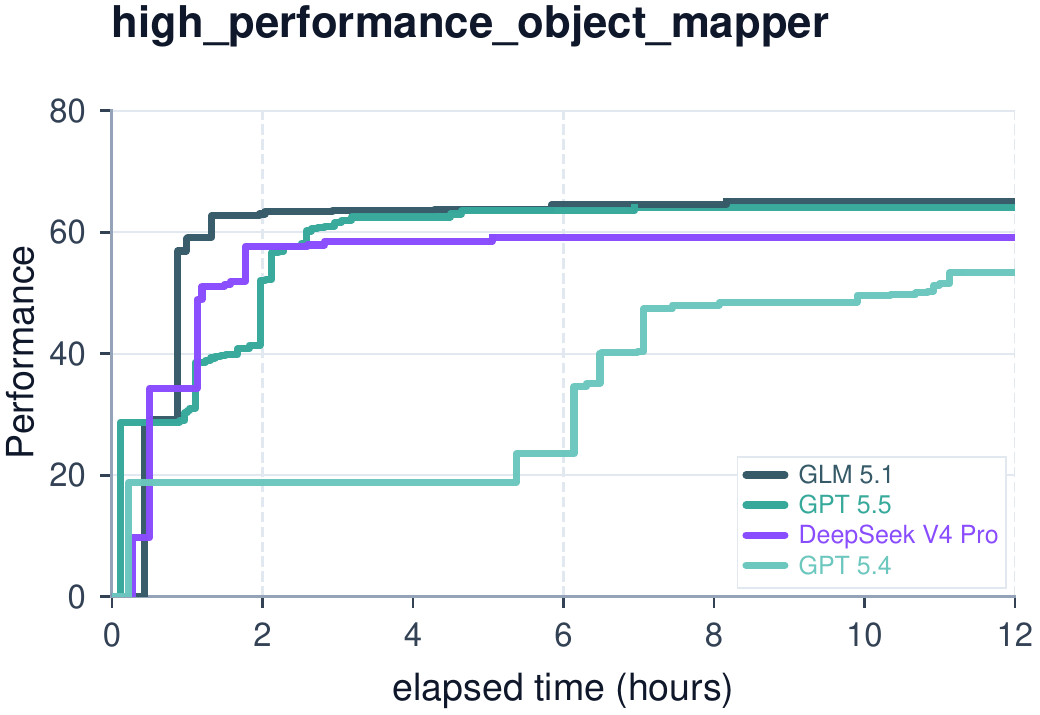}
\end{subfigure}
\hfill
\begin{subfigure}[b]{0.48\linewidth}
\includegraphics[width=\linewidth]{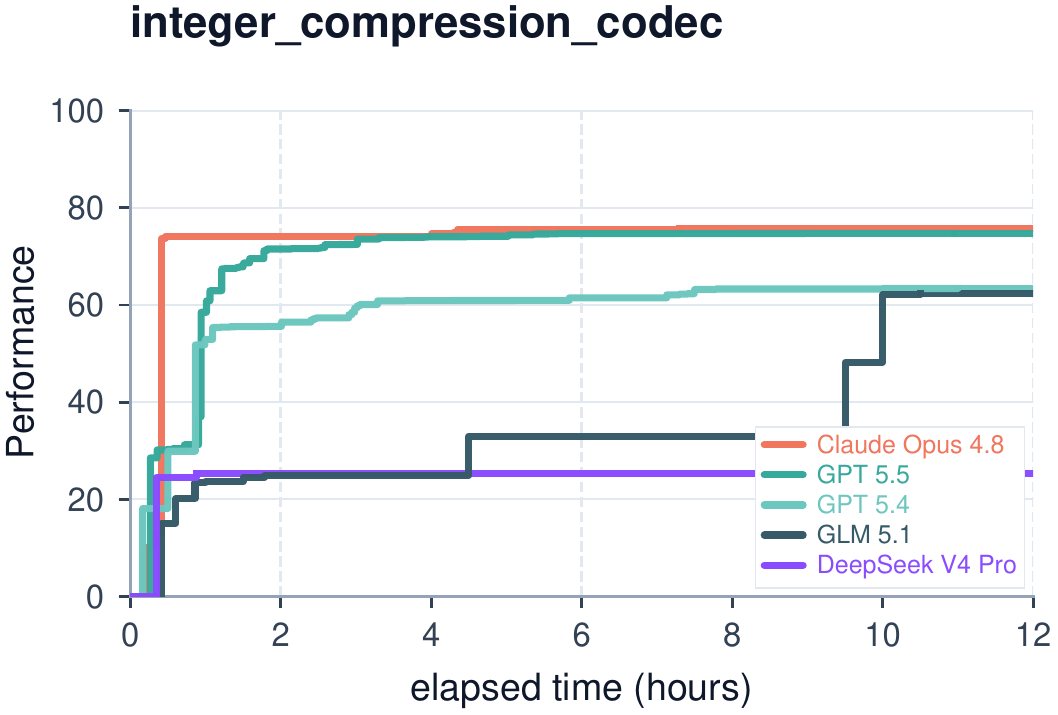}
\end{subfigure}
\vspace{0.5em}
\begin{subfigure}[b]{0.48\linewidth}
\includegraphics[width=\linewidth]{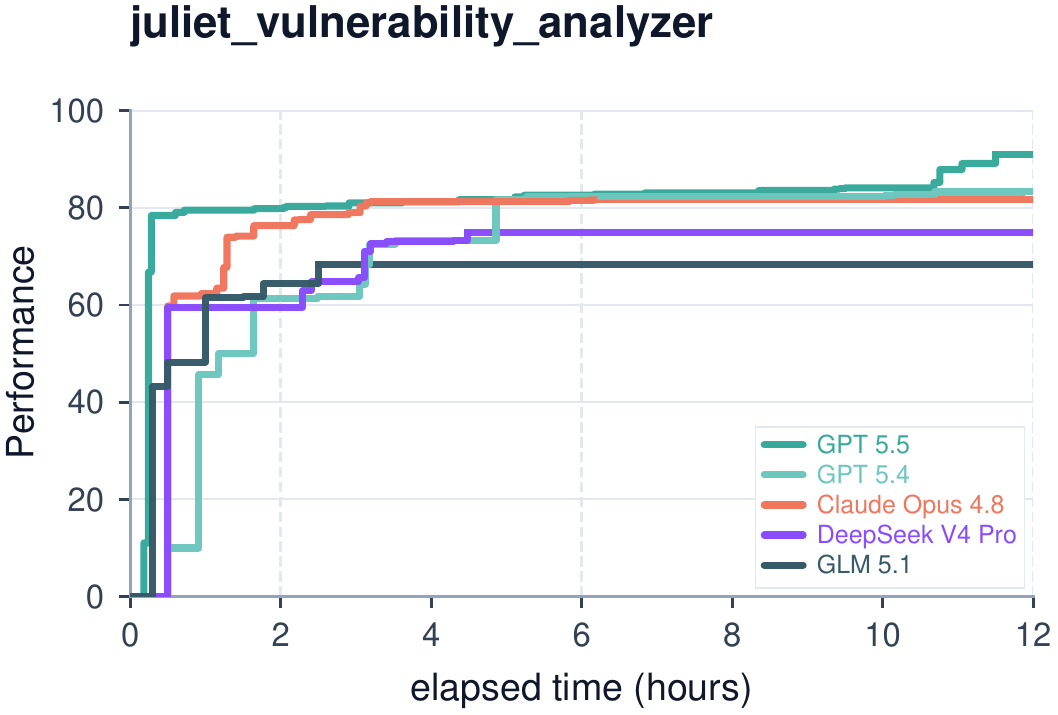}
\end{subfigure}
\hfill
\begin{subfigure}[b]{0.48\linewidth}
\includegraphics[width=\linewidth]{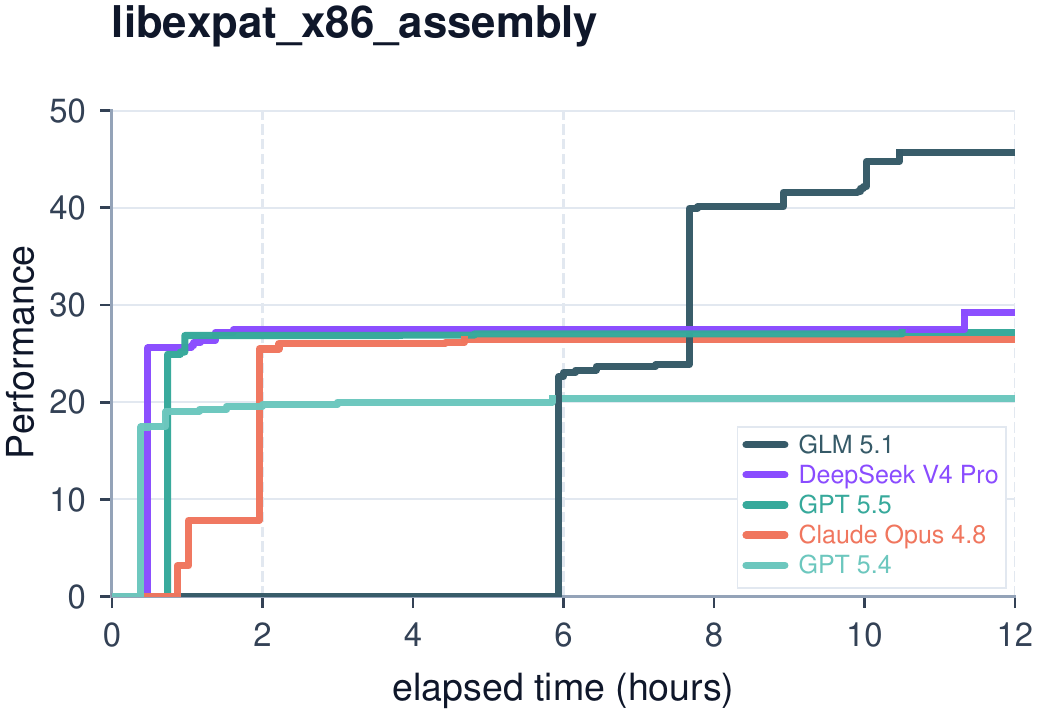}
\end{subfigure}
\vspace{0.5em}
\begin{subfigure}[b]{0.48\linewidth}
\includegraphics[width=\linewidth]{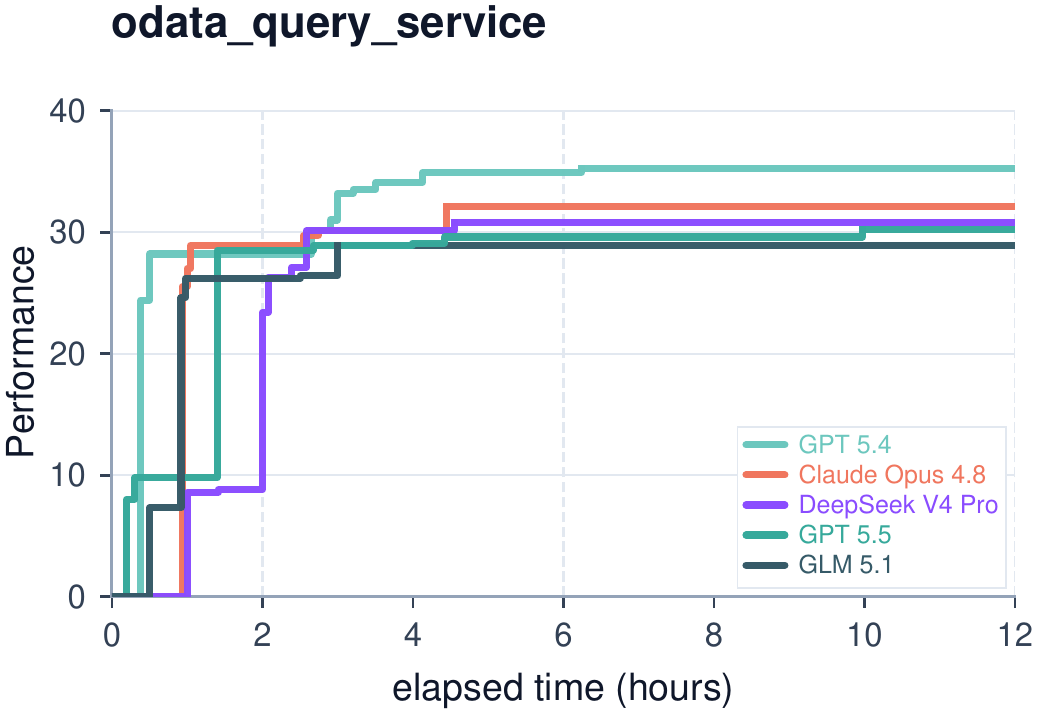}
\end{subfigure}
\hfill
\begin{subfigure}[b]{0.48\linewidth}
\includegraphics[width=\linewidth]{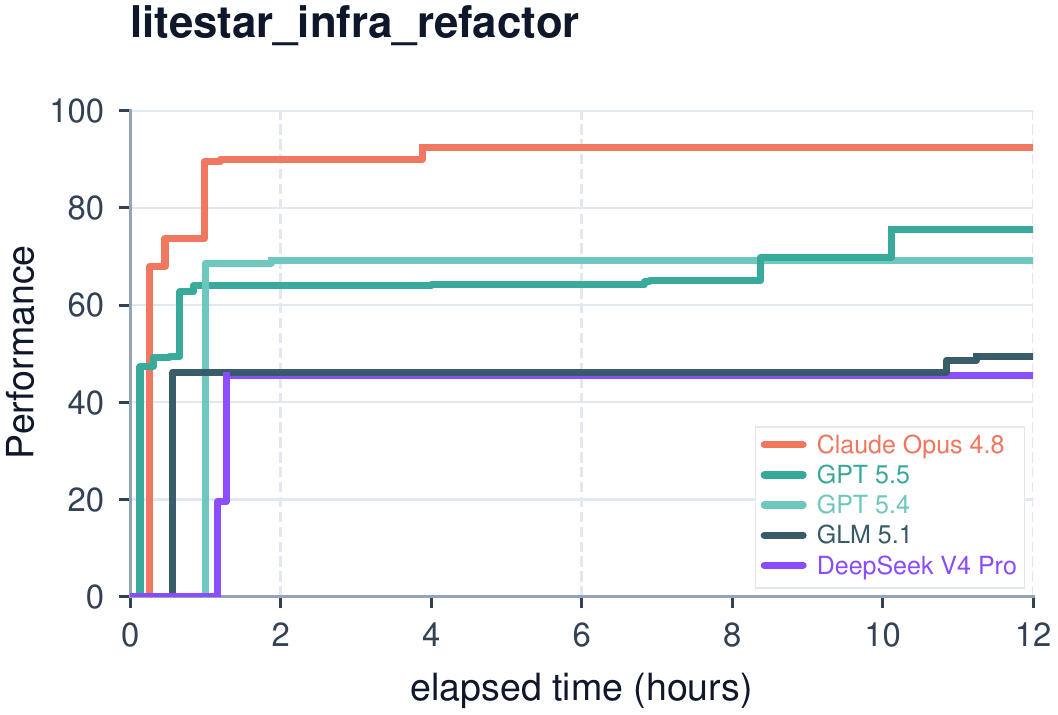}
\end{subfigure}
\caption{Per-task learning curves: Systems \& Software Engineering cont. (8/21).}
\label{fig:curves-all-8}
\end{figure}

\begin{figure}[p]
\centering
\begin{subfigure}[b]{0.48\linewidth}
\includegraphics[width=\linewidth]{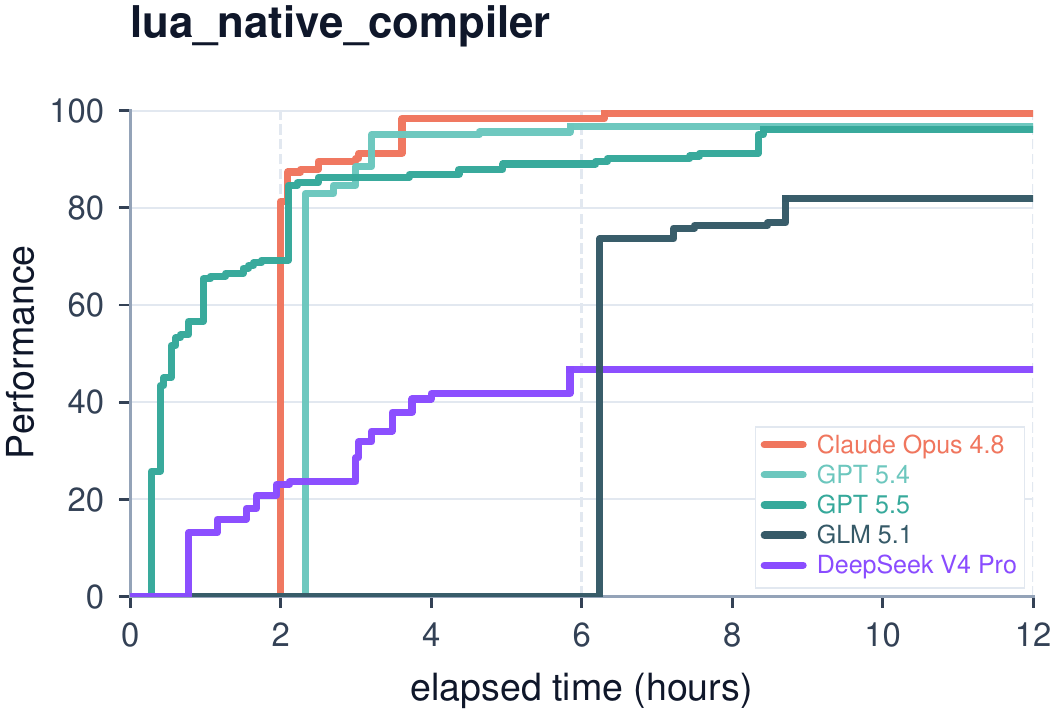}
\end{subfigure}
\hfill
\begin{subfigure}[b]{0.48\linewidth}
\includegraphics[width=\linewidth]{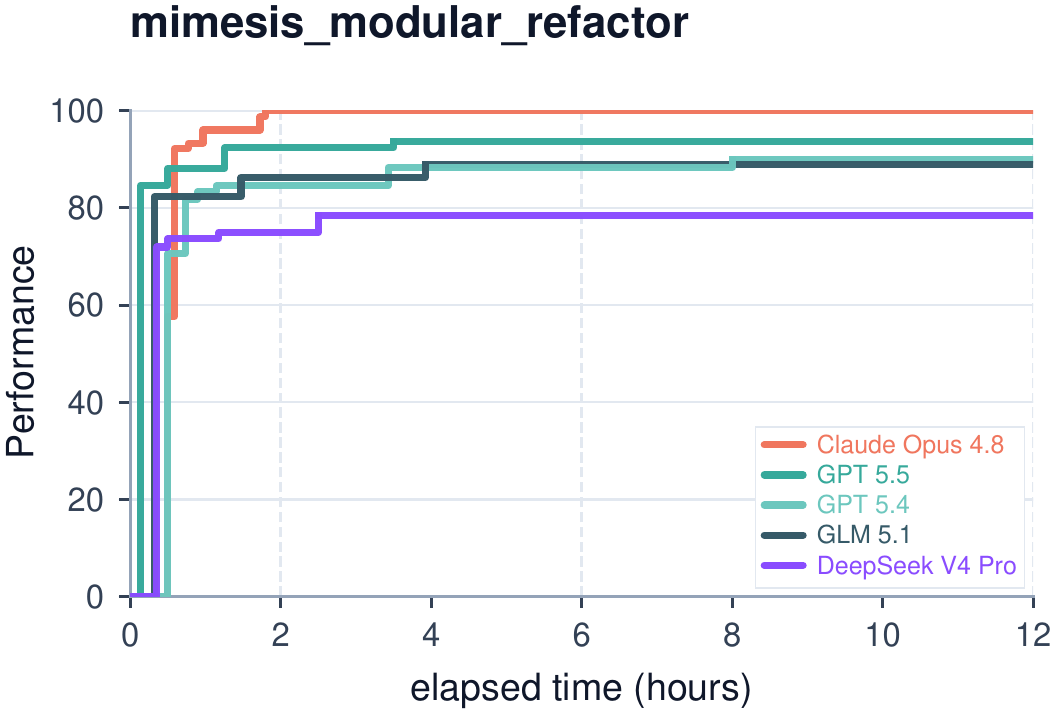}
\end{subfigure}
\vspace{0.5em}
\begin{subfigure}[b]{0.48\linewidth}
\includegraphics[width=\linewidth]{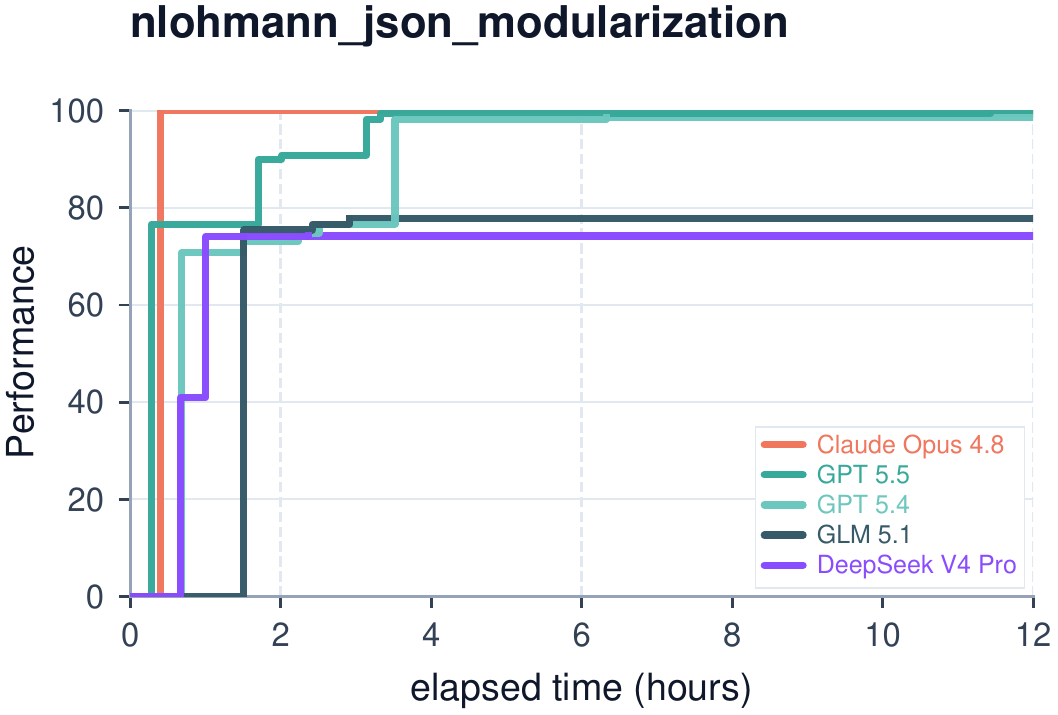}
\end{subfigure}
\hfill
\begin{subfigure}[b]{0.48\linewidth}
\includegraphics[width=\linewidth]{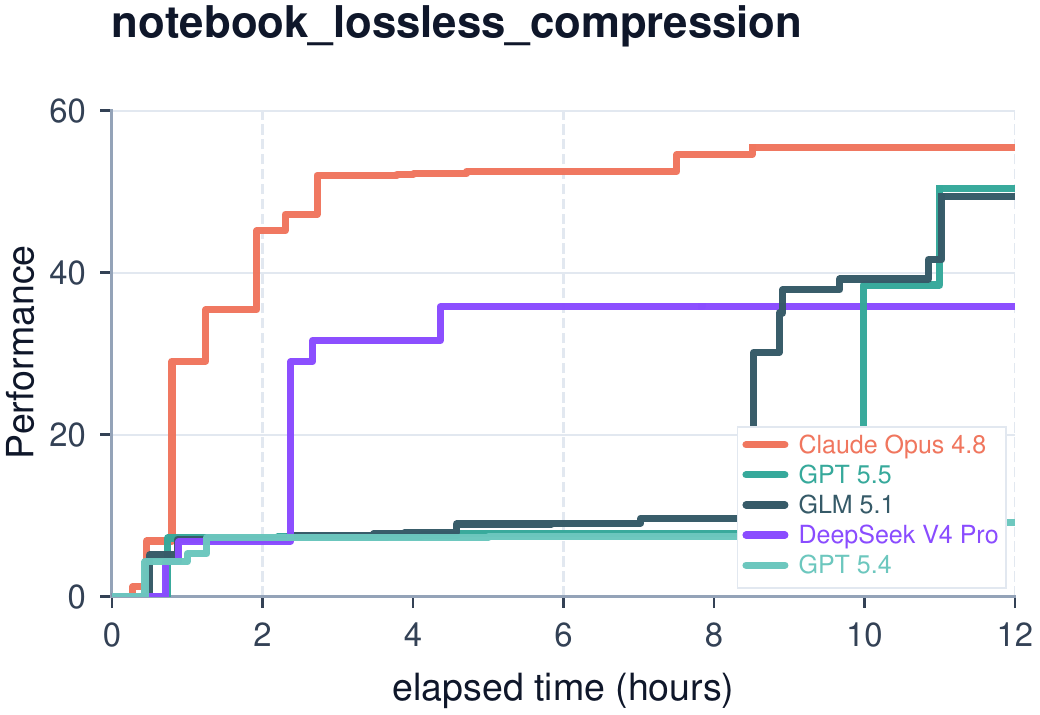}
\end{subfigure}
\vspace{0.5em}
\begin{subfigure}[b]{0.48\linewidth}
\includegraphics[width=\linewidth]{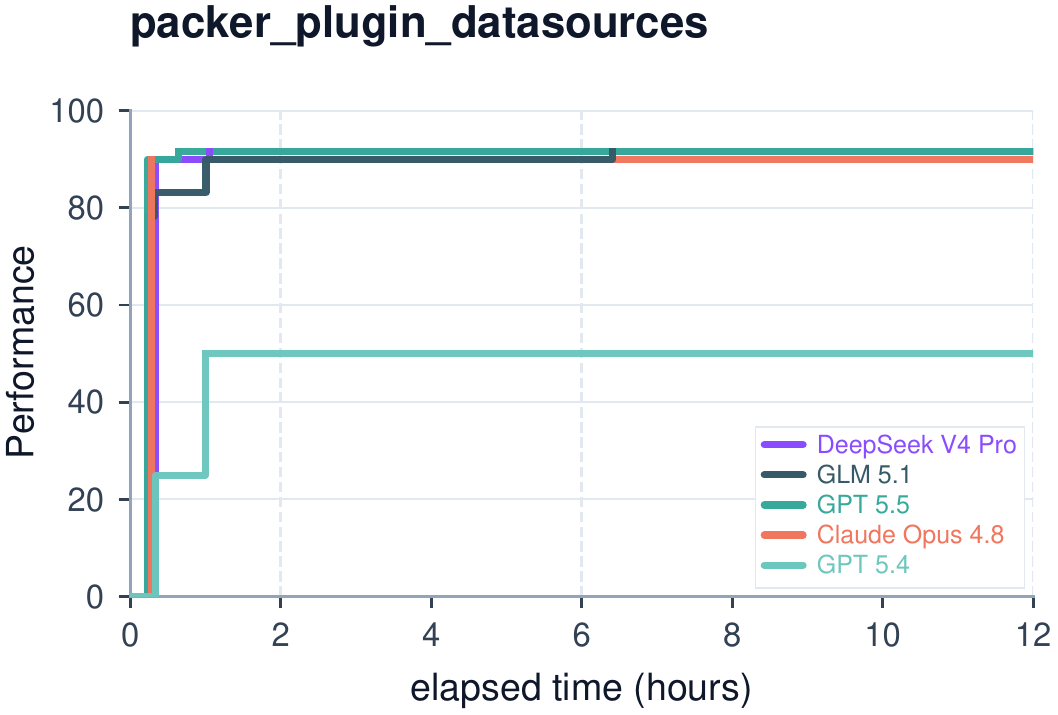}
\end{subfigure}
\hfill
\begin{subfigure}[b]{0.48\linewidth}
\includegraphics[width=\linewidth]{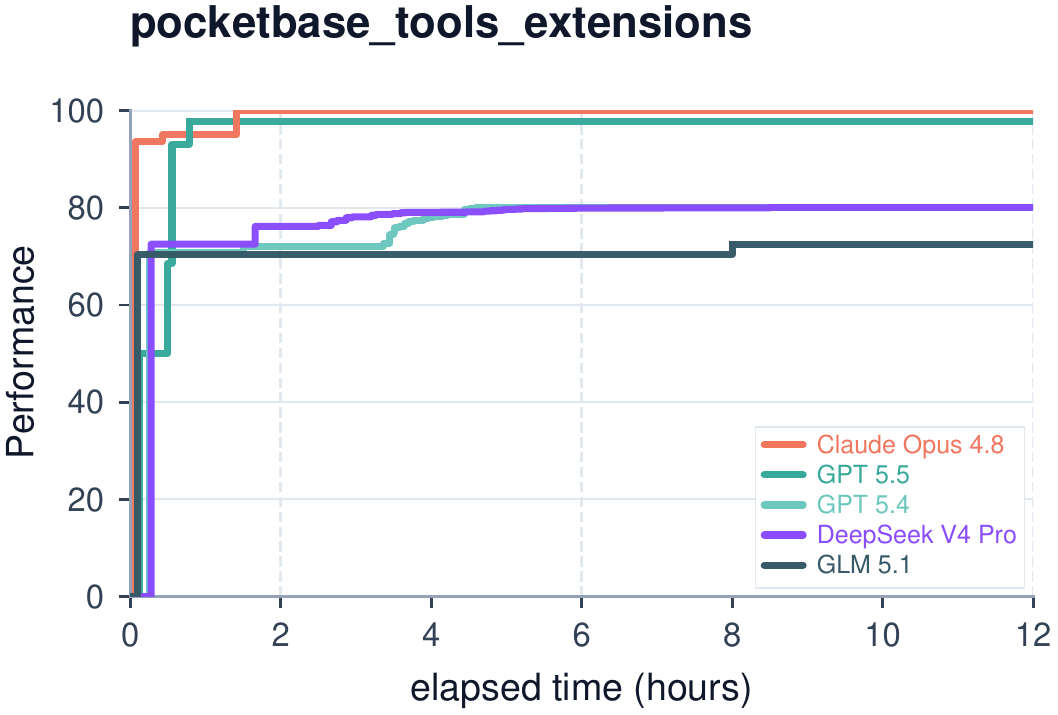}
\end{subfigure}
\caption{Per-task learning curves: Systems \& Software Engineering cont. (9/21).}
\label{fig:curves-all-9}
\end{figure}

\begin{figure}[p]
\centering
\begin{subfigure}[b]{0.48\linewidth}
\includegraphics[width=\linewidth]{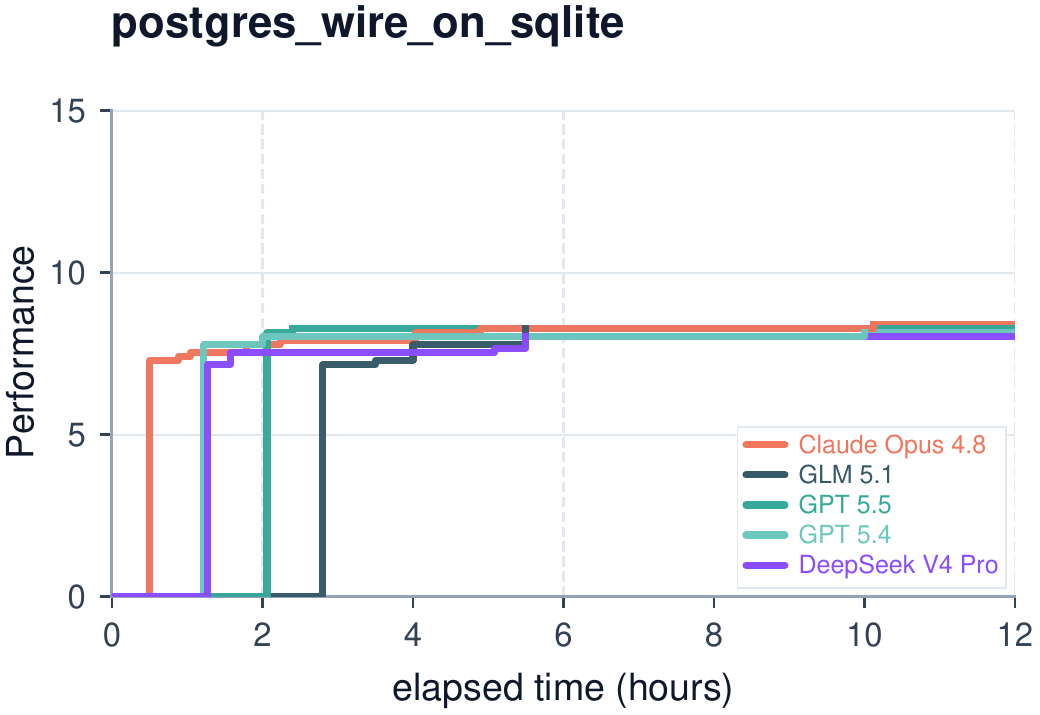}
\end{subfigure}
\hfill
\begin{subfigure}[b]{0.48\linewidth}
\includegraphics[width=\linewidth]{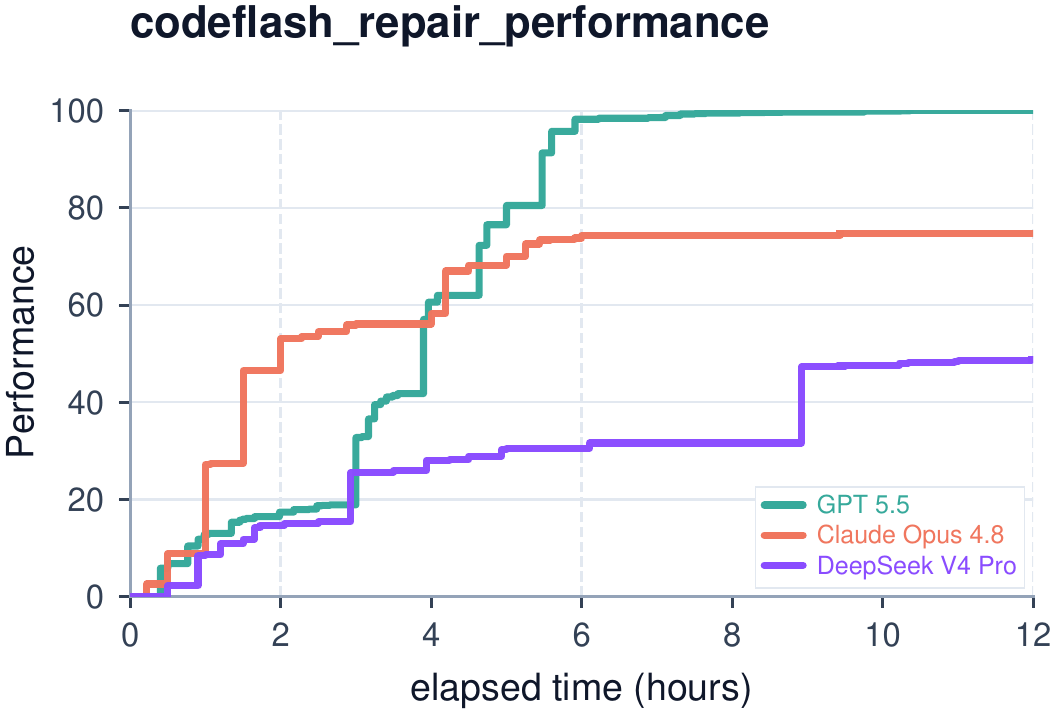}
\end{subfigure}
\vspace{0.5em}
\begin{subfigure}[b]{0.48\linewidth}
\includegraphics[width=\linewidth]{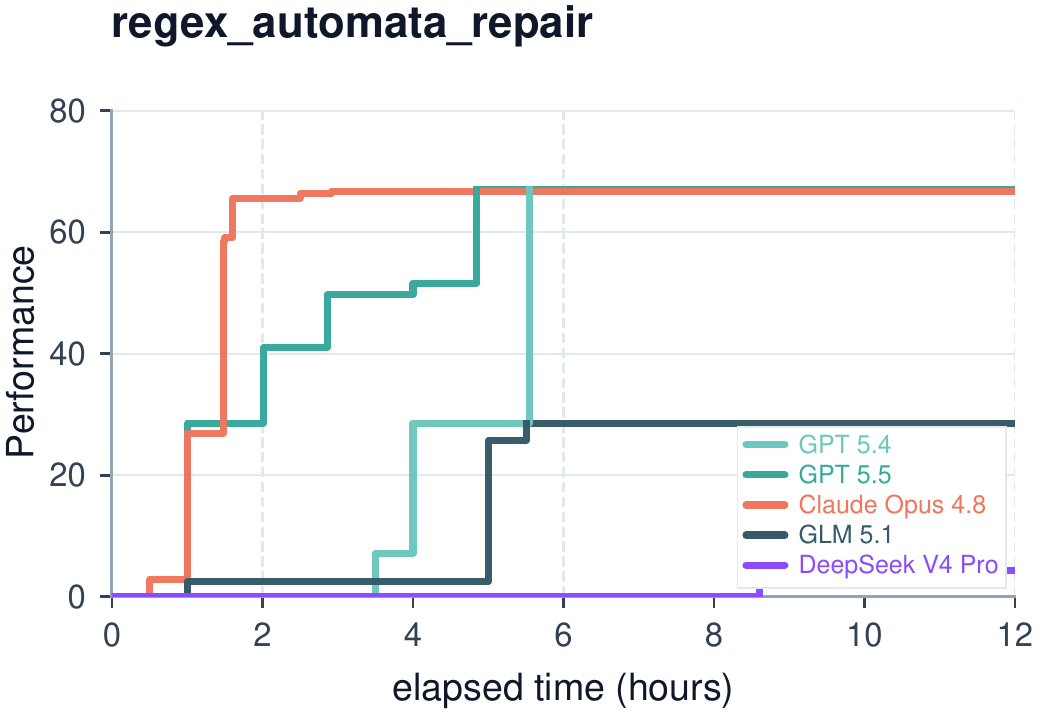}
\end{subfigure}
\hfill
\begin{subfigure}[b]{0.48\linewidth}
\includegraphics[width=\linewidth]{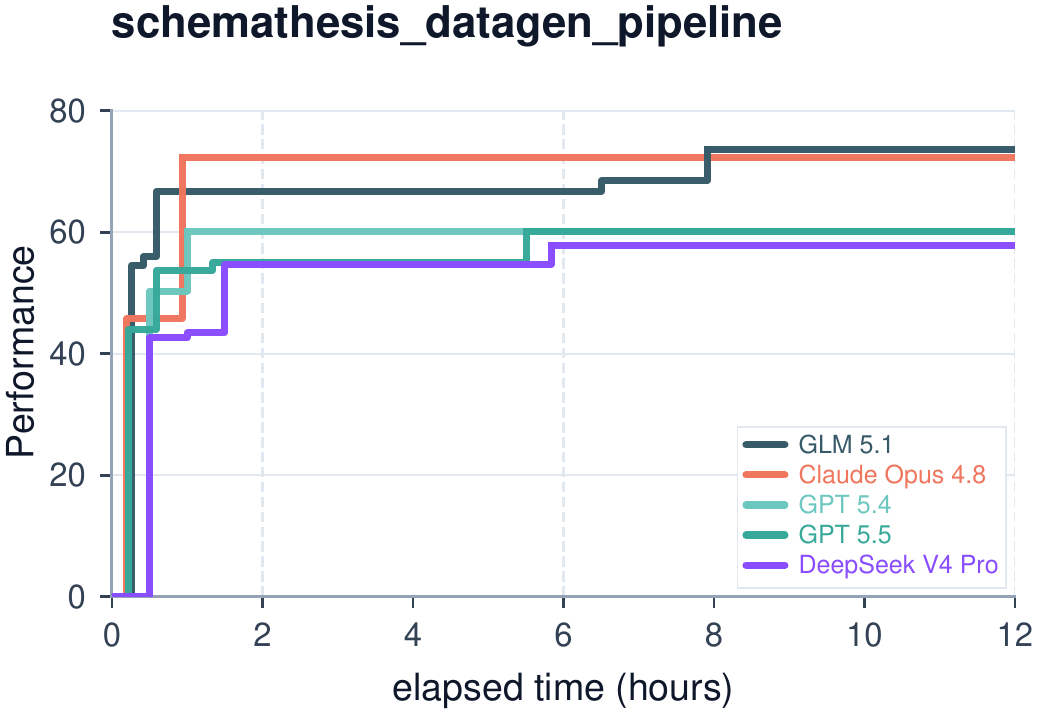}
\end{subfigure}
\vspace{0.5em}
\begin{subfigure}[b]{0.48\linewidth}
\includegraphics[width=\linewidth]{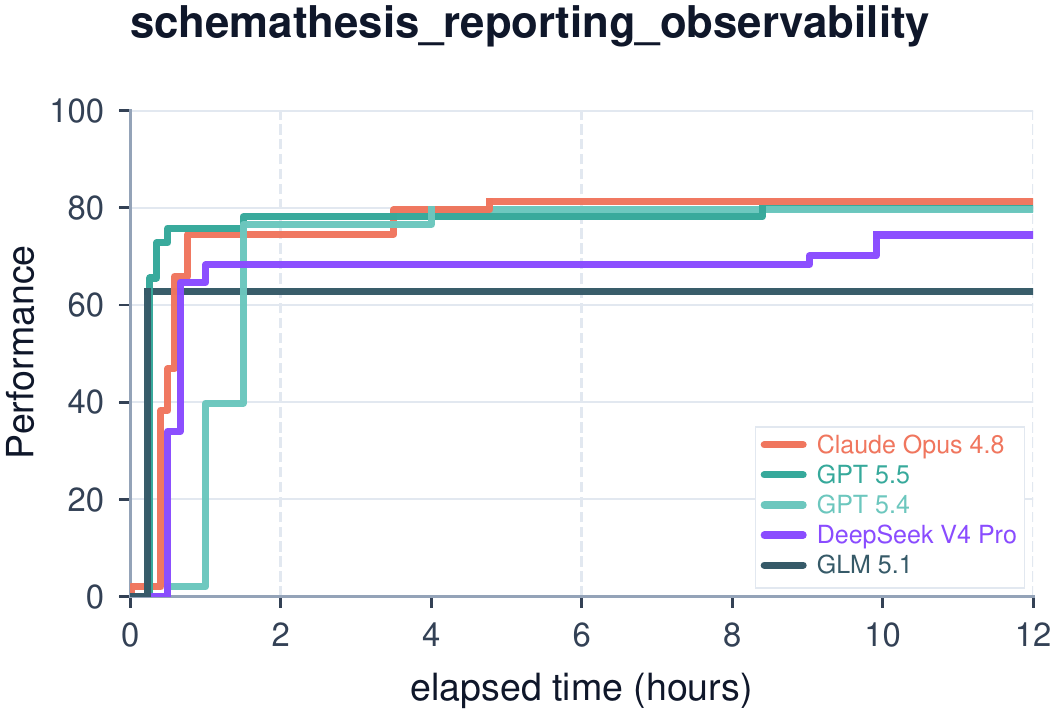}
\end{subfigure}
\hfill
\begin{subfigure}[b]{0.48\linewidth}
\includegraphics[width=\linewidth]{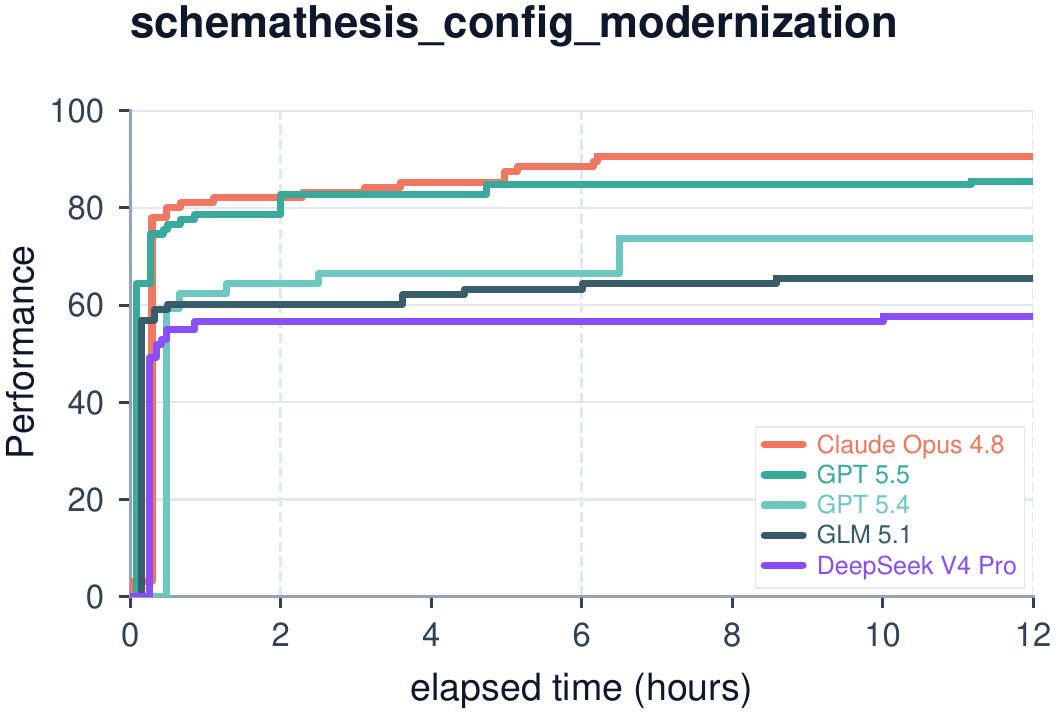}
\end{subfigure}
\caption{Per-task learning curves: Systems \& Software Engineering cont. (10/21).}
\label{fig:curves-all-10}
\end{figure}

\begin{figure}[p]
\centering
\begin{subfigure}[b]{0.48\linewidth}
\includegraphics[width=\linewidth]{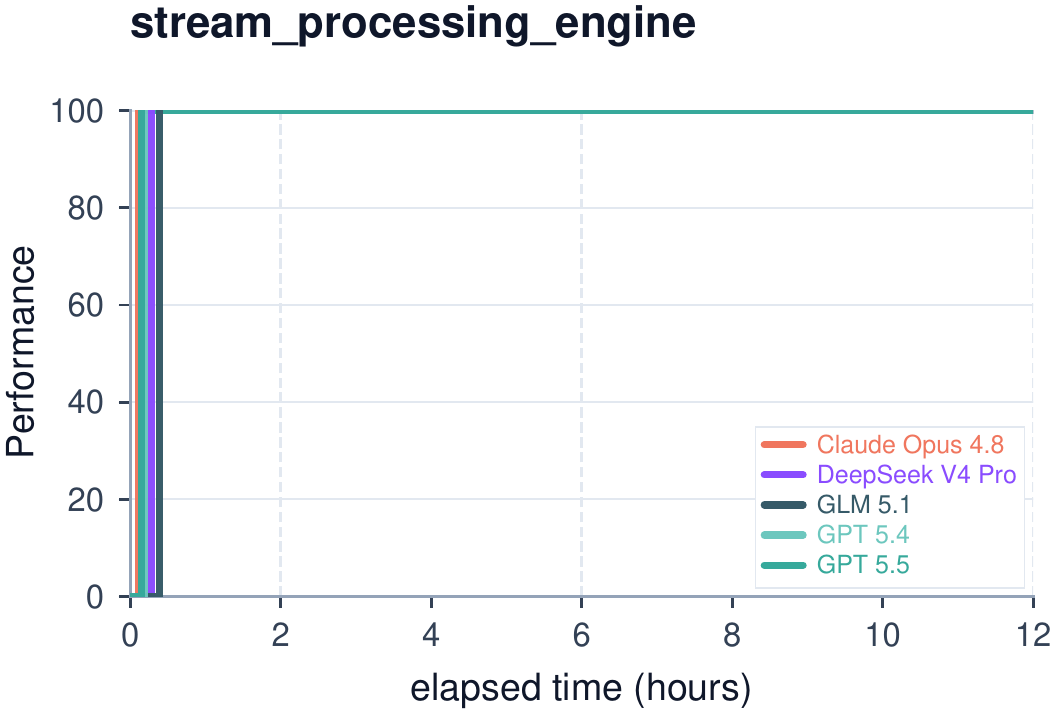}
\end{subfigure}
\hfill
\begin{subfigure}[b]{0.48\linewidth}
\includegraphics[width=\linewidth]{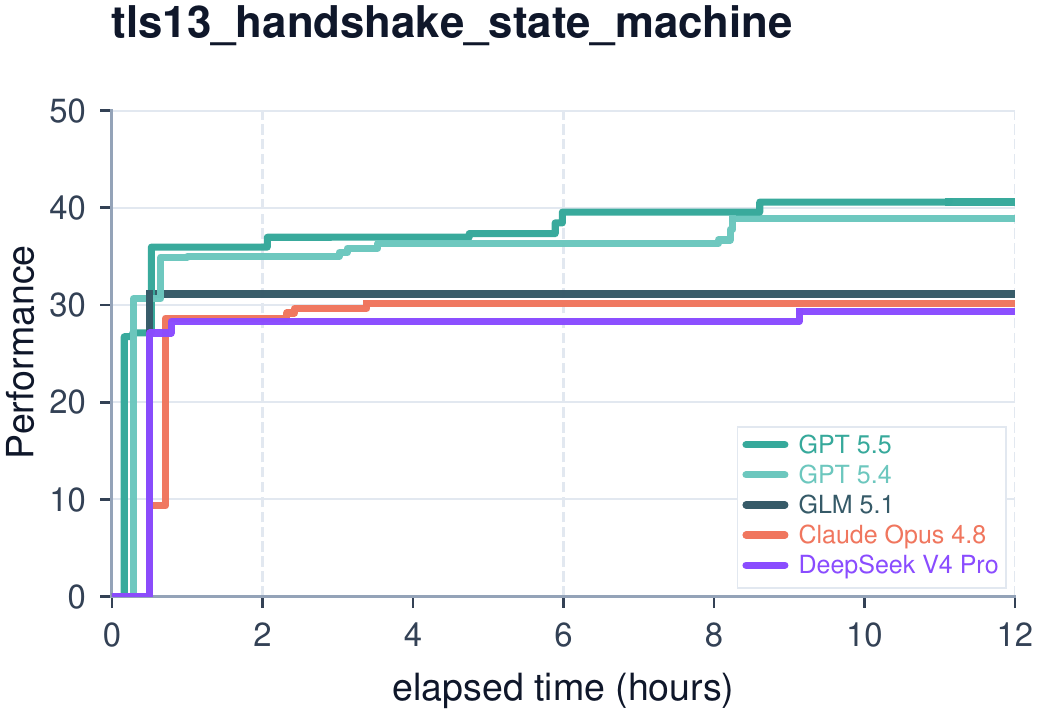}
\end{subfigure}
\vspace{0.5em}
\begin{subfigure}[b]{0.48\linewidth}
\includegraphics[width=\linewidth]{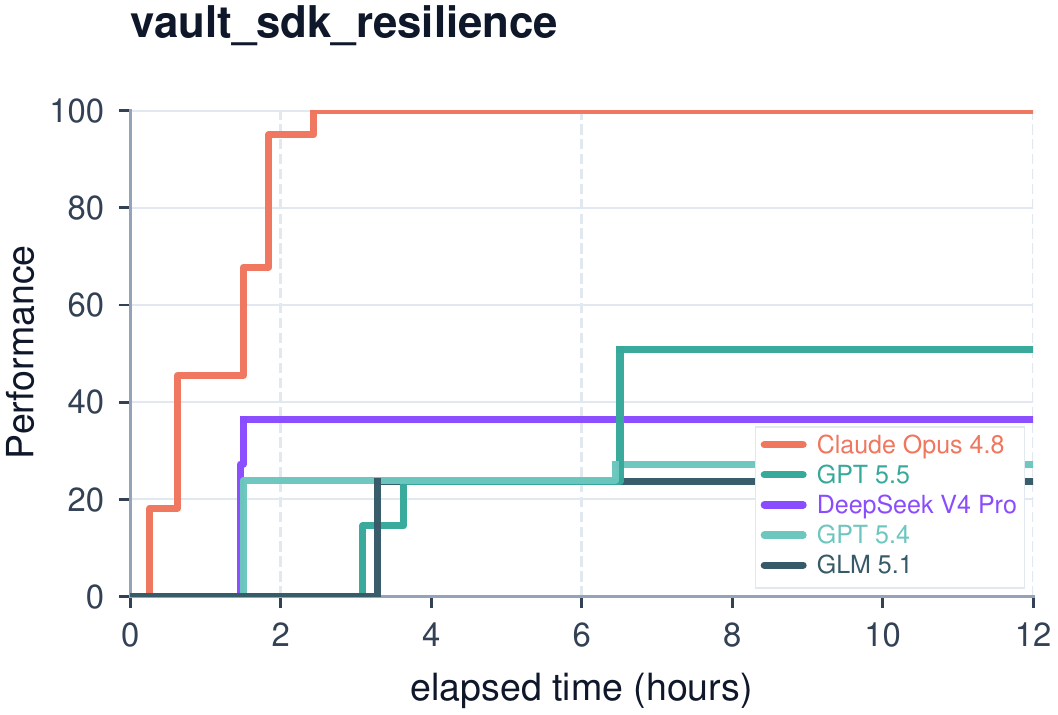}
\end{subfigure}
\hfill
\begin{subfigure}[b]{0.48\linewidth}
\includegraphics[width=\linewidth]{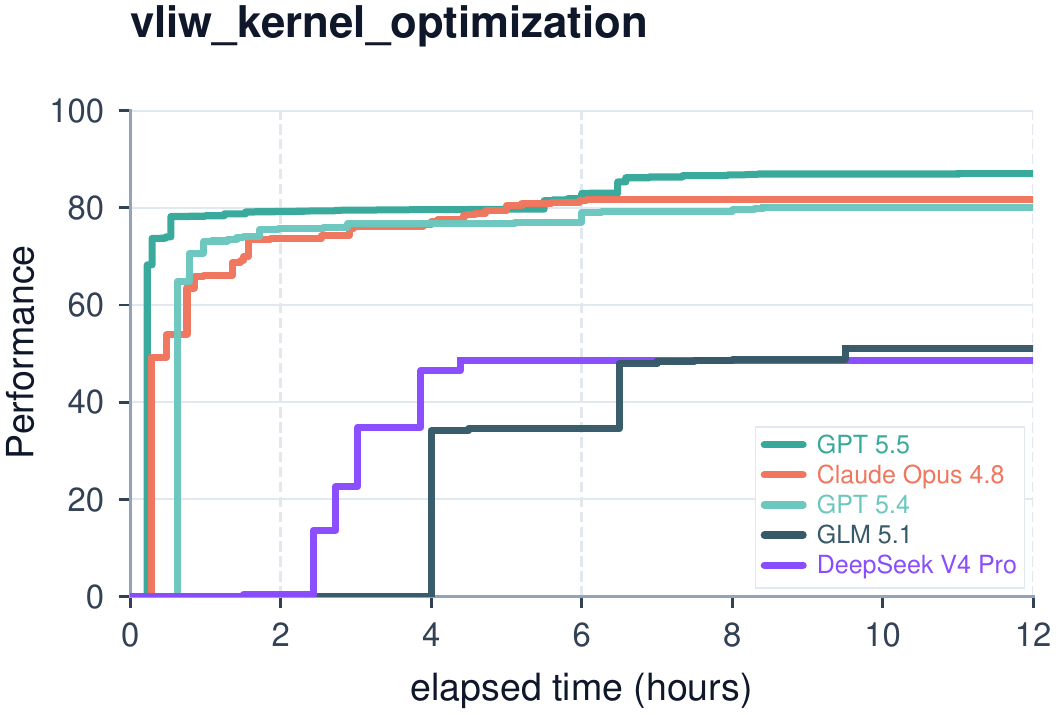}
\end{subfigure}
\vspace{0.5em}
\begin{subfigure}[b]{0.48\linewidth}
\includegraphics[width=\linewidth]{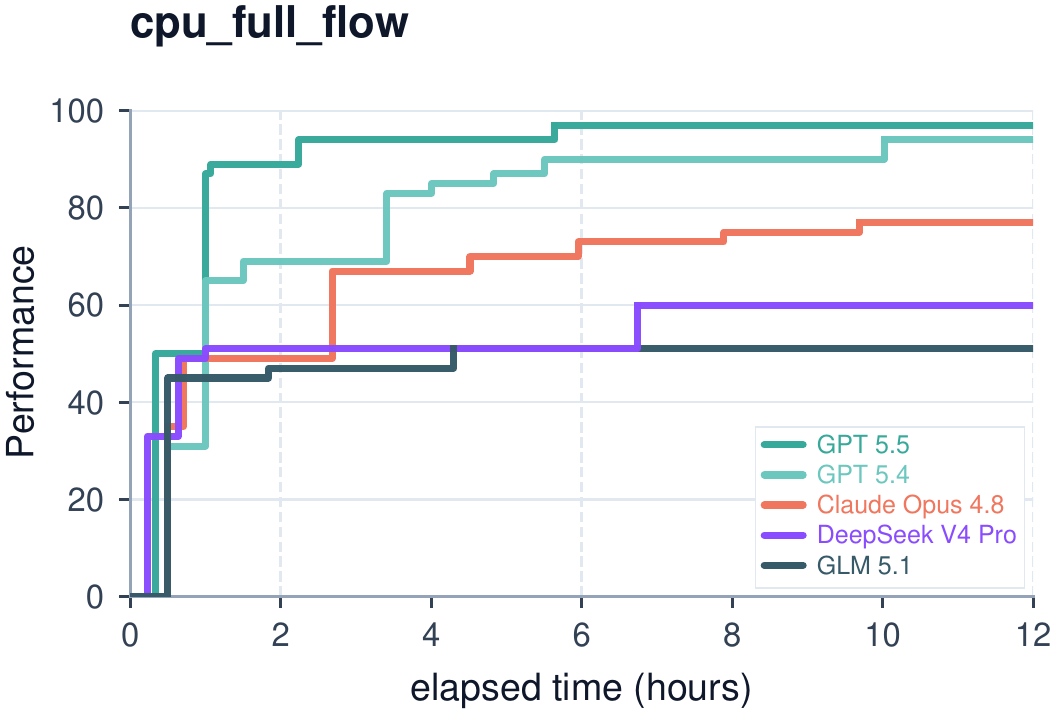}
\end{subfigure}
\hfill
\begin{subfigure}[b]{0.48\linewidth}
\includegraphics[width=\linewidth]{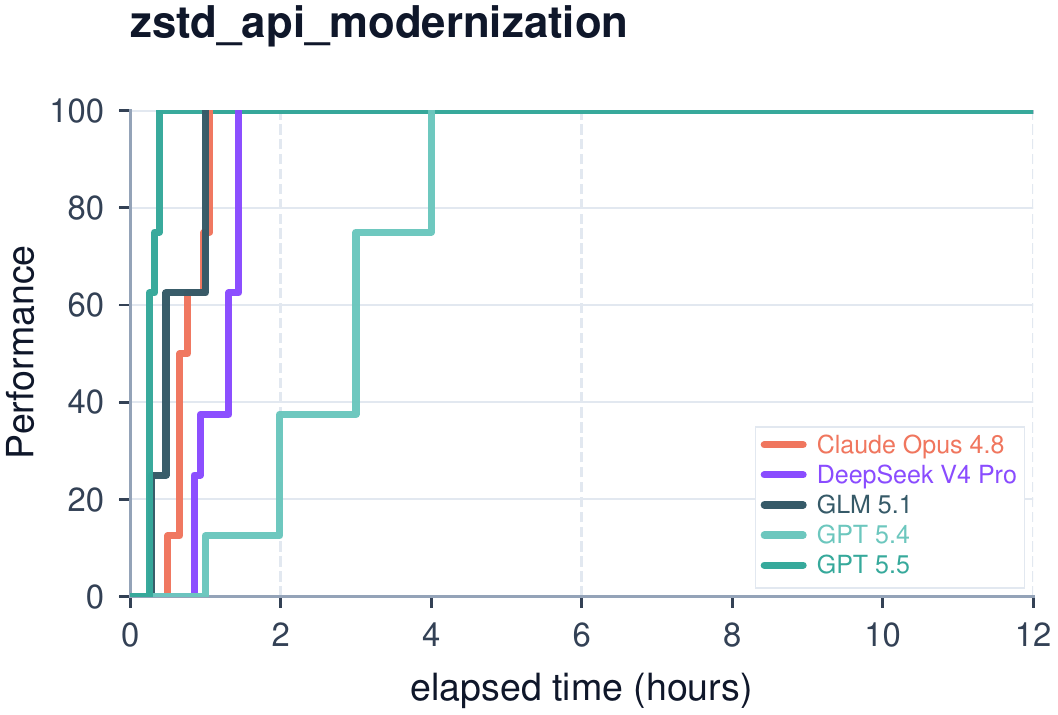}
\end{subfigure}
\caption{Per-task learning curves: Systems \& Software Engineering cont. (11/21).}
\label{fig:curves-all-11}
\end{figure}

\begin{figure}[p]
\centering
\begin{subfigure}[b]{0.48\linewidth}
\includegraphics[width=\linewidth]{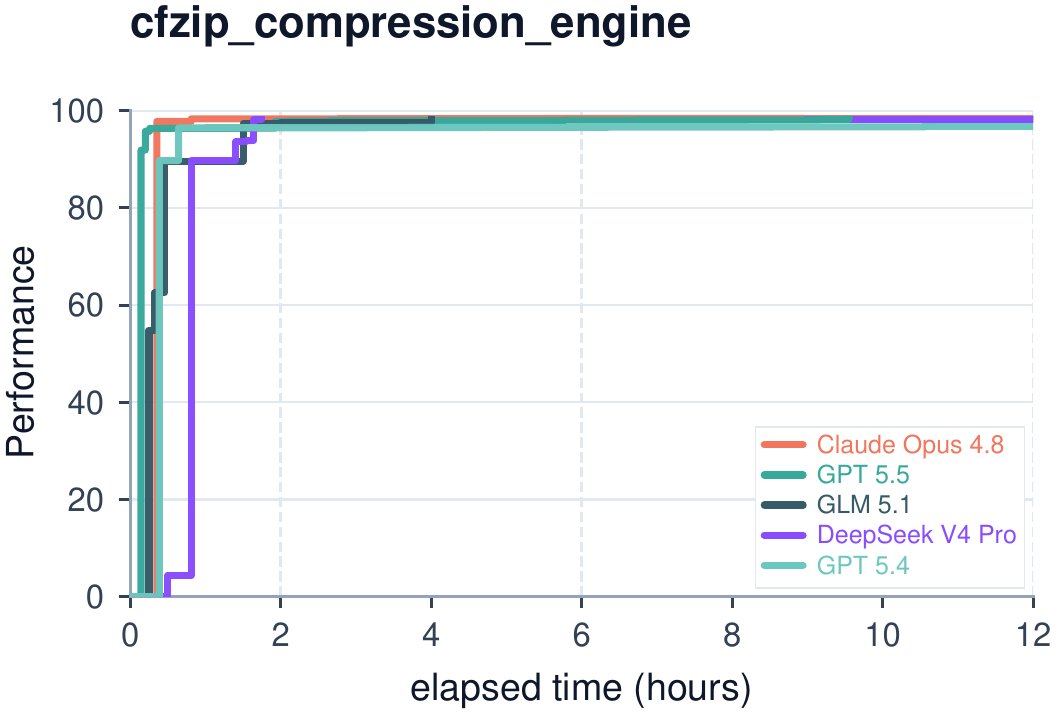}
\end{subfigure}
\hfill
\begin{subfigure}[b]{0.48\linewidth}
\includegraphics[width=\linewidth]{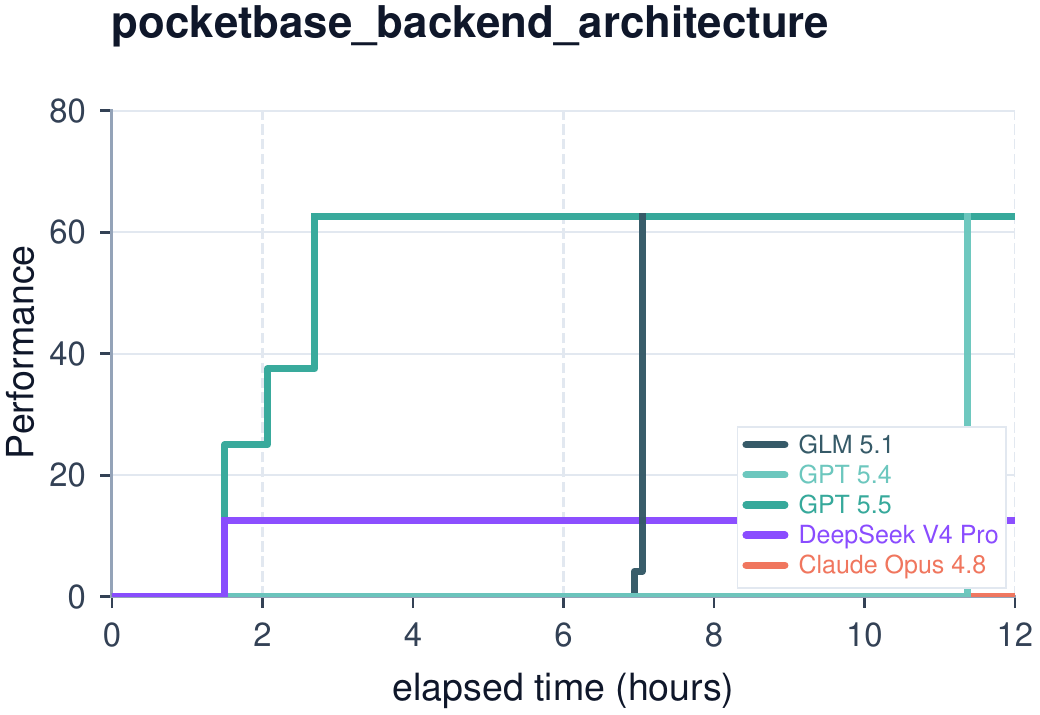}
\end{subfigure}
\vspace{0.5em}
\begin{subfigure}[b]{0.48\linewidth}
\includegraphics[width=\linewidth]{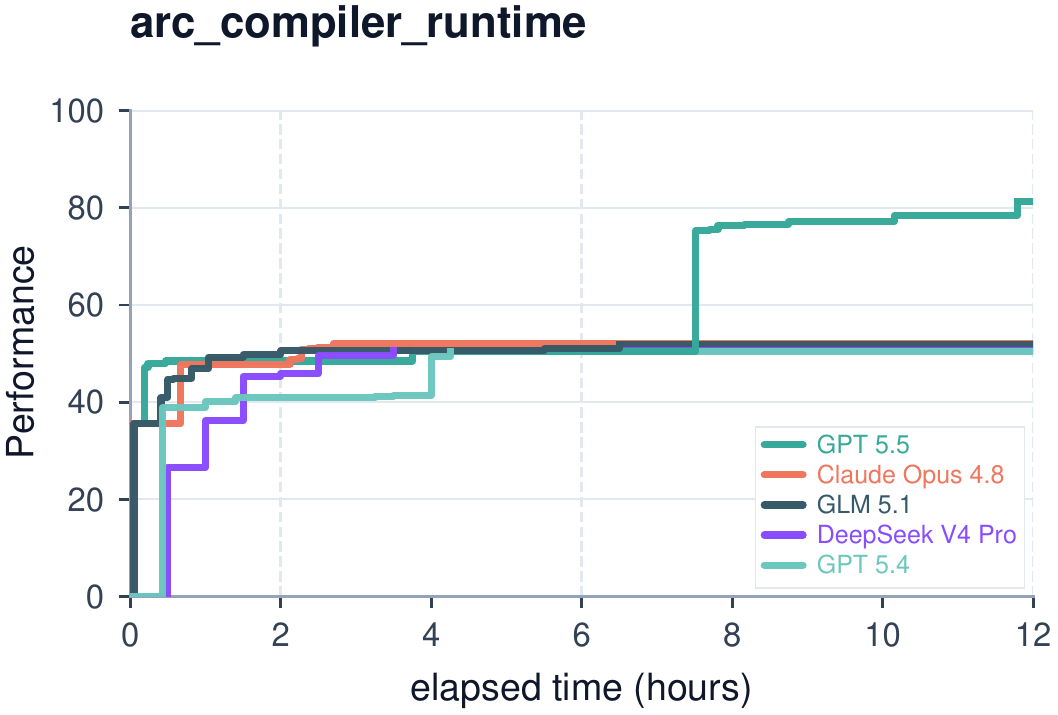}
\end{subfigure}
\caption{Per-task learning curves: Systems \& Software Engineering cont. (12/21).}
\label{fig:curves-all-12}
\end{figure}

\begin{figure}[p]
\centering
\begin{subfigure}[b]{0.48\linewidth}
\includegraphics[width=\linewidth]{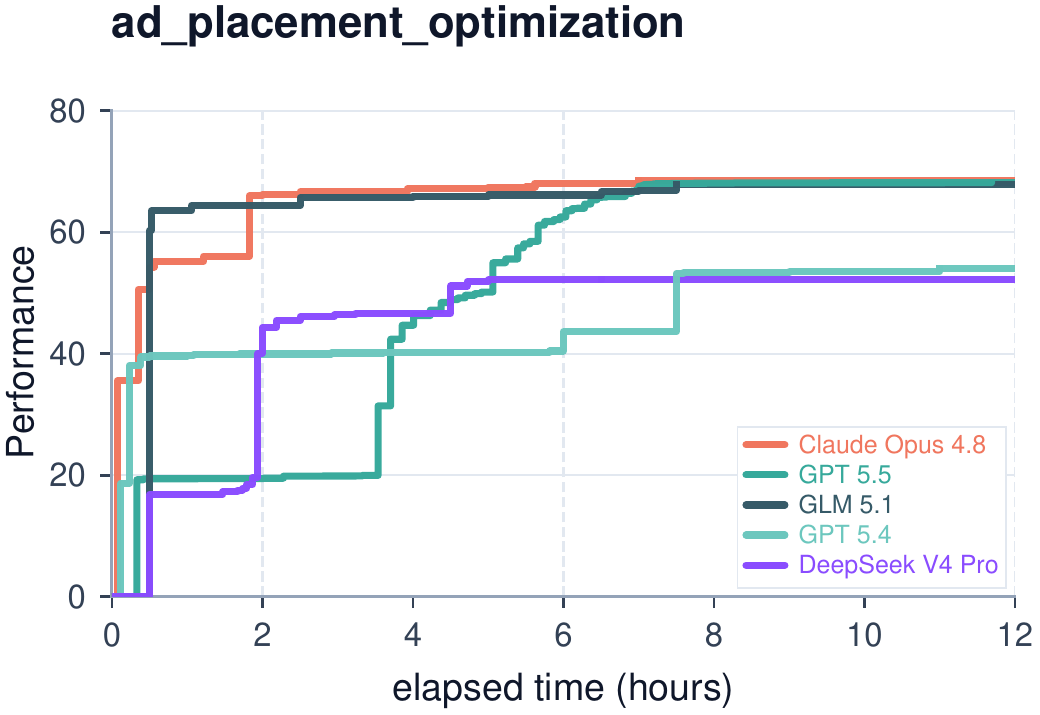}
\end{subfigure}
\hfill
\begin{subfigure}[b]{0.48\linewidth}
\includegraphics[width=\linewidth]{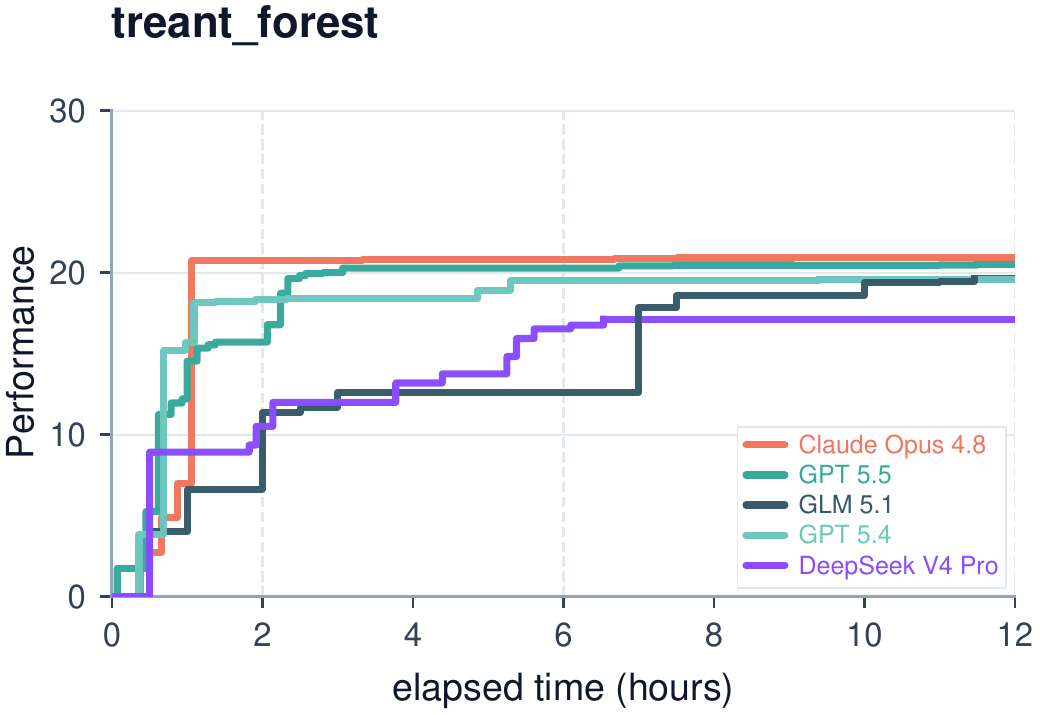}
\end{subfigure}
\vspace{0.5em}
\begin{subfigure}[b]{0.48\linewidth}
\includegraphics[width=\linewidth]{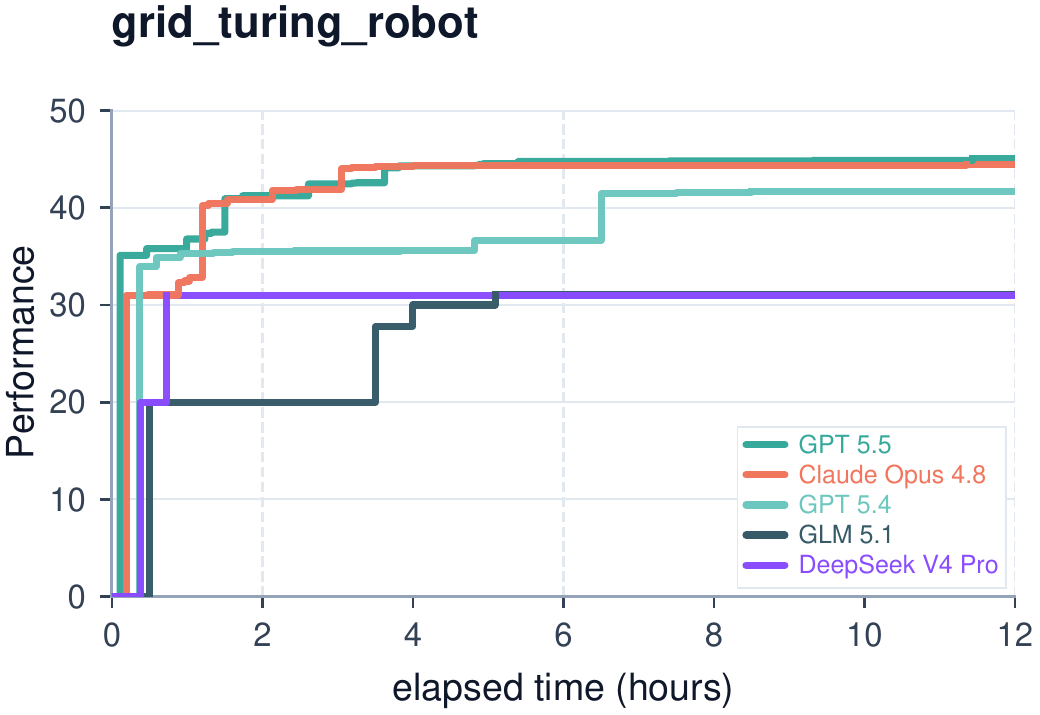}
\end{subfigure}
\hfill
\begin{subfigure}[b]{0.48\linewidth}
\includegraphics[width=\linewidth]{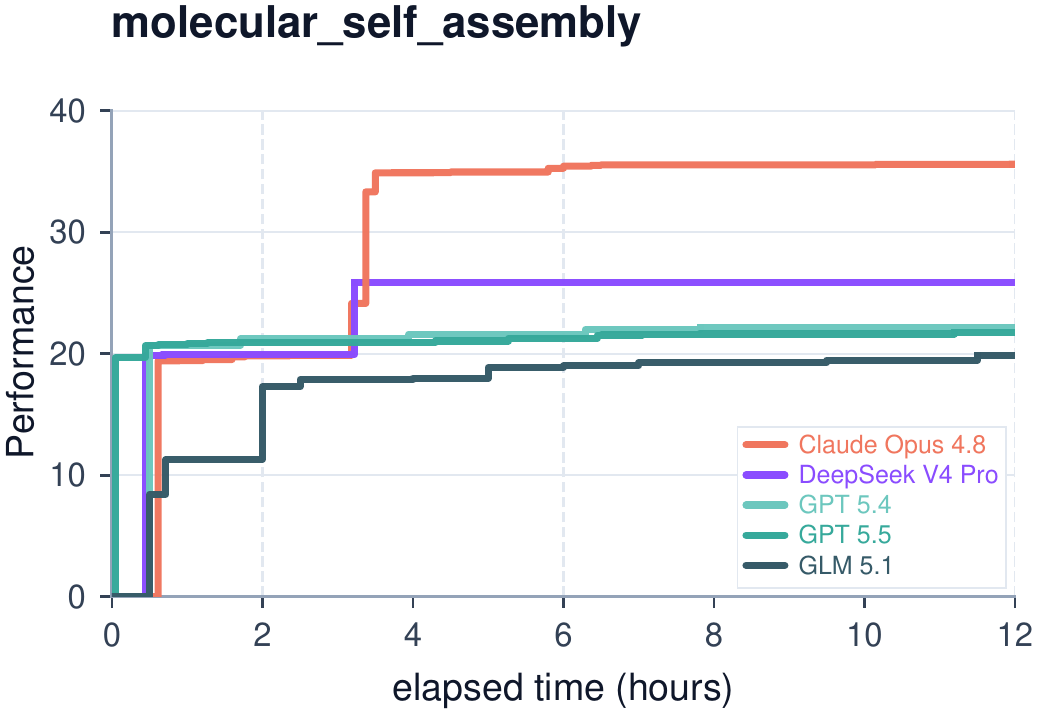}
\end{subfigure}
\vspace{0.5em}
\begin{subfigure}[b]{0.48\linewidth}
\includegraphics[width=\linewidth]{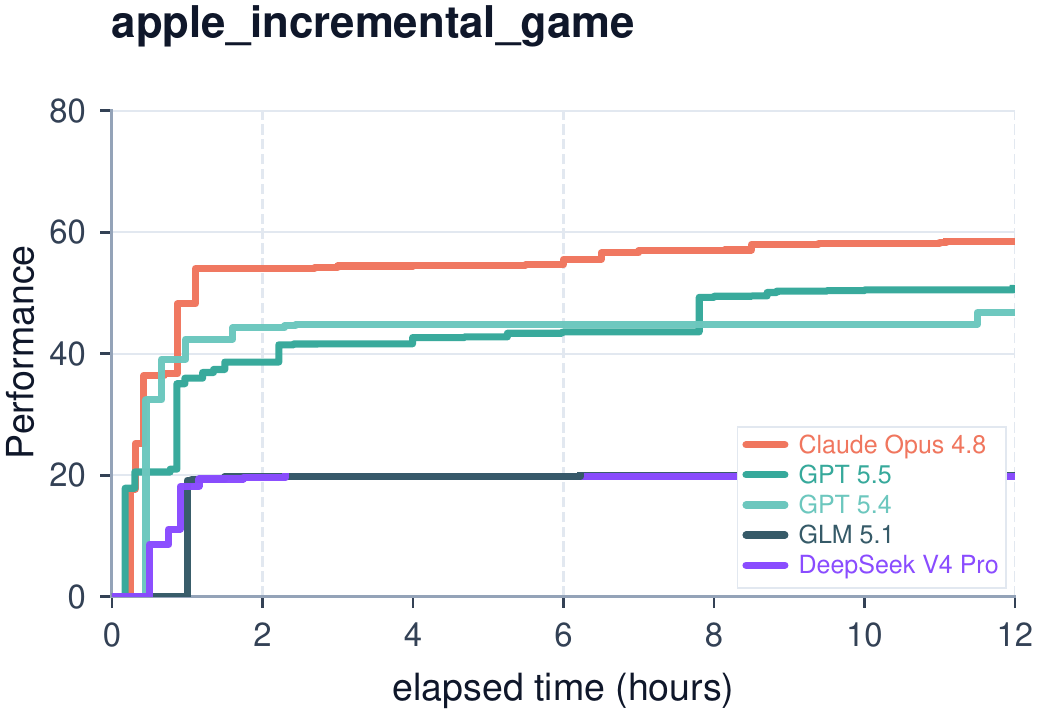}
\end{subfigure}
\hfill
\begin{subfigure}[b]{0.48\linewidth}
\includegraphics[width=\linewidth]{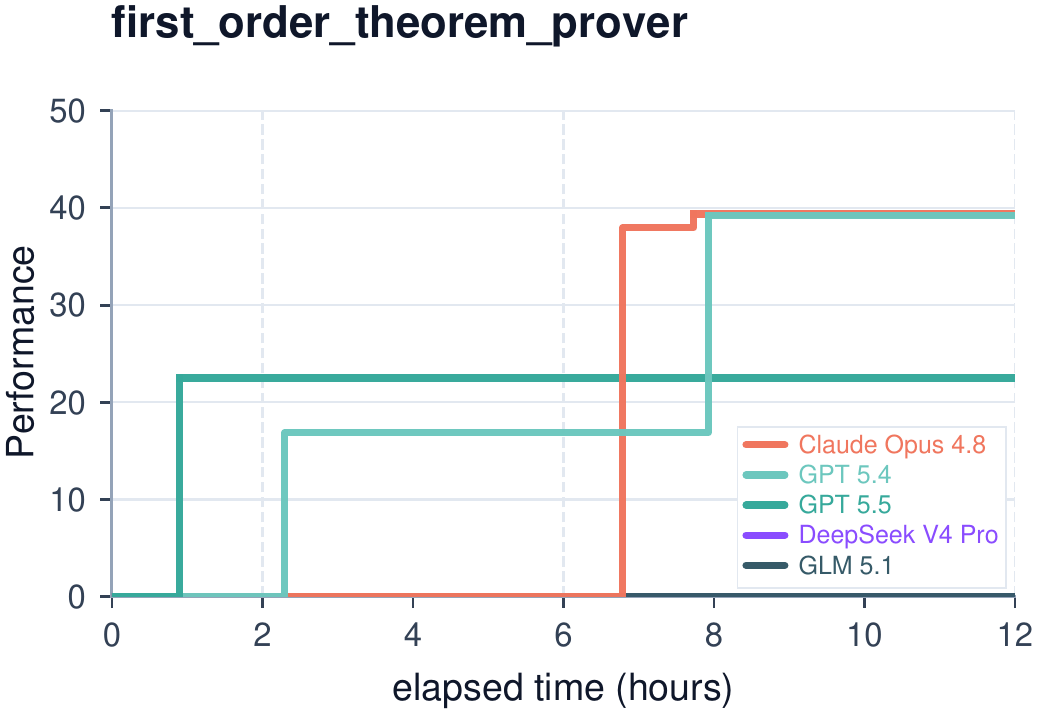}
\end{subfigure}
\caption{Per-task learning curves: Combinatorial Optimization \& Planning (13/21).}
\label{fig:curves-all-13}
\end{figure}

\begin{figure}[p]
\centering
\begin{subfigure}[b]{0.48\linewidth}
\includegraphics[width=\linewidth]{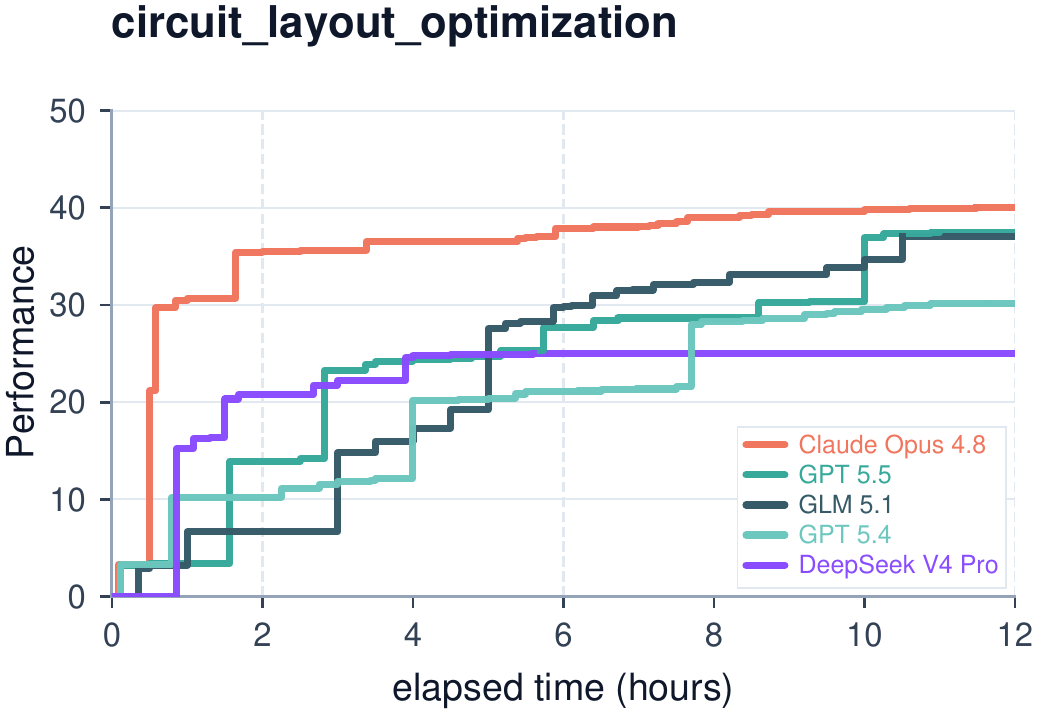}
\end{subfigure}
\hfill
\begin{subfigure}[b]{0.48\linewidth}
\includegraphics[width=\linewidth]{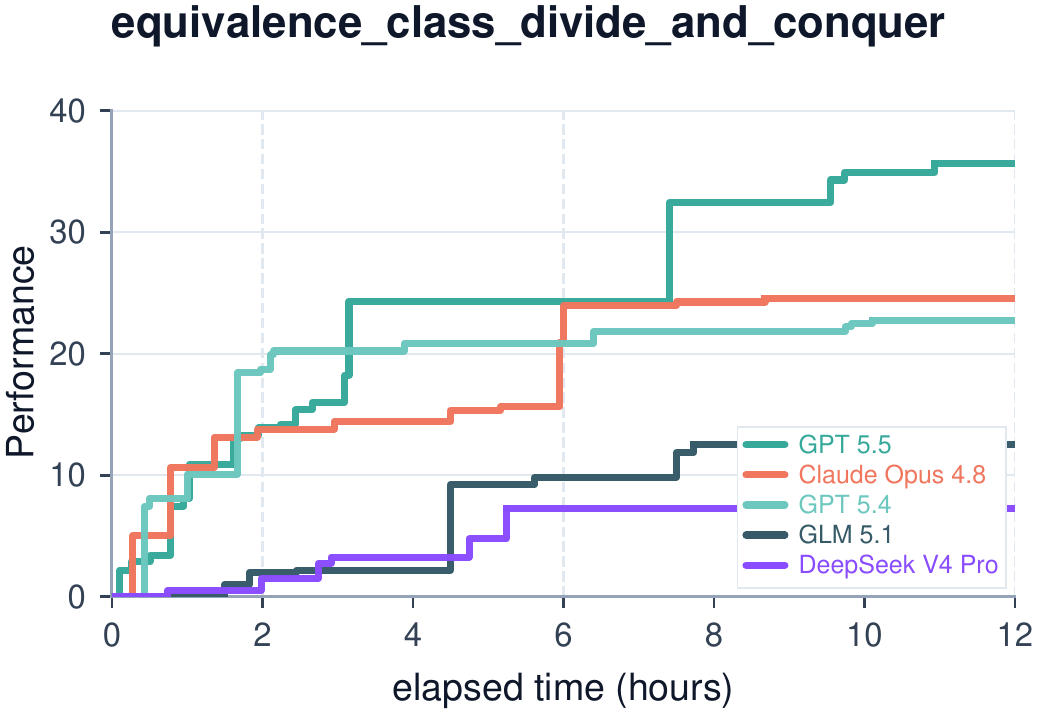}
\end{subfigure}
\vspace{0.5em}
\begin{subfigure}[b]{0.48\linewidth}
\includegraphics[width=\linewidth]{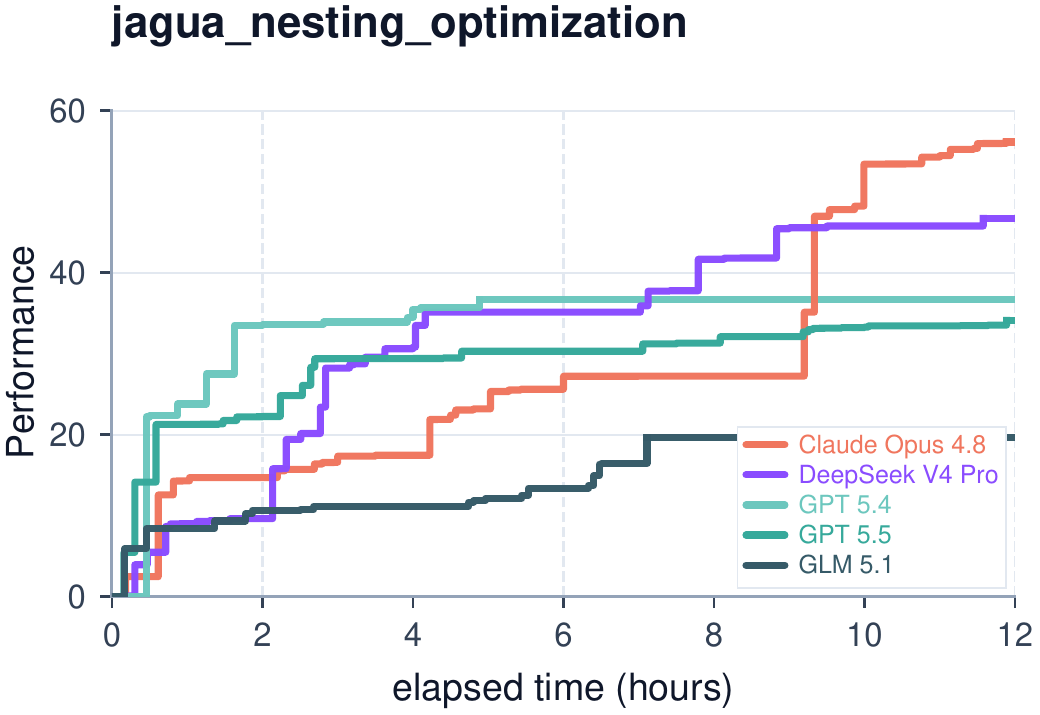}
\end{subfigure}
\hfill
\begin{subfigure}[b]{0.48\linewidth}
\includegraphics[width=\linewidth]{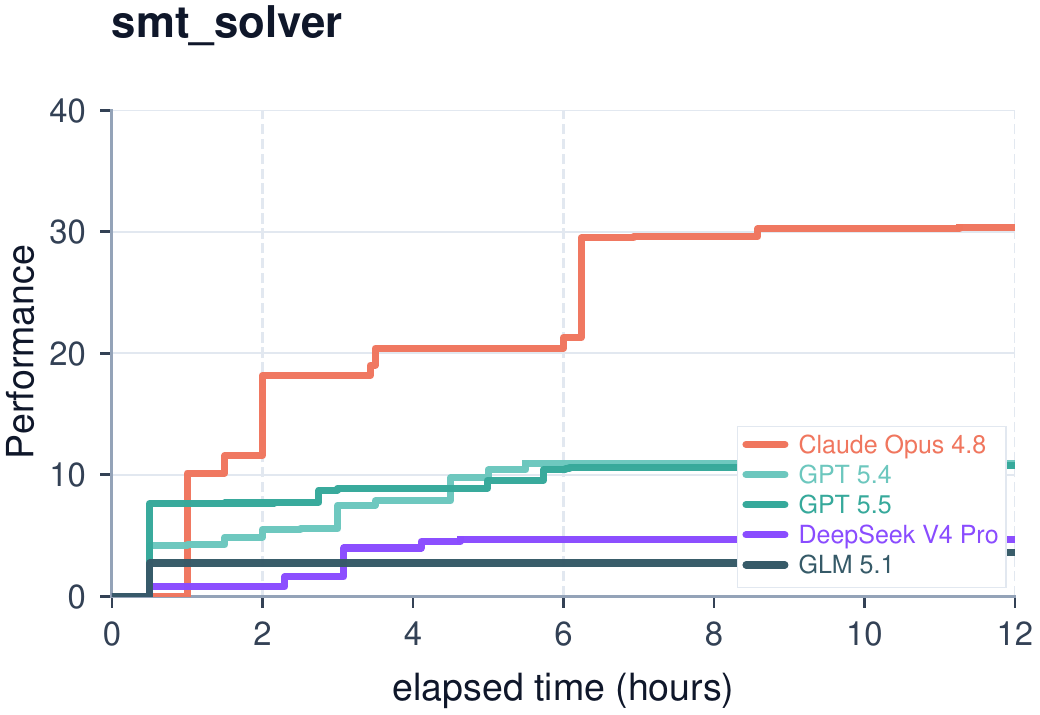}
\end{subfigure}
\vspace{0.5em}
\begin{subfigure}[b]{0.48\linewidth}
\includegraphics[width=\linewidth]{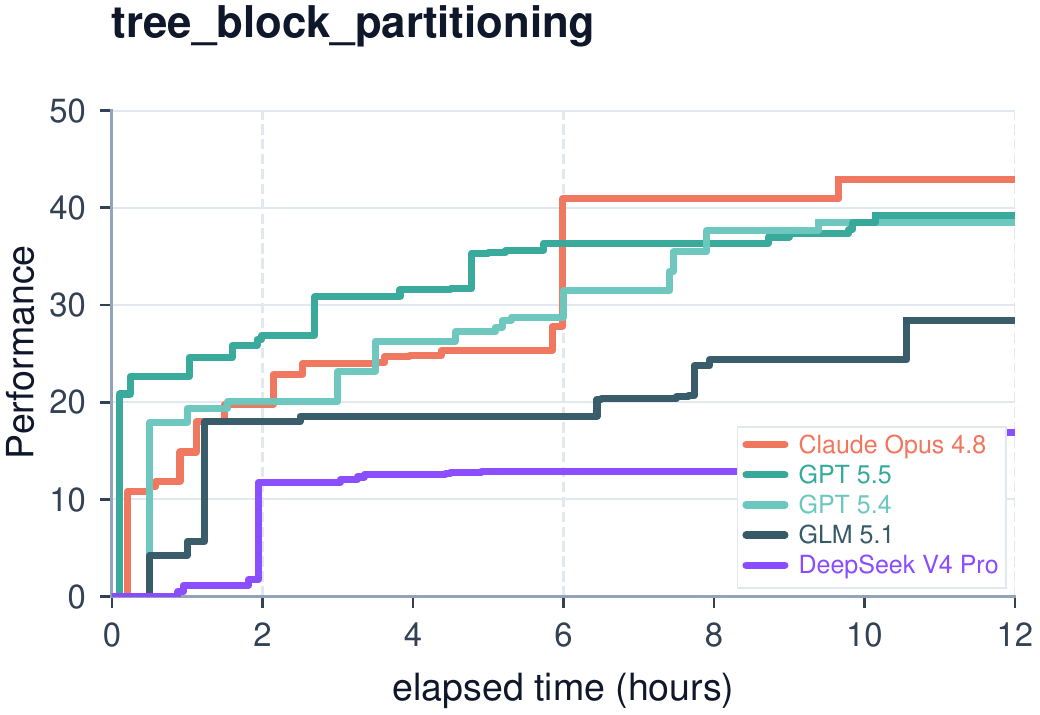}
\end{subfigure}
\hfill
\begin{subfigure}[b]{0.48\linewidth}
\includegraphics[width=\linewidth]{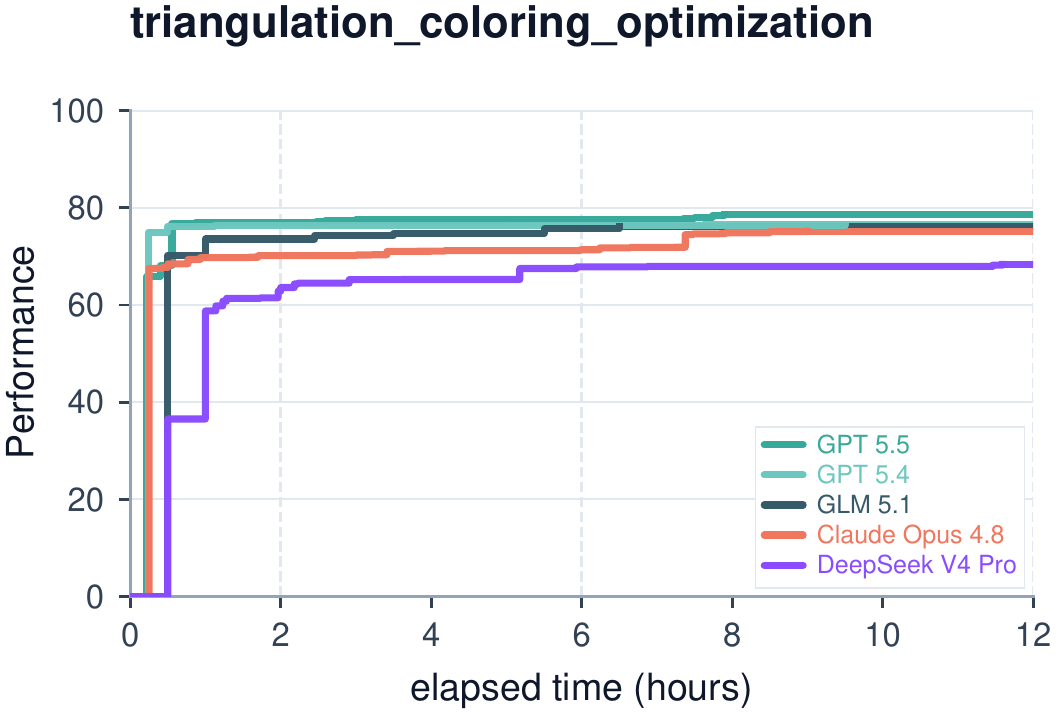}
\end{subfigure}
\caption{Per-task learning curves: Combinatorial Optimization \& Planning cont. (14/21).}
\label{fig:curves-all-14}
\end{figure}

\begin{figure}[p]
\centering
\begin{subfigure}[b]{0.48\linewidth}
\includegraphics[width=\linewidth]{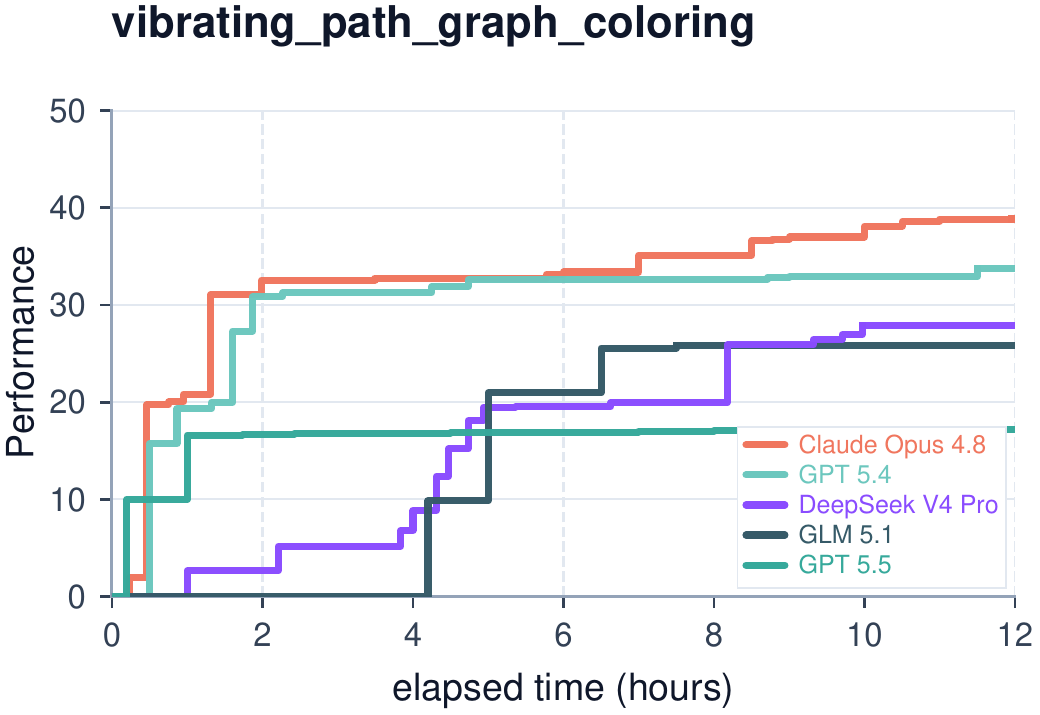}
\end{subfigure}
\hfill
\begin{subfigure}[b]{0.48\linewidth}
\includegraphics[width=\linewidth]{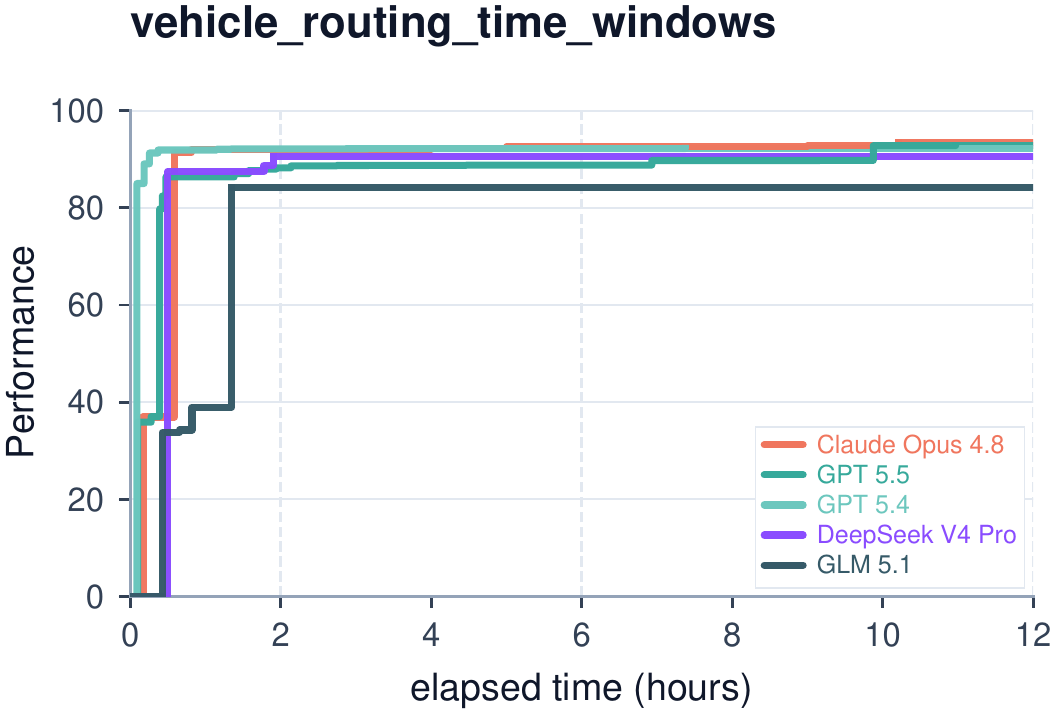}
\end{subfigure}
\vspace{0.5em}
\begin{subfigure}[b]{0.48\linewidth}
\includegraphics[width=\linewidth]{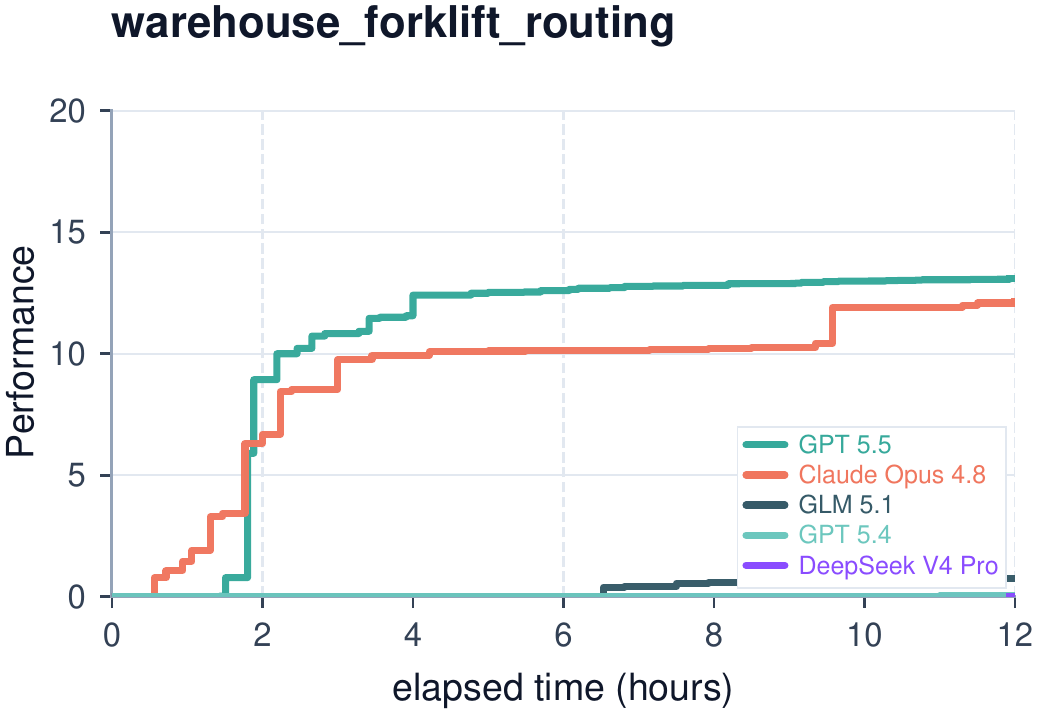}
\end{subfigure}
\hfill
\begin{subfigure}[b]{0.48\linewidth}
\includegraphics[width=\linewidth]{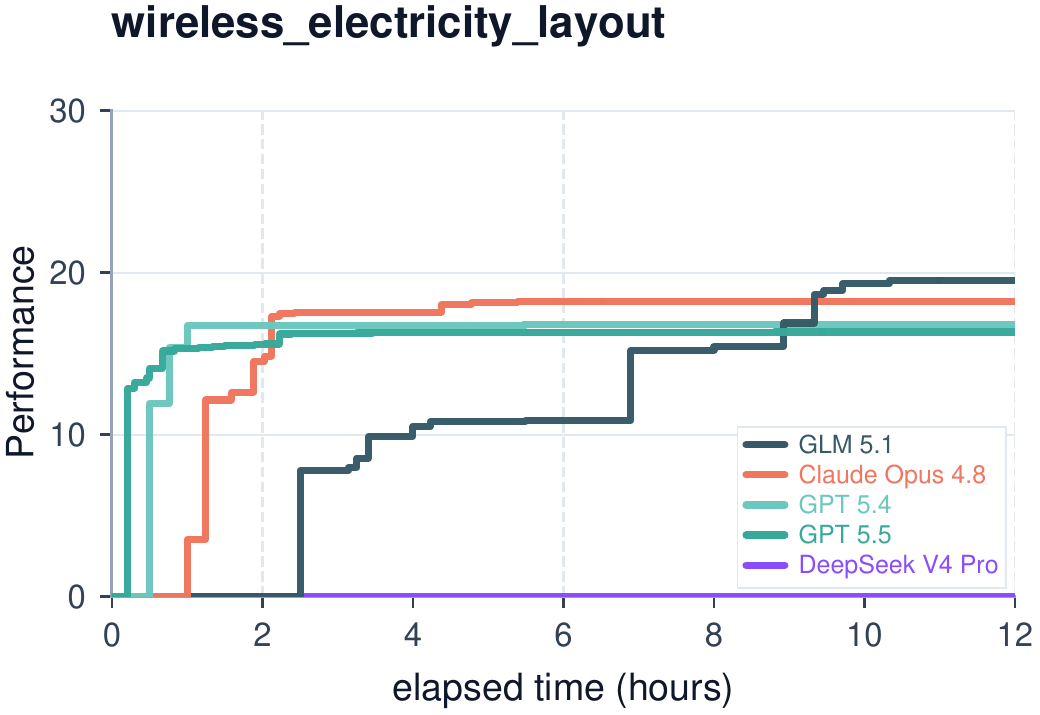}
\end{subfigure}
\caption{Per-task learning curves: Combinatorial Optimization \& Planning cont. (15/21).}
\label{fig:curves-all-15}
\end{figure}

\begin{figure}[p]
\centering
\begin{subfigure}[b]{0.48\linewidth}
\includegraphics[width=\linewidth]{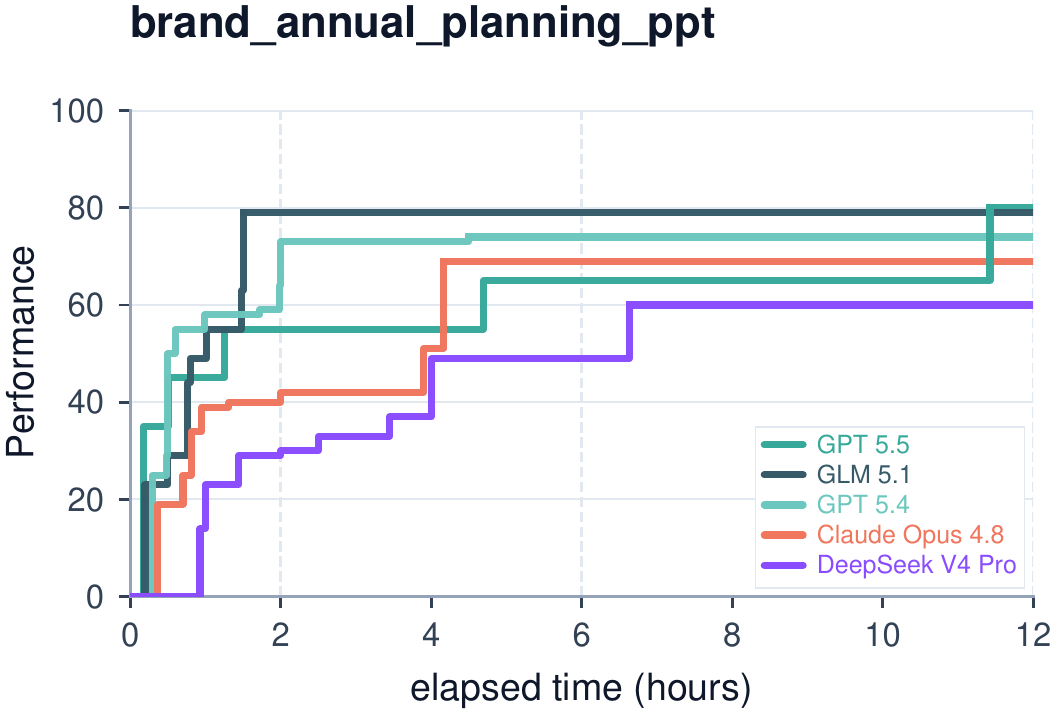}
\end{subfigure}
\hfill
\begin{subfigure}[b]{0.48\linewidth}
\includegraphics[width=\linewidth]{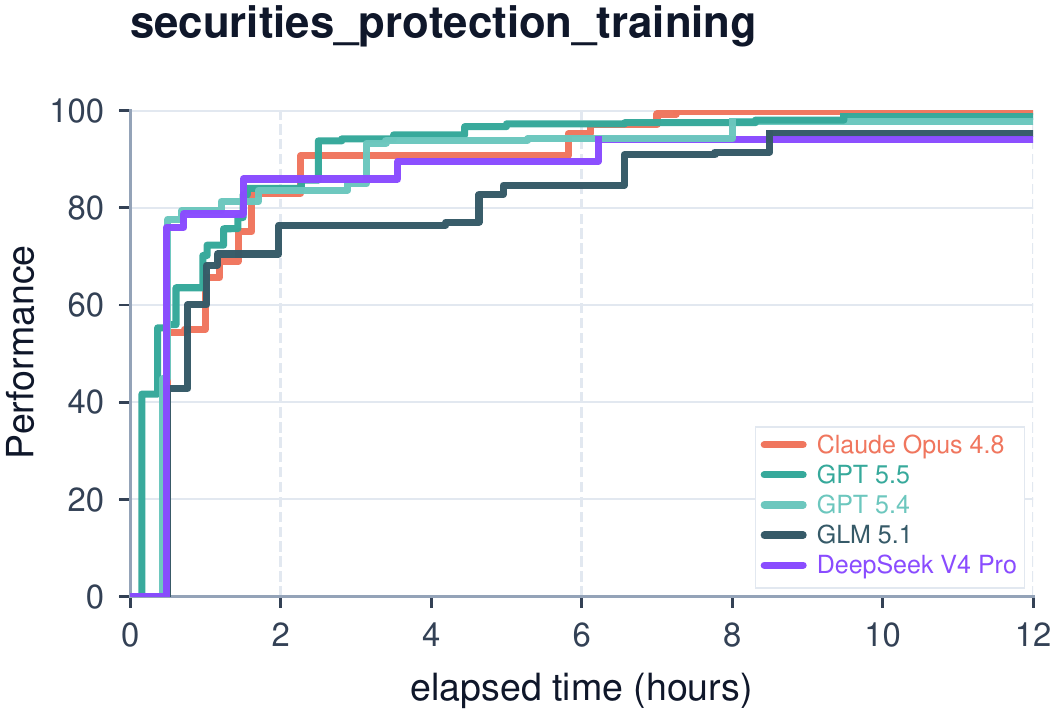}
\end{subfigure}
\vspace{0.5em}
\begin{subfigure}[b]{0.48\linewidth}
\includegraphics[width=\linewidth]{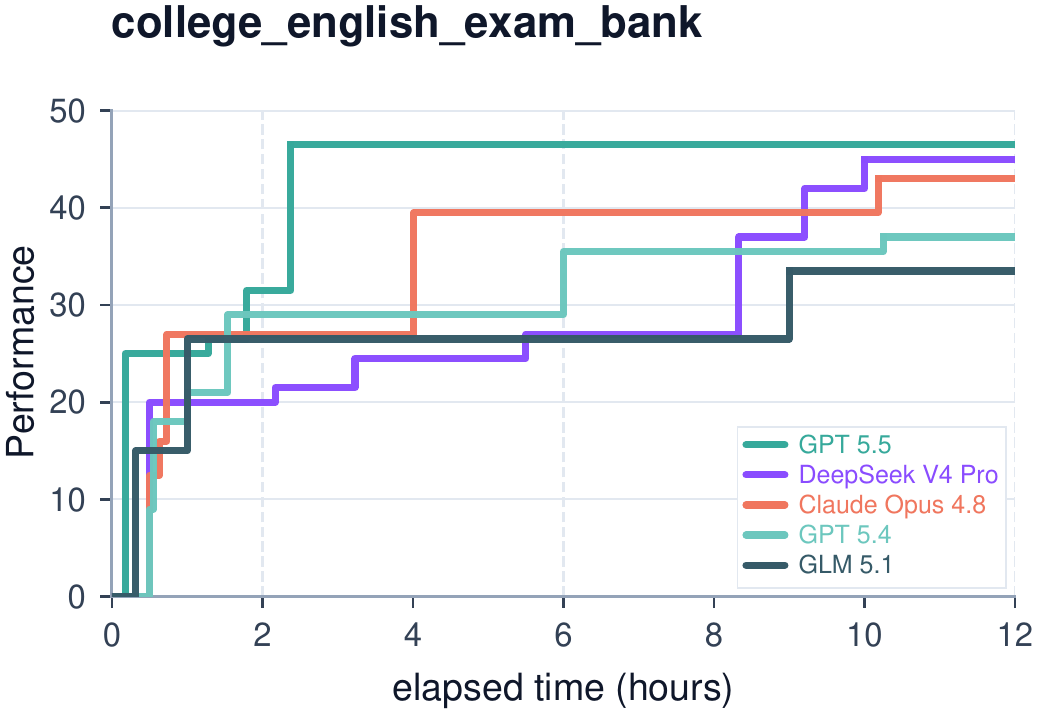}
\end{subfigure}
\hfill
\begin{subfigure}[b]{0.48\linewidth}
\includegraphics[width=\linewidth]{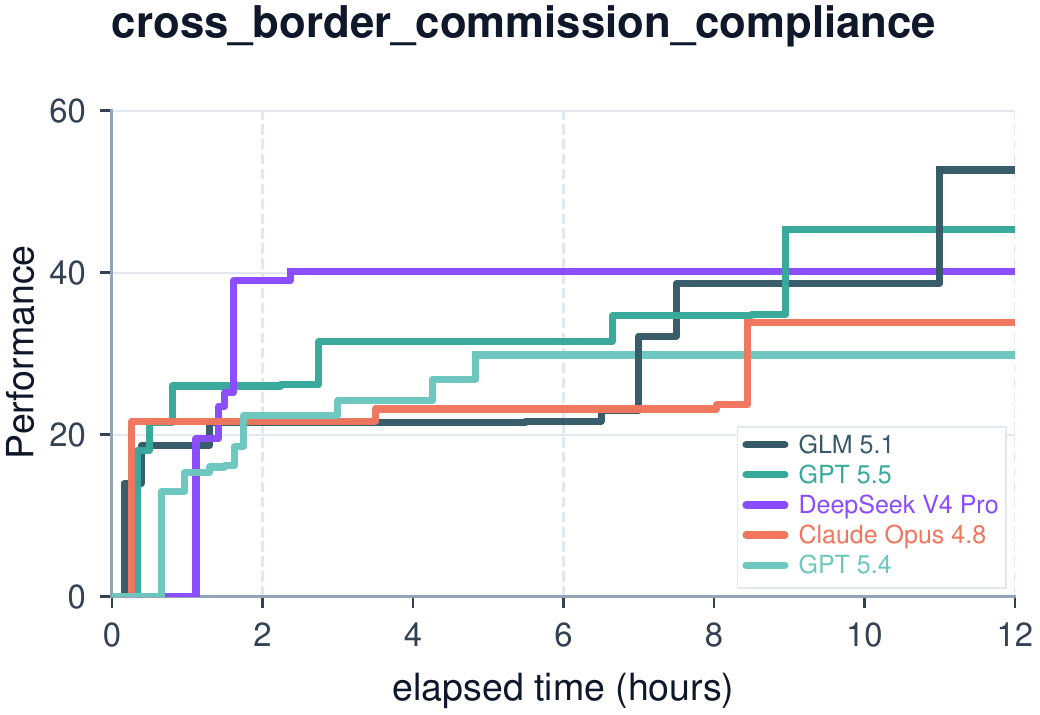}
\end{subfigure}
\vspace{0.5em}
\begin{subfigure}[b]{0.48\linewidth}
\includegraphics[width=\linewidth]{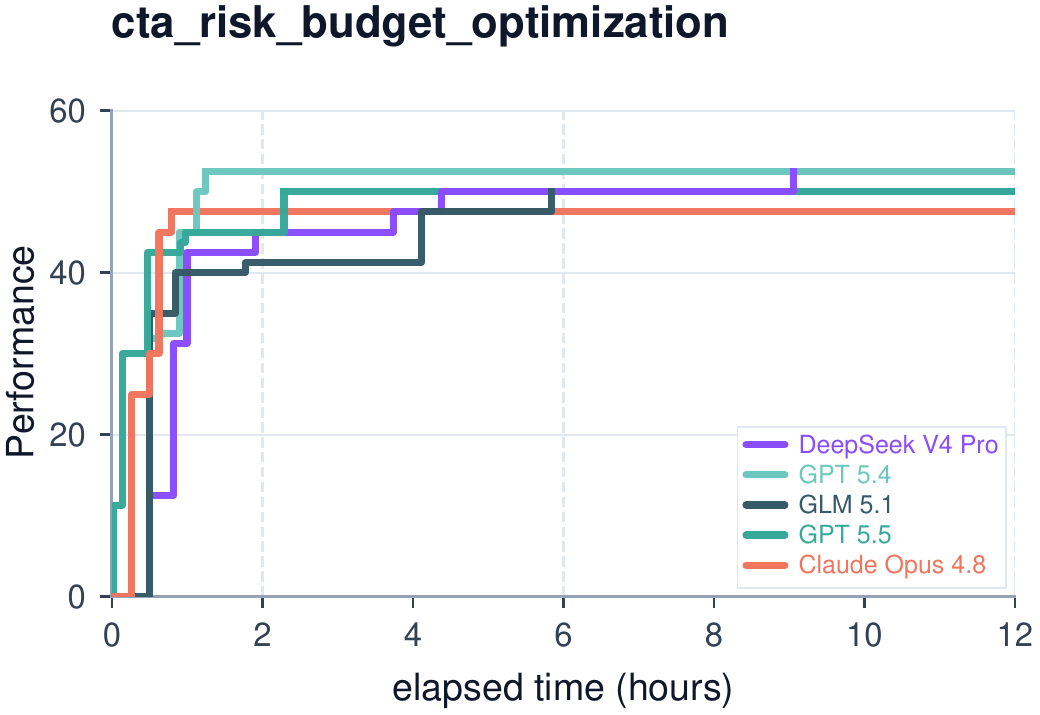}
\end{subfigure}
\hfill
\begin{subfigure}[b]{0.48\linewidth}
\includegraphics[width=\linewidth]{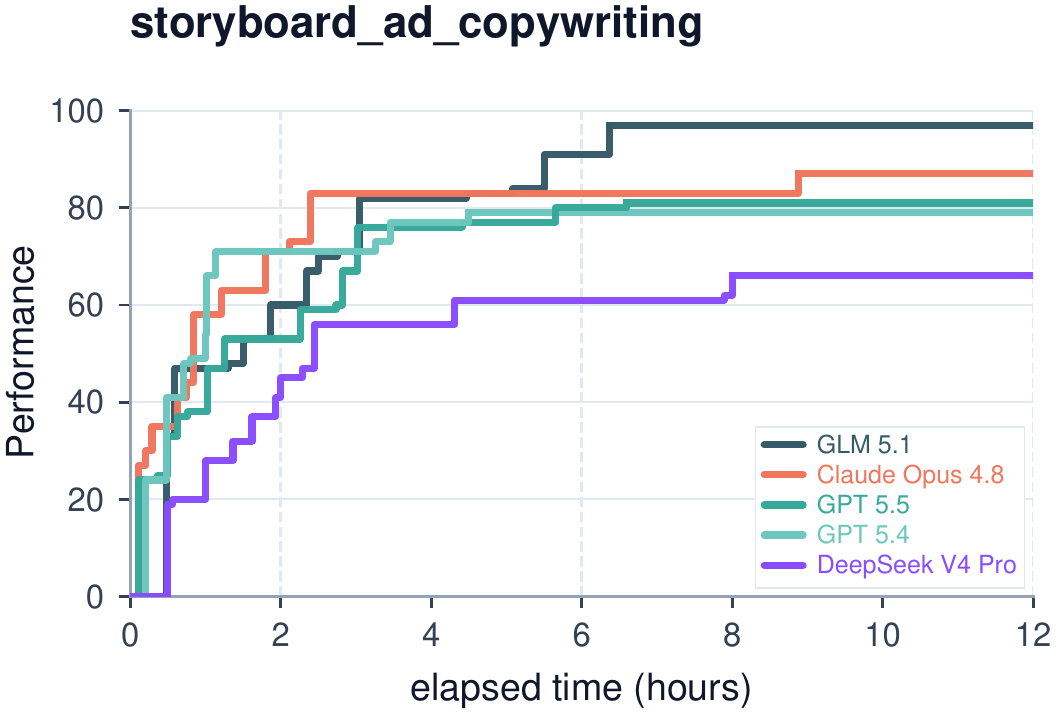}
\end{subfigure}
\caption{Per-task learning curves: Professional Knowledge Work (16/21).}
\label{fig:curves-all-16}
\end{figure}

\begin{figure}[p]
\centering
\begin{subfigure}[b]{0.48\linewidth}
\includegraphics[width=\linewidth]{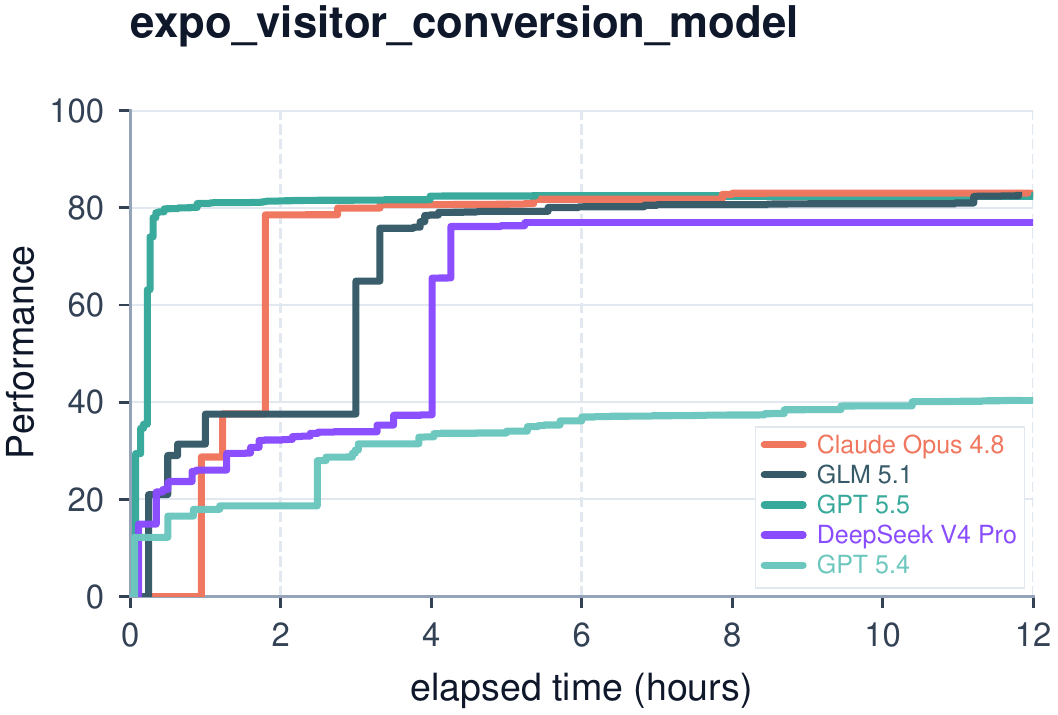}
\end{subfigure}
\hfill
\begin{subfigure}[b]{0.48\linewidth}
\includegraphics[width=\linewidth]{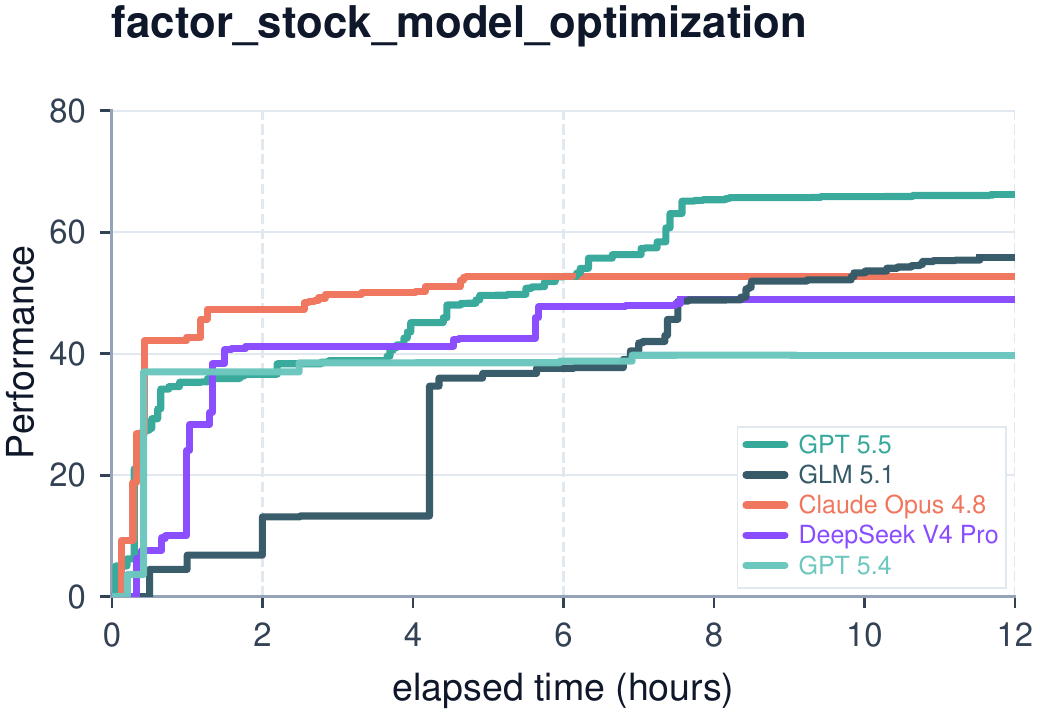}
\end{subfigure}
\vspace{0.5em}
\begin{subfigure}[b]{0.48\linewidth}
\includegraphics[width=\linewidth]{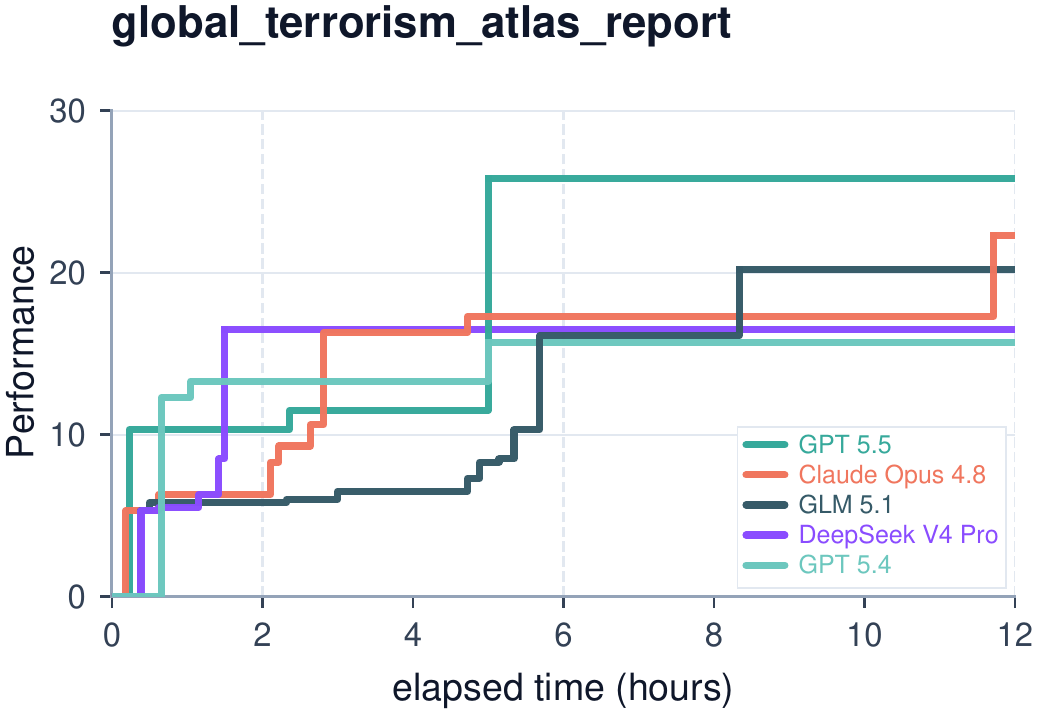}
\end{subfigure}
\hfill
\begin{subfigure}[b]{0.48\linewidth}
\includegraphics[width=\linewidth]{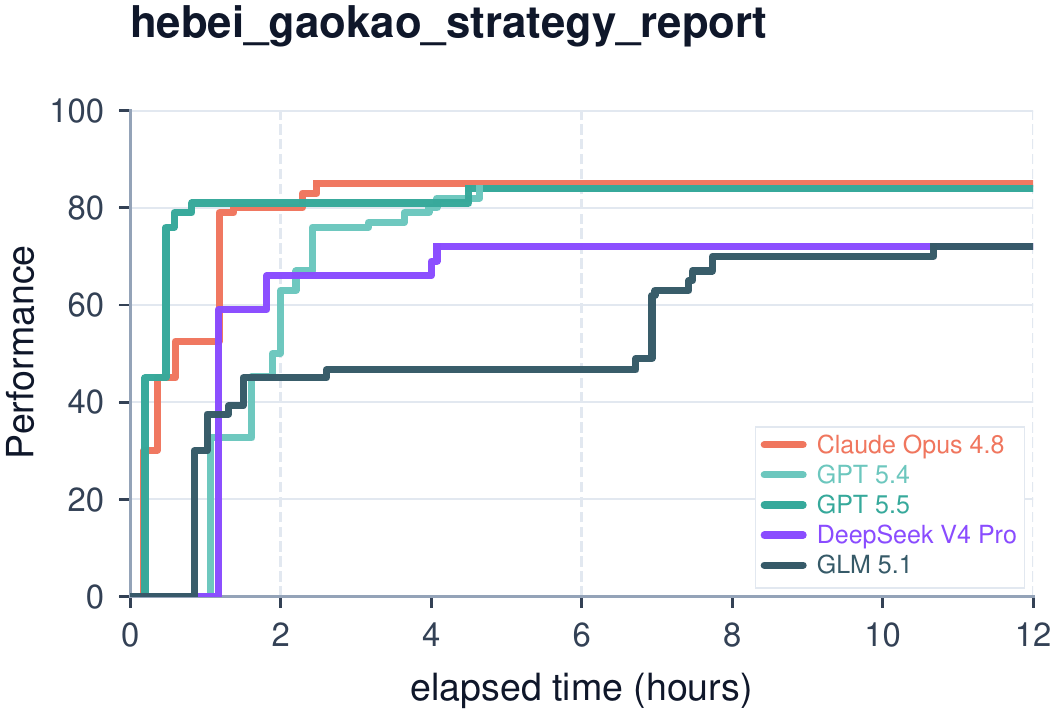}
\end{subfigure}
\vspace{0.5em}
\begin{subfigure}[b]{0.48\linewidth}
\includegraphics[width=\linewidth]{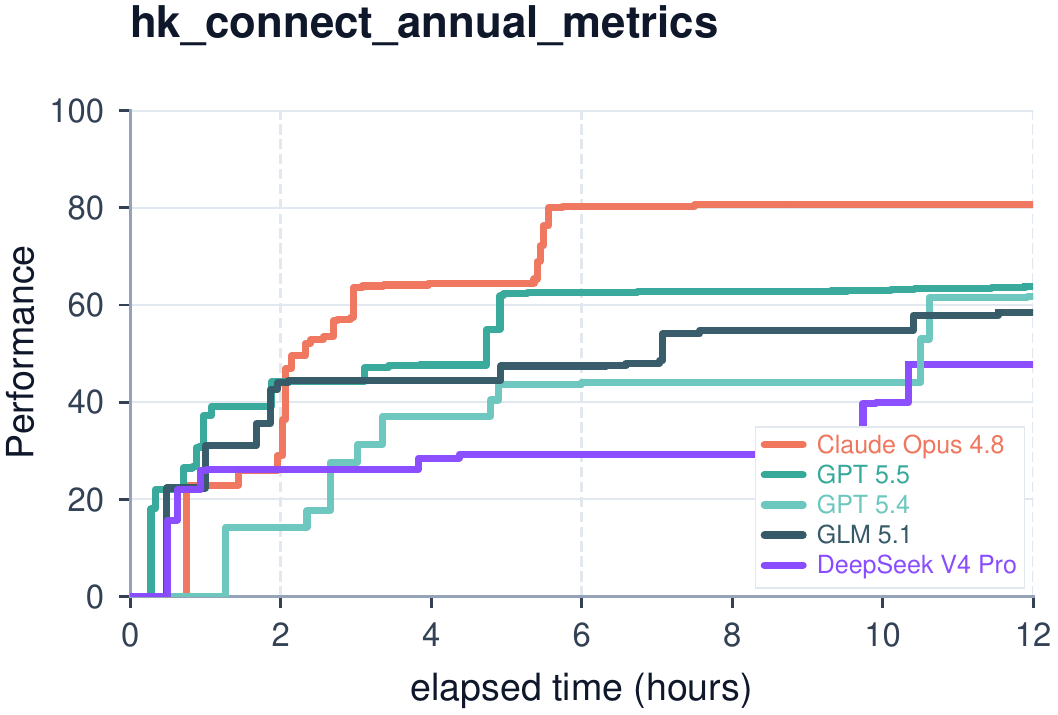}
\end{subfigure}
\hfill
\begin{subfigure}[b]{0.48\linewidth}
\includegraphics[width=\linewidth]{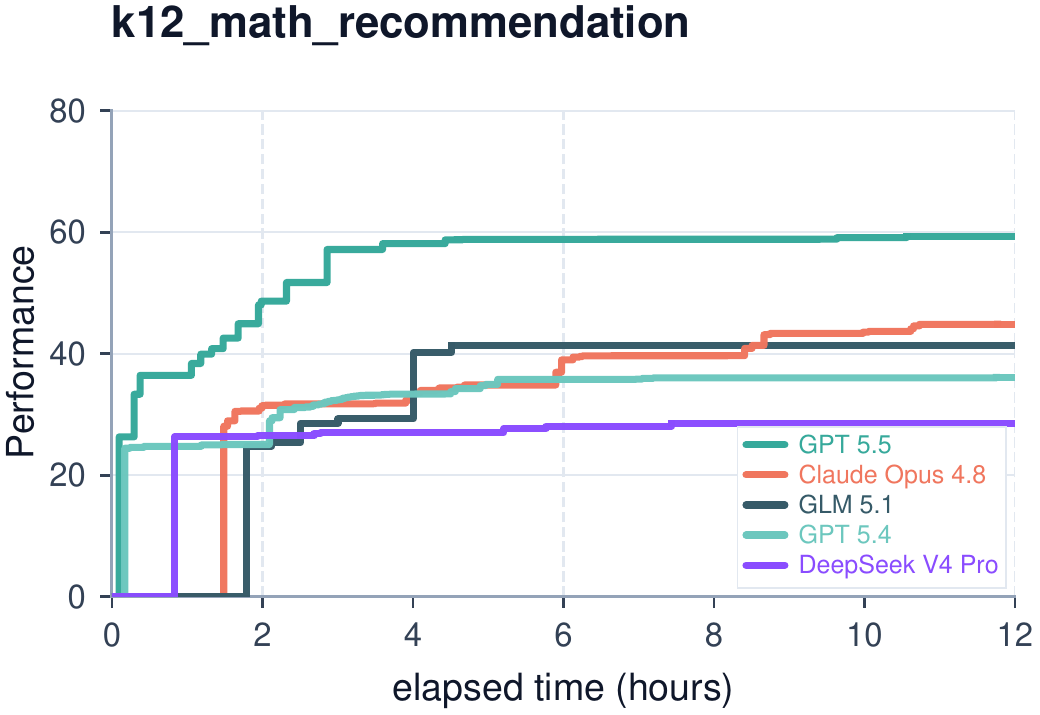}
\end{subfigure}
\caption{Per-task learning curves: Professional Knowledge Work cont. (17/21).}
\label{fig:curves-all-17}
\end{figure}

\begin{figure}[p]
\centering
\begin{subfigure}[b]{0.48\linewidth}
\includegraphics[width=\linewidth]{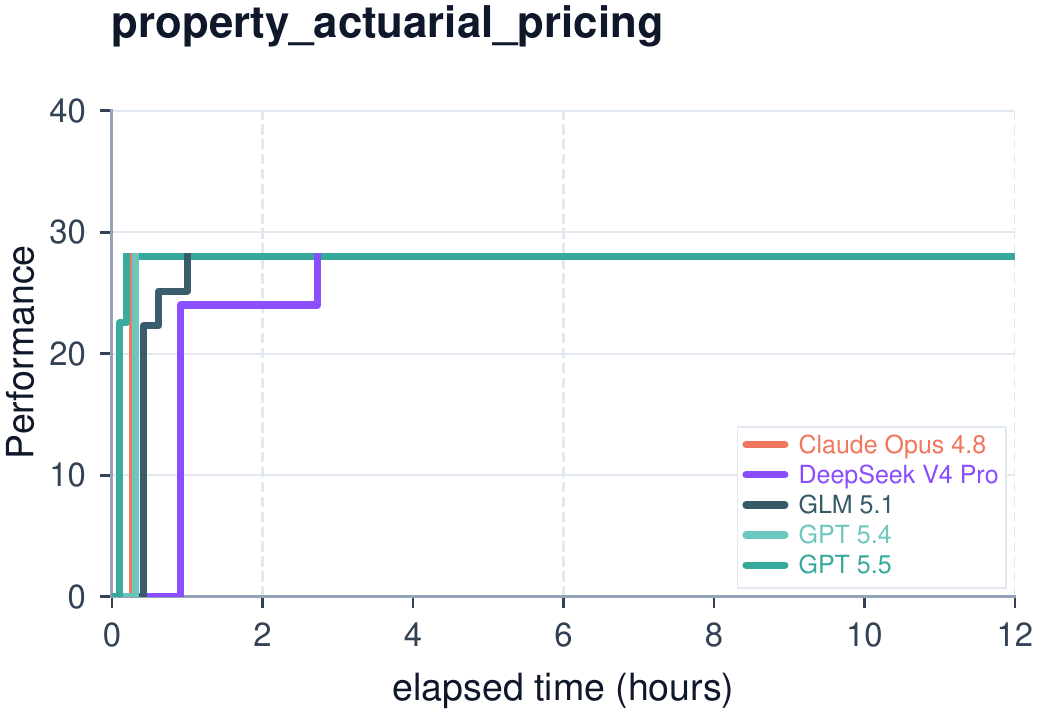}
\end{subfigure}
\hfill
\begin{subfigure}[b]{0.48\linewidth}
\includegraphics[width=\linewidth]{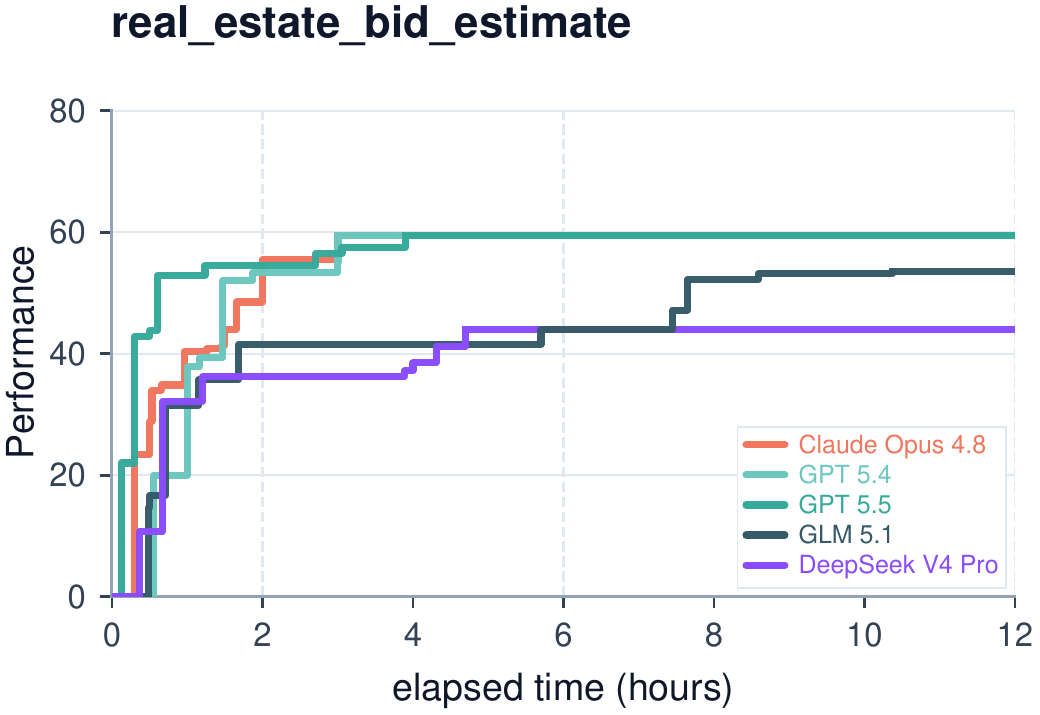}
\end{subfigure}
\vspace{0.5em}
\begin{subfigure}[b]{0.48\linewidth}
\includegraphics[width=\linewidth]{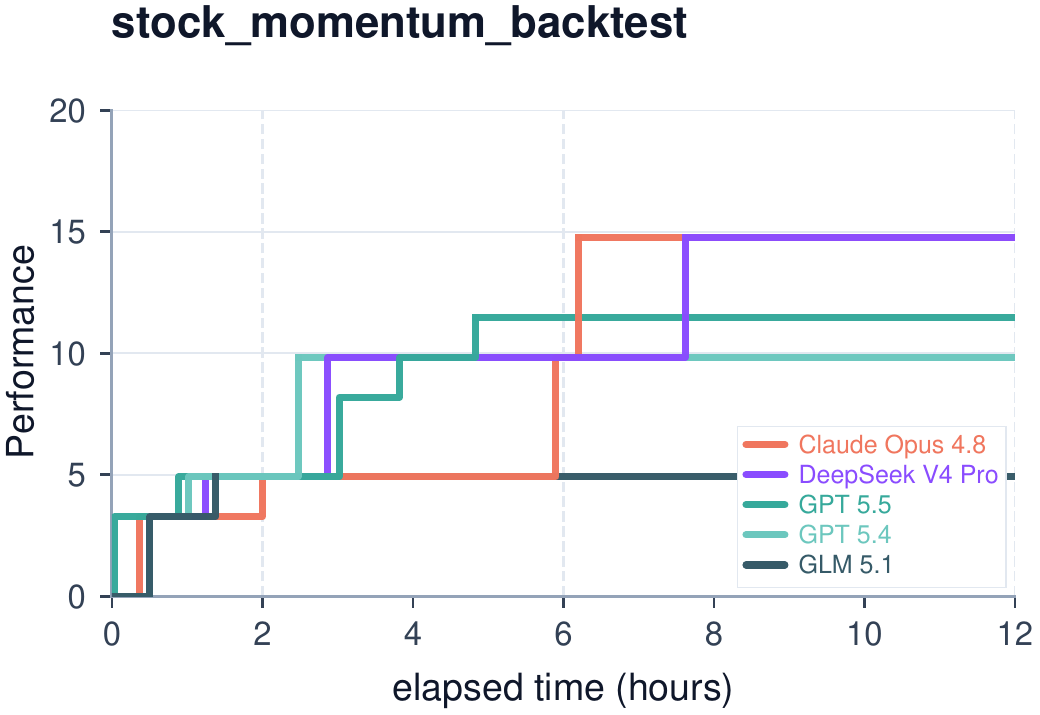}
\end{subfigure}
\hfill
\begin{subfigure}[b]{0.48\linewidth}
\includegraphics[width=\linewidth]{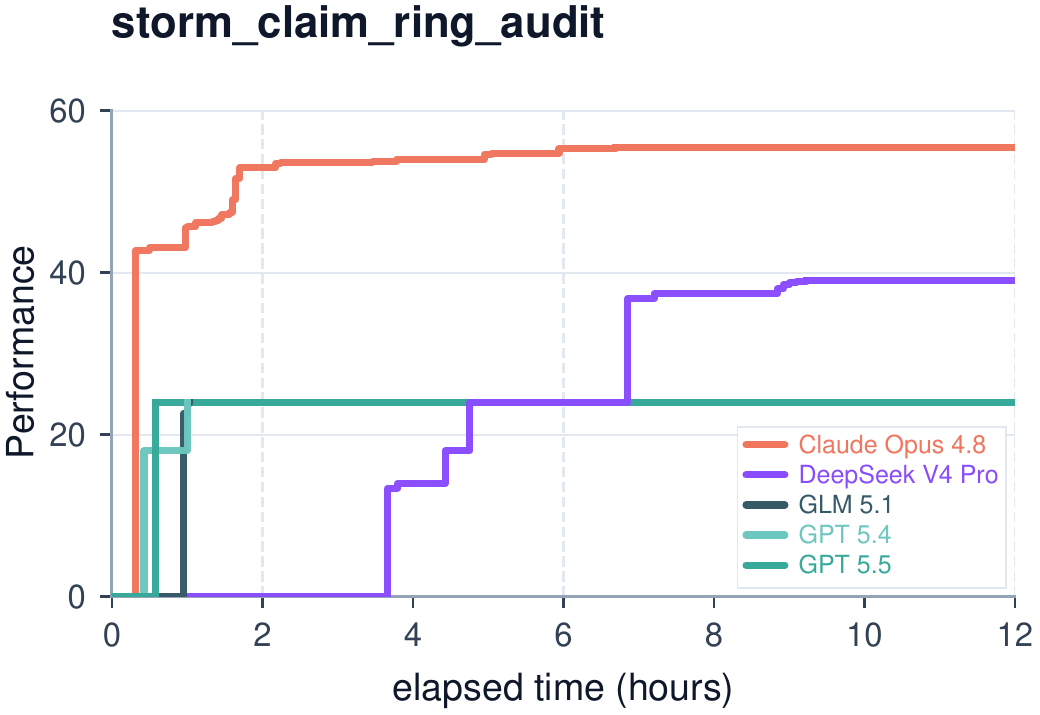}
\end{subfigure}
\caption{Per-task learning curves: Professional Knowledge Work cont. (18/21).}
\label{fig:curves-all-18}
\end{figure}

\begin{figure}[p]
\centering
\begin{subfigure}[b]{0.48\linewidth}
\includegraphics[width=\linewidth]{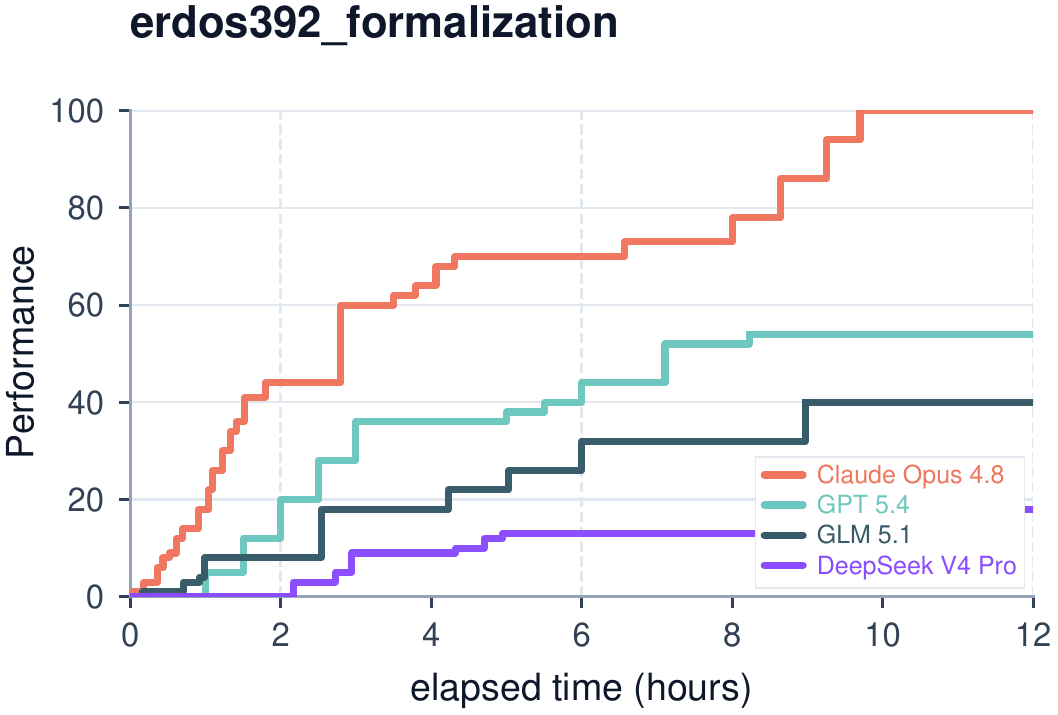}
\end{subfigure}
\hfill
\begin{subfigure}[b]{0.48\linewidth}
\includegraphics[width=\linewidth]{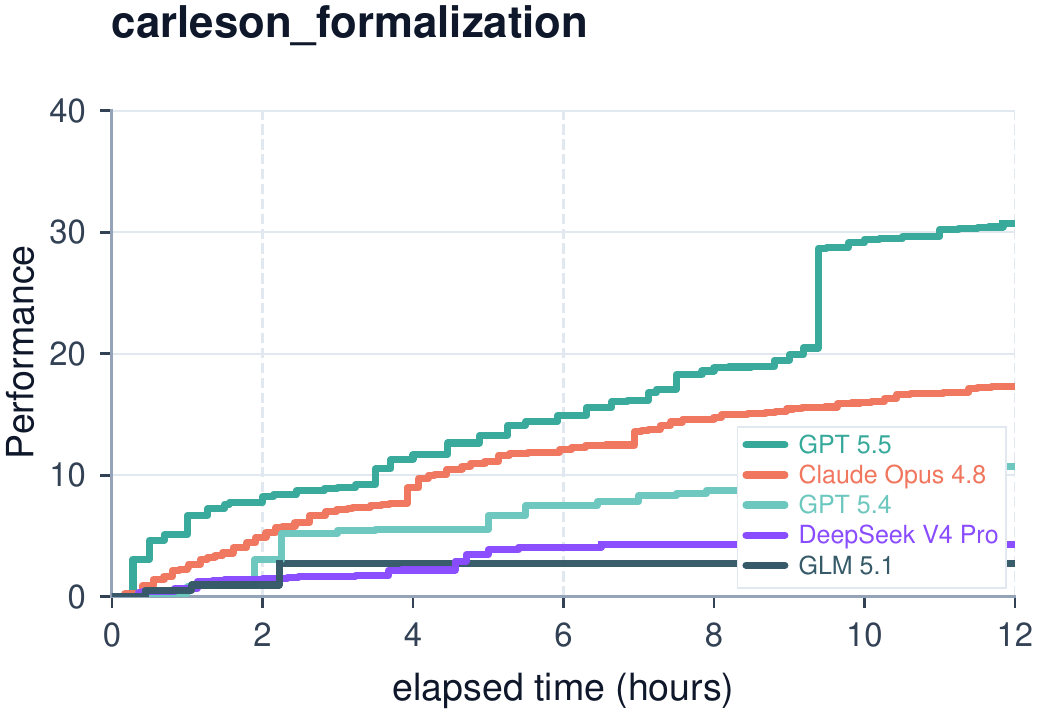}
\end{subfigure}
\vspace{0.5em}
\begin{subfigure}[b]{0.48\linewidth}
\includegraphics[width=\linewidth]{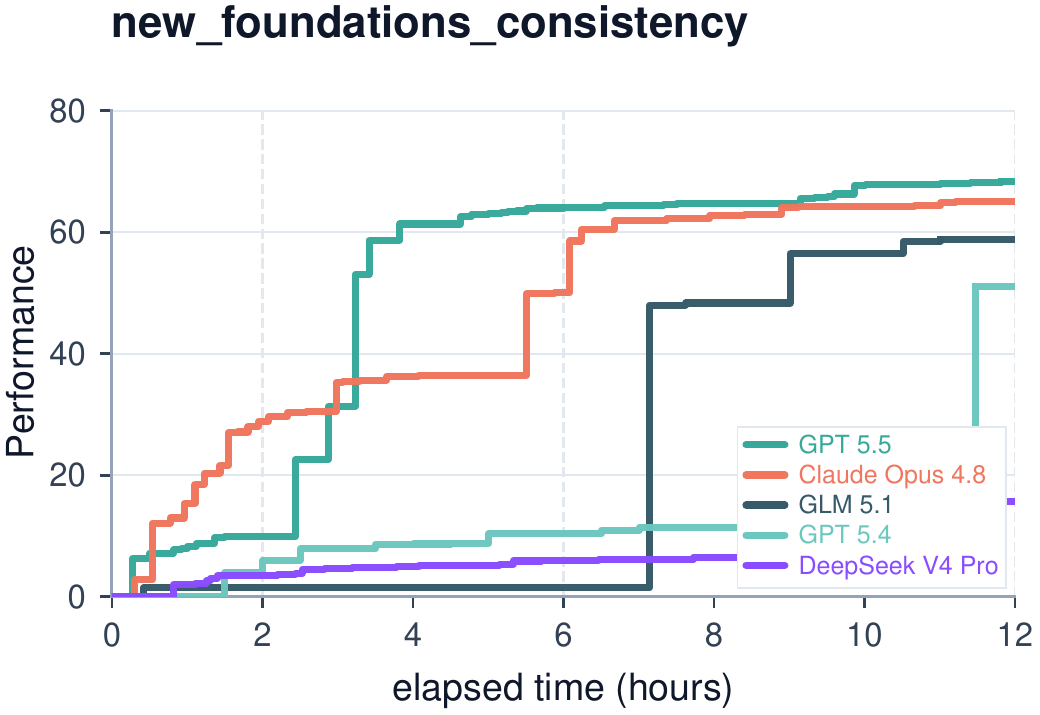}
\end{subfigure}
\hfill
\begin{subfigure}[b]{0.48\linewidth}
\includegraphics[width=\linewidth]{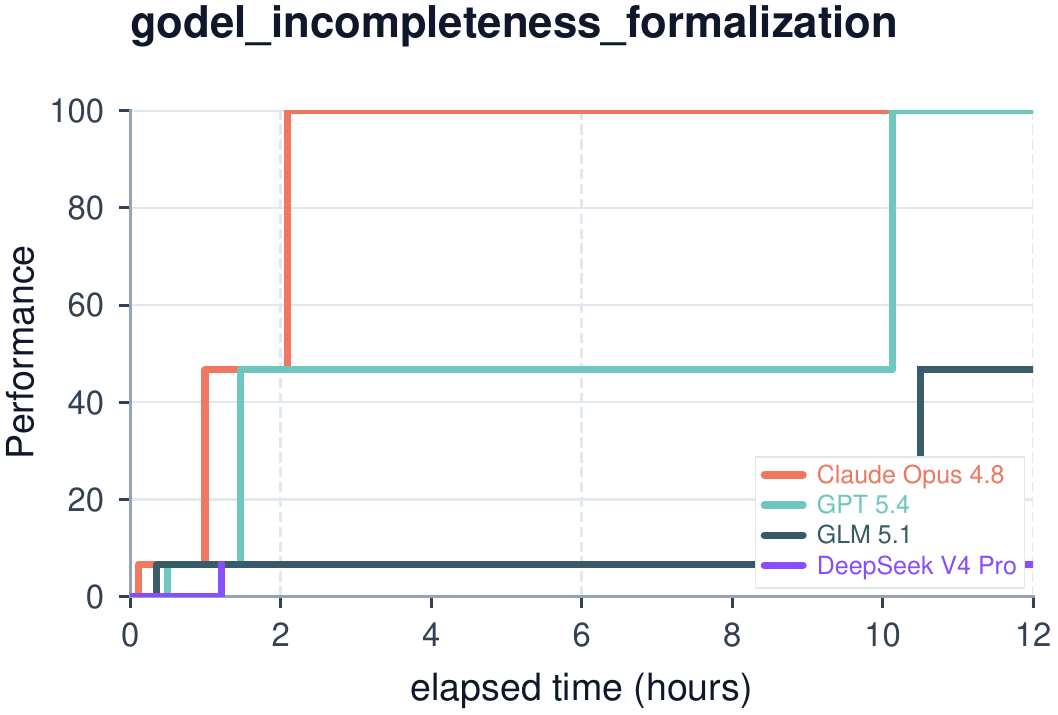}
\end{subfigure}
\vspace{0.5em}
\begin{subfigure}[b]{0.48\linewidth}
\includegraphics[width=\linewidth]{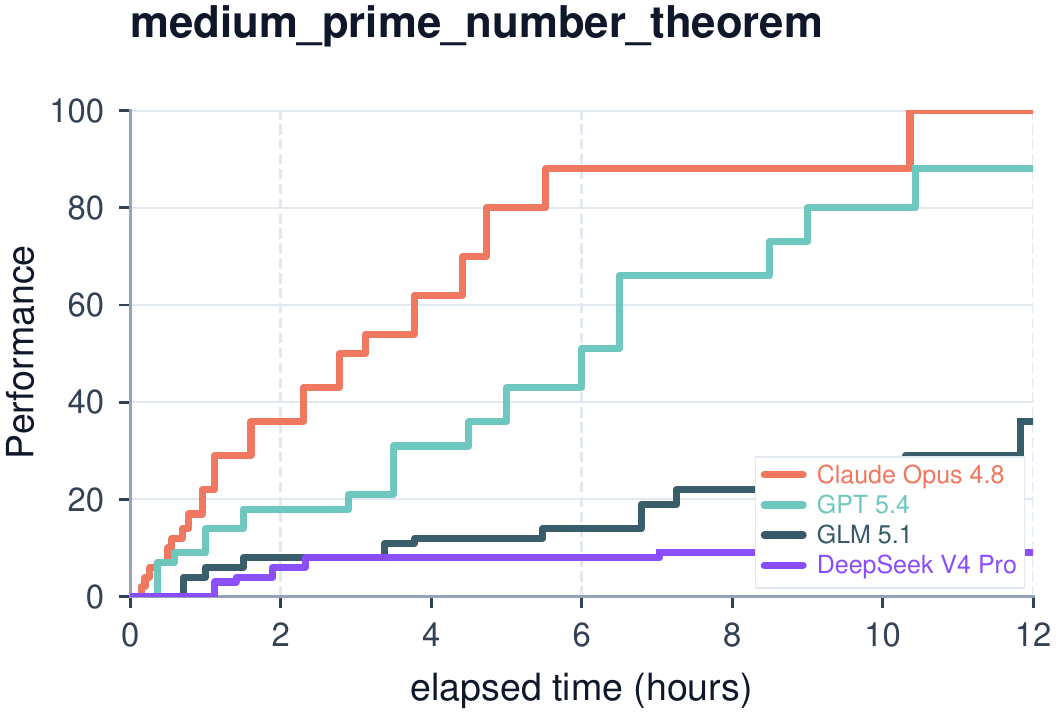}
\end{subfigure}
\hfill
\begin{subfigure}[b]{0.48\linewidth}
\includegraphics[width=\linewidth]{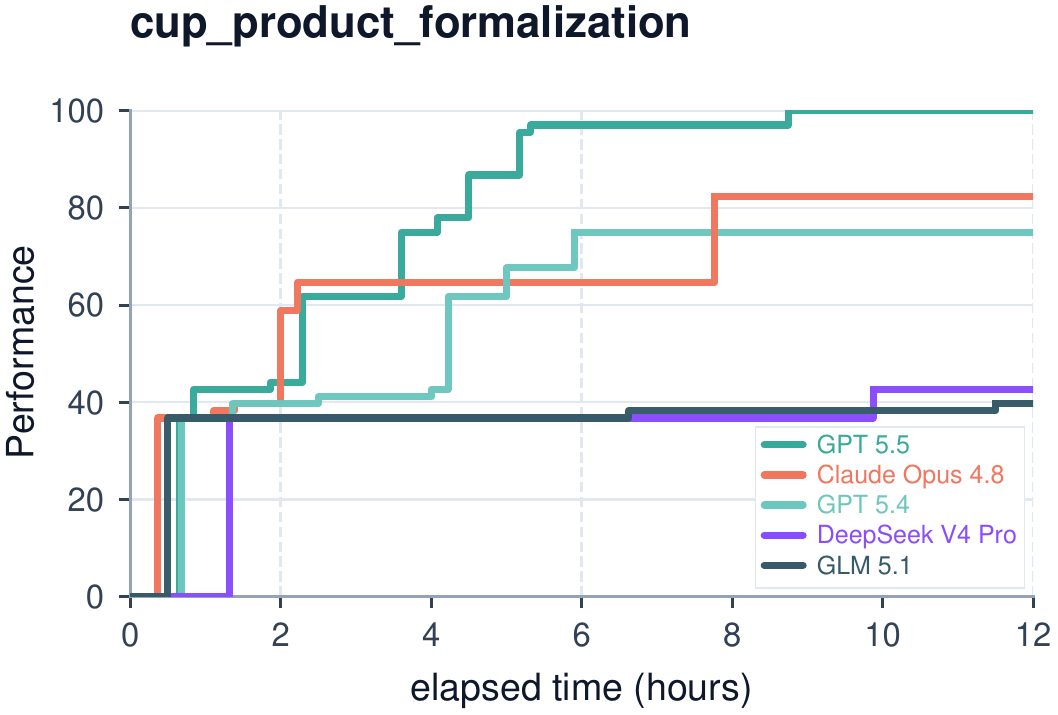}
\end{subfigure}
\caption{Per-task learning curves: Formal Math \& Theorem Proving (19/21).}
\label{fig:curves-all-19}
\end{figure}

\begin{figure}[p]
\centering
\begin{subfigure}[b]{0.48\linewidth}
\includegraphics[width=\linewidth]{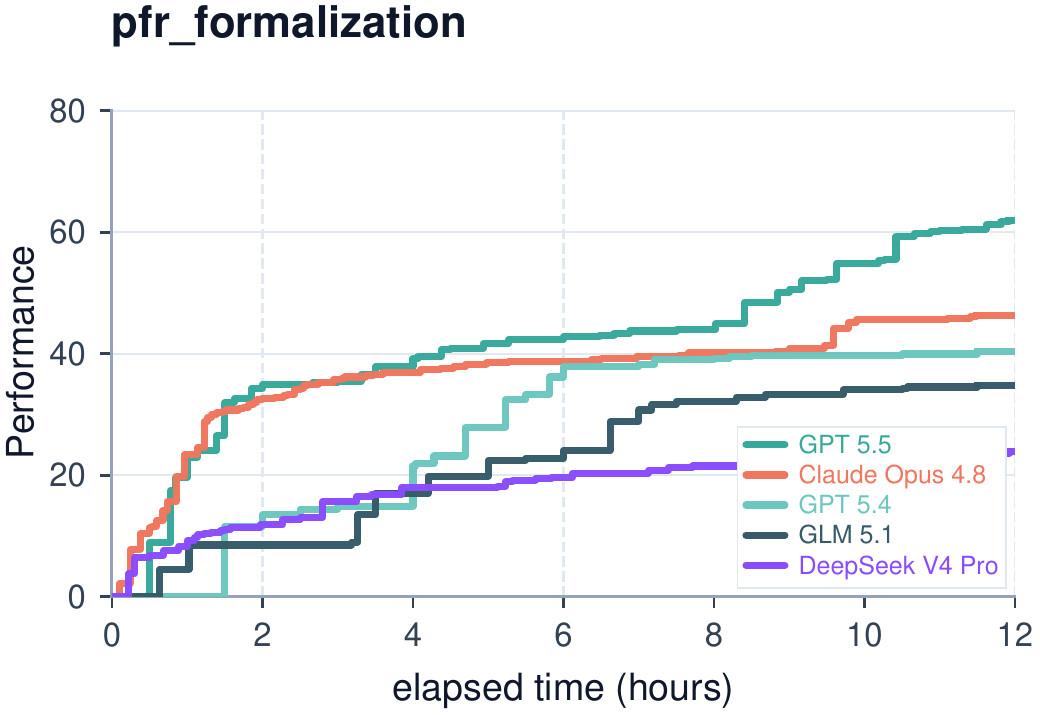}
\end{subfigure}
\hfill
\begin{subfigure}[b]{0.48\linewidth}
\includegraphics[width=\linewidth]{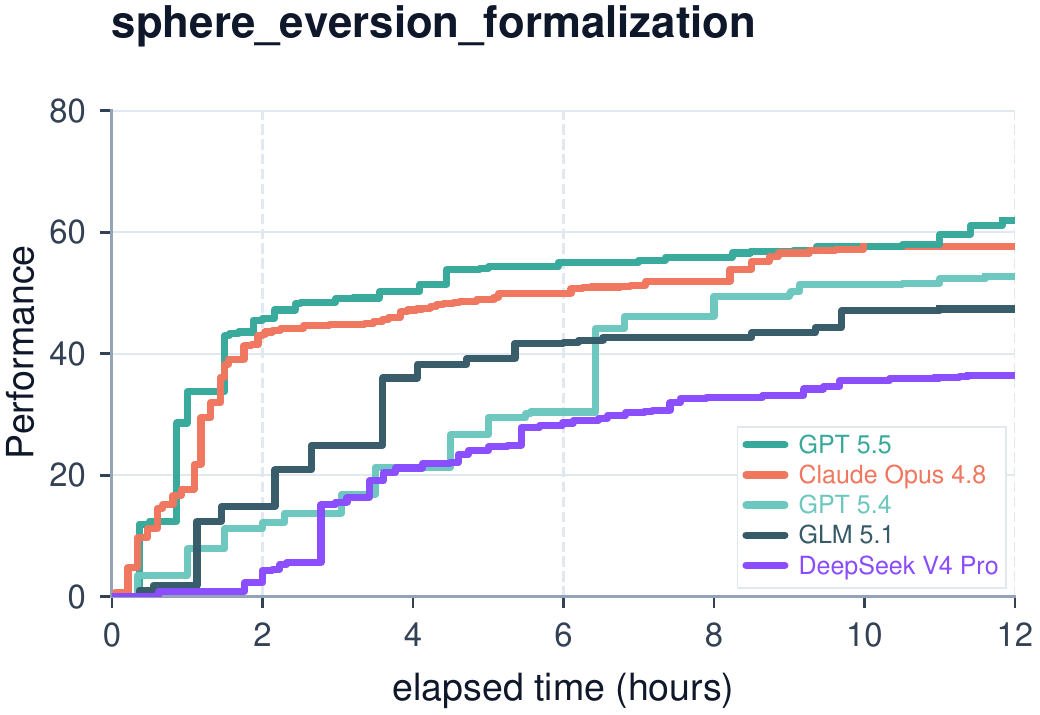}
\end{subfigure}
\vspace{0.5em}
\begin{subfigure}[b]{0.48\linewidth}
\includegraphics[width=\linewidth]{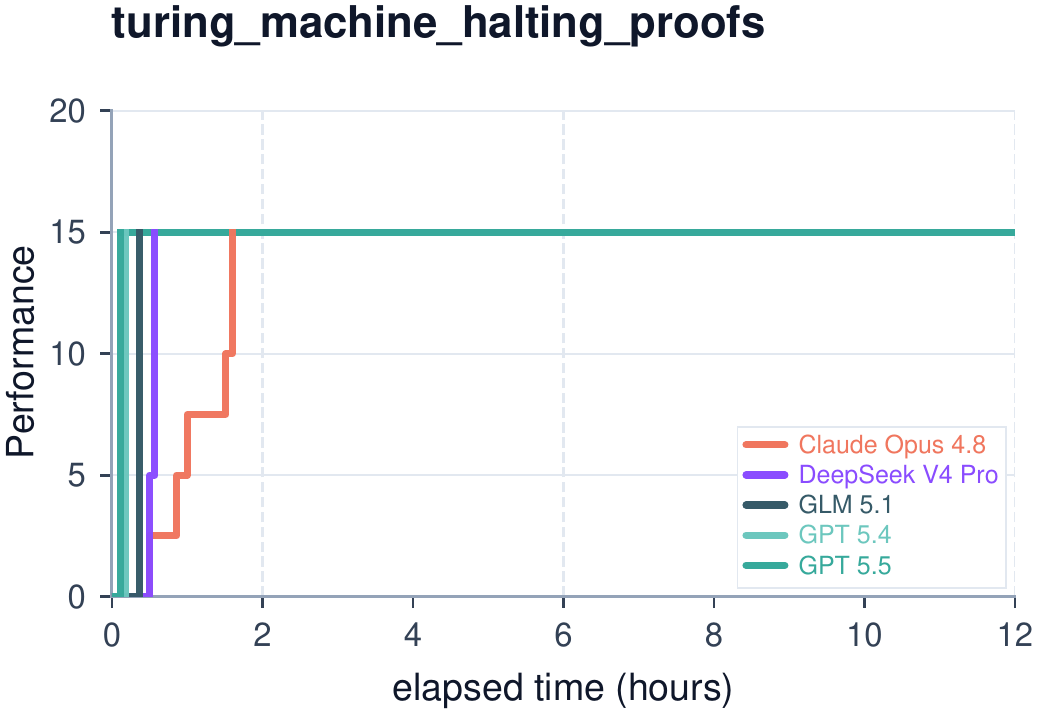}
\end{subfigure}
\hfill
\begin{subfigure}[b]{0.48\linewidth}
\includegraphics[width=\linewidth]{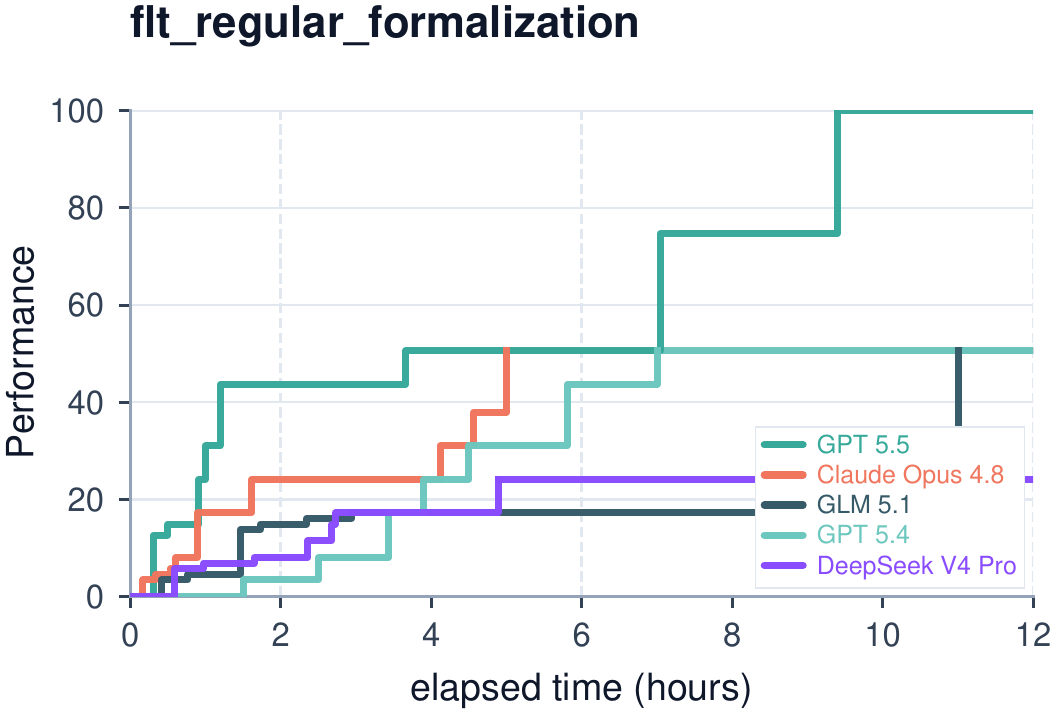}
\end{subfigure}
\caption{Per-task learning curves: Formal Math \& Theorem Proving cont. (20/21).}
\label{fig:curves-all-20}
\end{figure}

\begin{figure}[p]
\centering
\begin{subfigure}[b]{0.48\linewidth}
\includegraphics[width=\linewidth]{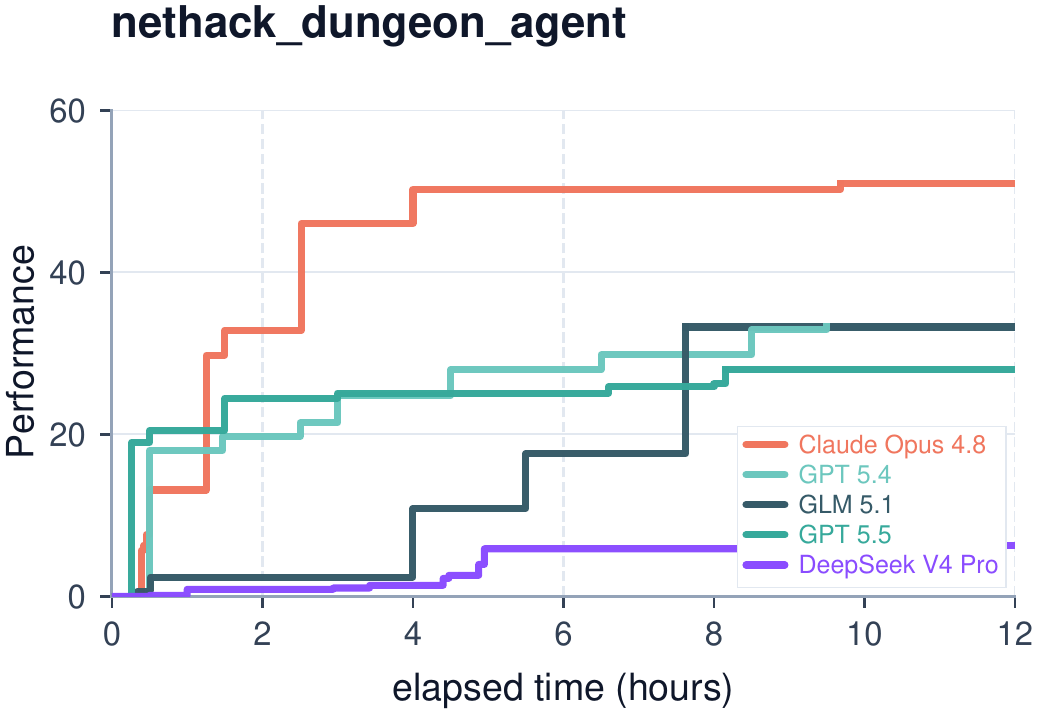}
\end{subfigure}
\hfill
\begin{subfigure}[b]{0.48\linewidth}
\includegraphics[width=\linewidth]{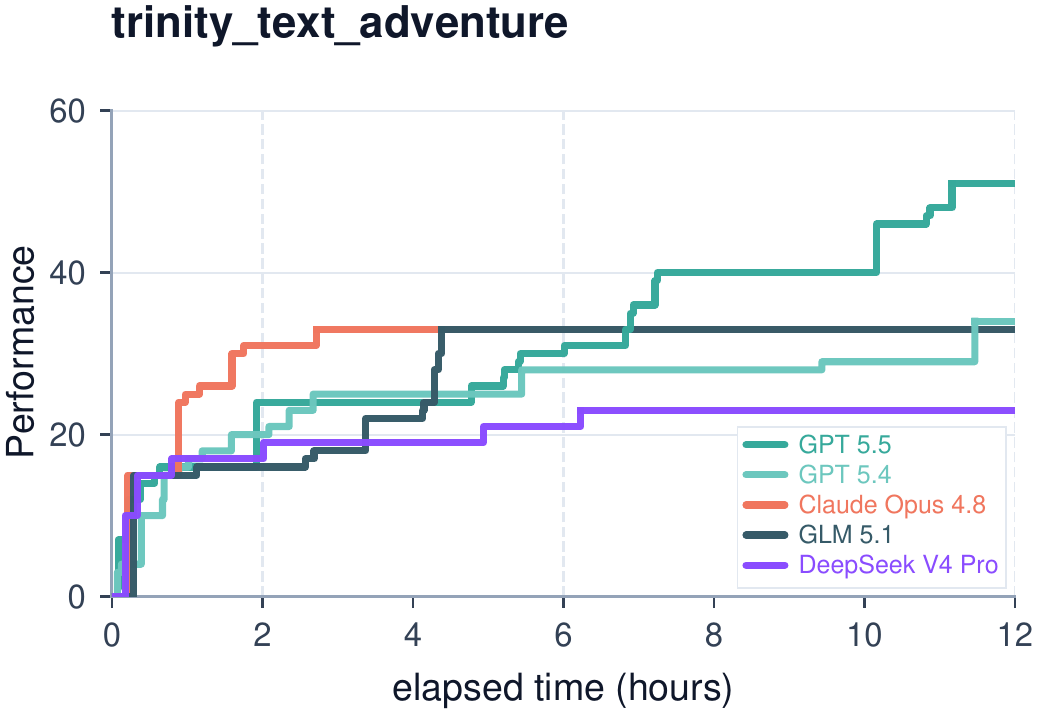}
\end{subfigure}
\vspace{0.5em}
\begin{subfigure}[b]{0.48\linewidth}
\includegraphics[width=\linewidth]{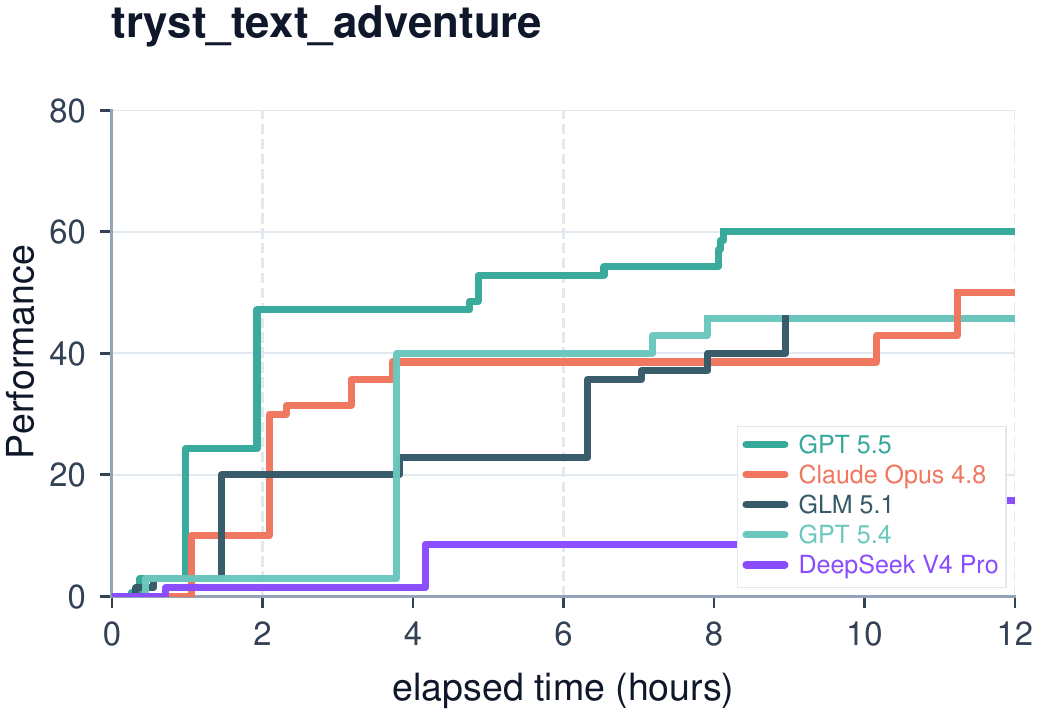}
\end{subfigure}
\hfill
\begin{subfigure}[b]{0.48\linewidth}
\includegraphics[width=\linewidth]{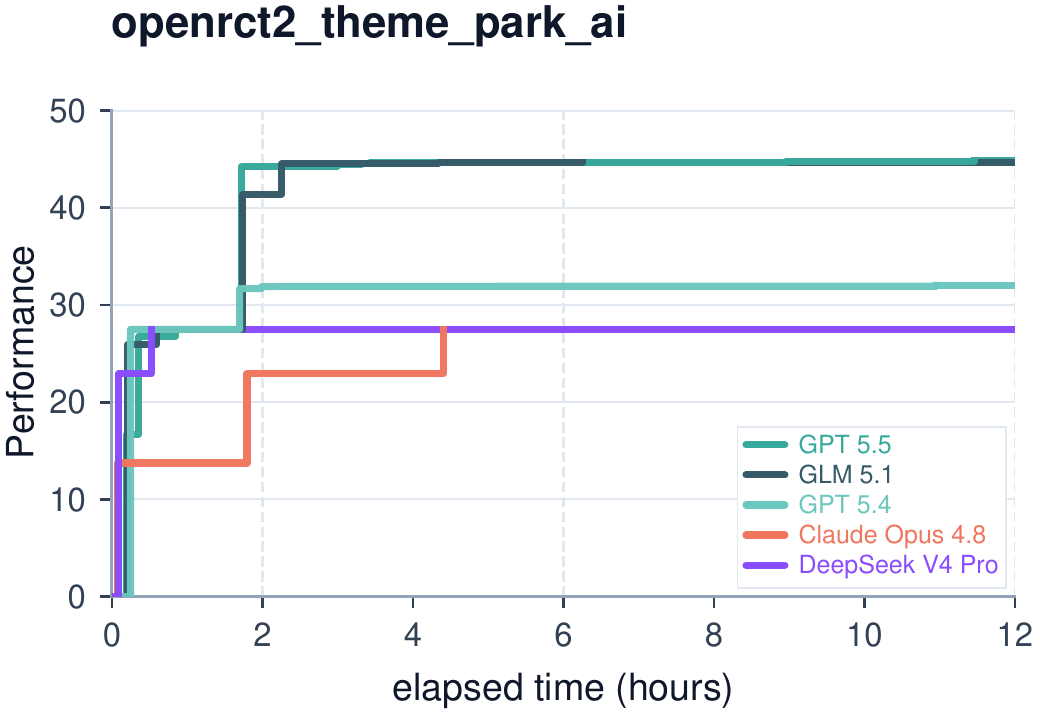}
\end{subfigure}
\vspace{0.5em}
\begin{subfigure}[b]{0.48\linewidth}
\includegraphics[width=\linewidth]{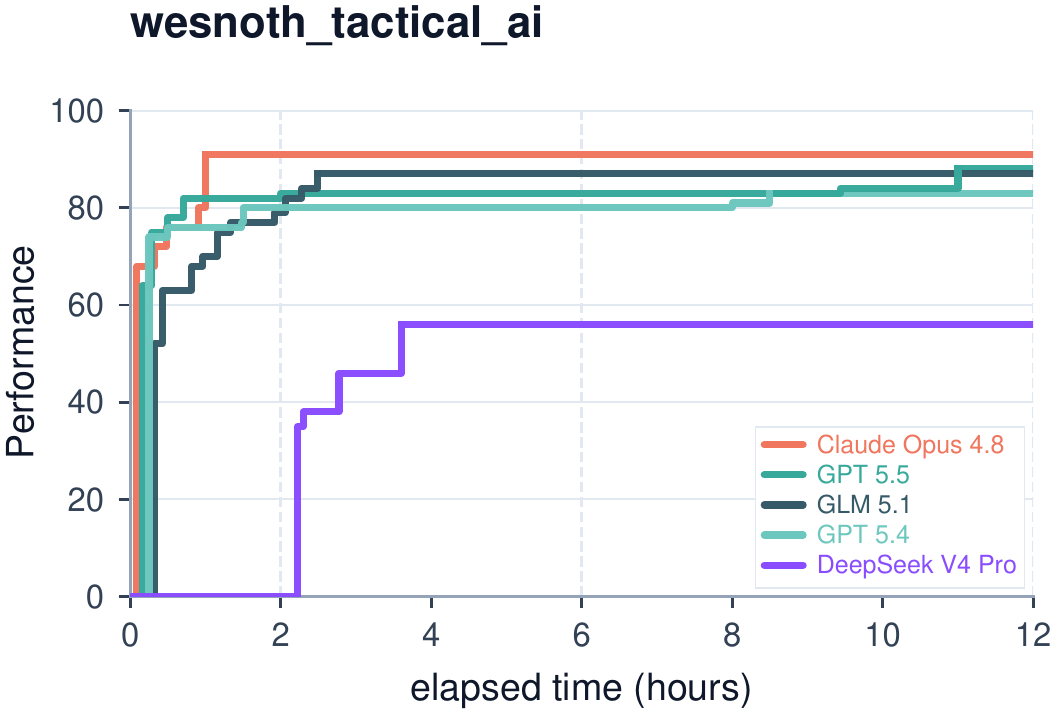}
\end{subfigure}
\caption{Per-task learning curves: Interactive Games \& Simulators (21/21).}
\label{fig:curves-all-21}
\end{figure}

\clearpage

\subsection{Per-Task Score Tables}
\label{sec:appendix-category-scores}

For every task--model configuration, we scheduled three independent long-horizon runs. In practice, because these 12-hour trajectories are sensitive to network instability and serving-side reliability limits, we carried out multiple rolling evaluation rounds to recover failed or incomplete runs and produced the final score tables below. A small number of reported task--model cells nevertheless still have fewer than three valid runs; those cells are marked with \textsuperscript{*}. Each run score is first rescaled to the 0--100 task scale; the adjacent $\pm s$ term reports the sample standard deviation across these rescaled run scores when at least two valid runs are available. This standard deviation is not clipped to the 0--100 range, so $\bar{x}\pm s$ should be read as run-to-run variation rather than as a bounded score interval.

\input{tables/category_score_tables}

%% file: sections/gw_case_study_details.tex
\subsection{Gravitational-Wave Case Study Details}
\label{sec:appendix-gw-case-study-details}

\textbf{The selected milestones compress the run into interpretable phase transitions.} Table~\ref{tab:gw-case-study-stages} summarizes the main phases behind the selected milestones rather than every submission. The first selected milestone already passes the protocol and entrypoint checks, generates all five required CSV files on the canonical grids, and obtains a score of 42.8. A subsequent key update raises the score to 47.1, mainly by improving the source dynamics. Around hour 1, another milestone reaches roughly 50 as the agent improves the H1/L1 spectrograms and early H1 alignment in the time domain. These gains reflect standard signal processing: filtering, normalization, windowing in the time-frequency plane, interpolation to the canonical grid, and cropping around the event.

\textbf{The largest score gain comes from the high-weight source-dynamics component.} The largest jump occurs around hours 4--5: the overall score rises from about 52.3 to 59.7, driven almost entirely by the source dynamics subscore (64.2 to 89.0). Because this subtask carries the largest weight, a compact physical model plus calibration produces a large total-score gain. After that point, the agent shifts to H1 waveform tuning, where the H1 time-series subscore climbs from roughly 47 to 95 by the end.

\textbf{The final solution is strong on H1 waveform and source dynamics, but it still leaves room to improve toward a human-standard LIGO-style analysis.} Table~\ref{tab:gw-case-study-subscores} breaks down the final score. GPT-5.5 scores high on H1 waveform reconstruction and the source dynamics but remains weak on the spectrograms and L1 reconstruction. Empirical calibration and parameter search produce a substantial score, but they do not replace a coherent end-to-end pipeline for strain preprocessing, whitening, time-frequency analysis, and waveform alignment across detectors.

\begin{table}[!htbp]
\centering
\footnotesize
\renewcommand{\arraystretch}{1.12}
\setlength{\tabcolsep}{4pt}
\begin{tabular}{lclp{0.46\linewidth}}
\BenchTopRule
Stage & Time & Best score & Main improvement \\
\BenchMidRule
Initial milestones & 0h & 42.8--47.1 & Built a reconstruction pipeline that can be evaluated; replacing a noisy frequency estimate improved the inferred source motion from 51 to 64 (\textbf{+13 pp.}) \\
Early signal processing & 1--4h & 52.3 & Recentered the time-frequency window on the merger; the Hanford detector waveform and both detector spectrograms improved together, raising the best score from 47.1 to 52.3 (\textbf{+5.2 pp.}) \\
Main breakthrough & 4--5h & 59.7 & Calibrated the binary-source model; the orbital velocity and separation score jumped from 64 to 89 (\textbf{+25 pp.}) \\
Late waveform tuning & 5--11.5h & 66.9 & Found the remaining mismatch in the Hanford detector waveform; time-shift alignment and error fitting raised Hanford from 47 to 95 (\textbf{+48 pp.}) \\
Final refinement & 12h & 67.0 & Consolidated final corrections for the Hanford and Livingston detectors; a small Livingston waveform gain raised the aggregate score from 66.9 to 67.0 (\textbf{+0.1 pp.}) \\
\BenchBottomRule
\end{tabular}
\caption{Main phases in the gravitational-wave evolving trajectory. The trajectory shows how the agent refines signal processing, source dynamics, and per-detector waveform reconstruction over the 12-hour run.}
\label{tab:gw-case-study-stages}
\end{table}

\begin{table}[!htbp]
\centering
\footnotesize
\renewcommand{\arraystretch}{1.12}
\setlength{\tabcolsep}{4pt}
\begin{tabular}{lr}
\BenchTopRule
Component & Final subscore \\
\BenchMidRule
H1 time series & 95.0 \\
L1 time series & 57.1 \\
H1 spectrogram & 42.7 \\
L1 spectrogram & 44.7 \\
Velocity/separation & 89.0 \\
\BenchMidRule
Aggregate score & 67.0 \\
\BenchBottomRule
\end{tabular}
\caption{Final subscore composition for the GPT-5.5 gravitational-wave 12h run.}
\label{tab:gw-case-study-subscores}
\end{table}

%% file: sections/harness_level_continuation_ablations.tex
\subsection{Harness-Level Continuation Ablations}
\label{sec:ablation-experiment-setting}
\label{sec:appendix-ablation-setting}

Long-running agent evaluation depends not only on model capability, but also on how the harness keeps the run alive and carries useful state across many hours. In a 12-hour task, an agent may stop its running unexpectedly due to idleness and need to resume. These continuation mechanisms are therefore part of the practical measurement setup: a weak scaffold can make a capable model stop working on the task, while a stronger scaffold may help the same model keep active and improving. As a supplementary harness-level ablation, we evaluate two such mechanisms, /goal mode and the Ralph loop, under a fixed 12-hour budget with GPT-5.5 and GPT-5.4.

Base uses the standard harness: one continuing agent session, a stop hook to prevent voluntary early exit, and auto-resume for abnormal exits. In /goal mode~\cite{openai2026codexgoals}, the harness prompts the agent to create a task-level goal at the beginning, keep it active during the run, and mark it complete only when the task is validated. The Ralph loop follows the file-backed fresh-context pattern~\cite{huntley2026ralphloop}: each loop starts a new agent invocation on the same workspace, asks it to read and update progress.md, then appends judge feedback to that file before the next loop. It uses up to 100 loops, a 7200-second per-loop cap, and the same 12-hour overall budget. Each cell schedules three runs and reports the mean score over valid runs.

\input{tables/ablations}

As shown in Table~\ref{tab:feature_ablation_goal_ralph}, the goal mode and the Ralph loop often outperform the base setting, suggesting that long-horizon agents benefit from harness support that preserves and updates task state across extended runs. In the displayed-task average, GPT-5.5 improves from 42.6 in Base to 43.1 with Goal and 43.4 with Ralph, while GPT-5.4 improves from 26.1 in Base to 31.8 with Goal and 27.6 with Ralph. The gains are not uniform across all tasks and models, so we treat this as an appendix-level harness diagnostic rather than a main result.

Most cells use all three scheduled runs. A few GPT-5.4 cells have fewer valid runs because of intermittent API or network instability: one-run cells are w/ Goal for \texttt{storyboard\_\allowbreak{}ad\_\allowbreak{}copywriting} and \texttt{flt\_\allowbreak{}regular\_\allowbreak{}formalization}, plus w/ Ralph for \texttt{symbolic\_integration\_engine}; two-run cells are \texttt{battery\_\allowbreak{}soh\_\allowbreak{}rul\_\allowbreak{}anomaly} (w/ Goal and w/ Ralph), \texttt{capecod\_\allowbreak{}plume\_\allowbreak{}reconstruction} (w/ Goal), \texttt{combinatorial\_\allowbreak{}games\_\allowbreak{}formalization} (w/ Ralph), \texttt{ffmpeg\_\allowbreak{}swscale\_\allowbreak{}reimplementation} (w/ Goal), \texttt{flt\_\allowbreak{}regular\_\allowbreak{}formalization} (w/ Ralph), \texttt{git\_\allowbreak{}rewrite\_\allowbreak{}in\_\allowbreak{}zig} (w/ Goal), and \texttt{symbolic\_integration\_engine} (w/ Goal). All reported GPT-5.5 cells use the full three runs.

%% file: tables/ablations.tex
\providecommand{\NA}{\textemdash}
\providecommand{\best}[1]{}
\renewcommand{\best}[1]{\begingroup\bfseries\boldmath #1\endgroup}
\providecommand{\second}[1]{}
\renewcommand{\second}[1]{\underline{#1}}
\providecommand{\sci}[2]{\ensuremath{#1\times 10^{#2}}}
\begin{table*}[t!]
\centering
\footnotesize
\renewcommand{\arraystretch}{1.08}
\setlength{\tabcolsep}{2.5pt}
\resizebox{\textwidth}{!}{%
\begin{tabular}{@{}>{\raggedright\arraybackslash}p{0.34\linewidth}*{6}{>{\centering\arraybackslash}p{0.10\linewidth}}@{}}
\BenchTopRule
\textbf{Task} & \multicolumn{3}{c}{\textbf{GPT-5.5}} & \multicolumn{3}{c}{\textbf{GPT-5.4}} \\
\BenchCMidRule{2-4} \BenchCMidRule{5-7}
 & \textbf{Base} & \textbf{w/ Goal} & \textbf{w/ Ralph} & \textbf{Base} & \textbf{w/ Goal} & \textbf{w/ Ralph} \\
\BenchMidRule
\texttt{portfolio\_\allowbreak{}risk\_\allowbreak{}calibration} & \ensuremath{\mathrm{25.0}} & \second{\ensuremath{\mathrm{27.5}}} & \best{\ensuremath{\mathrm{34.3}}} & \ensuremath{\mathrm{10.7}} & \best{\ensuremath{\mathrm{19.8}}} & \second{\ensuremath{\mathrm{12.2}}} \\
\texttt{storyboard\_\allowbreak{}ad\_\allowbreak{}copywriting} & \ensuremath{\mathrm{77.0}} & \second{\ensuremath{\mathrm{88.3}}} & \best{\ensuremath{\mathrm{97.3}}} & \ensuremath{\mathrm{65.7}} & \second{\ensuremath{\mathrm{78.0}}} & \best{\ensuremath{\mathrm{83.3}}} \\
\texttt{arc\_\allowbreak{}compiler\_\allowbreak{}runtime} & \best{\ensuremath{\mathrm{72.4}}} & \second{\ensuremath{\mathrm{70.6}}} & \ensuremath{\mathrm{52.6}} & \second{\ensuremath{\mathrm{50.0}}} & \ensuremath{\mathrm{47.2}} & \best{\ensuremath{\mathrm{51.2}}} \\
\texttt{battery\_\allowbreak{}soh\_\allowbreak{}rul\_\allowbreak{}anomaly} & \second{\ensuremath{\mathrm{30.2}}} & \ensuremath{\mathrm{23.8}} & \best{\ensuremath{\mathrm{48.8}}} & \best{\ensuremath{\mathrm{14.7}}} & \second{\ensuremath{\mathrm{14.6}}} & \ensuremath{\mathrm{13.4}} \\
\texttt{borden\_\allowbreak{}source\_\allowbreak{}inversion} & \second{\ensuremath{\mathrm{38.5}}} & \ensuremath{\mathrm{32.8}} & \best{\ensuremath{\mathrm{62.2}}} & \ensuremath{\mathrm{8.0}} & \second{\ensuremath{\mathrm{17.3}}} & \best{\ensuremath{\mathrm{23.2}}} \\
\texttt{capecod\_\allowbreak{}plume\_\allowbreak{}reconstruction} & \second{\ensuremath{\mathrm{16.4}}} & \ensuremath{\mathrm{16.0}} & \best{\ensuremath{\mathrm{17.3}}} & \ensuremath{\mathrm{12.6}} & \second{\ensuremath{\mathrm{13.3}}} & \best{\ensuremath{\mathrm{13.7}}} \\
\texttt{combinatorial\_\allowbreak{}games\_\allowbreak{}formalization} & \ensuremath{\mathrm{38.2}} & \best{\ensuremath{\mathrm{45.3}}} & \second{\ensuremath{\mathrm{39.8}}} & \second{\ensuremath{\mathrm{17.8}}} & \ensuremath{\mathrm{13.0}} & \best{\ensuremath{\mathrm{20.2}}} \\
\texttt{ffmpeg\_\allowbreak{}swscale\_\allowbreak{}reimplementation} & \ensuremath{\mathrm{15.3}} & \second{\ensuremath{\mathrm{16.4}}} & \best{\ensuremath{\mathrm{25.2}}} & \ensuremath{\mathrm{13.9}} & \best{\ensuremath{\mathrm{28.3}}} & \second{\ensuremath{\mathrm{21.8}}} \\
\texttt{flt\_\allowbreak{}regular\_\allowbreak{}formalization} & \best{\ensuremath{\mathrm{75.1}}} & \second{\ensuremath{\mathrm{66.7}}} & \ensuremath{\mathrm{58.6}} & \best{\ensuremath{\mathrm{48.3}}} & \second{\ensuremath{\mathrm{43.7}}} & \second{\ensuremath{\mathrm{43.7}}} \\
\texttt{git\_\allowbreak{}rewrite\_\allowbreak{}in\_\allowbreak{}zig} & \ensuremath{\mathrm{18.4}} & \second{\ensuremath{\mathrm{20.5}}} & \best{\ensuremath{\mathrm{22.0}}} & \ensuremath{\mathrm{15.4}} & \best{\ensuremath{\mathrm{18.4}}} & \second{\ensuremath{\mathrm{15.8}}} \\
\texttt{tryst\_\allowbreak{}text\_\allowbreak{}adventure} & \second{\ensuremath{\mathrm{55.7}}} & \best{\ensuremath{\mathrm{56.2}}} & \ensuremath{\mathrm{40.2}} & \best{\ensuremath{\mathrm{44.3}}} & \ensuremath{\mathrm{32.9}} & \second{\ensuremath{\mathrm{42.4}}} \\
\texttt{openttd\_\allowbreak{}transport\_\allowbreak{}ai} & \ensuremath{\mathrm{28.1}} & \best{\ensuremath{\mathrm{39.8}}} & \second{\ensuremath{\mathrm{30.6}}} & \best{\ensuremath{\mathrm{11.9}}} & \ensuremath{\mathrm{1.2}} & \second{\ensuremath{\mathrm{8.1}}} \\
\texttt{pocketbase\_\allowbreak{}backend\_\allowbreak{}architecture} & \best{\ensuremath{\mathrm{62.5}}} & \best{\ensuremath{\mathrm{62.5}}} & \second{\ensuremath{\mathrm{41.7}}} & \second{\ensuremath{\mathrm{20.8}}} & \best{\ensuremath{\mathrm{54.2}}} & \ensuremath{\mathrm{0.0}} \\
\texttt{symbolic\_\allowbreak{}integration\_\allowbreak{}engine} & \best{\ensuremath{\mathrm{44.0}}} & \ensuremath{\mathrm{36.6}} & \second{\ensuremath{\mathrm{37.7}}} & \ensuremath{\mathrm{30.9}} & \best{\ensuremath{\mathrm{63.7}}} & \second{\ensuremath{\mathrm{37.2}}} \\
\BenchMidRule
\textbf{Avg.} & \ensuremath{\mathrm{42.6}} & \second{\ensuremath{\mathrm{43.1}}} & \best{\ensuremath{\mathrm{43.4}}} & \ensuremath{\mathrm{26.1}} & \best{\ensuremath{\mathrm{31.8}}} & \second{\ensuremath{\mathrm{27.6}}} \\
\BenchBottomRule
\end{tabular}%
}
\caption{Harness-level continuation ablation of /goal mode and the Ralph loop. GPT-5.5 and GPT-5.4 are evaluated with Base, Goal, and Ralph settings under the same 12-hour budget; each value is the mean score over valid runs, and the Avg. row averages the displayed task rows. Incomplete cells are detailed in Appendix~\ref{sec:appendix-ablation-setting}. Bold marks the best setting and underlining marks the second-best setting.}
\label{tab:feature_ablation_goal_ralph}
\end{table*}

%% file: sections/task_by_task_specifications.tex
\subsection{Per-Task Design Notes}
\label{sec:task-by-task-design-notes}

This appendix gives per-task design notes for the 134-task \benchmark{} suite.

\begingroup
\scriptsize
\setlength{\tabcolsep}{2pt}
\renewcommand{\arraystretch}{1.0}
\setlength{\LTpre}{4pt}
\setlength{\LTpost}{4pt}
\newcommand{\tasknoterule}{\specialrule{0.12pt}{0.6pt}{0.6pt}}
\begin{longtable}{@{}>{\raggedright\arraybackslash}p{0.20\linewidth}>{\raggedright\arraybackslash}p{0.77\linewidth}@{}}
\BenchTopRule
Task & Design notes \\
\BenchMidRule
\endfirsthead
\BenchTopRule
Task & Design notes \\
\BenchMidRule
\endhead
\BenchMidRule
\multicolumn{2}{r}{\emph{Continued on next page}} \\
\endfoot
\BenchBottomRule
\noalign{\vskip 0.75em}
\caption{Per-task design notes for all 134 \benchmark{} tasks.}\label{tab:task-notes-all}\\
\endlastfoot
\multicolumn{2}{@{}l}{\rule{0pt}{1.2em}\textsc{\scriptsize Systems \& Software Engineering}} \\
\BenchMidRule
\texttt{ann\_\allowbreak{}vector\_\allowbreak{}search\_\allowbreak{}qps} & Replace a brute-force NumPy nearest-neighbor baseline with a high-performance approximate nearest-neighbor implementation under a hard recall constraint. Scored by queries per second. \\
\tasknoterule
\texttt{arc\_\allowbreak{}compiler\_\allowbreak{}runtime} & Implement a complete TypeScript compiler pipeline (lexer, parser, type checker, code generator) for a novel programming language defined by specification documents. \\
\tasknoterule
\texttt{rust\_\allowbreak{}multicrate\_\allowbreak{}reconstruction} & Reconstruct missing Rust implementations across a multi-crate content-addressable storage workspace, given only type signatures and public API contracts. \\
\tasknoterule
\texttt{codeflash\_\allowbreak{}repair\_\allowbreak{}performance} & Diagnose and repair intentionally degraded modules in a Python code-optimization tool spanning CST manipulation, profiling infrastructure, and test harness integration. Scored on both functional correctness and runtime performance. \\
\tasknoterule
\texttt{copier\_\allowbreak{}modular\_\allowbreak{}refactor} & Implement six architectural targets in the Copier Python scaffolding library: structured exceptions, template management, version compatibility, rendering modes, a worker class, and a user-data layer. \\
\tasknoterule
\texttt{dependent\_\allowbreak{}type\_\allowbreak{}checker} & Build a complete dependent type checker in Rust for a subset of Martin-L\"{o}f Type Theory, supporting cumulative universes, Pi/Sigma types, general inductive types with positivity checking, and universe polymorphism. Scored on both correctness and normalization throughput. \\
\tasknoterule
\texttt{entt\_\allowbreak{}graph\_\allowbreak{}module} & Implement seven feature targets in the EnTT C++ entity-component-system framework, including a graph module with adjacency matrix, a task-graph builder with transitive reduction, DOT export, and several core utility additions. \\
\tasknoterule
\texttt{exchange\_\allowbreak{}core\_\allowbreak{}throughput} & Maximize peak throughput of a Java financial matching engine built on the LMAX Disruptor by tuning thread topology, wait strategies, ring-buffer sizing, order-book implementation, and JVM configuration. \\
\tasknoterule
\texttt{ffmpeg\_\allowbreak{}swscale\_\allowbreak{}reimplementation} & Reimplement FFmpeg's libswscale pixel-format conversion and scaling library in Rust, handling multiple pixel formats and scaling algorithms. A correctness-passing scaffold is provided; the agent must optimize for speed via SIMD. \\
\tasknoterule
\texttt{git\_\allowbreak{}rewrite\_\allowbreak{}in\_\allowbreak{}zig} & Reimplement git as a drop-in Zig binary producing identical CLI output, exit codes, and repository state as the C reference implementation. The C source is available for reading but cannot be compiled. \\
\tasknoterule
\texttt{high\_\allowbreak{}performance\_\allowbreak{}object\_\allowbreak{}mapper} & Implement a .NET object-to-object mapping library using expression-tree compilation and IL emission, handling flat/nested mapping, collections, custom type converters, nullables, and inheritance hierarchies. \\
\tasknoterule
\texttt{integer\_\allowbreak{}compression\_\allowbreak{}codec} & Improve a C++ integer compression codec for better compression ratio and decode throughput on uint32 datasets via techniques such as delta encoding, bit-packing, and SIMD vectorization. Exact round-trip correctness is mandatory. \\
\tasknoterule
\texttt{juliet\_\allowbreak{}vulnerability\_\allowbreak{}analyzer} & Implement a deterministic static analyzer that processes structured program facts to detect vulnerabilities across six CWE categories (stack/heap overflow, integer overflow, null dereference, use-after-free, command injection). \\
\tasknoterule
\texttt{libexpat\_\allowbreak{}x86\_\allowbreak{}assembly} & Reimplement the libexpat XML parser entirely in x86-64 assembly, producing an ABI-compatible shared library. No C compiler is available---only assembler, linker, and libc. \\
\tasknoterule
\texttt{odata\_\allowbreak{}query\_\allowbreak{}service} & Complete a .NET OData query processing library: query-string parsing, LINQ expression-tree generation for filtering, multi-field sorting, pagination, field projection, and ASP.NET Core middleware integration. \\
\tasknoterule
\texttt{litestar\_\allowbreak{}infra\_\allowbreak{}refactor} & Implement five async infrastructure subsystems in the Litestar Python web framework: key-value stores with expiry, an event bus, WebSocket listener abstractions, a DTO framework, and a channels pub/sub system. \\
\tasknoterule
\texttt{lua\_\allowbreak{}native\_\allowbreak{}compiler} & Build a Lua 5.4 ahead-of-time native compiler that reads Lua source and emits standalone ELF executables with real x86-64 machine code---not C API call sequences or bytecode dispatch loops. Output must match the reference interpreter byte-for-byte. \\
\tasknoterule
\texttt{mimesis\_\allowbreak{}modular\_\allowbreak{}refactor} & Implement seven refactoring targets in the Mimesis Python fake-data library: a declarative schema class, an enum system with metaclass-driven random selection, and five data providers spanning payments, cryptography, science, internet, and unit systems. \\
\tasknoterule
\texttt{nlohmann\_\allowbreak{}json\_\allowbreak{}modularization} & Implement five feature targets in the nlohmann/json C++ library: UBJSON binary serialization, JSON Merge Patch (RFC 7386), Grisu2 shortest-representation float formatting, a range-based iteration interface, and library metadata macros. \\
\tasknoterule
\texttt{notebook\_\allowbreak{}lossless\_\allowbreak{}compression} & Build a lossless compression pipeline for Jupyter notebooks with a training phase (dictionary learning from a visible corpus) and per-file compression/decompression. Scored by overall compression ratio; byte-exact reconstruction is mandatory. \\
\tasknoterule
\texttt{packer\_\allowbreak{}plugin\_\allowbreak{}datasources} & Implement six targets in HashiCorp Packer's Go plugin ecosystem: address parsing with DNS validation, filesystem-based plugin discovery with checksum verification, HCL2 data-block integration, sensitive value handling, and new template functions. \\
\tasknoterule
\texttt{pocketbase\_\allowbreak{}backend\_\allowbreak{}architecture} & Implement four architectural targets in PocketBase (Go): a chain-based hook/event system, an HTTP router with middleware stacking, a dynamic OAuth2 provider registry, and a regex-based random string generator. \\
\tasknoterule
\texttt{pocketbase\_\allowbreak{}tools\_\allowbreak{}extensions} & Implement four Go utility packages for PocketBase: a REST JSON serializer with nested field picking, a zip archive utility, a cron scheduler with timezone support, and a SQL index parser/builder. \\
\tasknoterule
\texttt{postgres\_\allowbreak{}wire\_\allowbreak{}on\_\allowbreak{}sqlite} & Implement a server that speaks the PostgreSQL wire protocol (v3) using SQLite as the storage backend, handling authentication, simple and extended query protocols, type mapping, transactions, and system catalog queries. Scored by pass rate on PostgreSQL's own test suites. \\
\tasknoterule
\texttt{quic\_\allowbreak{}transport\_\allowbreak{}stack} & Implement a subset of the QUIC transport protocol (RFC 9000): connection establishment, TLS 1.3 key derivation, AEAD encryption, packet number encoding, header protection, and core frame types. \\
\tasknoterule
\texttt{regex\_\allowbreak{}automata\_\allowbreak{}repair} & Repair broken implementations in Rust's regex-automata crate spanning NFA compilation, DFA transition construction, Unicode handling, capture groups, and the hybrid matching engine. \\
\tasknoterule
\texttt{schemathesis\_\allowbreak{}datagen\_\allowbreak{}pipeline} & Implement eight feature targets in the Schemathesis Python API testing framework, including structured HTTP header generation strategies, coverage-phase hooks, discriminator-aware validation and data generation, and schema-driven code generation fixes. \\
\tasknoterule
\texttt{schemathesis\_\allowbreak{}reporting\_\allowbreak{}observability} & Implement five targets in Schemathesis: a post-validation hook system, multi-format test report writers (VCR, HAR, JUnit, NDJSON), pytest plugin integration, and schema-branch-aware example generation. \\
\tasknoterule
\texttt{schemathesis\_\allowbreak{}config\_\allowbreak{}modernization} & Implement six modernization targets in Schemathesis: a TOML-based configuration system with auto-discovery, API namespace reorganization, a metrics framework, transport and response abstractions, and a redesigned check registration system. \\
\tasknoterule
\texttt{stream\_\allowbreak{}processing\_\allowbreak{}engine} & Implement a Rust-based stream processing engine supporting windowed aggregations, filtering, projection, and stateful operators over JSON event streams. Scored on correctness, robustness to malformed input, and throughput. \\
\tasknoterule
\texttt{tls13\_\allowbreak{}handshake\_\allowbreak{}state\_\allowbreak{}machine} & Complete a Python TLS 1.3 protocol state machine that processes handshake message traces, implementing state transitions, key schedule computation, and message validation using a provided crypto API. \\
\tasknoterule
\texttt{vault\_\allowbreak{}sdk\_\allowbreak{}resilience} & Implement five targets in HashiCorp Vault's Go SDK: a string template engine, a fair-share job scheduler, persistent cache storage for agent tokens, certificate utility enhancements, and API client request/response hooks. \\
\tasknoterule
\texttt{vliw\_\allowbreak{}kernel\_\allowbreak{}optimization} & Optimize a VLIW/SIMD kernel generator for correctness and minimum cycle count on a custom architecture simulator. Hard-coded answers for specific inputs are forbidden. \\
\tasknoterule
\texttt{cpu\_\allowbreak{}full\_\allowbreak{}flow} & Work through a full RISC-V CPU design curriculum: ISA emulator implementation, hardware abstraction layer, Verilog processor design, and SoC integration with Verilator simulation. \\
\tasknoterule
\texttt{zstd\_\allowbreak{}api\_\allowbreak{}modernization} & Implement five API evolution targets in Zstandard's C library: stable struct alignment and query functions, a generic parameter-driven compression interface, static allocation support, dictionary enhancements, and memory estimation utilities. \\
\tasknoterule
\texttt{cfzip\_\allowbreak{}compression\_\allowbreak{}engine} & Implement a complete compression engine from scratch in C++17 without external libraries: custom archive format, CLI with streaming mode, dictionary training and use, and integrity verification. Scored on compression ratio, speed, memory usage, and correctness. \\
\BenchMidRule
\multicolumn{2}{@{}l}{\rule{0pt}{1.2em}\textsc{\scriptsize Scientific Problems \& ML}} \\
\BenchMidRule
\texttt{battery\_\allowbreak{}soh\_\allowbreak{}rul\_\allowbreak{}anomaly} & Predict battery state-of-health, remaining useful life, and anomaly type/severity per cycle for unseen cells, given multi-cell degradation training data. Evaluated under distributional shift with regime and calibration changes not present in training. \\
\tasknoterule
\texttt{borden\_\allowbreak{}pump\_\allowbreak{}treat\_\allowbreak{}dispatch} & Build a physics-based groundwater flow model and solve a constrained multi-objective optimization for pump-and-treat remediation on the Borden aquifer, selecting wells, treatment types, and multi-phase schedules under budget and capacity constraints. \\
\tasknoterule
\texttt{borden\_\allowbreak{}source\_\allowbreak{}inversion} & Infer a finite-duration rectangular contaminant source from sparse, noisy monitoring-well observations in a 3D hydrogeological scene. The agent must implement its own forward model and inversion optimizer from scratch. \\
\tasknoterule
\texttt{borden\_\allowbreak{}sensor\_\allowbreak{}fault\_\allowbreak{}diagnosis} & Classify sensor faults (spikes, drift, stuck-at-zero, unit errors, etc.) versus true plume arrivals in groundwater monitoring records using physics-informed checks such as travel-time ordering and neighbor-well consistency. \\
\tasknoterule
\texttt{bridge\_\allowbreak{}gnss\_\allowbreak{}state\_\allowbreak{}forecast} & Process dirty bridge GNSS displacement time series (timestamp jitter, duplicates, spikes, drift) and produce cleaned reconstruction, state estimates, and short-term displacement forecasts. \\
\tasknoterule
\texttt{capecod\_\allowbreak{}plume\_\allowbreak{}reconstruction} & Reconstruct a multi-analyte groundwater plume from sparse monitoring wells: predict concentrations at withheld locations and times, compute plume metrics, and propose an optimal monitoring network under budget constraints. \\
\tasknoterule
\texttt{vsg\_\allowbreak{}stability\_\allowbreak{}parameter\_\allowbreak{}optimization} & Integrate current-limited Virtual Synchronous Generators into the IEEE 39-bus power system, generate transient stability data, build a physics-informed neural network, and train a reinforcement learning agent to optimize VSG parameters. \\
\tasknoterule
\texttt{cylinder\_\allowbreak{}wake\_\allowbreak{}prediction} & Implement a CPU-only 2D cylinder wake solver for the incompressible Navier-Stokes equations. Evaluated on unseen Reynolds numbers and domain configurations; scored on velocity-field accuracy, pressure-field accuracy, and flow-regime prediction. \\
\tasknoterule
\texttt{dabic\_\allowbreak{}gravity\_\allowbreak{}inversion} & Implement the D-ABIC regularization method for 3D gravity inversion within the SimPEG framework, run on both synthetic and real Vinton salt dome data under L0 and L1 sparse norms, and compare against a Hamiltonian Monte Carlo baseline. \\
\tasknoterule
\texttt{noisy\_\allowbreak{}product\_\allowbreak{}matching\_\allowbreak{}pipeline} & Determine whether pairs of product listings from different sources refer to the same real-world item, given noisy and incomplete attributes. Evaluated on hidden variants with different noise characteristics and data scales; no ground-truth labels are available during development. \\
\tasknoterule
\texttt{neural\_\allowbreak{}net\_\allowbreak{}weight\_\allowbreak{}recovery} & Reconstruct the correct layer ordering of a neural network from 97 shuffled weight files and input-output historical data, using dimensional constraints and reconstruction-error analysis. \\
\tasknoterule
\texttt{nanophotonic\_\allowbreak{}simulation\_\allowbreak{}reproduction} & Reproduce published nanophotonic simulation results (multi-source electromagnetic field distributions) from a research paper, implementing the solver from scratch using only NumPy. \\
\tasknoterule
\texttt{ftir\_\allowbreak{}polymer\_\allowbreak{}identification} & Identify polyimide monomers from FTIR spectra by correlating IR absorption peaks with functional groups, optionally aided by quantum chemistry simulations. Submission attempts are rate-limited to prevent brute-force enumeration. \\
\tasknoterule
\texttt{molecular\_\allowbreak{}property\_\allowbreak{}regression} & Predict molecular properties from graph representations without graph neural network libraries. Evaluated on multi-domain data with perturbations that penalize solutions overfitting to the development distribution. \\
\tasknoterule
\texttt{graph\_\allowbreak{}node\_\allowbreak{}classification} & Implement graph neural networks from scratch using only base PyTorch for semi-supervised node classification on an unseen graph under CPU-only constraints. \\
\tasknoterule
\texttt{gravitational\_\allowbreak{}wave\_\allowbreak{}signal\_\allowbreak{}detection} & Reproduce the LIGO GW150914 gravitational-wave detection pipeline: whitening, bandpass filtering, matched filtering against a numerical-relativity template, time-frequency analysis, and residual computation. \\
\tasknoterule
\texttt{herbal\_\allowbreak{}depression\_\allowbreak{}target\_\allowbreak{}screening} & Screen active components and protein targets for four traditional Chinese medicine herbs used in depression treatment, producing component lists, target mappings, disease-target intersections, and a PPI network visualization. \\
\tasknoterule
\texttt{polyimide\_\allowbreak{}homo\_\allowbreak{}lumo\_\allowbreak{}prediction} & Compute HOMO/LUMO energies of polyimide repeating units using DFT calculations, analyze substituent and conjugation effects on electronic properties, and identify the monomer pair with the widest band gap. \\
\tasknoterule
\texttt{industrial\_\allowbreak{}anomaly\_\allowbreak{}detection} & Detect anomalies in multivariate industrial sensor time series using only classical ML libraries. Evaluated on hidden variants including time-series crops, reversals, and normal-only segments to prevent hard-coding. \\
\tasknoterule
\texttt{bipedalwalker\_\allowbreak{}locomotion\_\allowbreak{}rl} & Train a CPU-only locomotion policy for BipedalWalker and its Hardcore variant. The judge evaluates only the submitted policy checkpoint, not the training process. Pre-trained policies and external RL libraries are prohibited. \\
\tasknoterule
\texttt{molecular\_\allowbreak{}solubility\_\allowbreak{}prediction} & Predict aqueous log-solubility from SMILES strings and molecular descriptors, improving upon a provided random forest baseline. Scored by prediction error on hidden test molecules. \\
\tasknoterule
\texttt{motor\_\allowbreak{}clutch\_\allowbreak{}model\_\allowbreak{}reproduction} & Implement a Gillespie/KMC stochastic simulation of the motor-clutch mechanotransduction model from sparse reference traces. Curve-fitting formulas and lookup tables are prohibited; the simulation must exhibit correct stochastic dynamics across unseen parameter combinations. \\
\tasknoterule
\texttt{streaming\_\allowbreak{}multilabel\_\allowbreak{}classification} & Implement streaming multi-label classification from scratch using only NumPy. Evaluated on subset accuracy, Hamming loss, and F1 metrics under CPU-only constraints. \\
\tasknoterule
\texttt{barnes\_\allowbreak{}hut\_\allowbreak{}nbody\_\allowbreak{}acceleration} & Implement a Barnes-Hut $N$-body gravitational simulation in C++17. Scored on both force accuracy relative to direct summation and speedup over a naive baseline across varying particle counts. \\
\tasknoterule
\texttt{blackbox\_\allowbreak{}numerical\_\allowbreak{}integration} & Integrate hidden black-box functions over the 10-dimensional unit hypercube via an oracle interface with bounded query budgets. Scored on accuracy relative to true integral values. \\
\tasknoterule
\texttt{pancreatic\_\allowbreak{}radiotherapy\_\allowbreak{}meta\_\allowbreak{}analysis} & Automate a two-stage systematic review pipeline: screen candidate PDFs and extract evidence, then perform inverse-variance meta-analysis with model selection and subgroup analysis. Evaluated on unseen publications. \\
\tasknoterule
\texttt{monge\_\allowbreak{}ampere\_\allowbreak{}pde\_\allowbreak{}solver} & Implement a numerical solver for the fully nonlinear Monge-Amp\`ere equation $\det(D^2 u) = f(x,y)$. Evaluated on unseen right-hand sides and boundary conditions; scored on solution accuracy and computational efficiency. \\
\tasknoterule
\texttt{pdf\_\allowbreak{}structured\_\allowbreak{}extraction} & Extract structured page blocks (text, tables, formulas, figures) with bounding boxes, reading order, and content from PDFs using only classical computer vision tools---no deep learning models. Evaluated on diverse enterprise document layouts. \\
\tasknoterule
\texttt{ocean\_\allowbreak{}mt\_\allowbreak{}lab\_\allowbreak{}inversion} & Implement 1D marine magnetotelluric forward modeling and Bayesian inversion with lateral coupling across ocean-bottom stations. Independent per-station or point-estimate-only solutions are penalized. \\
\tasknoterule
\texttt{pv\_\allowbreak{}power\_\allowbreak{}forecasting} & Forecast multi-site photovoltaic power generation from historical output and weather features. Evaluated on unseen data across multiple domain-specific metrics including ramp handling, peak accuracy, and multi-horizon performance. \\
\tasknoterule
\texttt{quantum\_\allowbreak{}architecture\_\allowbreak{}search} & Implement noise-adaptive quantum architecture search balancing classification accuracy and molecular ground-state energy estimation against hardware-realistic circuit depth and gate count constraints. \\
\tasknoterule
\texttt{collaborative\_\allowbreak{}filtering\_\allowbreak{}recommender} & Implement a top-K recommender system that jointly optimizes recommendation quality (NDCG) and runtime efficiency, including cold-start users with no training interactions. \\
\tasknoterule
\texttt{touchstone\_\allowbreak{}vna\_\allowbreak{}diagnostics} & Parse Touchstone S-parameter files in various representations (real-imaginary, magnitude-angle, dB-angle), compute derived RF metrics (return loss, VSWR, group delay, impedance), and produce structured diagnostic reports. \\
\tasknoterule
\texttt{ecg\_\allowbreak{}signal\_\allowbreak{}processing\_\allowbreak{}pipeline} & Implement a three-stage ECG processing pipeline (denoising, QRS complex detection, beat classification) from scratch using only NumPy. Evaluated on unseen recordings with different noise profiles and arrhythmia distributions. \\
\tasknoterule
\texttt{sketch\_\allowbreak{}solve\_\allowbreak{}least\_\allowbreak{}squares} & Solve large-scale overdetermined least-squares problems by choosing among direct solvers, iterative methods, and randomized sketching based on matrix properties. Scored on solution accuracy and speed across varied problem structures. \\
\tasknoterule
\texttt{substrate\_\allowbreak{}interface\_\allowbreak{}simulation} & Implement a coupled interface response simulation with correct parameter dependence. Hard-coded outputs are prohibited; evaluated on physical consistency, energy conservation, and statistical properties of stochastic trajectories. \\
\tasknoterule
\texttt{thermo\_\allowbreak{}fluid\_\allowbreak{}field\_\allowbreak{}prediction} & Predict 2D velocity and temperature fields for thermo-fluid coupling problems under varying boundary conditions and dimensionless parameters, using only NumPy on CPU. \\
\tasknoterule
\texttt{csi\_\allowbreak{}time\_\allowbreak{}series\_\allowbreak{}forecasting} & Forecast future wireless channel state information tensors from historical observations. Evaluated on unseen channel conditions and mobility scenarios under CPU-only constraints. \\
\tasknoterule
\texttt{roof\_\allowbreak{}damage\_\allowbreak{}active\_\allowbreak{}learning} & Design an active learning pipeline for satellite-imagery roof damage detection: start from a small labeled seed set, strategically query an oracle for additional labels under a fixed budget, and train an object detector for evaluation on unlabeled images. \\
\BenchMidRule
\multicolumn{2}{@{}l}{\rule{0pt}{1.2em}\textsc{\scriptsize Combinatorial Optimization}} \\
\BenchMidRule
\texttt{ad\_\allowbreak{}placement\_\allowbreak{}optimization} & Partition a large integer grid into non-overlapping rectangles, each containing a designated anchor point, maximizing total satisfaction from how closely each rectangle's area matches its target. \\
\tasknoterule
\texttt{treant\_\allowbreak{}forest} & Strategically place obstacles in a grid maze to maximize the shortest-path length between start and goal, or block the path entirely. \\
\tasknoterule
\texttt{grid\_\allowbreak{}turing\_\allowbreak{}robot} & Design transition rules and initial coloring for a Turing-like robot on a colored grid to maximize the number of distinct cells visited while minimizing the rule set size. \\
\tasknoterule
\texttt{molecular\_\allowbreak{}self\_\allowbreak{}assembly} & Schedule bonding operations over discrete time steps to assemble atoms into a specified number of connected molecules, respecting spatial proximity and temporal ordering constraints. \\
\tasknoterule
\texttt{apple\_\allowbreak{}incremental\_\allowbreak{}game} & Decide each turn whether to invest in machines or collect output in an incremental production game, balancing short-term gains against long-horizon compounding. \\
\tasknoterule
\texttt{first\_\allowbreak{}order\_\allowbreak{}theorem\_\allowbreak{}prover} & Build a first-order automated theorem prover from scratch (parsing, clausification, unification, saturation) that produces verified proof or model witnesses. External provers and benchmark fingerprinting are forbidden. \\
\tasknoterule
\texttt{circuit\_\allowbreak{}layout\_\allowbreak{}optimization} & Implement a VLSI standard-cell placement solver minimizing half-perimeter wire length on industry-standard benchmarks. Scored on wire-length quality and runtime. \\
\tasknoterule
\texttt{order\_\allowbreak{}addition\_\allowbreak{}permutation\_\allowbreak{}optimization} & Find a permutation of 1,000 elements that minimizes a black-box cost function, using metaheuristic search (simulated annealing, genetic algorithms, local search) without access to the cost function's internals. \\
\tasknoterule
\texttt{equivalence\_\allowbreak{}class\_\allowbreak{}divide\_\allowbreak{}and\_\allowbreak{}conquer} & Solve six progressive competitive-programming problems centered on equivalence classes and divide-and-conquer, where techniques from earlier problems inform solutions to harder ones. \\
\tasknoterule
\texttt{jagua\_\allowbreak{}nesting\_\allowbreak{}optimization} & Improve a Rust-based 2D irregular bin packing optimizer for non-convex polygonal pieces. Solution geometry is independently verified; improvements below a minimum threshold receive no credit. \\
\tasknoterule
\texttt{sat\_\allowbreak{}solver} & Build a SAT solver from scratch implementing conflict-driven clause learning, watched literals, and restart strategies. External solvers and benchmark-aware heuristics are forbidden; scoring is balanced across difficulty tiers. \\
\tasknoterule
\texttt{smt\_\allowbreak{}solver} & Build an SMT solver from scratch for four quantifier-free theories (uninterpreted functions, linear real and integer arithmetic, and their combination). External SMT solvers are forbidden; model witnesses are independently validated. \\
\tasknoterule
\texttt{symbolic\_\allowbreak{}integration\_\allowbreak{}engine} & Extend a starter symbolic integration engine (with its own parser, simplifier, and differentiator) to handle a broader class of integrands. All computer algebra systems and numerical libraries are forbidden; correctness is verified by differentiating the returned antiderivative. \\
\tasknoterule
\texttt{tree\_\allowbreak{}block\_\allowbreak{}partitioning} & Solve six progressive problems on tree decomposition and block partitioning, where algorithmic ideas discovered in simpler variants transfer to harder ones. \\
\tasknoterule
\texttt{triangulation\_\allowbreak{}coloring\_\allowbreak{}optimization} & Minimize a cost function over a triangulation by jointly recoloring vertices and flipping edges, where the dominant term is a quadratic penalty on monochromatic (``ugly'') triangles. \\
\tasknoterule
\texttt{vibrating\_\allowbreak{}path\_\allowbreak{}graph\_\allowbreak{}coloring} & Color graph vertices and selectively remove edges to minimize a cost that penalizes both removed edges and monochromatic surviving edges. \\
\tasknoterule
\texttt{vehicle\_\allowbreak{}routing\_\allowbreak{}time\_\allowbreak{}windows} & Implement a capacitated vehicle routing solver with time windows for Solomon-style benchmarks. Scored against best-known solutions on vehicle count and total travel distance. \\
\tasknoterule
\texttt{warehouse\_\allowbreak{}forklift\_\allowbreak{}routing} & Route a forklift in a grid warehouse to receive goods arriving in random order, store them, and dispatch them in sequential order, minimizing total movement. \\
\tasknoterule
\texttt{wireless\_\allowbreak{}electricity\_\allowbreak{}layout} & Position wire segments on a 2D plane to deliver wireless electricity from two fixed source plates to thousands of cities, minimizing a quadratic cost over city-to-wire distances and wire displacements while avoiding short circuits. \\
\BenchMidRule
\multicolumn{2}{@{}l}{\rule{0pt}{1.2em}\textsc{\scriptsize Formal Math \& Theorem Proving}} \\
\BenchMidRule
\texttt{lean\_\allowbreak{}analysis\_\allowbreak{}proofs} & Complete proof obligations across a multi-file Lean 4 project formalizing results in real and functional analysis. Proofs are checked transitively: a theorem counts only if its entire dependency chain is fully proved. \\
\tasknoterule
\texttt{carleson\_\allowbreak{}formalization} & Fill proof obligations in the Lean 4 formalization of Carleson's theorem on pointwise convergence of $L^2$ Fourier series. Transitive axiom checking ensures no dependence on unproved prerequisites. \\
\tasknoterule
\texttt{combinatorial\_\allowbreak{}games\_\allowbreak{}formalization} & Resolve proof obligations in a Lean 4 formalization of combinatorial game theory, covering surreal numbers, game arithmetic, and the Sprague-Grundy theorem. \\
\tasknoterule
\texttt{new\_\allowbreak{}foundations\_\allowbreak{}consistency} & Complete proof obligations in the ConNF Lean 4 project formalizing the consistency of Quine's New Foundations, involving permutation models and tangled type theory. \\
\tasknoterule
\texttt{cup\_\allowbreak{}product\_\allowbreak{}formalization} & Fill proof obligations in a Lean 4 formalization of the cup product in singular cohomology: cochain-level multiplication, the Leibniz rule, and the induced ring structure on cohomology groups. \\
\tasknoterule
\texttt{erdos392\_\allowbreak{}formalization} & Complete proof obligations for Erd\H{o}s Problem 392 (asymptotic prime distribution) within a Lean 4 analytic number theory project. Weighted scoring reflects relative difficulty of each proof. \\
\tasknoterule
\texttt{flt\_\allowbreak{}regular\_\allowbreak{}formalization} & Resolve proof obligations in a Lean 4 formalization of Fermat's Last Theorem for regular primes via Kummer's cyclotomic theory. Top-level results earn no credit unless foundational prerequisites are also fully proved. \\
\tasknoterule
\texttt{godel\_\allowbreak{}incompleteness\_\allowbreak{}formalization} & Complete proof obligations in a Lean 4 formalization of G\"{o}del's First Incompleteness Theorem, spanning G\"{o}del numbering, the fixed-point lemma, and the self-referential undecidable sentence. \\
\tasknoterule
\texttt{medium\_\allowbreak{}prime\_\allowbreak{}number\_\allowbreak{}theorem} & Complete proof obligations for the Prime Number Theorem with an explicit error term, requiring complex-analytic techniques including contour integration and zero-free regions of the Riemann zeta function. \\
\tasknoterule
\texttt{ordinal\_\allowbreak{}notation\_\allowbreak{}well\_\allowbreak{}foundedness} & Construct well-foundedness proofs for ordinal notation systems in Coq, involving Cantor Normal Form and ordinal arithmetic. \\
\tasknoterule
\texttt{pfr\_\allowbreak{}formalization} & Resolve proof obligations in the Lean 4 formalization of the Polynomial Freiman--Ruzsa conjecture (Gowers--Green--Manners--Tao 2023), involving Shannon entropy, Ruzsa distance, and subgroup covering arguments. \\
\tasknoterule
\texttt{sphere\_\allowbreak{}eversion\_\allowbreak{}formalization} & Complete proof obligations in a Lean 4 formalization of sphere eversion, spanning smooth immersions, jet bundles, ample differential relations, and convex integration. \\
\tasknoterule
\texttt{turing\_\allowbreak{}machine\_\allowbreak{}halting\_\allowbreak{}proofs} & Prove halting behavior for specific 6-state, 2-symbol Turing machines in Coq, contributing to the Busy Beaver frontier. Proofs must be mechanically verified with no admitted axioms. \\
\BenchMidRule
\multicolumn{2}{@{}l}{\rule{0pt}{1.2em}\textsc{\scriptsize Professional Knowledge Work}} \\
\BenchMidRule
\texttt{portfolio\_\allowbreak{}risk\_\allowbreak{}calibration} & Implement a multi-module portfolio management system (risk calibration, constrained optimization, execution cost modeling, dynamic rebalancing) for a cross-asset ETF portfolio. Evaluated out-of-sample on risk-adjusted return metrics. \\
\tasknoterule
\texttt{storyboard\_\allowbreak{}ad\_\allowbreak{}copywriting} & Produce a promotional video script with dual variants and a shot-by-shot storyboard for a state-owned enterprise exhibition appearance, adhering to strict political communication standards and advertising compliance. \\
\tasknoterule
\texttt{brand\_\allowbreak{}annual\_\allowbreak{}planning\_\allowbreak{}ppt} & Produce a comprehensive annual brand management plan as a presentation deck and companion data tables for a mid-size IPTV company, covering strategy, monthly action roadmaps, budgets with market rate benchmarks, and competitive analysis. \\
\tasknoterule
\texttt{securities\_\allowbreak{}protection\_\allowbreak{}training} & Produce a complete set of training deliverables for a securities investor protection seminar: legal framework overview, core systems analysis, case studies, comparative study, policy analysis, presentation deck, lecture script, and bibliography. \\
\tasknoterule
\texttt{college\_\allowbreak{}english\_\allowbreak{}exam\_\allowbreak{}bank} & Produce five parallel examination papers with answer keys for a college English course, plus a blueprint table and an overlap self-check matrix ensuring cross-paper diversity meets pedagogical thresholds. \\
\tasknoterule
\texttt{cross\_\allowbreak{}border\_\allowbreak{}commission\_\allowbreak{}compliance} & Produce a multi-module legal compliance report for a cross-border commission payment, covering multi-jurisdictional anti-corruption risk (FCPA, UK Bribery Act, Chinese Criminal Law), export control, transfer pricing, and evidence preservation strategy. \\
\tasknoterule
\texttt{cross\_\allowbreak{}border\_\allowbreak{}investment\_\allowbreak{}ppt} & Produce a presentation covering 15 months of global cross-border investment policy developments with chronological policy inventory, original data charts, multi-sector industry analysis with deal case studies, and forward-looking risk assessment. \\
\tasknoterule
\texttt{cta\_\allowbreak{}risk\_\allowbreak{}budget\_\allowbreak{}optimization} & Build a complete CTA multi-strategy futures trading system: multiple signal classes, dynamic risk budgeting, a multi-currency backtest engine with transaction costs, drawdown control, and performance attribution. \\
\tasknoterule
\texttt{equity\_\allowbreak{}objection\_\allowbreak{}report} & Produce a court-submission-ready legal research report analyzing whether beneficial owners under equity proxy-holding arrangements can defeat forced execution in Chinese civil enforcement law, citing specific judicial guidance and case precedents. \\
\tasknoterule
\texttt{expo\_\allowbreak{}visitor\_\allowbreak{}conversion\_\allowbreak{}model} & Clean tens of thousands of messy visitor registration records (entity resolution, address parsing, company-name merging), tag professional visitors, and build a calibrated exhibitor conversion scoring model. \\
\tasknoterule
\texttt{factor\_\allowbreak{}stock\_\allowbreak{}model\_\allowbreak{}optimization} & Implement an adaptive multi-factor stock selection model with IC-based factor selection, dynamic weighting, industry/market-cap neutralization, and constrained portfolio construction. Evaluated out-of-sample on information ratio, excess return, drawdown, and turnover. \\
\tasknoterule
\texttt{global\_\allowbreak{}terrorism\_\allowbreak{}atlas\_\allowbreak{}report} & Transform a 200K+ record terrorism database into a single-PDF infographic atlas with maps, ranking charts, shaped word clouds, a sunburst diagram, and country-level fatality trends. \\
\tasknoterule
\texttt{hebei\_\allowbreak{}gaokao\_\allowbreak{}strategy\_\allowbreak{}report} & Produce a multi-tier college admission preference plan with alternative strategies, given complex personal constraints (banned fields, health disqualifications, career priorities) and three years of historical admission data. \\
\tasknoterule
\texttt{hk\_\allowbreak{}connect\_\allowbreak{}annual\_\allowbreak{}metrics} & Programmatically retrieve and populate financial metrics for 255 Hong Kong Stock Connect-eligible securities from annual reports and market data providers. Scored on per-cell accuracy. \\
\tasknoterule
\texttt{k12\_\allowbreak{}math\_\allowbreak{}recommendation} & Build a knowledge-tracing and question-recommendation system from hundreds of thousands of student interaction records, evaluated on prediction accuracy, mastery calibration, learning gain, and pedagogical constraint satisfaction. \\
\tasknoterule
\texttt{property\_\allowbreak{}actuarial\_\allowbreak{}pricing} & Build an actuarial pricing model for SME property insurance: loss modeling, risk grading, reinsurance cost allocation, and premium smoothing with renewal constraints. Evaluated on predictive accuracy and premium adequacy. \\
\tasknoterule
\texttt{real\_\allowbreak{}estate\_\allowbreak{}bid\_\allowbreak{}estimate} & Conduct feasibility analysis for a supertall office land tender: market research, quarterly cash flow models under two bidding structures, sensitivity analysis, and a reasoned investment recommendation. \\
\tasknoterule
\texttt{stock\_\allowbreak{}momentum\_\allowbreak{}backtest} & Fetch live A-share market data, compute risk-adjusted momentum scores, apply multi-factor screening filters, and calculate a market-cap-weighted portfolio return for a specific holding period. Scored on numerical accuracy of each computation step. \\
\tasknoterule
\texttt{storm\_\allowbreak{}claim\_\allowbreak{}ring\_\allowbreak{}audit} & Build a fraud detection pipeline for post-disaster insurance claims: produce hold/release/escalate decisions and fraud-ring cluster assignments from interconnected claim, payment, survey, and media datasets, distinguishing genuine disaster patterns from coordinated fraud. \\
\BenchMidRule
\multicolumn{2}{@{}l}{\rule{0pt}{1.2em}\textsc{\scriptsize Interactive Games \& Simulators}} \\
\BenchMidRule
\texttt{dcss\_\allowbreak{}dungeon\_\allowbreak{}ai} & Write a Lua bot for Dungeon Crawl Stone Soup that autonomously explores, fights, and descends dungeon levels as a Minotaur Berserker under a wall-clock time budget. Scored by mean in-game score across multiple runs. \\
\tasknoterule
\texttt{anchorhead\_\allowbreak{}text\_\allowbreak{}adventure} & Play the Lovecraftian interactive fiction game \emph{Anchorhead} via an HTTP API, sending text commands and receiving prose observations. Scored by peak in-game score, reflecting progression through the multi-day narrative and puzzle chain. \\
\tasknoterule
\texttt{trinity\_\allowbreak{}text\_\allowbreak{}adventure} & Play Infocom's \emph{Trinity} via an HTTP game API. The game requires precise object manipulation and understanding of symbolic and temporal clues across interconnected zones. \\
\tasknoterule
\texttt{tryst\_\allowbreak{}text\_\allowbreak{}adventure} & Play \emph{Tryst of Fate} via an HTTP game API. The branching narrative with irreversible choices requires strategic exploration to reach high-scoring endings. \\
\tasknoterule
\texttt{nethack\_\allowbreak{}dungeon\_\allowbreak{}agent} & Implement a decision policy for NetHack via the NLE harness, parsing ASCII map observations and stat vectors to navigate, fight, and survive across multiple procedurally generated runs. \\
\tasknoterule
\texttt{openrct2\_\allowbreak{}theme\_\allowbreak{}park\_\allowbreak{}ai} & Write a JavaScript plugin for OpenRCT2 that autonomously builds rides, hires staff, sets pricing, and grows park company value across multiple scenarios of increasing complexity. \\
\tasknoterule
\texttt{openttd\_\allowbreak{}transport\_\allowbreak{}ai} & Write an AI script for OpenTTD that builds road, rail, and air transport networks to connect towns and industries and grow company value across diverse procedurally generated maps. \\
\tasknoterule
\texttt{wesnoth\_\allowbreak{}tactical\_\allowbreak{}ai} & Write tactical AI logic for Battle for Wesnoth that defeats the built-in AI through custom recruitment, focus-fire targeting, terrain exploitation, and village-capture timing across multiple maps. \\
\end{longtable}
\endgroup

%% file: tables/category_score_tables.tex
\providecommand{\NA}{\textemdash}
\providecommand{\best}[1]{}
\renewcommand{\best}[1]{\begingroup\bfseries\boldmath #1\endgroup}
\providecommand{\second}[1]{}
\renewcommand{\second}[1]{\underline{#1}}
\providecommand{\sci}[2]{\ensuremath{#1\times 10^{#2}}}
\providecommand{\partialrun}{\textsuperscript{*}}
\providecommand{\nostd}[1]{\multicolumn{2}{c}{#1}}
\begin{table}[p]
\centering
\footnotesize
\renewcommand{\arraystretch}{1.08}
\setlength{\tabcolsep}{2.5pt}
\resizebox{\ifdim\width>\linewidth\linewidth\else\width\fi}{!}{%
\begin{tabular}{@{}>{\raggedright\arraybackslash}p{0.30\linewidth}*{5}{r@{\,}>{\scriptsize}l}@{}}
\BenchTopRule
\textbf{Systems \& Software Engineering Tasks} & \multicolumn{2}{c}{\textbf{Opus 4.8}} & \multicolumn{2}{c}{\textbf{GPT-5.5}} & \multicolumn{2}{c}{\textbf{GPT-5.4}} & \multicolumn{2}{c}{\textbf{GLM-5.1}} & \multicolumn{2}{c}{\textbf{DS-V4-Pro}} \\
\BenchMidRule
\texttt{codeflash\_\allowbreak{}repair\_\allowbreak{}performance} & \multicolumn{2}{c}{\second{\ensuremath{\mathrm{57.2}}\,{\scriptsize \ensuremath{\pm 21.9}}}} & \multicolumn{2}{c}{\best{\ensuremath{\mathrm{100.0}}\partialrun\,{\scriptsize \ensuremath{\pm 0.0}}}} & \nostd{\NA} & \nostd{\NA} & \ensuremath{\mathrm{30.5}} & \ensuremath{\pm 17.9} \\
\texttt{ffmpeg\_\allowbreak{}swscale\_\allowbreak{}reimplementation} & \multicolumn{2}{c}{\best{\ensuremath{\mathrm{21.1}}\,{\scriptsize \ensuremath{\pm 7.9}}}} & \multicolumn{2}{c}{\second{\ensuremath{\mathrm{15.3}}\,{\scriptsize \ensuremath{\pm 2.8}}}} & \ensuremath{\mathrm{13.9}} & \ensuremath{\pm 3.1} & \ensuremath{\mathrm{2.2}} & \ensuremath{\pm 3.9} & \ensuremath{\mathrm{3.8}} & \ensuremath{\pm 5.9} \\
\texttt{git\_\allowbreak{}rewrite\_\allowbreak{}in\_\allowbreak{}zig} & \multicolumn{2}{c}{\second{\ensuremath{\mathrm{23.1}}\,{\scriptsize \ensuremath{\pm 2.2}}}} & \ensuremath{\mathrm{18.4}} & \ensuremath{\pm 0.4} & \ensuremath{\mathrm{15.4}} & \ensuremath{\pm 2.4} & \multicolumn{2}{c}{\best{\ensuremath{\mathrm{23.5}}\,{\scriptsize \ensuremath{\pm 1.9}}}} & \ensuremath{\mathrm{17.9}} & \ensuremath{\pm 2.1} \\
\texttt{arc\_\allowbreak{}compiler\_\allowbreak{}runtime} & \multicolumn{2}{c}{\second{\ensuremath{\mathrm{52.0}}\,{\scriptsize \ensuremath{\pm 0.1}}}} & \multicolumn{2}{c}{\best{\ensuremath{\mathrm{72.4}}\,{\scriptsize \ensuremath{\pm 15.1}}}} & \ensuremath{\mathrm{50.0}} & \ensuremath{\pm 0.7} & \ensuremath{\mathrm{48.7}} & \ensuremath{\pm 3.5} & \ensuremath{\mathrm{44.2}} & \ensuremath{\pm 4.7} \\
\texttt{pdf\_\allowbreak{}structured\_\allowbreak{}extraction} & \multicolumn{2}{c}{\second{\ensuremath{\mathrm{36.5}}\,{\scriptsize \ensuremath{\pm 4.5}}}} & \ensuremath{\mathrm{29.8}} & \ensuremath{\pm 2.1} & \multicolumn{2}{c}{\best{\ensuremath{\mathrm{36.9}}\partialrun\,{\scriptsize \ensuremath{\pm 1.7}}}} & \ensuremath{\mathrm{26.3}} & \ensuremath{\pm 2.9} & \ensuremath{\mathrm{9.5}} & \ensuremath{\pm 2.3} \\
\texttt{ann\_\allowbreak{}vector\_\allowbreak{}search\_\allowbreak{}qps} & \multicolumn{2}{c}{\best{\ensuremath{\mathrm{59.7}}\,{\scriptsize \ensuremath{\pm 2.1}}}} & \ensuremath{\mathrm{40.7}} & \ensuremath{\pm 18.8} & \multicolumn{2}{c}{\second{\ensuremath{\mathrm{50.2}}\,{\scriptsize \ensuremath{\pm 17.4}}}} & \ensuremath{\mathrm{38.3}} & \ensuremath{\pm 17.5} & \ensuremath{\mathrm{23.8}} & \ensuremath{\pm 3.0} \\
\texttt{rust\_\allowbreak{}multicrate\_\allowbreak{}reconstruction} & \nostd{\NA} & \multicolumn{2}{c}{\best{\ensuremath{\mathrm{57.8}}\,{\scriptsize \ensuremath{\pm 16.8}}}} & \ensuremath{\mathrm{21.4}} & \ensuremath{\pm 3.0} & \multicolumn{2}{c}{\second{\ensuremath{\mathrm{38.5}}\,{\scriptsize \ensuremath{\pm 21.4}}}} & \ensuremath{\mathrm{23.6}} & \ensuremath{\pm 2.4} \\
\texttt{copier\_\allowbreak{}modular\_\allowbreak{}refactor} & \multicolumn{2}{c}{\best{\ensuremath{\mathrm{98.9}}\,{\scriptsize \ensuremath{\pm 0.0}}}} & \ensuremath{\mathrm{97.8}} & \ensuremath{\pm 1.2} & \multicolumn{2}{c}{\second{\ensuremath{\mathrm{98.1}}\,{\scriptsize \ensuremath{\pm 0.6}}}} & \ensuremath{\mathrm{98.0}} & \ensuremath{\pm 0.9} & \ensuremath{\mathrm{87.2}} & \ensuremath{\pm 9.4} \\
\texttt{dependent\_\allowbreak{}type\_\allowbreak{}checker} & \multicolumn{2}{c}{\best{\ensuremath{\mathrm{44.7}}\,{\scriptsize \ensuremath{\pm 3.7}}}} & \multicolumn{2}{c}{\second{\ensuremath{\mathrm{24.7}}\,{\scriptsize \ensuremath{\pm 21.4}}}} & \ensuremath{\mathrm{3.8}} & \ensuremath{\pm 6.7} & \ensuremath{\mathrm{1.9}} & \ensuremath{\pm 3.2} & \ensuremath{\mathrm{0.0}} & \ensuremath{\pm 0.0} \\
\texttt{entt\_\allowbreak{}graph\_\allowbreak{}module} & \multicolumn{2}{c}{\best{\ensuremath{\mathrm{100.0}}\,{\scriptsize \ensuremath{\pm 0.0}}}} & \multicolumn{2}{c}{\second{\ensuremath{\mathrm{94.3}}\,{\scriptsize \ensuremath{\pm 3.0}}}} & \ensuremath{\mathrm{78.4}} & \ensuremath{\pm 25.7} & \multicolumn{2}{c}{\best{\ensuremath{\mathrm{100.0}}\,{\scriptsize \ensuremath{\pm 0.0}}}} & \ensuremath{\mathrm{92.0}} & \ensuremath{\pm 10.5} \\
\texttt{exchange\_\allowbreak{}core\_\allowbreak{}throughput} & \multicolumn{2}{c}{\best{\ensuremath{\mathrm{59.7}}\,{\scriptsize \ensuremath{\pm 2.7}}}} & \multicolumn{2}{c}{\second{\ensuremath{\mathrm{53.2}}\,{\scriptsize \ensuremath{\pm 6.8}}}} & \ensuremath{\mathrm{47.3}} & \ensuremath{\pm 10.2} & \ensuremath{\mathrm{52.6}} & \ensuremath{\pm 12.9} & \ensuremath{\mathrm{48.6}} & \ensuremath{\pm 17.9} \\
\texttt{integer\_\allowbreak{}compression\_\allowbreak{}codec} & \multicolumn{2}{c}{\best{\ensuremath{\mathrm{75.3}}\,{\scriptsize \ensuremath{\pm 0.3}}}} & \multicolumn{2}{c}{\second{\ensuremath{\mathrm{74.4}}\,{\scriptsize \ensuremath{\pm 0.5}}}} & \ensuremath{\mathrm{42.3}} & \ensuremath{\pm 18.4} & \ensuremath{\mathrm{28.9}} & \ensuremath{\pm 4.3} & \ensuremath{\mathrm{16.2}} & \ensuremath{\pm 8.1} \\
\texttt{juliet\_\allowbreak{}vulnerability\_\allowbreak{}analyzer} & \ensuremath{\mathrm{75.6}} & \ensuremath{\pm 6.7} & \multicolumn{2}{c}{\best{\ensuremath{\mathrm{89.8}}\,{\scriptsize \ensuremath{\pm 1.9}}}} & \multicolumn{2}{c}{\second{\ensuremath{\mathrm{77.2}}\,{\scriptsize \ensuremath{\pm 5.5}}}} & \ensuremath{\mathrm{63.5}} & \ensuremath{\pm 4.6} & \ensuremath{\mathrm{66.2}} & \ensuremath{\pm 8.4} \\
\texttt{libexpat\_\allowbreak{}x86\_\allowbreak{}assembly} & \multicolumn{2}{c}{\second{\ensuremath{\mathrm{15.8}}\,{\scriptsize \ensuremath{\pm 14.0}}}} & \ensuremath{\mathrm{13.5}} & \ensuremath{\pm 11.8} & \ensuremath{\mathrm{11.9}} & \ensuremath{\pm 8.4} & \multicolumn{2}{c}{\best{\ensuremath{\mathrm{18.1}}\,{\scriptsize \ensuremath{\pm 24.3}}}} & \ensuremath{\mathrm{11.2}} & \ensuremath{\pm 15.7} \\
\texttt{litestar\_\allowbreak{}infra\_\allowbreak{}refactor} & \nostd{\best{\ensuremath{\mathrm{92.5}}\partialrun}} & \multicolumn{2}{c}{\second{\ensuremath{\mathrm{66.4}}\,{\scriptsize \ensuremath{\pm 15.9}}}} & \ensuremath{\mathrm{60.8}} & \ensuremath{\pm 7.7} & \ensuremath{\mathrm{46.0}} & \ensuremath{\pm 3.7} & \ensuremath{\mathrm{43.6}} & \ensuremath{\pm 3.3} \\
\texttt{lua\_\allowbreak{}native\_\allowbreak{}compiler} & \multicolumn{2}{c}{\best{\ensuremath{\mathrm{98.9}}\partialrun\,{\scriptsize \ensuremath{\pm 0.8}}}} & \ensuremath{\mathrm{78.2}} & \ensuremath{\pm 27.4} & \multicolumn{2}{c}{\second{\ensuremath{\mathrm{90.7}}\,{\scriptsize \ensuremath{\pm 8.6}}}} & \ensuremath{\mathrm{40.9}}\partialrun & \ensuremath{\pm 57.9} & \ensuremath{\mathrm{41.2}}\partialrun & \ensuremath{\pm 7.8} \\
\texttt{mimesis\_\allowbreak{}modular\_\allowbreak{}refactor} & \multicolumn{2}{c}{\best{\ensuremath{\mathrm{100.0}}\,{\scriptsize \ensuremath{\pm 0.0}}}} & \multicolumn{2}{c}{\second{\ensuremath{\mathrm{91.0}}\,{\scriptsize \ensuremath{\pm 2.7}}}} & \ensuremath{\mathrm{87.8}} & \ensuremath{\pm 1.9} & \ensuremath{\mathrm{82.0}} & \ensuremath{\pm 6.9} & \ensuremath{\mathrm{65.4}} & \ensuremath{\pm 12.1} \\
\texttt{nlohmann\_\allowbreak{}json\_\allowbreak{}modularization} & \multicolumn{2}{c}{\best{\ensuremath{\mathrm{100.0}}\,{\scriptsize \ensuremath{\pm 0.0}}}} & \multicolumn{2}{c}{\second{\ensuremath{\mathrm{98.9}}\,{\scriptsize \ensuremath{\pm 0.5}}}} & \ensuremath{\mathrm{87.3}} & \ensuremath{\pm 11.0} & \ensuremath{\mathrm{77.8}} & \ensuremath{\pm 0.0} & \ensuremath{\mathrm{73.3}} & \ensuremath{\pm 0.9} \\
\texttt{notebook\_\allowbreak{}lossless\_\allowbreak{}compression} & \multicolumn{2}{c}{\best{\ensuremath{\mathrm{53.0}}\,{\scriptsize \ensuremath{\pm 2.3}}}} & \ensuremath{\mathrm{19.3}} & \ensuremath{\pm 27.2} & \ensuremath{\mathrm{7.3}} & \ensuremath{\pm 2.0} & \multicolumn{2}{c}{\second{\ensuremath{\mathrm{35.1}}\,{\scriptsize \ensuremath{\pm 24.3}}}} & \ensuremath{\mathrm{16.0}} & \ensuremath{\pm 17.2} \\
\texttt{packer\_\allowbreak{}plugin\_\allowbreak{}datasources} & \ensuremath{\mathrm{90.0}} & \ensuremath{\pm 0.0} & \multicolumn{2}{c}{\second{\ensuremath{\mathrm{90.8}}\partialrun\,{\scriptsize \ensuremath{\pm 1.2}}}} & \ensuremath{\mathrm{33.3}} & \ensuremath{\pm 14.4} & \multicolumn{2}{c}{\best{\ensuremath{\mathrm{91.1}}\,{\scriptsize \ensuremath{\pm 1.0}}}} & \ensuremath{\mathrm{90.6}} & \ensuremath{\pm 1.0} \\
\texttt{pocketbase\_\allowbreak{}tools\_\allowbreak{}extensions} & \multicolumn{2}{c}{\best{\ensuremath{\mathrm{100.0}}\,{\scriptsize \ensuremath{\pm 0.0}}}} & \multicolumn{2}{c}{\second{\ensuremath{\mathrm{94.8}}\,{\scriptsize \ensuremath{\pm 2.6}}}} & \ensuremath{\mathrm{66.5}} & \ensuremath{\pm 15.2} & \ensuremath{\mathrm{57.4}} & \ensuremath{\pm 13.0} & \ensuremath{\mathrm{60.0}} & \ensuremath{\pm 17.3} \\
\texttt{postgres\_\allowbreak{}wire\_\allowbreak{}on\_\allowbreak{}sqlite} & \multicolumn{2}{c}{\best{\ensuremath{\mathrm{8.2}}\,{\scriptsize \ensuremath{\pm 0.3}}}} & \ensuremath{\mathrm{7.7}} & \ensuremath{\pm 0.6} & \ensuremath{\mathrm{7.8}} & \ensuremath{\pm 0.3} & \multicolumn{2}{c}{\second{\ensuremath{\mathrm{8.1}}\partialrun\,{\scriptsize \ensuremath{\pm 0.3}}}} & \ensuremath{\mathrm{7.9}}\partialrun & \ensuremath{\pm 0.2} \\
\texttt{quic\_\allowbreak{}transport\_\allowbreak{}stack} & \ensuremath{\mathrm{48.3}} & \ensuremath{\pm 14.3} & \multicolumn{2}{c}{\best{\ensuremath{\mathrm{63.6}}\,{\scriptsize \ensuremath{\pm 8.9}}}} & \ensuremath{\mathrm{47.1}} & \ensuremath{\pm 16.1} & \multicolumn{2}{c}{\second{\ensuremath{\mathrm{52.5}}\partialrun\,{\scriptsize \ensuremath{\pm 22.4}}}} & \ensuremath{\mathrm{33.4}} & \ensuremath{\pm 10.0} \\
\texttt{regex\_\allowbreak{}automata\_\allowbreak{}repair} & \multicolumn{2}{c}{\second{\ensuremath{\mathrm{66.7}}\partialrun\,{\scriptsize \ensuremath{\pm 0.1}}}} & \multicolumn{2}{c}{\best{\ensuremath{\mathrm{67.0}}\,{\scriptsize \ensuremath{\pm 0.0}}}} & \ensuremath{\mathrm{61.0}} & \ensuremath{\pm 10.3} & \nostd{\ensuremath{\mathrm{28.6}}\partialrun} & \ensuremath{\mathrm{2.3}} & \ensuremath{\pm 2.2} \\
\texttt{schemathesis\_\allowbreak{}datagen\_\allowbreak{}pipeline} & \multicolumn{2}{c}{\best{\ensuremath{\mathrm{70.2}}\,{\scriptsize \ensuremath{\pm 2.7}}}} & \ensuremath{\mathrm{56.7}} & \ensuremath{\pm 3.0} & \ensuremath{\mathrm{56.6}} & \ensuremath{\pm 3.3} & \multicolumn{2}{c}{\second{\ensuremath{\mathrm{67.0}}\,{\scriptsize \ensuremath{\pm 7.0}}}} & \ensuremath{\mathrm{52.3}} & \ensuremath{\pm 5.3} \\
\texttt{schemathesis\_\allowbreak{}reporting\_\allowbreak{}observability} & \multicolumn{2}{c}{\second{\ensuremath{\mathrm{76.2}}\,{\scriptsize \ensuremath{\pm 4.7}}}} & \multicolumn{2}{c}{\best{\ensuremath{\mathrm{77.1}}\,{\scriptsize \ensuremath{\pm 3.5}}}} & \multicolumn{2}{c}{\second{\ensuremath{\mathrm{76.2}}\,{\scriptsize \ensuremath{\pm 2.9}}}} & \ensuremath{\mathrm{61.9}} & \ensuremath{\pm 1.5} & \ensuremath{\mathrm{65.0}} & \ensuremath{\pm 11.7} \\
\texttt{schemathesis\_\allowbreak{}config\_\allowbreak{}modernization} & \multicolumn{2}{c}{\best{\ensuremath{\mathrm{87.7}}\,{\scriptsize \ensuremath{\pm 2.6}}}} & \multicolumn{2}{c}{\second{\ensuremath{\mathrm{84.0}}\,{\scriptsize \ensuremath{\pm 1.5}}}} & \ensuremath{\mathrm{71.9}} & \ensuremath{\pm 1.8} & \ensuremath{\mathrm{61.7}} & \ensuremath{\pm 4.2} & \ensuremath{\mathrm{55.6}} & \ensuremath{\pm 2.9} \\
\texttt{stream\_\allowbreak{}processing\_\allowbreak{}engine} & \multicolumn{2}{c}{\best{\ensuremath{\mathrm{100.0}}\,{\scriptsize \ensuremath{\pm 0.0}}}} & \multicolumn{2}{c}{\best{\ensuremath{\mathrm{100.0}}\,{\scriptsize \ensuremath{\pm 0.0}}}} & \multicolumn{2}{c}{\best{\ensuremath{\mathrm{100.0}}\,{\scriptsize \ensuremath{\pm 0.0}}}} & \multicolumn{2}{c}{\best{\ensuremath{\mathrm{100.0}}\,{\scriptsize \ensuremath{\pm 0.0}}}} & \multicolumn{2}{c}{\best{\ensuremath{\mathrm{100.0}}\,{\scriptsize \ensuremath{\pm 0.0}}}} \\
\texttt{tls13\_\allowbreak{}handshake\_\allowbreak{}state\_\allowbreak{}machine} & \ensuremath{\mathrm{29.4}} & \ensuremath{\pm 0.9} & \multicolumn{2}{c}{\best{\ensuremath{\mathrm{39.3}}\,{\scriptsize \ensuremath{\pm 1.1}}}} & \multicolumn{2}{c}{\second{\ensuremath{\mathrm{37.9}}\,{\scriptsize \ensuremath{\pm 0.9}}}} & \ensuremath{\mathrm{29.2}} & \ensuremath{\pm 1.8} & \ensuremath{\mathrm{29.0}} & \ensuremath{\pm 0.6} \\
\texttt{vault\_\allowbreak{}sdk\_\allowbreak{}resilience} & \multicolumn{2}{c}{\best{\ensuremath{\mathrm{100.0}}\,{\scriptsize \ensuremath{\pm 0.0}}}} & \multicolumn{2}{c}{\second{\ensuremath{\mathrm{29.1}}\,{\scriptsize \ensuremath{\pm 21.0}}}} & \ensuremath{\mathrm{18.2}} & \ensuremath{\pm 15.7} & \ensuremath{\mathrm{13.9}} & \ensuremath{\pm 8.4} & \ensuremath{\mathrm{15.2}} & \ensuremath{\pm 18.9} \\
\texttt{vliw\_\allowbreak{}kernel\_\allowbreak{}optimization} & \multicolumn{2}{c}{\second{\ensuremath{\mathrm{80.9}}\,{\scriptsize \ensuremath{\pm 0.9}}}} & \multicolumn{2}{c}{\best{\ensuremath{\mathrm{85.6}}\,{\scriptsize \ensuremath{\pm 1.9}}}} & \ensuremath{\mathrm{79.1}} & \ensuremath{\pm 1.5} & \ensuremath{\mathrm{35.9}} & \ensuremath{\pm 25.1} & \ensuremath{\mathrm{34.1}} & \ensuremath{\pm 19.1} \\
\texttt{cpu\_\allowbreak{}full\_\allowbreak{}flow} & \ensuremath{\mathrm{72.0}}\partialrun & \ensuremath{\pm 7.1} & \multicolumn{2}{c}{\best{\ensuremath{\mathrm{88.5}}\partialrun\,{\scriptsize \ensuremath{\pm 12.0}}}} & \multicolumn{2}{c}{\second{\ensuremath{\mathrm{85.3}}\,{\scriptsize \ensuremath{\pm 11.7}}}} & \nostd{\ensuremath{\mathrm{51.0}}\partialrun} & \ensuremath{\mathrm{56.3}} & \ensuremath{\pm 4.7} \\
\texttt{zstd\_\allowbreak{}api\_\allowbreak{}modernization} & \multicolumn{2}{c}{\best{\ensuremath{\mathrm{100.0}}\,{\scriptsize \ensuremath{\pm 0.0}}}} & \multicolumn{2}{c}{\second{\ensuremath{\mathrm{95.8}}\,{\scriptsize \ensuremath{\pm 7.2}}}} & \multicolumn{2}{c}{\best{\ensuremath{\mathrm{100.0}}\,{\scriptsize \ensuremath{\pm 0.0}}}} & \multicolumn{2}{c}{\best{\ensuremath{\mathrm{100.0}}\,{\scriptsize \ensuremath{\pm 0.0}}}} & \ensuremath{\mathrm{91.7}} & \ensuremath{\pm 14.4} \\
\texttt{cfzip\_\allowbreak{}compression\_\allowbreak{}engine} & \multicolumn{2}{c}{\best{\ensuremath{\mathrm{98.4}}\,{\scriptsize \ensuremath{\pm 0.0}}}} & \multicolumn{2}{c}{\second{\ensuremath{\mathrm{97.4}}\,{\scriptsize \ensuremath{\pm 1.3}}}} & \ensuremath{\mathrm{96.2}}\partialrun & \ensuremath{\pm 0.7} & \ensuremath{\mathrm{87.4}} & \ensuremath{\pm 12.0} & \ensuremath{\mathrm{92.5}} & \ensuremath{\pm 4.8} \\
\texttt{pocketbase\_\allowbreak{}backend\_\allowbreak{}architecture} & \ensuremath{\mathrm{0.0}} & \ensuremath{\pm 0.0} & \multicolumn{2}{c}{\best{\ensuremath{\mathrm{62.5}}\,{\scriptsize \ensuremath{\pm 0.0}}}} & \ensuremath{\mathrm{20.8}} & \ensuremath{\pm 36.1} & \multicolumn{2}{c}{\second{\ensuremath{\mathrm{61.1}}\,{\scriptsize \ensuremath{\pm 2.4}}}} & \ensuremath{\mathrm{4.2}} & \ensuremath{\pm 7.2} \\
\BenchBottomRule
\end{tabular}%
}
\caption{Model performance on Systems \& Software Engineering tasks. Values are mean scores over up to three valid runs; adjacent entries show $\pm s$ when at least two valid runs are available. Bold marks the best model for each task, underlining marks the second-best model, \NA{} indicates no valid result, and \textsuperscript{*} marks fewer than three valid runs.}
\label{tab:appendix-software_systems}
\end{table}
\begin{table}[p]
\centering
\footnotesize
\renewcommand{\arraystretch}{1.08}
\setlength{\tabcolsep}{2.5pt}
\resizebox{\ifdim\width>\linewidth\linewidth\else\width\fi}{!}{%
\begin{tabular}{@{}>{\raggedright\arraybackslash}p{0.30\linewidth}*{5}{r@{\,}>{\scriptsize}l}@{}}
\BenchTopRule
\textbf{Scientific Problems \& ML Tasks} & \multicolumn{2}{c}{\textbf{Opus 4.8}} & \multicolumn{2}{c}{\textbf{GPT-5.5}} & \multicolumn{2}{c}{\textbf{GPT-5.4}} & \multicolumn{2}{c}{\textbf{GLM-5.1}} & \multicolumn{2}{c}{\textbf{DS-V4-Pro}} \\
\BenchMidRule
\texttt{capecod\_\allowbreak{}plume\_\allowbreak{}reconstruction} & \multicolumn{2}{c}{\best{\ensuremath{\mathrm{19.9}}\,{\scriptsize \ensuremath{\pm 11.1}}}} & \multicolumn{2}{c}{\second{\ensuremath{\mathrm{16.4}}\,{\scriptsize \ensuremath{\pm 5.5}}}} & \ensuremath{\mathrm{12.6}} & \ensuremath{\pm 4.6} & \ensuremath{\mathrm{10.9}} & \ensuremath{\pm 2.5} & \ensuremath{\mathrm{8.8}} & \ensuremath{\pm 0.4} \\
\texttt{battery\_\allowbreak{}soh\_\allowbreak{}rul\_\allowbreak{}anomaly} & \multicolumn{2}{c}{\best{\ensuremath{\mathrm{37.2}}\,{\scriptsize \ensuremath{\pm 10.7}}}} & \multicolumn{2}{c}{\second{\ensuremath{\mathrm{30.2}}\,{\scriptsize \ensuremath{\pm 14.8}}}} & \ensuremath{\mathrm{14.7}} & \ensuremath{\pm 1.8} & \ensuremath{\mathrm{18.2}} & \ensuremath{\pm 4.3} & \ensuremath{\mathrm{14.0}} & \ensuremath{\pm 0.6} \\
\texttt{borden\_\allowbreak{}source\_\allowbreak{}inversion} & \multicolumn{2}{c}{\best{\ensuremath{\mathrm{48.4}}\,{\scriptsize \ensuremath{\pm 7.3}}}} & \multicolumn{2}{c}{\second{\ensuremath{\mathrm{38.5}}\,{\scriptsize \ensuremath{\pm 14.3}}}} & \ensuremath{\mathrm{8.0}} & \ensuremath{\pm 1.4} & \ensuremath{\mathrm{15.1}} & \ensuremath{\pm 13.7} & \ensuremath{\mathrm{38.2}} & \ensuremath{\pm 3.6} \\
\texttt{borden\_\allowbreak{}pump\_\allowbreak{}treat\_\allowbreak{}dispatch} & \multicolumn{2}{c}{\second{\ensuremath{\mathrm{14.9}}\,{\scriptsize \ensuremath{\pm 2.5}}}} & \multicolumn{2}{c}{\best{\ensuremath{\mathrm{16.0}}\,{\scriptsize \ensuremath{\pm 1.6}}}} & \ensuremath{\mathrm{10.7}} & \ensuremath{\pm 0.6} & \ensuremath{\mathrm{14.5}} & \ensuremath{\pm 1.7} & \ensuremath{\mathrm{13.6}} & \ensuremath{\pm 2.1} \\
\texttt{borden\_\allowbreak{}sensor\_\allowbreak{}fault\_\allowbreak{}diagnosis} & \multicolumn{2}{c}{\second{\ensuremath{\mathrm{5.4}}\,{\scriptsize \ensuremath{\pm 2.6}}}} & \multicolumn{2}{c}{\best{\ensuremath{\mathrm{12.2}}\,{\scriptsize \ensuremath{\pm 8.0}}}} & \ensuremath{\mathrm{3.4}} & \ensuremath{\pm 0.2} & \ensuremath{\mathrm{3.4}} & \ensuremath{\pm 0.3} & \ensuremath{\mathrm{3.0}} & \ensuremath{\pm 0.0} \\
\texttt{bridge\_\allowbreak{}gnss\_\allowbreak{}state\_\allowbreak{}forecast} & \multicolumn{2}{c}{\second{\ensuremath{\mathrm{22.0}}\,{\scriptsize \ensuremath{\pm 1.4}}}} & \ensuremath{\mathrm{21.8}} & \ensuremath{\pm 2.3} & \ensuremath{\mathrm{21.0}} & \ensuremath{\pm 1.1} & \multicolumn{2}{c}{\best{\ensuremath{\mathrm{23.7}}\,{\scriptsize \ensuremath{\pm 1.8}}}} & \ensuremath{\mathrm{21.3}} & \ensuremath{\pm 0.5} \\
\texttt{vsg\_\allowbreak{}stability\_\allowbreak{}parameter\_\allowbreak{}optimization} & \multicolumn{2}{c}{\second{\ensuremath{\mathrm{27.9}}\,{\scriptsize \ensuremath{\pm 16.6}}}} & \multicolumn{2}{c}{\best{\ensuremath{\mathrm{47.0}}\,{\scriptsize \ensuremath{\pm 34.0}}}} & \ensuremath{\mathrm{5.8}} & \ensuremath{\pm 2.9} & \ensuremath{\mathrm{4.5}}\partialrun & \ensuremath{\pm 0.8} & \ensuremath{\mathrm{4.4}} & \ensuremath{\pm 0.5} \\
\texttt{cylinder\_\allowbreak{}wake\_\allowbreak{}prediction} & \multicolumn{2}{c}{\second{\ensuremath{\mathrm{66.6}}\,{\scriptsize \ensuremath{\pm 4.3}}}} & \multicolumn{2}{c}{\best{\ensuremath{\mathrm{69.9}}\,{\scriptsize \ensuremath{\pm 4.6}}}} & \ensuremath{\mathrm{39.8}} & \ensuremath{\pm 16.6} & \ensuremath{\mathrm{36.9}}\partialrun & \ensuremath{\pm 36.0} & \ensuremath{\mathrm{24.2}} & \ensuremath{\pm 14.6} \\
\texttt{dabic\_\allowbreak{}gravity\_\allowbreak{}inversion} & \multicolumn{2}{c}{\best{\ensuremath{\mathrm{17.5}}\,{\scriptsize \ensuremath{\pm 3.0}}}} & \multicolumn{2}{c}{\second{\ensuremath{\mathrm{17.3}}\,{\scriptsize \ensuremath{\pm 0.7}}}} & \ensuremath{\mathrm{15.0}}\partialrun & \ensuremath{\pm 0.7} & \ensuremath{\mathrm{17.1}} & \ensuremath{\pm 0.7} & \ensuremath{\mathrm{13.8}}\partialrun & \ensuremath{\pm 2.9} \\
\texttt{noisy\_\allowbreak{}product\_\allowbreak{}matching\_\allowbreak{}pipeline} & \multicolumn{2}{c}{\best{\ensuremath{\mathrm{68.1}}\,{\scriptsize \ensuremath{\pm 6.5}}}} & \multicolumn{2}{c}{\second{\ensuremath{\mathrm{64.3}}\,{\scriptsize \ensuremath{\pm 6.8}}}} & \ensuremath{\mathrm{27.1}} & \ensuremath{\pm 24.7} & \ensuremath{\mathrm{44.7}} & \ensuremath{\pm 3.4} & \ensuremath{\mathrm{54.2}} & \ensuremath{\pm 9.1} \\
\texttt{neural\_\allowbreak{}net\_\allowbreak{}weight\_\allowbreak{}recovery} & \multicolumn{2}{c}{\second{\ensuremath{\mathrm{94.9}}\,{\scriptsize \ensuremath{\pm 8.9}}}} & \nostd{\best{\ensuremath{\mathrm{100.0}}\partialrun}} & \multicolumn{2}{c}{\best{\ensuremath{\mathrm{100.0}}\partialrun\,{\scriptsize \ensuremath{\pm 0.0}}}} & \ensuremath{\mathrm{69.2}} & \ensuremath{\pm 0.0} & \ensuremath{\mathrm{65.4}}\partialrun & \ensuremath{\pm 5.4} \\
\texttt{nanophotonic\_\allowbreak{}simulation\_\allowbreak{}reproduction} & \multicolumn{2}{c}{\best{\ensuremath{\mathrm{64.8}}\,{\scriptsize \ensuremath{\pm 1.0}}}} & \ensuremath{\mathrm{38.2}}\partialrun & \ensuremath{\pm 1.8} & \nostd{\NA} & \multicolumn{2}{c}{\second{\ensuremath{\mathrm{43.1}}\,{\scriptsize \ensuremath{\pm 9.5}}}} & \ensuremath{\mathrm{42.8}} & \ensuremath{\pm 3.5} \\
\texttt{ftir\_\allowbreak{}polymer\_\allowbreak{}identification} & \multicolumn{2}{c}{\best{\ensuremath{\mathrm{45.0}}\,{\scriptsize \ensuremath{\pm 14.0}}}} & \multicolumn{2}{c}{\second{\ensuremath{\mathrm{32.0}}\partialrun\,{\scriptsize \ensuremath{\pm 5.7}}}} & \ensuremath{\mathrm{19.3}} & \ensuremath{\pm 16.4} & \ensuremath{\mathrm{12.3}} & \ensuremath{\pm 4.0} & \ensuremath{\mathrm{12.3}} & \ensuremath{\pm 7.1} \\
\texttt{molecular\_\allowbreak{}property\_\allowbreak{}regression} & \multicolumn{2}{c}{\best{\ensuremath{\mathrm{49.5}}\,{\scriptsize \ensuremath{\pm 16.9}}}} & \multicolumn{2}{c}{\second{\ensuremath{\mathrm{43.2}}\,{\scriptsize \ensuremath{\pm 5.9}}}} & \ensuremath{\mathrm{23.3}} & \ensuremath{\pm 0.8} & \ensuremath{\mathrm{24.2}} & \ensuremath{\pm 3.7} & \ensuremath{\mathrm{27.1}} & \ensuremath{\pm 7.2} \\
\texttt{graph\_\allowbreak{}node\_\allowbreak{}classification} & \multicolumn{2}{c}{\best{\ensuremath{\mathrm{66.6}}\,{\scriptsize \ensuremath{\pm 3.0}}}} & \ensuremath{\mathrm{56.0}} & \ensuremath{\pm 15.8} & \multicolumn{2}{c}{\second{\ensuremath{\mathrm{57.6}}\,{\scriptsize \ensuremath{\pm 2.0}}}} & \ensuremath{\mathrm{52.3}} & \ensuremath{\pm 8.7} & \ensuremath{\mathrm{51.8}} & \ensuremath{\pm 8.1} \\
\texttt{gravitational\_\allowbreak{}wave\_\allowbreak{}signal\_\allowbreak{}detection} & \multicolumn{2}{c}{\second{\ensuremath{\mathrm{61.5}}\,{\scriptsize \ensuremath{\pm 8.1}}}} & \multicolumn{2}{c}{\best{\ensuremath{\mathrm{64.5}}\,{\scriptsize \ensuremath{\pm 2.2}}}} & \ensuremath{\mathrm{50.0}} & \ensuremath{\pm 3.7} & \ensuremath{\mathrm{44.8}} & \ensuremath{\pm 15.7} & \ensuremath{\mathrm{58.8}} & \ensuremath{\pm 2.9} \\
\texttt{polyimide\_\allowbreak{}homo\_\allowbreak{}lumo\_\allowbreak{}prediction} & \multicolumn{2}{c}{\second{\ensuremath{\mathrm{86.7}}\,{\scriptsize \ensuremath{\pm 23.1}}}} & \multicolumn{2}{c}{\best{\ensuremath{\mathrm{100.0}}\partialrun\,{\scriptsize \ensuremath{\pm 0.0}}}} & \multicolumn{2}{c}{\best{\ensuremath{\mathrm{100.0}}\partialrun\,{\scriptsize \ensuremath{\pm 0.0}}}} & \ensuremath{\mathrm{60.0}}\partialrun & \ensuremath{\pm 0.0} & \ensuremath{\mathrm{80.0}}\partialrun & \ensuremath{\pm 28.3} \\
\texttt{industrial\_\allowbreak{}anomaly\_\allowbreak{}detection} & \multicolumn{2}{c}{\best{\ensuremath{\mathrm{49.3}}\,{\scriptsize \ensuremath{\pm 5.2}}}} & \multicolumn{2}{c}{\second{\ensuremath{\mathrm{40.8}}\,{\scriptsize \ensuremath{\pm 1.7}}}} & \ensuremath{\mathrm{20.4}} & \ensuremath{\pm 19.6} & \ensuremath{\mathrm{34.4}} & \ensuremath{\pm 4.6} & \ensuremath{\mathrm{35.1}} & \ensuremath{\pm 5.9} \\
\texttt{bipedalwalker\_\allowbreak{}locomotion\_\allowbreak{}rl} & \multicolumn{2}{c}{\best{\ensuremath{\mathrm{23.3}}\,{\scriptsize \ensuremath{\pm 3.8}}}} & \ensuremath{\mathrm{21.0}} & \ensuremath{\pm 8.5} & \ensuremath{\mathrm{17.5}} & \ensuremath{\pm 1.2} & \multicolumn{2}{c}{\second{\ensuremath{\mathrm{22.5}}\,{\scriptsize \ensuremath{\pm 2.1}}}} & \ensuremath{\mathrm{20.6}} & \ensuremath{\pm 4.3} \\
\texttt{molecular\_\allowbreak{}solubility\_\allowbreak{}prediction} & \multicolumn{2}{c}{\second{\ensuremath{\mathrm{37.3}}\,{\scriptsize \ensuremath{\pm 1.9}}}} & \multicolumn{2}{c}{\best{\ensuremath{\mathrm{51.7}}\,{\scriptsize \ensuremath{\pm 9.9}}}} & \ensuremath{\mathrm{35.6}} & \ensuremath{\pm 5.0} & \ensuremath{\mathrm{36.8}} & \ensuremath{\pm 2.3} & \ensuremath{\mathrm{33.0}} & \ensuremath{\pm 4.0} \\
\texttt{motor\_\allowbreak{}clutch\_\allowbreak{}model\_\allowbreak{}reproduction} & \multicolumn{2}{c}{\best{\ensuremath{\mathrm{100.0}}\,{\scriptsize \ensuremath{\pm 0.0}}}} & \multicolumn{2}{c}{\second{\ensuremath{\mathrm{21.7}}\,{\scriptsize \ensuremath{\pm 37.5}}}} & \nostd{\NA} & \ensuremath{\mathrm{20.0}}\partialrun & \ensuremath{\pm 28.3} & \ensuremath{\mathrm{20.0}} & \ensuremath{\pm 34.6} \\
\texttt{streaming\_\allowbreak{}multilabel\_\allowbreak{}classification} & \multicolumn{2}{c}{\second{\ensuremath{\mathrm{67.2}}\,{\scriptsize \ensuremath{\pm 2.0}}}} & \ensuremath{\mathrm{63.1}} & \ensuremath{\pm 3.1} & \multicolumn{2}{c}{\best{\ensuremath{\mathrm{68.5}}\,{\scriptsize \ensuremath{\pm 5.2}}}} & \ensuremath{\mathrm{52.9}} & \ensuremath{\pm 6.2} & \ensuremath{\mathrm{60.4}} & \ensuremath{\pm 6.5} \\
\texttt{barnes\_\allowbreak{}hut\_\allowbreak{}nbody\_\allowbreak{}acceleration} & \ensuremath{\mathrm{78.4}} & \ensuremath{\pm 6.2} & \multicolumn{2}{c}{\best{\ensuremath{\mathrm{91.6}}\,{\scriptsize \ensuremath{\pm 14.5}}}} & \multicolumn{2}{c}{\second{\ensuremath{\mathrm{87.4}}\,{\scriptsize \ensuremath{\pm 15.4}}}} & \ensuremath{\mathrm{64.1}} & \ensuremath{\pm 8.5} & \ensuremath{\mathrm{67.1}} & \ensuremath{\pm 2.2} \\
\texttt{blackbox\_\allowbreak{}numerical\_\allowbreak{}integration} & \multicolumn{2}{c}{\best{\ensuremath{\mathrm{45.1}}\,{\scriptsize \ensuremath{\pm 9.1}}}} & \multicolumn{2}{c}{\second{\ensuremath{\mathrm{44.3}}\,{\scriptsize \ensuremath{\pm 12.9}}}} & \ensuremath{\mathrm{33.4}} & \ensuremath{\pm 5.4} & \ensuremath{\mathrm{40.2}} & \ensuremath{\pm 5.6} & \ensuremath{\mathrm{40.2}} & \ensuremath{\pm 3.4} \\
\texttt{monge\_\allowbreak{}ampere\_\allowbreak{}pde\_\allowbreak{}solver} & \multicolumn{2}{c}{\best{\ensuremath{\mathrm{63.3}}\,{\scriptsize \ensuremath{\pm 5.3}}}} & \ensuremath{\mathrm{37.1}} & \ensuremath{\pm 8.8} & \ensuremath{\mathrm{9.6}} & \ensuremath{\pm 16.5} & \multicolumn{2}{c}{\second{\ensuremath{\mathrm{48.1}}\,{\scriptsize \ensuremath{\pm 11.2}}}} & \ensuremath{\mathrm{32.9}} & \ensuremath{\pm 30.2} \\
\texttt{ocean\_\allowbreak{}mt\_\allowbreak{}lab\_\allowbreak{}inversion} & \multicolumn{2}{c}{\second{\ensuremath{\mathrm{41.9}}\,{\scriptsize \ensuremath{\pm 47.4}}}} & \ensuremath{\mathrm{39.7}} & \ensuremath{\pm 43.7} & \ensuremath{\mathrm{15.2}} & \ensuremath{\pm 3.3} & \multicolumn{2}{c}{\best{\ensuremath{\mathrm{60.8}}\,{\scriptsize \ensuremath{\pm 40.2}}}} & \ensuremath{\mathrm{14.5}} & \ensuremath{\pm 0.0} \\
\texttt{pv\_\allowbreak{}power\_\allowbreak{}forecasting} & \ensuremath{\mathrm{15.3}} & \ensuremath{\pm 1.0} & \multicolumn{2}{c}{\best{\ensuremath{\mathrm{17.2}}\,{\scriptsize \ensuremath{\pm 1.0}}}} & \ensuremath{\mathrm{12.6}} & \ensuremath{\pm 0.5} & \multicolumn{2}{c}{\second{\ensuremath{\mathrm{16.6}}\,{\scriptsize \ensuremath{\pm 0.5}}}} & \ensuremath{\mathrm{14.1}} & \ensuremath{\pm 2.4} \\
\texttt{collaborative\_\allowbreak{}filtering\_\allowbreak{}recommender} & \multicolumn{2}{c}{\best{\ensuremath{\mathrm{55.0}}\,{\scriptsize \ensuremath{\pm 2.9}}}} & \multicolumn{2}{c}{\second{\ensuremath{\mathrm{46.4}}\,{\scriptsize \ensuremath{\pm 9.3}}}} & \ensuremath{\mathrm{14.7}} & \ensuremath{\pm 12.2} & \ensuremath{\mathrm{39.9}} & \ensuremath{\pm 20.5} & \ensuremath{\mathrm{8.5}} & \ensuremath{\pm 1.8} \\
\texttt{ecg\_\allowbreak{}signal\_\allowbreak{}processing\_\allowbreak{}pipeline} & \multicolumn{2}{c}{\best{\ensuremath{\mathrm{58.7}}\,{\scriptsize \ensuremath{\pm 7.3}}}} & \multicolumn{2}{c}{\second{\ensuremath{\mathrm{44.7}}\,{\scriptsize \ensuremath{\pm 8.7}}}} & \ensuremath{\mathrm{39.2}} & \ensuremath{\pm 7.8} & \ensuremath{\mathrm{31.8}} & \ensuremath{\pm 1.9} & \ensuremath{\mathrm{14.3}} & \ensuremath{\pm 8.6} \\
\texttt{sketch\_\allowbreak{}solve\_\allowbreak{}least\_\allowbreak{}squares} & \ensuremath{\mathrm{59.4}} & \ensuremath{\pm 0.8} & \multicolumn{2}{c}{\second{\ensuremath{\mathrm{61.1}}\,{\scriptsize \ensuremath{\pm 2.3}}}} & \ensuremath{\mathrm{60.2}} & \ensuremath{\pm 1.3} & \multicolumn{2}{c}{\second{\ensuremath{\mathrm{61.1}}\,{\scriptsize \ensuremath{\pm 0.8}}}} & \multicolumn{2}{c}{\best{\ensuremath{\mathrm{61.3}}\,{\scriptsize \ensuremath{\pm 1.1}}}} \\
\texttt{substrate\_\allowbreak{}interface\_\allowbreak{}simulation} & \multicolumn{2}{c}{\best{\ensuremath{\mathrm{0.0}}\partialrun\,{\scriptsize \ensuremath{\pm 0.0}}}} & \multicolumn{2}{c}{\best{\ensuremath{\mathrm{0.0}}\,{\scriptsize \ensuremath{\pm 0.0}}}} & \multicolumn{2}{c}{\best{\ensuremath{\mathrm{0.0}}\,{\scriptsize \ensuremath{\pm 0.0}}}} & \multicolumn{2}{c}{\best{\ensuremath{\mathrm{0.0}}\,{\scriptsize \ensuremath{\pm 0.0}}}} & \multicolumn{2}{c}{\best{\ensuremath{\mathrm{0.0}}\,{\scriptsize \ensuremath{\pm 0.0}}}} \\
\texttt{thermo\_\allowbreak{}fluid\_\allowbreak{}field\_\allowbreak{}prediction} & \multicolumn{2}{c}{\second{\ensuremath{\mathrm{69.6}}\,{\scriptsize \ensuremath{\pm 20.8}}}} & \ensuremath{\mathrm{57.8}} & \ensuremath{\pm 15.1} & \ensuremath{\mathrm{39.1}} & \ensuremath{\pm 5.3} & \multicolumn{2}{c}{\best{\ensuremath{\mathrm{76.8}}\,{\scriptsize \ensuremath{\pm 17.8}}}} & \ensuremath{\mathrm{22.5}} & \ensuremath{\pm 15.7} \\
\texttt{csi\_\allowbreak{}time\_\allowbreak{}series\_\allowbreak{}forecasting} & \multicolumn{2}{c}{\best{\ensuremath{\mathrm{83.4}}\,{\scriptsize \ensuremath{\pm 10.4}}}} & \ensuremath{\mathrm{76.3}} & \ensuremath{\pm 2.5} & \multicolumn{2}{c}{\second{\ensuremath{\mathrm{77.3}}\,{\scriptsize \ensuremath{\pm 0.2}}}} & \ensuremath{\mathrm{44.6}} & \ensuremath{\pm 3.1} & \ensuremath{\mathrm{40.4}} & \ensuremath{\pm 12.3} \\
\texttt{roof\_\allowbreak{}damage\_\allowbreak{}active\_\allowbreak{}learning} & \multicolumn{2}{c}{\second{\ensuremath{\mathrm{4.7}}\,{\scriptsize \ensuremath{\pm 8.2}}}} & \multicolumn{2}{c}{\best{\ensuremath{\mathrm{22.8}}\partialrun\,{\scriptsize \ensuremath{\pm 7.1}}}} & \nostd{\NA} & \ensuremath{\mathrm{4.1}} & \ensuremath{\pm 7.0} & \ensuremath{\mathrm{0.0}} & \ensuremath{\pm 0.0} \\
\BenchBottomRule
\end{tabular}%
}
\caption{Model performance on Scientific Problems \& ML tasks. Values are mean scores over up to three valid runs; adjacent entries show $\pm s$ when at least two valid runs are available. Bold marks the best model for each task, underlining marks the second-best model, \NA{} indicates no valid result, and \textsuperscript{*} marks fewer than three valid runs.}
\label{tab:appendix-scientific_computing_ml}
\end{table}
\begin{table}[p]
\centering
\footnotesize
\renewcommand{\arraystretch}{1.08}
\setlength{\tabcolsep}{2.5pt}
\resizebox{\ifdim\width>\linewidth\linewidth\else\width\fi}{!}{%
\begin{tabular}{@{}>{\raggedright\arraybackslash}p{0.30\linewidth}*{5}{r@{\,}>{\scriptsize}l}@{}}
\BenchTopRule
\textbf{Combinatorial Optimization Tasks} & \multicolumn{2}{c}{\textbf{Opus 4.8}} & \multicolumn{2}{c}{\textbf{GPT-5.5}} & \multicolumn{2}{c}{\textbf{GPT-5.4}} & \multicolumn{2}{c}{\textbf{GLM-5.1}} & \multicolumn{2}{c}{\textbf{DS-V4-Pro}} \\
\BenchMidRule
\texttt{symbolic\_\allowbreak{}integration\_\allowbreak{}engine} & \multicolumn{2}{c}{\best{\ensuremath{\mathrm{57.7}}\,{\scriptsize \ensuremath{\pm 1.1}}}} & \multicolumn{2}{c}{\second{\ensuremath{\mathrm{44.0}}\,{\scriptsize \ensuremath{\pm 4.2}}}} & \ensuremath{\mathrm{30.9}} & \ensuremath{\pm 1.7} & \ensuremath{\mathrm{26.7}} & \ensuremath{\pm 8.9} & \ensuremath{\mathrm{18.4}} & \ensuremath{\pm 2.9} \\
\texttt{order\_\allowbreak{}addition\_\allowbreak{}permutation\_\allowbreak{}optimization} & \multicolumn{2}{c}{\best{\ensuremath{\mathrm{36.4}}\,{\scriptsize \ensuremath{\pm 5.7}}}} & \ensuremath{\mathrm{23.3}} & \ensuremath{\pm 1.4} & \ensuremath{\mathrm{14.3}} & \ensuremath{\pm 12.4} & \multicolumn{2}{c}{\second{\ensuremath{\mathrm{33.2}}\,{\scriptsize \ensuremath{\pm 9.8}}}} & \ensuremath{\mathrm{30.8}} & \ensuremath{\pm 11.3} \\
\texttt{sat\_\allowbreak{}solver} & \multicolumn{2}{c}{\best{\ensuremath{\mathrm{14.4}}\,{\scriptsize \ensuremath{\pm 7.6}}}} & \ensuremath{\mathrm{8.9}} & \ensuremath{\pm 0.6} & \multicolumn{2}{c}{\second{\ensuremath{\mathrm{13.8}}\,{\scriptsize \ensuremath{\pm 4.7}}}} & \ensuremath{\mathrm{13.6}} & \ensuremath{\pm 5.7} & \ensuremath{\mathrm{8.0}} & \ensuremath{\pm 6.4} \\
\texttt{quantum\_\allowbreak{}architecture\_\allowbreak{}search} & \ensuremath{\mathrm{12.7}} & \ensuremath{\pm 1.5} & \multicolumn{2}{c}{\best{\ensuremath{\mathrm{68.3}}\,{\scriptsize \ensuremath{\pm 22.2}}}} & \ensuremath{\mathrm{12.5}}\partialrun & \ensuremath{\pm 2.1} & \multicolumn{2}{c}{\second{\ensuremath{\mathrm{25.3}}\,{\scriptsize \ensuremath{\pm 23.1}}}} & \ensuremath{\mathrm{12.0}} & \ensuremath{\pm 1.7} \\
\texttt{ad\_\allowbreak{}placement\_\allowbreak{}optimization} & \multicolumn{2}{c}{\best{\ensuremath{\mathrm{67.7}}\,{\scriptsize \ensuremath{\pm 1.0}}}} & \multicolumn{2}{c}{\second{\ensuremath{\mathrm{62.9}}\,{\scriptsize \ensuremath{\pm 6.2}}}} & \ensuremath{\mathrm{48.1}} & \ensuremath{\pm 6.2} & \ensuremath{\mathrm{58.8}} & \ensuremath{\pm 14.3} & \ensuremath{\mathrm{36.2}} & \ensuremath{\pm 16.3} \\
\texttt{first\_\allowbreak{}order\_\allowbreak{}theorem\_\allowbreak{}prover} & \multicolumn{2}{c}{\best{\ensuremath{\mathrm{31.9}}\,{\scriptsize \ensuremath{\pm 13.0}}}} & \ensuremath{\mathrm{11.2}} & \ensuremath{\pm 11.2} & \multicolumn{2}{c}{\second{\ensuremath{\mathrm{13.1}}\,{\scriptsize \ensuremath{\pm 22.7}}}} & \nostd{\ensuremath{\mathrm{0.0}}\partialrun} & \ensuremath{\mathrm{0.0}} & \ensuremath{\pm 0.0} \\
\texttt{circuit\_\allowbreak{}layout\_\allowbreak{}optimization} & \multicolumn{2}{c}{\best{\ensuremath{\mathrm{37.3}}\,{\scriptsize \ensuremath{\pm 2.4}}}} & \multicolumn{2}{c}{\second{\ensuremath{\mathrm{33.0}}\,{\scriptsize \ensuremath{\pm 3.9}}}} & \ensuremath{\mathrm{26.0}} & \ensuremath{\pm 3.9} & \ensuremath{\mathrm{31.3}} & \ensuremath{\pm 5.6} & \ensuremath{\mathrm{22.5}} & \ensuremath{\pm 3.5} \\
\texttt{equivalence\_\allowbreak{}class\_\allowbreak{}divide\_\allowbreak{}and\_\allowbreak{}conquer} & \multicolumn{2}{c}{\second{\ensuremath{\mathrm{21.3}}\,{\scriptsize \ensuremath{\pm 4.9}}}} & \multicolumn{2}{c}{\best{\ensuremath{\mathrm{22.4}}\,{\scriptsize \ensuremath{\pm 12.2}}}} & \ensuremath{\mathrm{20.3}} & \ensuremath{\pm 2.4} & \ensuremath{\mathrm{10.6}} & \ensuremath{\pm 2.0} & \ensuremath{\mathrm{3.4}} & \ensuremath{\pm 3.5} \\
\texttt{jagua\_\allowbreak{}nesting\_\allowbreak{}optimization} & \multicolumn{2}{c}{\best{\ensuremath{\mathrm{44.2}}\,{\scriptsize \ensuremath{\pm 16.7}}}} & \ensuremath{\mathrm{21.6}} & \ensuremath{\pm 10.9} & \ensuremath{\mathrm{24.1}} & \ensuremath{\pm 11.0} & \ensuremath{\mathrm{12.4}} & \ensuremath{\pm 8.2} & \multicolumn{2}{c}{\second{\ensuremath{\mathrm{28.4}}\,{\scriptsize \ensuremath{\pm 16.7}}}} \\
\texttt{smt\_\allowbreak{}solver} & \multicolumn{2}{c}{\best{\ensuremath{\mathrm{23.9}}\,{\scriptsize \ensuremath{\pm 7.0}}}} & \ensuremath{\mathrm{8.6}} & \ensuremath{\pm 3.1} & \multicolumn{2}{c}{\second{\ensuremath{\mathrm{9.2}}\,{\scriptsize \ensuremath{\pm 2.8}}}} & \nostd{\ensuremath{\mathrm{3.6}}\partialrun} & \ensuremath{\mathrm{3.3}} & \ensuremath{\pm 1.7} \\
\texttt{tree\_\allowbreak{}block\_\allowbreak{}partitioning} & \multicolumn{2}{c}{\best{\ensuremath{\mathrm{37.7}}\,{\scriptsize \ensuremath{\pm 6.6}}}} & \multicolumn{2}{c}{\second{\ensuremath{\mathrm{36.4}}\,{\scriptsize \ensuremath{\pm 2.8}}}} & \ensuremath{\mathrm{34.3}} & \ensuremath{\pm 4.7} & \ensuremath{\mathrm{23.4}} & \ensuremath{\pm 4.5} & \ensuremath{\mathrm{16.1}} & \ensuremath{\pm 0.9} \\
\texttt{triangulation\_\allowbreak{}coloring\_\allowbreak{}optimization} & \ensuremath{\mathrm{73.4}} & \ensuremath{\pm 2.4} & \multicolumn{2}{c}{\best{\ensuremath{\mathrm{75.2}}\,{\scriptsize \ensuremath{\pm 3.0}}}} & \multicolumn{2}{c}{\second{\ensuremath{\mathrm{74.3}}\,{\scriptsize \ensuremath{\pm 3.3}}}} & \ensuremath{\mathrm{73.0}} & \ensuremath{\pm 1.5} & \ensuremath{\mathrm{59.3}} & \ensuremath{\pm 10.3} \\
\texttt{vibrating\_\allowbreak{}path\_\allowbreak{}graph\_\allowbreak{}coloring} & \multicolumn{2}{c}{\best{\ensuremath{\mathrm{25.3}}\,{\scriptsize \ensuremath{\pm 11.7}}}} & \ensuremath{\mathrm{11.4}} & \ensuremath{\pm 5.1} & \multicolumn{2}{c}{\second{\ensuremath{\mathrm{24.1}}\,{\scriptsize \ensuremath{\pm 8.5}}}} & \ensuremath{\mathrm{22.9}} & \ensuremath{\pm 3.2} & \ensuremath{\mathrm{22.1}} & \ensuremath{\pm 5.1} \\
\texttt{vehicle\_\allowbreak{}routing\_\allowbreak{}time\_\allowbreak{}windows} & \ensuremath{\mathrm{74.0}} & \ensuremath{\pm 16.9} & \multicolumn{2}{c}{\best{\ensuremath{\mathrm{90.8}}\,{\scriptsize \ensuremath{\pm 2.3}}}} & \multicolumn{2}{c}{\second{\ensuremath{\mathrm{89.6}}\,{\scriptsize \ensuremath{\pm 2.3}}}} & \ensuremath{\mathrm{77.9}} & \ensuremath{\pm 8.1} & \ensuremath{\mathrm{83.1}} & \ensuremath{\pm 6.6} \\
\texttt{warehouse\_\allowbreak{}forklift\_\allowbreak{}routing} & \multicolumn{2}{c}{\second{\ensuremath{\mathrm{11.2}}\,{\scriptsize \ensuremath{\pm 1.1}}}} & \multicolumn{2}{c}{\best{\ensuremath{\mathrm{12.6}}\,{\scriptsize \ensuremath{\pm 0.6}}}} & \ensuremath{\mathrm{0.0}} & \ensuremath{\pm 0.0} & \ensuremath{\mathrm{0.5}}\partialrun & \ensuremath{\pm 0.3} & \ensuremath{\mathrm{0.0}}\partialrun & \ensuremath{\pm 0.0} \\
\texttt{wireless\_\allowbreak{}electricity\_\allowbreak{}layout} & \multicolumn{2}{c}{\best{\ensuremath{\mathrm{14.5}}\,{\scriptsize \ensuremath{\pm 6.1}}}} & \ensuremath{\mathrm{7.2}} & \ensuremath{\pm 7.9} & \multicolumn{2}{c}{\second{\ensuremath{\mathrm{11.1}}\,{\scriptsize \ensuremath{\pm 9.6}}}} & \ensuremath{\mathrm{9.5}} & \ensuremath{\pm 9.8} & \ensuremath{\mathrm{0.0}} & \ensuremath{\pm 0.0} \\
\BenchBottomRule
\end{tabular}%
}
\caption{Model performance on Combinatorial Optimization tasks. Values are mean scores over up to three valid runs; adjacent entries show $\pm s$ when at least two valid runs are available. Bold marks the best model for each task, underlining marks the second-best model, \NA{} indicates no valid result, and \textsuperscript{*} marks fewer than three valid runs.}
\label{tab:appendix-optimization_search_or}
\end{table}
\begin{table}[p]
\centering
\footnotesize
\renewcommand{\arraystretch}{1.08}
\setlength{\tabcolsep}{2.5pt}
\resizebox{\ifdim\width>\linewidth\linewidth\else\width\fi}{!}{%
\begin{tabular}{@{}>{\raggedright\arraybackslash}p{0.30\linewidth}*{5}{r@{\,}>{\scriptsize}l}@{}}
\BenchTopRule
\textbf{Formal Math \& Theorem Proving Tasks} & \multicolumn{2}{c}{\textbf{Opus 4.8}} & \multicolumn{2}{c}{\textbf{GPT-5.5}} & \multicolumn{2}{c}{\textbf{GPT-5.4}} & \multicolumn{2}{c}{\textbf{GLM-5.1}} & \multicolumn{2}{c}{\textbf{DS-V4-Pro}} \\
\BenchMidRule
\texttt{combinatorial\_\allowbreak{}games\_\allowbreak{}formalization} & \multicolumn{2}{c}{\second{\ensuremath{\mathrm{35.5}}\,{\scriptsize \ensuremath{\pm 3.9}}}} & \multicolumn{2}{c}{\best{\ensuremath{\mathrm{38.2}}\,{\scriptsize \ensuremath{\pm 10.1}}}} & \ensuremath{\mathrm{17.8}} & \ensuremath{\pm 6.8} & \ensuremath{\mathrm{16.2}}\partialrun & \ensuremath{\pm 3.1} & \ensuremath{\mathrm{7.8}} & \ensuremath{\pm 0.9} \\
\texttt{erdos392\_\allowbreak{}formalization} & \multicolumn{2}{c}{\best{\ensuremath{\mathrm{98.0}}\,{\scriptsize \ensuremath{\pm 3.5}}}} & \nostd{\NA} & \multicolumn{2}{c}{\second{\ensuremath{\mathrm{48.0}}\,{\scriptsize \ensuremath{\pm 5.3}}}} & \ensuremath{\mathrm{32.7}} & \ensuremath{\pm 11.0} & \ensuremath{\mathrm{10.7}} & \ensuremath{\pm 6.7} \\
\texttt{cup\_\allowbreak{}product\_\allowbreak{}formalization} & \multicolumn{2}{c}{\second{\ensuremath{\mathrm{74.0}}\,{\scriptsize \ensuremath{\pm 8.9}}}} & \multicolumn{2}{c}{\best{\ensuremath{\mathrm{93.1}}\,{\scriptsize \ensuremath{\pm 6.1}}}} & \multicolumn{2}{c}{\second{\ensuremath{\mathrm{74.0}}\,{\scriptsize \ensuremath{\pm 1.7}}}} & \ensuremath{\mathrm{38.7}} & \ensuremath{\pm 0.8} & \ensuremath{\mathrm{40.7}} & \ensuremath{\pm 3.4} \\
\texttt{lean\_\allowbreak{}analysis\_\allowbreak{}proofs} & \nostd{\second{\ensuremath{\mathrm{33.0}}\partialrun}} & \multicolumn{2}{c}{\best{\ensuremath{\mathrm{42.5}}\,{\scriptsize \ensuremath{\pm 5.1}}}} & \ensuremath{\mathrm{16.4}} & \ensuremath{\pm 4.3} & \nostd{\ensuremath{\mathrm{5.9}}\partialrun} & \ensuremath{\mathrm{9.5}} & \ensuremath{\pm 2.7} \\
\texttt{carleson\_\allowbreak{}formalization} & \multicolumn{2}{c}{\second{\ensuremath{\mathrm{16.8}}\,{\scriptsize \ensuremath{\pm 0.5}}}} & \multicolumn{2}{c}{\best{\ensuremath{\mathrm{26.5}}\,{\scriptsize \ensuremath{\pm 6.5}}}} & \ensuremath{\mathrm{7.1}} & \ensuremath{\pm 3.2} & \ensuremath{\mathrm{2.2}} & \ensuremath{\pm 0.7} & \ensuremath{\mathrm{2.5}} & \ensuremath{\pm 1.5} \\
\texttt{new\_\allowbreak{}foundations\_\allowbreak{}consistency} & \nostd{\second{\ensuremath{\mathrm{65.1}}\partialrun}} & \multicolumn{2}{c}{\best{\ensuremath{\mathrm{66.5}}\,{\scriptsize \ensuremath{\pm 2.0}}}} & \ensuremath{\mathrm{39.8}} & \ensuremath{\pm 12.2} & \ensuremath{\mathrm{27.0}} & \ensuremath{\pm 27.5} & \ensuremath{\mathrm{11.4}} & \ensuremath{\pm 4.9} \\
\texttt{godel\_\allowbreak{}incompleteness\_\allowbreak{}formalization} & \multicolumn{2}{c}{\best{\ensuremath{\mathrm{100.0}}\,{\scriptsize \ensuremath{\pm 0.0}}}} & \nostd{\NA} & \multicolumn{2}{c}{\second{\ensuremath{\mathrm{64.4}}\,{\scriptsize \ensuremath{\pm 30.8}}}} & \nostd{\ensuremath{\mathrm{46.7}}\partialrun} & \ensuremath{\mathrm{6.7}} & \ensuremath{\pm 0.0} \\
\texttt{medium\_\allowbreak{}prime\_\allowbreak{}number\_\allowbreak{}theorem} & \multicolumn{2}{c}{\best{\ensuremath{\mathrm{100.0}}\partialrun\,{\scriptsize \ensuremath{\pm 0.0}}}} & \nostd{\NA} & \multicolumn{2}{c}{\second{\ensuremath{\mathrm{88.0}}\,{\scriptsize \ensuremath{\pm 0.0}}}} & \ensuremath{\mathrm{26.7}} & \ensuremath{\pm 8.1} & \ensuremath{\mathrm{9.0}} & \ensuremath{\pm 0.0} \\
\texttt{ordinal\_\allowbreak{}notation\_\allowbreak{}well\_\allowbreak{}foundedness} & \multicolumn{2}{c}{\best{\ensuremath{\mathrm{24.7}}\,{\scriptsize \ensuremath{\pm 0.0}}}} & \multicolumn{2}{c}{\best{\ensuremath{\mathrm{24.7}}\,{\scriptsize \ensuremath{\pm 0.0}}}} & \multicolumn{2}{c}{\second{\ensuremath{\mathrm{21.6}}\,{\scriptsize \ensuremath{\pm 5.4}}}} & \ensuremath{\mathrm{5.9}} & \ensuremath{\pm 0.0} & \ensuremath{\mathrm{4.7}}\partialrun & \ensuremath{\pm 1.7} \\
\texttt{pfr\_\allowbreak{}formalization} & \nostd{\second{\ensuremath{\mathrm{46.3}}\partialrun}} & \multicolumn{2}{c}{\best{\ensuremath{\mathrm{60.0}}\,{\scriptsize \ensuremath{\pm 2.7}}}} & \ensuremath{\mathrm{38.9}} & \ensuremath{\pm 1.3} & \ensuremath{\mathrm{33.5}} & \ensuremath{\pm 1.4} & \ensuremath{\mathrm{19.1}} & \ensuremath{\pm 5.4} \\
\texttt{sphere\_\allowbreak{}eversion\_\allowbreak{}formalization} & \multicolumn{2}{c}{\second{\ensuremath{\mathrm{55.4}}\,{\scriptsize \ensuremath{\pm 3.7}}}} & \multicolumn{2}{c}{\best{\ensuremath{\mathrm{58.5}}\,{\scriptsize \ensuremath{\pm 2.9}}}} & \ensuremath{\mathrm{51.4}} & \ensuremath{\pm 1.7} & \ensuremath{\mathrm{30.2}} & \ensuremath{\pm 21.3} & \ensuremath{\mathrm{29.3}} & \ensuremath{\pm 7.7} \\
\texttt{turing\_\allowbreak{}machine\_\allowbreak{}halting\_\allowbreak{}proofs} & \multicolumn{2}{c}{\best{\ensuremath{\mathrm{15.0}}\,{\scriptsize \ensuremath{\pm 0.0}}}} & \multicolumn{2}{c}{\best{\ensuremath{\mathrm{15.0}}\,{\scriptsize \ensuremath{\pm 0.0}}}} & \multicolumn{2}{c}{\best{\ensuremath{\mathrm{15.0}}\,{\scriptsize \ensuremath{\pm 0.0}}}} & \multicolumn{2}{c}{\best{\ensuremath{\mathrm{15.0}}\,{\scriptsize \ensuremath{\pm 0.0}}}} & \multicolumn{2}{c}{\best{\ensuremath{\mathrm{15.0}}\,{\scriptsize \ensuremath{\pm 0.0}}}} \\
\texttt{flt\_\allowbreak{}regular\_\allowbreak{}formalization} & \multicolumn{2}{c}{\second{\ensuremath{\mathrm{50.6}}\,{\scriptsize \ensuremath{\pm 0.0}}}} & \multicolumn{2}{c}{\best{\ensuremath{\mathrm{75.1}}\,{\scriptsize \ensuremath{\pm 24.7}}}} & \ensuremath{\mathrm{48.3}} & \ensuremath{\pm 4.0} & \ensuremath{\mathrm{38.7}} & \ensuremath{\pm 13.4} & \ensuremath{\mathrm{17.6}} & \ensuremath{\pm 11.3} \\
\BenchBottomRule
\end{tabular}%
}
\caption{Model performance on Formal Math \& Theorem Proving tasks. Values are mean scores over up to three valid runs; adjacent entries show $\pm s$ when at least two valid runs are available. Bold marks the best model for each task, underlining marks the second-best model, \NA{} indicates no valid result, and \textsuperscript{*} marks fewer than three valid runs.}
\label{tab:appendix-formal_math}
\end{table}
\begin{table}[p]
\centering
\footnotesize
\renewcommand{\arraystretch}{1.08}
\setlength{\tabcolsep}{2.5pt}
\resizebox{\ifdim\width>\linewidth\linewidth\else\width\fi}{!}{%
\begin{tabular}{@{}>{\raggedright\arraybackslash}p{0.30\linewidth}*{5}{r@{\,}>{\scriptsize}l}@{}}
\BenchTopRule
\textbf{Professional Knowledge Work Tasks} & \multicolumn{2}{c}{\textbf{Opus 4.8}} & \multicolumn{2}{c}{\textbf{GPT-5.5}} & \multicolumn{2}{c}{\textbf{GPT-5.4}} & \multicolumn{2}{c}{\textbf{GLM-5.1}} & \multicolumn{2}{c}{\textbf{DS-V4-Pro}} \\
\BenchMidRule
\texttt{portfolio\_\allowbreak{}risk\_\allowbreak{}calibration} & \multicolumn{2}{c}{\second{\ensuremath{\mathrm{24.5}}\,{\scriptsize \ensuremath{\pm 7.5}}}} & \multicolumn{2}{c}{\best{\ensuremath{\mathrm{25.0}}\,{\scriptsize \ensuremath{\pm 6.5}}}} & \ensuremath{\mathrm{10.7}} & \ensuremath{\pm 3.9} & \ensuremath{\mathrm{9.4}} & \ensuremath{\pm 14.2} & \ensuremath{\mathrm{23.7}} & \ensuremath{\pm 6.9} \\
\texttt{storyboard\_\allowbreak{}ad\_\allowbreak{}copywriting} & \multicolumn{2}{c}{\second{\ensuremath{\mathrm{79.7}}\,{\scriptsize \ensuremath{\pm 8.1}}}} & \ensuremath{\mathrm{77.0}} & \ensuremath{\pm 5.3} & \ensuremath{\mathrm{65.7}} & \ensuremath{\pm 14.6} & \multicolumn{2}{c}{\best{\ensuremath{\mathrm{92.0}}\,{\scriptsize \ensuremath{\pm 5.0}}}} & \ensuremath{\mathrm{56.5}}\partialrun & \ensuremath{\pm 7.8} \\
\texttt{cross\_\allowbreak{}border\_\allowbreak{}investment\_\allowbreak{}ppt} & \multicolumn{2}{c}{\second{\ensuremath{\mathrm{38.3}}\,{\scriptsize \ensuremath{\pm 4.7}}}} & \ensuremath{\mathrm{34.0}}\partialrun & \ensuremath{\pm 2.2} & \ensuremath{\mathrm{35.7}} & \ensuremath{\pm 6.5} & \multicolumn{2}{c}{\best{\ensuremath{\mathrm{45.9}}\,{\scriptsize \ensuremath{\pm 7.0}}}} & \ensuremath{\mathrm{32.3}} & \ensuremath{\pm 5.9} \\
\texttt{herbal\_\allowbreak{}depression\_\allowbreak{}target\_\allowbreak{}screening} & \multicolumn{2}{c}{\best{\ensuremath{\mathrm{70.3}}\,{\scriptsize \ensuremath{\pm 5.9}}}} & \multicolumn{2}{c}{\second{\ensuremath{\mathrm{68.1}}\,{\scriptsize \ensuremath{\pm 3.1}}}} & \ensuremath{\mathrm{58.9}} & \ensuremath{\pm 7.9} & \ensuremath{\mathrm{57.0}} & \ensuremath{\pm 12.8} & \ensuremath{\mathrm{60.1}} & \ensuremath{\pm 3.3} \\
\texttt{pancreatic\_\allowbreak{}radiotherapy\_\allowbreak{}meta\_\allowbreak{}analysis} & \multicolumn{2}{c}{\best{\ensuremath{\mathrm{40.2}}\,{\scriptsize \ensuremath{\pm 5.7}}}} & \multicolumn{2}{c}{\second{\ensuremath{\mathrm{39.8}}\,{\scriptsize \ensuremath{\pm 3.2}}}} & \ensuremath{\mathrm{30.6}} & \ensuremath{\pm 4.1} & \ensuremath{\mathrm{36.5}} & \ensuremath{\pm 5.2} & \ensuremath{\mathrm{34.5}} & \ensuremath{\pm 7.2} \\
\texttt{touchstone\_\allowbreak{}vna\_\allowbreak{}diagnostics} & \multicolumn{2}{c}{\best{\ensuremath{\mathrm{39.9}}\,{\scriptsize \ensuremath{\pm 2.8}}}} & \multicolumn{2}{c}{\second{\ensuremath{\mathrm{11.5}}\,{\scriptsize \ensuremath{\pm 4.9}}}} & \ensuremath{\mathrm{8.4}} & \ensuremath{\pm 1.1} & \ensuremath{\mathrm{4.7}} & \ensuremath{\pm 1.8} & \ensuremath{\mathrm{4.5}} & \ensuremath{\pm 2.4} \\
\texttt{high\_\allowbreak{}performance\_\allowbreak{}object\_\allowbreak{}mapper} & \nostd{\NA} & \multicolumn{2}{c}{\second{\ensuremath{\mathrm{62.2}}\,{\scriptsize \ensuremath{\pm 2.9}}}} & \ensuremath{\mathrm{45.6}} & \ensuremath{\pm 8.9} & \multicolumn{2}{c}{\best{\ensuremath{\mathrm{64.0}}\,{\scriptsize \ensuremath{\pm 1.3}}}} & \ensuremath{\mathrm{51.0}} & \ensuremath{\pm 7.2} \\
\texttt{odata\_\allowbreak{}query\_\allowbreak{}service} & \multicolumn{2}{c}{\second{\ensuremath{\mathrm{30.9}}\,{\scriptsize \ensuremath{\pm 1.1}}}} & \ensuremath{\mathrm{27.8}} & \ensuremath{\pm 3.3} & \multicolumn{2}{c}{\best{\ensuremath{\mathrm{31.1}}\,{\scriptsize \ensuremath{\pm 4.9}}}} & \ensuremath{\mathrm{26.1}} & \ensuremath{\pm 3.2} & \ensuremath{\mathrm{20.8}} & \ensuremath{\pm 11.0} \\
\texttt{brand\_\allowbreak{}annual\_\allowbreak{}planning\_\allowbreak{}ppt} & \ensuremath{\mathrm{51.3}} & \ensuremath{\pm 19.8} & \ensuremath{\mathrm{69.0}} & \ensuremath{\pm 11.0} & \multicolumn{2}{c}{\second{\ensuremath{\mathrm{72.3}}\,{\scriptsize \ensuremath{\pm 2.1}}}} & \multicolumn{2}{c}{\best{\ensuremath{\mathrm{73.7}}\,{\scriptsize \ensuremath{\pm 9.2}}}} & \ensuremath{\mathrm{50.0}} & \ensuremath{\pm 10.0} \\
\texttt{securities\_\allowbreak{}protection\_\allowbreak{}training} & \multicolumn{2}{c}{\best{\ensuremath{\mathrm{97.2}}\,{\scriptsize \ensuremath{\pm 4.3}}}} & \ensuremath{\mathrm{89.4}} & \ensuremath{\pm 8.2} & \multicolumn{2}{c}{\second{\ensuremath{\mathrm{94.2}}\,{\scriptsize \ensuremath{\pm 5.2}}}} & \ensuremath{\mathrm{93.8}} & \ensuremath{\pm 1.5} & \ensuremath{\mathrm{85.0}} & \ensuremath{\pm 9.5} \\
\texttt{college\_\allowbreak{}english\_\allowbreak{}exam\_\allowbreak{}bank} & \multicolumn{2}{c}{\best{\ensuremath{\mathrm{39.8}}\,{\scriptsize \ensuremath{\pm 4.3}}}} & \multicolumn{2}{c}{\second{\ensuremath{\mathrm{37.8}}\,{\scriptsize \ensuremath{\pm 7.8}}}} & \ensuremath{\mathrm{34.5}} & \ensuremath{\pm 2.3} & \ensuremath{\mathrm{32.5}} & \ensuremath{\pm 1.3} & \ensuremath{\mathrm{34.7}} & \ensuremath{\pm 10.3} \\
\texttt{cross\_\allowbreak{}border\_\allowbreak{}commission\_\allowbreak{}compliance} & \ensuremath{\mathrm{32.1}} & \ensuremath{\pm 2.4} & \multicolumn{2}{c}{\best{\ensuremath{\mathrm{43.8}}\partialrun\,{\scriptsize \ensuremath{\pm 2.1}}}} & \nostd{\ensuremath{\mathrm{29.8}}\partialrun} & \multicolumn{2}{c}{\second{\ensuremath{\mathrm{39.2}}\,{\scriptsize \ensuremath{\pm 12.3}}}} & \ensuremath{\mathrm{35.2}} & \ensuremath{\pm 5.4} \\
\texttt{cta\_\allowbreak{}risk\_\allowbreak{}budget\_\allowbreak{}optimization} & \ensuremath{\mathrm{46.1}} & \ensuremath{\pm 1.4} & \ensuremath{\mathrm{46.7}} & \ensuremath{\pm 3.8} & \multicolumn{2}{c}{\best{\ensuremath{\mathrm{49.8}}\,{\scriptsize \ensuremath{\pm 2.8}}}} & \multicolumn{2}{c}{\second{\ensuremath{\mathrm{49.6}}\,{\scriptsize \ensuremath{\pm 0.7}}}} & \ensuremath{\mathrm{48.1}} & \ensuremath{\pm 4.1} \\
\texttt{equity\_\allowbreak{}objection\_\allowbreak{}report} & \multicolumn{2}{c}{\second{\ensuremath{\mathrm{15.5}}\,{\scriptsize \ensuremath{\pm 4.3}}}} & \multicolumn{2}{c}{\best{\ensuremath{\mathrm{22.5}}\,{\scriptsize \ensuremath{\pm 2.3}}}} & \ensuremath{\mathrm{15.3}} & \ensuremath{\pm 6.5} & \ensuremath{\mathrm{14.7}} & \ensuremath{\pm 3.8} & \ensuremath{\mathrm{14.8}} & \ensuremath{\pm 11.2} \\
\texttt{expo\_\allowbreak{}visitor\_\allowbreak{}conversion\_\allowbreak{}model} & \multicolumn{2}{c}{\best{\ensuremath{\mathrm{81.9}}\,{\scriptsize \ensuremath{\pm 1.6}}}} & \ensuremath{\mathrm{54.9}} & \ensuremath{\pm 23.9} & \ensuremath{\mathrm{29.1}} & \ensuremath{\pm 10.0} & \multicolumn{2}{c}{\second{\ensuremath{\mathrm{79.0}}\,{\scriptsize \ensuremath{\pm 3.8}}}} & \ensuremath{\mathrm{46.2}} & \ensuremath{\pm 26.7} \\
\texttt{factor\_\allowbreak{}stock\_\allowbreak{}model\_\allowbreak{}optimization} & \ensuremath{\mathrm{51.0}} & \ensuremath{\pm 1.5} & \multicolumn{2}{c}{\best{\ensuremath{\mathrm{58.8}}\,{\scriptsize \ensuremath{\pm 6.9}}}} & \ensuremath{\mathrm{36.3}} & \ensuremath{\pm 3.2} & \multicolumn{2}{c}{\second{\ensuremath{\mathrm{54.3}}\,{\scriptsize \ensuremath{\pm 1.5}}}} & \ensuremath{\mathrm{35.6}} & \ensuremath{\pm 11.9} \\
\texttt{global\_\allowbreak{}terrorism\_\allowbreak{}atlas\_\allowbreak{}report} & \multicolumn{2}{c}{\second{\ensuremath{\mathrm{19.4}}\,{\scriptsize \ensuremath{\pm 2.6}}}} & \multicolumn{2}{c}{\best{\ensuremath{\mathrm{21.5}}\,{\scriptsize \ensuremath{\pm 5.5}}}} & \nostd{\ensuremath{\mathrm{15.7}}\partialrun} & \ensuremath{\mathrm{17.3}} & \ensuremath{\pm 4.4} & \ensuremath{\mathrm{15.2}} & \ensuremath{\pm 1.7} \\
\texttt{hebei\_\allowbreak{}gaokao\_\allowbreak{}strategy\_\allowbreak{}report} & \ensuremath{\mathrm{69.9}} & \ensuremath{\pm 16.3} & \multicolumn{2}{c}{\best{\ensuremath{\mathrm{82.0}}\partialrun\,{\scriptsize \ensuremath{\pm 2.8}}}} & \multicolumn{2}{c}{\second{\ensuremath{\mathrm{78.7}}\,{\scriptsize \ensuremath{\pm 4.7}}}} & \ensuremath{\mathrm{71.0}} & \ensuremath{\pm 1.7} & \ensuremath{\mathrm{62.5}}\partialrun & \ensuremath{\pm 13.5} \\
\texttt{hk\_\allowbreak{}connect\_\allowbreak{}annual\_\allowbreak{}metrics} & \multicolumn{2}{c}{\best{\ensuremath{\mathrm{73.4}}\,{\scriptsize \ensuremath{\pm 11.5}}}} & \multicolumn{2}{c}{\second{\ensuremath{\mathrm{61.8}}\,{\scriptsize \ensuremath{\pm 3.2}}}} & \ensuremath{\mathrm{45.2}} & \ensuremath{\pm 14.4} & \ensuremath{\mathrm{44.8}} & \ensuremath{\pm 18.4} & \ensuremath{\mathrm{41.5}} & \ensuremath{\pm 10.9} \\
\texttt{k12\_\allowbreak{}math\_\allowbreak{}recommendation} & \multicolumn{2}{c}{\best{\ensuremath{\mathrm{44.3}}\,{\scriptsize \ensuremath{\pm 0.5}}}} & \multicolumn{2}{c}{\second{\ensuremath{\mathrm{44.0}}\,{\scriptsize \ensuremath{\pm 13.3}}}} & \ensuremath{\mathrm{31.4}} & \ensuremath{\pm 4.2} & \ensuremath{\mathrm{32.7}} & \ensuremath{\pm 8.2} & \ensuremath{\mathrm{26.3}} & \ensuremath{\pm 2.2} \\
\texttt{property\_\allowbreak{}actuarial\_\allowbreak{}pricing} & \multicolumn{2}{c}{\best{\ensuremath{\mathrm{28.0}}\partialrun\,{\scriptsize \ensuremath{\pm 0.0}}}} & \multicolumn{2}{c}{\best{\ensuremath{\mathrm{28.0}}\partialrun\,{\scriptsize \ensuremath{\pm 0.0}}}} & \multicolumn{2}{c}{\best{\ensuremath{\mathrm{28.0}}\,{\scriptsize \ensuremath{\pm 0.0}}}} & \multicolumn{2}{c}{\best{\ensuremath{\mathrm{28.0}}\,{\scriptsize \ensuremath{\pm 0.0}}}} & \multicolumn{2}{c}{\best{\ensuremath{\mathrm{28.0}}\,{\scriptsize \ensuremath{\pm 0.0}}}} \\
\texttt{real\_\allowbreak{}estate\_\allowbreak{}bid\_\allowbreak{}estimate} & \multicolumn{2}{c}{\second{\ensuremath{\mathrm{52.5}}\,{\scriptsize \ensuremath{\pm 6.1}}}} & \multicolumn{2}{c}{\best{\ensuremath{\mathrm{57.3}}\,{\scriptsize \ensuremath{\pm 1.9}}}} & \ensuremath{\mathrm{51.5}} & \ensuremath{\pm 7.0} & \ensuremath{\mathrm{49.6}} & \ensuremath{\pm 3.9} & \ensuremath{\mathrm{40.6}} & \ensuremath{\pm 3.0} \\
\texttt{stock\_\allowbreak{}momentum\_\allowbreak{}backtest} & \ensuremath{\mathrm{9.8}} & \ensuremath{\pm 4.9} & \multicolumn{2}{c}{\second{\ensuremath{\mathrm{10.7}}\partialrun\,{\scriptsize \ensuremath{\pm 1.2}}}} & \ensuremath{\mathrm{8.2}} & \ensuremath{\pm 2.8} & \ensuremath{\mathrm{4.1}}\partialrun & \ensuremath{\pm 1.2} & \multicolumn{2}{c}{\best{\ensuremath{\mathrm{11.5}}\,{\scriptsize \ensuremath{\pm 2.8}}}} \\
\texttt{storm\_\allowbreak{}claim\_\allowbreak{}ring\_\allowbreak{}audit} & \multicolumn{2}{c}{\best{\ensuremath{\mathrm{43.9}}\,{\scriptsize \ensuremath{\pm 14.3}}}} & \ensuremath{\mathrm{24.0}} & \ensuremath{\pm 0.0} & \ensuremath{\mathrm{24.0}} & \ensuremath{\pm 0.0} & \ensuremath{\mathrm{24.0}} & \ensuremath{\pm 0.0} & \multicolumn{2}{c}{\second{\ensuremath{\mathrm{30.3}}\,{\scriptsize \ensuremath{\pm 7.8}}}} \\
\BenchBottomRule
\end{tabular}%
}
\caption{Model performance on Professional Knowledge Work tasks. Values are mean scores over up to three valid runs; adjacent entries show $\pm s$ when at least two valid runs are available. Bold marks the best model for each task, underlining marks the second-best model, \NA{} indicates no valid result, and \textsuperscript{*} marks fewer than three valid runs.}
\label{tab:appendix-domain_work}
\end{table}
\begin{table}[p]
\centering
\footnotesize
\renewcommand{\arraystretch}{1.08}
\setlength{\tabcolsep}{2.5pt}
\resizebox{\ifdim\width>\linewidth\linewidth\else\width\fi}{!}{%
\begin{tabular}{@{}>{\raggedright\arraybackslash}p{0.30\linewidth}*{5}{r@{\,}>{\scriptsize}l}@{}}
\BenchTopRule
\textbf{Interactive Games \& Simulators Tasks} & \multicolumn{2}{c}{\textbf{Opus 4.8}} & \multicolumn{2}{c}{\textbf{GPT-5.5}} & \multicolumn{2}{c}{\textbf{GPT-5.4}} & \multicolumn{2}{c}{\textbf{GLM-5.1}} & \multicolumn{2}{c}{\textbf{DS-V4-Pro}} \\
\BenchMidRule
\texttt{openttd\_\allowbreak{}transport\_\allowbreak{}ai} & \multicolumn{2}{c}{\best{\ensuremath{\mathrm{52.0}}\,{\scriptsize \ensuremath{\pm 8.5}}}} & \multicolumn{2}{c}{\second{\ensuremath{\mathrm{28.1}}\,{\scriptsize \ensuremath{\pm 24.4}}}} & \ensuremath{\mathrm{11.9}} & \ensuremath{\pm 20.5} & \ensuremath{\mathrm{0.0}} & \ensuremath{\pm 0.0} & \ensuremath{\mathrm{15.2}} & \ensuremath{\pm 26.4} \\
\texttt{nethack\_\allowbreak{}dungeon\_\allowbreak{}agent} & \multicolumn{2}{c}{\best{\ensuremath{\mathrm{41.9}}\,{\scriptsize \ensuremath{\pm 9.1}}}} & \multicolumn{2}{c}{\second{\ensuremath{\mathrm{22.5}}\,{\scriptsize \ensuremath{\pm 4.8}}}} & \ensuremath{\mathrm{20.4}} & \ensuremath{\pm 11.2} & \ensuremath{\mathrm{21.6}} & \ensuremath{\pm 16.2} & \ensuremath{\mathrm{3.3}} & \ensuremath{\pm 2.3} \\
\texttt{treant\_\allowbreak{}forest} & \multicolumn{2}{c}{\best{\ensuremath{\mathrm{18.0}}\,{\scriptsize \ensuremath{\pm 4.9}}}} & \ensuremath{\mathrm{15.6}} & \ensuremath{\pm 6.7} & \ensuremath{\mathrm{13.3}} & \ensuremath{\pm 7.9} & \multicolumn{2}{c}{\second{\ensuremath{\mathrm{16.9}}\,{\scriptsize \ensuremath{\pm 2.4}}}} & \ensuremath{\mathrm{13.5}} & \ensuremath{\pm 4.8} \\
\texttt{grid\_\allowbreak{}turing\_\allowbreak{}robot} & \multicolumn{2}{c}{\second{\ensuremath{\mathrm{40.3}}\,{\scriptsize \ensuremath{\pm 3.6}}}} & \multicolumn{2}{c}{\best{\ensuremath{\mathrm{42.2}}\,{\scriptsize \ensuremath{\pm 2.7}}}} & \ensuremath{\mathrm{28.9}} & \ensuremath{\pm 11.3} & \ensuremath{\mathrm{25.7}} & \ensuremath{\pm 5.9} & \ensuremath{\mathrm{24.2}} & \ensuremath{\pm 6.0} \\
\texttt{molecular\_\allowbreak{}self\_\allowbreak{}assembly} & \multicolumn{2}{c}{\best{\ensuremath{\mathrm{34.7}}\,{\scriptsize \ensuremath{\pm 0.9}}}} & \ensuremath{\mathrm{20.7}} & \ensuremath{\pm 1.0} & \ensuremath{\mathrm{21.6}} & \ensuremath{\pm 0.6} & \ensuremath{\mathrm{13.2}} & \ensuremath{\pm 5.9} & \multicolumn{2}{c}{\second{\ensuremath{\mathrm{21.9}}\,{\scriptsize \ensuremath{\pm 3.5}}}} \\
\texttt{apple\_\allowbreak{}incremental\_\allowbreak{}game} & \multicolumn{2}{c}{\best{\ensuremath{\mathrm{50.6}}\,{\scriptsize \ensuremath{\pm 8.1}}}} & \ensuremath{\mathrm{33.6}} & \ensuremath{\pm 15.9} & \multicolumn{2}{c}{\second{\ensuremath{\mathrm{34.9}}\,{\scriptsize \ensuremath{\pm 14.7}}}} & \ensuremath{\mathrm{19.1}} & \ensuremath{\pm 1.1} & \ensuremath{\mathrm{19.7}} & \ensuremath{\pm 0.1} \\
\texttt{dcss\_\allowbreak{}dungeon\_\allowbreak{}ai} & \multicolumn{2}{c}{\second{\ensuremath{\mathrm{8.3}}\,{\scriptsize \ensuremath{\pm 3.5}}}} & \multicolumn{2}{c}{\best{\ensuremath{\mathrm{13.4}}\,{\scriptsize \ensuremath{\pm 0.5}}}} & \ensuremath{\mathrm{6.1}} & \ensuremath{\pm 4.7} & \ensuremath{\mathrm{7.6}} & \ensuremath{\pm 4.3} & \ensuremath{\mathrm{5.7}} & \ensuremath{\pm 3.7} \\
\texttt{anchorhead\_\allowbreak{}text\_\allowbreak{}adventure} & \multicolumn{2}{c}{\second{\ensuremath{\mathrm{22.3}}\,{\scriptsize \ensuremath{\pm 4.5}}}} & \multicolumn{2}{c}{\best{\ensuremath{\mathrm{36.3}}\,{\scriptsize \ensuremath{\pm 6.4}}}} & \ensuremath{\mathrm{17.7}} & \ensuremath{\pm 2.3} & \ensuremath{\mathrm{20.3}} & \ensuremath{\pm 0.6} & \ensuremath{\mathrm{14.7}} & \ensuremath{\pm 8.1} \\
\texttt{trinity\_\allowbreak{}text\_\allowbreak{}adventure} & \multicolumn{2}{c}{\second{\ensuremath{\mathrm{30.0}}\,{\scriptsize \ensuremath{\pm 2.6}}}} & \multicolumn{2}{c}{\best{\ensuremath{\mathrm{40.0}}\,{\scriptsize \ensuremath{\pm 10.5}}}} & \ensuremath{\mathrm{27.0}} & \ensuremath{\pm 7.0} & \ensuremath{\mathrm{26.7}} & \ensuremath{\pm 5.7} & \ensuremath{\mathrm{20.3}} & \ensuremath{\pm 3.8} \\
\texttt{tryst\_\allowbreak{}text\_\allowbreak{}adventure} & \multicolumn{2}{c}{\second{\ensuremath{\mathrm{44.3}}\,{\scriptsize \ensuremath{\pm 8.7}}}} & \multicolumn{2}{c}{\best{\ensuremath{\mathrm{55.7}}\,{\scriptsize \ensuremath{\pm 3.8}}}} & \multicolumn{2}{c}{\second{\ensuremath{\mathrm{44.3}}\,{\scriptsize \ensuremath{\pm 7.6}}}} & \ensuremath{\mathrm{43.3}} & \ensuremath{\pm 4.4} & \ensuremath{\mathrm{13.8}} & \ensuremath{\pm 3.3} \\
\texttt{openrct2\_\allowbreak{}theme\_\allowbreak{}park\_\allowbreak{}ai} & \ensuremath{\mathrm{27.5}} & \ensuremath{\pm 0.0} & \multicolumn{2}{c}{\best{\ensuremath{\mathrm{37.6}}\,{\scriptsize \ensuremath{\pm 9.0}}}} & \ensuremath{\mathrm{23.1}} & \ensuremath{\pm 11.7} & \multicolumn{2}{c}{\second{\ensuremath{\mathrm{36.2}}\,{\scriptsize \ensuremath{\pm 9.5}}}} & \ensuremath{\mathrm{26.0}} & \ensuremath{\pm 2.6} \\
\texttt{wesnoth\_\allowbreak{}tactical\_\allowbreak{}ai} & \multicolumn{2}{c}{\best{\ensuremath{\mathrm{88.0}}\,{\scriptsize \ensuremath{\pm 2.6}}}} & \ensuremath{\mathrm{79.3}} & \ensuremath{\pm 6.4} & \multicolumn{2}{c}{\second{\ensuremath{\mathrm{81.3}}\,{\scriptsize \ensuremath{\pm 1.5}}}} & \ensuremath{\mathrm{78.3}} & \ensuremath{\pm 11.0} & \ensuremath{\mathrm{36.3}} & \ensuremath{\pm 31.5} \\
\BenchBottomRule
\end{tabular}%
}
\caption{Model performance on Interactive Games \& Simulators tasks. Values are mean scores over up to three valid runs; adjacent entries show $\pm s$ when at least two valid runs are available. Bold marks the best model for each task, underlining marks the second-best model, \NA{} indicates no valid result, and \textsuperscript{*} marks fewer than three valid runs.}
\label{tab:appendix-games_simulators}
\end{table}

%% file: sections/acknowledgements.tex
\section{Acknowledgements}
\label{sec:acknowledgements}

We thank UniPat for collaborating with us on the development of 13 tasks in the Systems \& Software Engineering portion of \benchmark{}.
We also thank Xiaoxing Wu, Xiang Gao, Gao Liu, Yue Yang, Wen Heng, Weinan Zhao, Mailun Gao, Zongbao Zhang, and Yuchen Wu for helpful supporting contributions during the project.
We thank Xinkai Zhou and Qi Zhao for meaningful discussions during the project.
We thank Tenglong Ao, Ao Zhang, Shengjie Luo, Zeyi Zhang, Guhao Feng, Tianle Cai, Xinrong Zhang, Yizhong Wang, and Ruinian Chang for meaningful discussions and collaborations that took place prior to the EdgeBench project.